\documentclass{article}
\usepackage{microtype}
\usepackage{graphicx}
\usepackage{subcaption}
\usepackage{booktabs}

\usepackage[colorlinks=true,allcolors=linkcolor,pageanchor=true,plainpages=false,pdfpagelabels,bookmarks,bookmarksnumbered]{hyperref}

\usepackage[preprint]{icml2026}

\usepackage{url}
\usepackage{booktabs}
\usepackage{amsfonts,amsmath,amssymb,amsthm}
\usepackage{nicefrac}
\usepackage{microtype}
\usepackage{xcolor}
\usepackage{dblfloatfix}

\usepackage{graphicx}
\usepackage{wrapfig}
\usepackage[export]{adjustbox}
\usepackage{caption}
\usepackage{makecell}
\usepackage{multirow}
\usepackage{bm}
\usepackage{enumitem}
\usepackage{array}
\usepackage{textcomp}
\usepackage{diagbox}
\usepackage{float}
\floatstyle{ruled}
\restylefloat{algorithm}
\newcolumntype{S}[1]{>{\small\raggedright\arraybackslash}p{#1}}
\newcolumntype{Z}{>{\small\centering\arraybackslash}c}
\newcolumntype{P}{Z@{\hspace{2pt}}Z}
\usepackage{algorithm}  

\newcommand{\msl}[2]{#1{\color{gray}\scriptsize\,\textpm\,#2}}

\usepackage[capitalize,noabbrev]{cleveref}
\theoremstyle{plain}
\newtheorem{theorem}{Theorem}[section]
\newtheorem{proposition}[theorem]{Proposition}
\newtheorem{lemma}[theorem]{Lemma}

\theoremstyle{definition}
\newtheorem{definition}[theorem]{Definition}

\theoremstyle{remark}
\newtheorem{remark}[theorem]{Remark}

\usepackage[textsize=tiny]{todonotes}

\definecolor{linkcolor}{RGB}{70, 90, 180}

\renewcommand{\eqref}[1]{\textup{Eq.~\ref{#1}}}

\icmltitlerunning{
Controlling Exploration-Exploitation in GFlowNets via Markov Chain Perspectives
}

\begin{document}

\twocolumn[
  \icmltitle{Controlling Exploration–Exploitation in GFlowNets \\ via Markov Chain Perspectives}
  \icmlsetsymbol{equal}{*}
  \icmlsetsymbol{correspond}{$\dagger$}
  \begin{icmlauthorlist}
    \icmlauthor{Lin Chen}{equal,lumia,sjtusms}
    \icmlauthor{Samuel Drapeau}{equal,sjtusms,sjtusaif}
    \icmlauthor{Fanghao Shao}{lumia}
    \icmlauthor{Xuekai Zhu}{lumia}
    \icmlauthor{Bo Xue}{sjtuscs}
    \icmlauthor{Yunchong Song}{ailab}
    \icmlauthor{Mathieu Lauri\`ere}{correspond,nyushscds,nyushms}
    \icmlauthor{Zhouhan Lin}{correspond,lumia,ailab,sii}
  \end{icmlauthorlist}

    \icmlaffiliation{lumia}{LUMIA Lab, School of Artificial Intelligence, SJTU}
  \icmlaffiliation{sjtusms}{School of Mathematical Sciences, SJTU}
  \icmlaffiliation{sjtusaif}{Shanghai Advanced Institute of Finance, SJTU}
  \icmlaffiliation{sjtuscs}{School of Computer Science, SJTU}
  \icmlaffiliation{ailab}{Shanghai AI Laboratory}
  \icmlaffiliation{sii}{Shanghai Innovation Institute}
  \icmlaffiliation{nyushscds}{Shanghai Center for Data Science, NYU Shanghai}
  \icmlaffiliation{nyushms}{NYU-ECNU Institute of Mathematical Sciences, NYU Shanghai}

  \icmlcorrespondingauthor{Lin Chen}{charliecl0526@gmail.com}
  \icmlcorrespondingauthor{Mathieu Laurière}{ml5197@nyu.edu}
  \icmlcorrespondingauthor{Zhouhan Lin}{lin.zhouhan@gmail.com}

  \icmlkeywords{Machine Learning, ICML}

  \vskip 0.3in
]

\printAffiliationsAndNotice{\icmlEqualContribution. 
\textsuperscript{$\dagger$}Corresponding authors.}

\begin{abstract}
Generative Flow Network (GFlowNet) objectives implicitly fix an equal mixing of forward and backward policies, potentially constraining the exploration-exploitation trade-off during training. 
By further exploring the link between GFlowNets and Markov chains, 
we establish an equivalence between GFlowNet objectives and Markov chain reversibility, thereby revealing the origin of such constraints,
and provide a framework for adapting Markov chain properties to GFlowNets.
Building on these theoretical findings, we propose 
\textbf{$\alpha$-GFNs}, which generalize the mixing via a tunable parameter $\alpha$. 
This generalization enables direct control over exploration-exploitation dynamics to enhance mode discovery capabilities, while ensuring convergence to unique flows.
Across various benchmarks, including Set, Bit Sequence, and Molecule Generation, $\alpha$-GFN objectives consistently outperform previous GFlowNet objectives,
achieving up to a $10 \times$ increase in the number of discovered modes.
\end{abstract}

\section{Introduction}\label{introduction}

Generative Flow Networks (GFlowNets)~\citep{bengio2021flow} are generative models that sample compositional objects from high-dimensional distributions with probabilities proportional to a reward function.
They are sampling methods that originate from the intersection of  reinforcement learning frameworks~\citep{tiapkin2024generative, mohammadpour2024maximum, deleu2024discrete} and flow networks~\citep{bengio2023gflownet}, offering an alternative to traditional approaches such as Markov Chain Monte Carlo (MCMC)~\citep{brooks1998markov}.
Since their introduction, GFlowNets 
have been applied in various domains including molecular discovery~\citep{bengio2021flow,jain2023multi,zhu2023sample}, diffusion models~\citep{zhang2023diffusion,venkatraman2024amortizing,liu2024efficient} and large language models~\citep{hu2023amortizing,song2024latent,yun2025learning,zhu2025flowrl}, demonstrating both mode-discovering and diversity-preserving abilities.

Alongside these empirical successes, the training objectives of GFlowNets largely draw upon a flow matching perspective~\citep{bengio2021flow}. 
While this paradigm provides a systematic way to match reward distributions~\citep{bengio2023gflownet,malkin2022trajectory,madan2023learning}, it induces a symmetric treatment of forward and backward transitions, implicitly entailing an equal weighting of forward and backward policies.
However, such a equal mixing scheme can be sub-optimal, as it potentially constrains the flexibility of the exploration-exploitation trade-off during training.
As shown in Fig.~\ref{fig:main_figure}, an adaptable weighting scheme yields significantly higher average per-sample rewards. 
This suggests a need for a broader theoretical treatment beyond the flow-matching view, particularly through GFlowNets' inherent connections to Markov chain (MC) theory.
While GFlowNets are primarily formulated as Markov Decision Processes (MDPs), recent research has begun to explore their relationship with the theory of MCs. 
Prior work by \citet{deleu2023generative} discussed this connection specifically for the Flow Matching (FM) objective.
\begin{figure*}[t]
  \centering
  \begin{tabular}{@{}c@{\hspace{10pt}}c@{}}
    \includegraphics[valign=b,width=.45\textwidth]{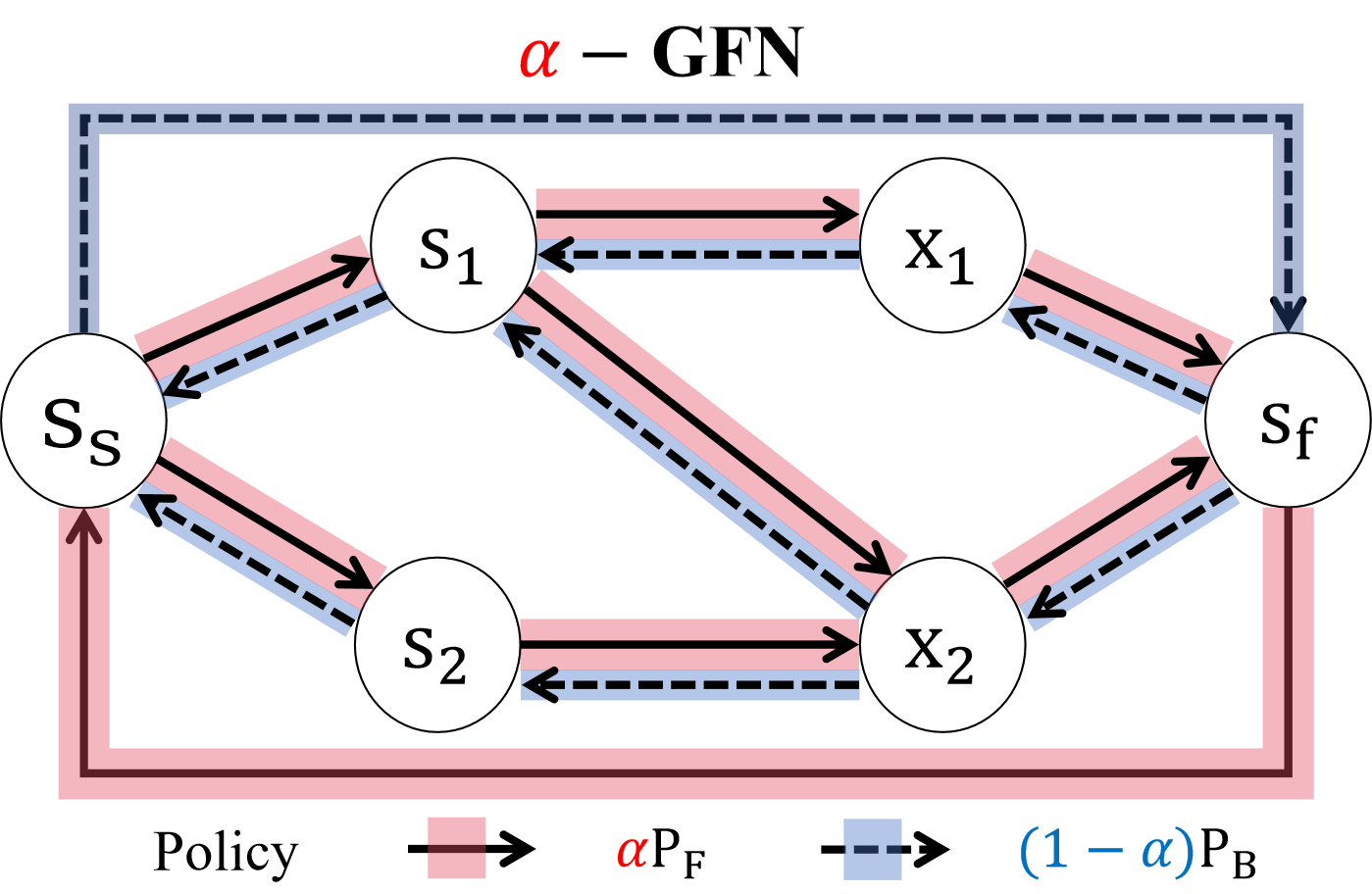} &
    \raisebox{-1.63mm}{\includegraphics[valign=b,width=.45\textwidth]{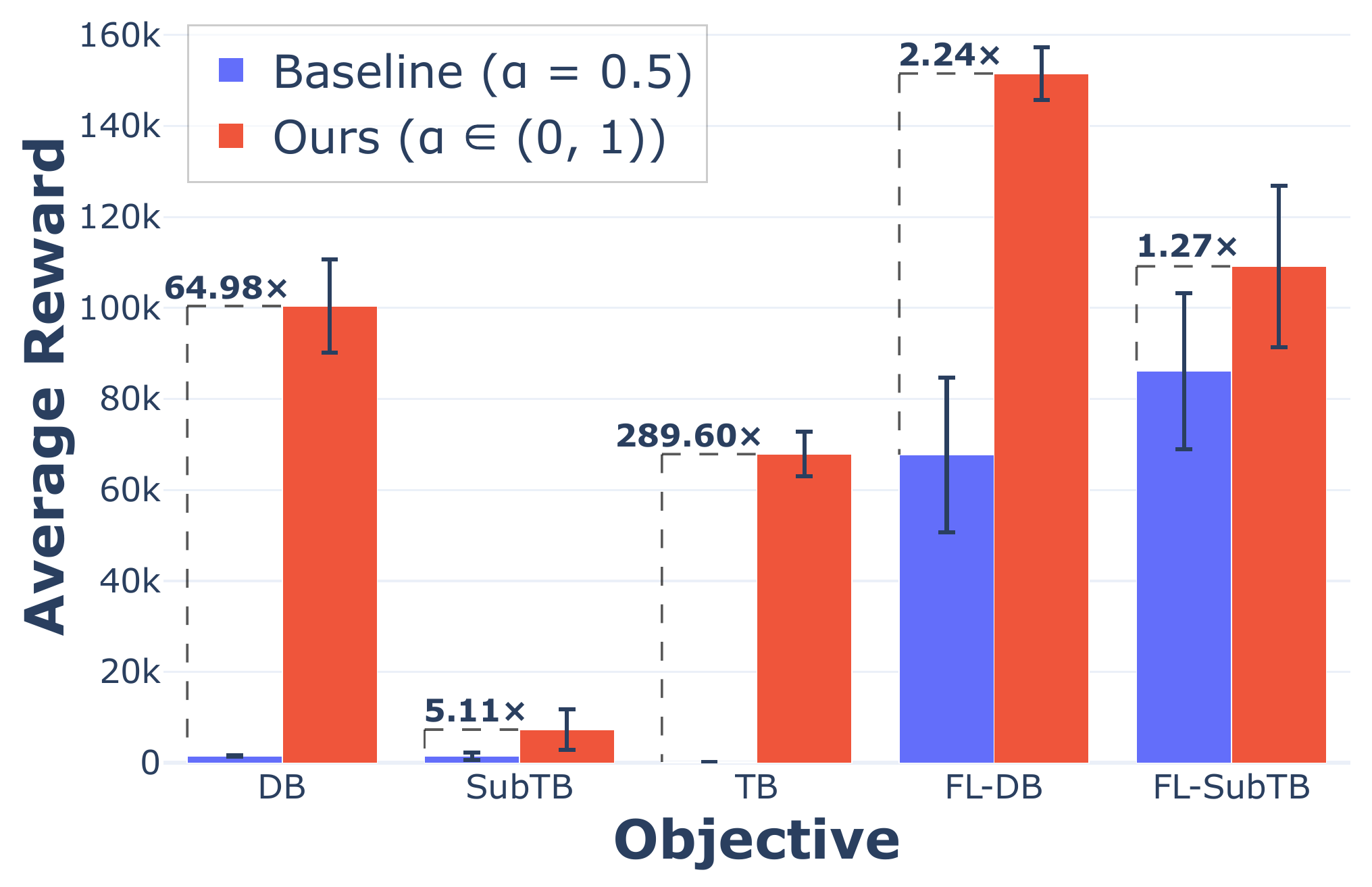}}
  \end{tabular}
  \caption{\textbf{(Left)} Illustration of $\alpha$-GFNs. Vanilla GFlowNets implicitly assigns equal weights (0.5/0.5) to the forward policy $P_F$ and the backward policy $P_B$, while $\alpha$-GFNs assign $\alpha$ and $1-\alpha$ to $P_{F}$ and $P_{B}$ respectively.
  \textbf{(Right)} The performance gain with $\alpha$-GFN objectives in Set Generation~\citep{pan2023better}. With flexible exploration-exploitation trade-offs enabled by $\alpha$, GFlowNet training achieves significantly higher average reward of all generated samples.}
  \label{fig:main_figure}
\end{figure*}

In this work, we further explore the theoretical link between GFlowNets and Markov chains and establish a framework that unifies multiple GFlowNet objectives. This framework provides a unified perspective that allows for adapting various MC properties to generalize the GFlowNet framework. For instance, we show that Markov chain reversibility provides a fundamental characterization of GFlowNet objectives. Based on these theoretical insights, we propose $\alpha$-GFN, a simple yet effective generalization that encompasses standard GFlowNet objectives as special cases. By introducing a single hyperparameter $\alpha \in (0,1)$, our formulation enables flexible mixing of forward and backward policies, relaxing the fixed-weighting constraints of traditional views. This flexibility facilitates a tunable balance between exploration and exploitation, thereby enhancing the mode discovery capabilities of GFlowNets.

More rigorously, we further prove that $\alpha$-GFN objectives converge to unique flow functions and analyze their gradient dynamics to provide a principled explanation for their empirical effectiveness.
Extensive experiments on a variety of benchmarks such as Set Generation, Bit Sequence Generation, and Molecule Generation demonstrate that $\alpha$-GFNs consistently outperform standard GFlowNet objectives in terms of mode discovery, increasing the number of distinct high-reward samples generated during training.
Additionally, we provide an ablation study of $\alpha$ values and discuss its potential influence on trajectory lengths as observed in specific settings, offering further insights into the empirical characteristics of our framework.

\textbf{Contributions.} 
 We summarize the main contributions of this paper as follows: 
 \begin{itemize}[leftmargin=*, topsep=0pt, noitemsep]
 \item \textbf{Theoretical Unification.} We further explore the theoretical link between GFlowNets and Markov chain theory, establishing a unified framework that encompasses multiple GFlowNet objectives. This unification enables the systematic integration of MC properties, such as reversibility, into the design of GFlowNets.
 \item \textbf{Generalized Training Objective.} We introduce $\alpha$-GFNs based on the theoretical connections, which utilize a mixing hyperparameter $\alpha$ to interpolate between forward and backward policy learning. We support this design with theoretical convergence proofs and a gradient-based analysis that elucidates how $\alpha$ modulates the exploration-exploitation trade-off.
 \item \textbf{Empirical Performance and Insights.} We show that our approach yields improved mode discovery results through extensive evaluations on various benchmarks, including Set, Bit Sequence and Molecule Generation. Furthermore, we demonstrate the robustness of $\alpha$-GFNs to hyperparameter selection through ablation studies, and report an observed impact on trajectory lengths in specific settings.
 \end{itemize}
\section{Preliminaries}
\label{sec:preliminaries}

\paragraph{Generative Flow Network (GFlowNet, GFN) preliminaries.}
A GFlowNet is specified by a tuple $(G,R,F,P_F,P_B)$. Here $G=(S,\mathbb{A})$ is a pointed directed acyclic graph with source $s_s$ and sink $s_f$, $R:\mathcal{X}\subseteq S\to\mathbb{R}_+$ is a reward, $F:S\to\mathbb{R}_+$ is a state flow; and $P_F, P_B:\mathbb{A}\to[0,1]$ are forward and backward policies. Let $\mathfrak{T}^{\text{flow}}$ be the set of complete trajectories $\mathfrak{t}^f=(s_0,\ldots,s_N)$ with $s_0=s_s$ and $s_N=s_f$. Samples are states $x\in\mathcal{X}$ sequentially constructed by $P_F$  and deconstructed by $P_B$. The generation of a sample terminates when $s_f$ is reached, yielding a complete trajectory $(s_0=s_s,\ldots,s_{N-1}=x,s_N=s_f)$. Ideally, the probability of generating $x$ is proportional to its reward,
$
P_F^{\top}(x)\;\triangleq\;\sum_{\mathfrak{t}\in\mathfrak{T}^{\text{flow}}:\,x\in\mathfrak{t}}\;\prod_{i=1}^{N} P_F(s_i\mid s_{i-1})\ \propto\ R(x).
$
Under the convergence of its training objectives, this proportionality is satisfied and the uniqueness of flows is achieved.

\paragraph{GFlowNet training objectives and variants.} 
Several loss functions have been introduced to train GFlowNets by enforcing flow balancing conditions, such as Flow Matching (FM,~\citet{bengio2021flow}), Detailed Balance (DB,~\citet{bengio2023gflownet}), Subtrajectory Balance (SubTB,~\citet{madan2023learning}), or Trajectory Balance (TB,~\citet{malkin2022trajectory}). Since SubTB provides a unifying view of DB and TB~\citep{madan2023learning}, we present subsequent derivations on the basis of SubTB for notational simplicity. 
Given any partial trajectory $\mathfrak{t}'=(s_k,s_{k+1},\dots,s_{k+m})\subset\mathfrak{t}^f\in\mathfrak{T}^{\text{flow}}$, SubTB aims at
\begin{equation}
\label{eq:subtb}
\begin{split}
&F(s_k)\prod_{i=1}^m P_F\!\bigl(s_{k+i}\!\mid\! s_{k+i-1}\bigr) \\
= &F(s_{k+m})\prod_{i=1}^{m} P_B\!\bigl(s_{k+i-1}\!\mid\! s_{k+i}\bigr)
\end{split}
\end{equation}
where $P_B(x\mid s_f)=\frac{R(x)}{F(s_f)}$ for all $x\in\mathcal{X}$ and $F(s_f)=\sum_{x'\in\mathcal{X}}R(x')=F(s_0)$. The loss is the log-square of~\eqref{eq:subtb}. 
A convex combination of SubTB losses across subtrajectory lengths is used for training, termed SubTB($\lambda$). 
While flow-balancing objectives provide a principled training signal, credit assignment can still be inefficient. 
Forward-looking (FL) variants~\citep{pan2023better,jang2023learning} incorporate intermediate energies via a reparameterization of flows. 
Concretely, given an intermediate energy $\mathcal{E}:\mathbb{A}\to\mathbb{R}$, FL-SubTB augments~\eqref{eq:subtb} by multiplying the right-hand side by
$\prod_{i=1}^m \exp\!\big(-\mathcal{E}(s_{k+i-1}, s_{k+i})\big)$.
Detailed definitions are available in App.~\ref{app:defs}.
\paragraph{Reversibility.} Reversibility is a classical concept in Markov chain theory, see e.g.~\citep{douc2018markov}. Given a probability measure over the state space $\pi: S \rightarrow [0,1]$ and the transition kernel of the chain $P: S \times S \rightarrow [0,1]$, the reversibility of $P$ implies that, at any sequence of states $(s_k,s_{k+1},\dots,s_{k+m})$ 
\begin{equation}
\label{eq:reversibility}
    \pi(s_k) \prod\limits_{i=1}^{m} P(s_{k+i}|s_{k+i-1}) = \pi(s_{k+m})\prod\limits_{i=1}^{m} P(s_{k+i-1}|s_{k+i}).
\end{equation}

\section{$\alpha$-GFN: Generalized GFlowNet Training}
This section is organized into four parts. First, we show that the target of the vanilla GFlowNet objectives corresponds to the reversibility condition of a Markov chain with the \textbf{equally mixed policy} $P_{0.5} = \tfrac{1}{2}P_F + \tfrac{1}{2}P_B$ as its transition kernel. Next, we extend the mixing to unequal weights, leading to the $\alpha$-GFN objectives, whose convergence is further clarified through the underlying MC formulation. We then introduce a scheduling algorithm that combines the advantages of different $\alpha$ values. Finally, we illustrate the flexible exploration–exploitation trade-off enabled by $\alpha$. 

\subsection{GFlowNets as Equally Mixed Markov Chains}

We begin with an intuitive comparison of~\eqref{eq:subtb} and~\eqref{eq:reversibility}, which reveals a structural similarity between GFlowNets objectives and the reversibility of Markov chains. 
Specifically, both equations share the following structure: a sequence of transitions is coupled with a probability measure on each side. 
Building on the equivalence between flows and probability measures (see~\citet{deleu2023generative} and~\eqref{eq:flows as probability measures}), this resemblance actually suggests a close connection between GFN objectives and MC reversibility even though GFNs only use the forward policy $P_F$ to generate samples. 
However, the policies on the two sides of~\eqref{eq:subtb} are not identical, which prevents a direct correspondence with~\eqref{eq:reversibility}. 
In reality, this obstacle can be overcome by introducing the equally mixed policy $P_{0.5}$. When applied to~\eqref{eq:reversibility},  $P_{0.5}$ naturally separates $P_F$ and $P_B$ onto different sides, eliminating the weights and recovering~\eqref{eq:subtb}:
\begin{proposition}
\label{prop:reversibility of equally mixed policy}
    The reversibility of $P_{0.5}$ means that for any partial trajectory 
    $\mathfrak{t}'=(s_k,\dots,s_{k+m})$, 
    \begin{equation}
    \begin{split}
      &\pi(s_k) \prod_{i=1}^{m} P_F(s_{k+i}|s_{k+i-1})\\
      = &\pi(s_{k+m})\prod_{i=1}^{m} P_B(s_{k+i-1}|s_{k+i}). 
    \end{split}
    \end{equation}
\end{proposition}
Applying Prop.~\ref{prop:reversibility of equally mixed policy} to DB, SubTB and TB, we extend the GFN-MC link in~\cite{deleu2023generative} from FM to these objectives, 
and turn the similarity into a theoretical equivalence between GFNs and MCs:
\begin{theorem}
\label{thm:gflownet objectives and reversibility}
    The target of SubTB is equivalent to the partial-trajectory-level reversibility of a MC with the transition kernel $P_{0.5}$, and that of DB or TB is similar. Moreover, their convergence to unique flows is related to corresponding properties in MC theory.
\end{theorem}
The proofs for Prop.~\ref{prop:reversibility of equally mixed policy} and Thm.~\ref{thm:gflownet objectives and reversibility} appear in App.~\ref{app:proof for reversibility of equally mixed policy} and App.~\ref{app:proof for gflownet objectives and reversibility}, respectively. App.~\ref{app:supporting theoretical framework} further clarifies the connection between Markov chains and GFlowNets, including the MC$\to$GFN link which is not elucidated in~\citet{deleu2023generative}.

\subsection{Generalizing GFlowNet Objectives via $\alpha$-Mixing}
Now, it is clear that the objectives of GFlowNets adopt an equal weight mixing of $P_F$ and $P_B$. 
However, the equal weights may not be ideal for all the settings.
By analogy with the way two policies can be mixed by taking a convex combination of the probability distributions, we propose a flexible mixing regime with a hyperparameter $\alpha \in (0,1)$ as the mixing ratio.
To be specific, by plugging an \textbf{arbitrarily mixed policy} $P_{\alpha}=\alpha P_F+(1-\alpha)P_B$ into \eqref{eq:subtb}, Prop.~\ref{prop:reversibility of equally mixed policy}, and Thm.~\ref{thm:gflownet objectives and reversibility}, we obtain objectives corresponding to reversibility of $P_{\alpha}$, termed $\alpha$-GFN objectives. 

\begin{definition}[\textbf{$\alpha$-GFN objectives}]
\label{def:alpha-gfn objectives}
Given $\alpha \in (0,1)$, the $\alpha$-SubTB loss aims at 
\begin{equation}
\begin{split}
    &\alpha^mF(s_k)\prod_{i=1}^m  P_F(s_{k+i} | s_{k+i-1}) \\
    = &(1-\alpha)^mF(s_{k+m}) \prod_{i=1}^{m} P_B(s_{k+i-1} | s_{k+i})
\end{split}
\end{equation}
for any partial trajectory $\mathfrak{t}'=(s_k,\dots,s_{k+m})$. Applying this hyperparameter to other variants, such as forward-looking ones, yields similar objectives, specifically, $\alpha$-FL-SubTB aims at ensuring that
\begin{equation}
\begin{split}
    &\alpha^m F(s_k)\prod_{i=1}^m P_F(s_{k+i} | s_{k+i-1}) \\
    = &(1-\alpha)^m F(s_{k+m}) \prod_{i=1}^{m} P_B(s_{k+i-1} | s_{k+i}) e^{-\mathcal{E}(s_{k+i-1},s_{k+i}) }.
\end{split}
\end{equation}
\end{definition}
Other objectives, e.g. $\alpha$-DB, $\alpha$-TB, $\alpha$-FL-DB and $\alpha$-FL-SubTB, are defined similarly. Although the objetives in Def.~\ref{def:alpha-gfn objectives} violate the balance condition of vanilla GFlowNets, their \textbf{convergence} to unique flows is described by the link to MC reversibility (proof in App.~\ref{app:proof for convergence of alpha-gfn objectives}):
\begin{proposition}
\label{prop:convergence of alpha-gfn objectives}
The targets of $\alpha$-SubTB is equivalent to the partial-trajectory-level reversibility of a MC with the transition kernel $P_{\alpha}$, and those of $\alpha$-DB, $\alpha$-TB and the variants are similar. Moreover, their convergence to unique flows is similar to vanilla GFN objectives for all $\alpha \in (0,1)$.
\end{proposition}

Despite the convergence of training with $\alpha \neq 0.5$ and the potential benefits (see Sec.~\ref{sec:effects of alpha}), such training may not yield a terminating policy $P_F^{\top}$ that matches the reward distribution. Theoretically, this is a possible result of flow imbalance, since  $P_F^{\top}(x) \propto R(x)$ is achieved when flows are balanced~\cite{bengio2023gflownet}. Hence, we propose a scheduling algorithm to combine the strengths of different $\alpha$ values.
\subsection{Scheduling of $\alpha$}
\label{sec:scheduling}

\begin{figure*}[t]
  \centering
  \setlength{\tabcolsep}{0pt} 
  \begin{tabular}{@{}c@{\hspace{17pt}}c@{\hspace{17pt}}c@{}}
    \includegraphics[width=.27\textwidth]{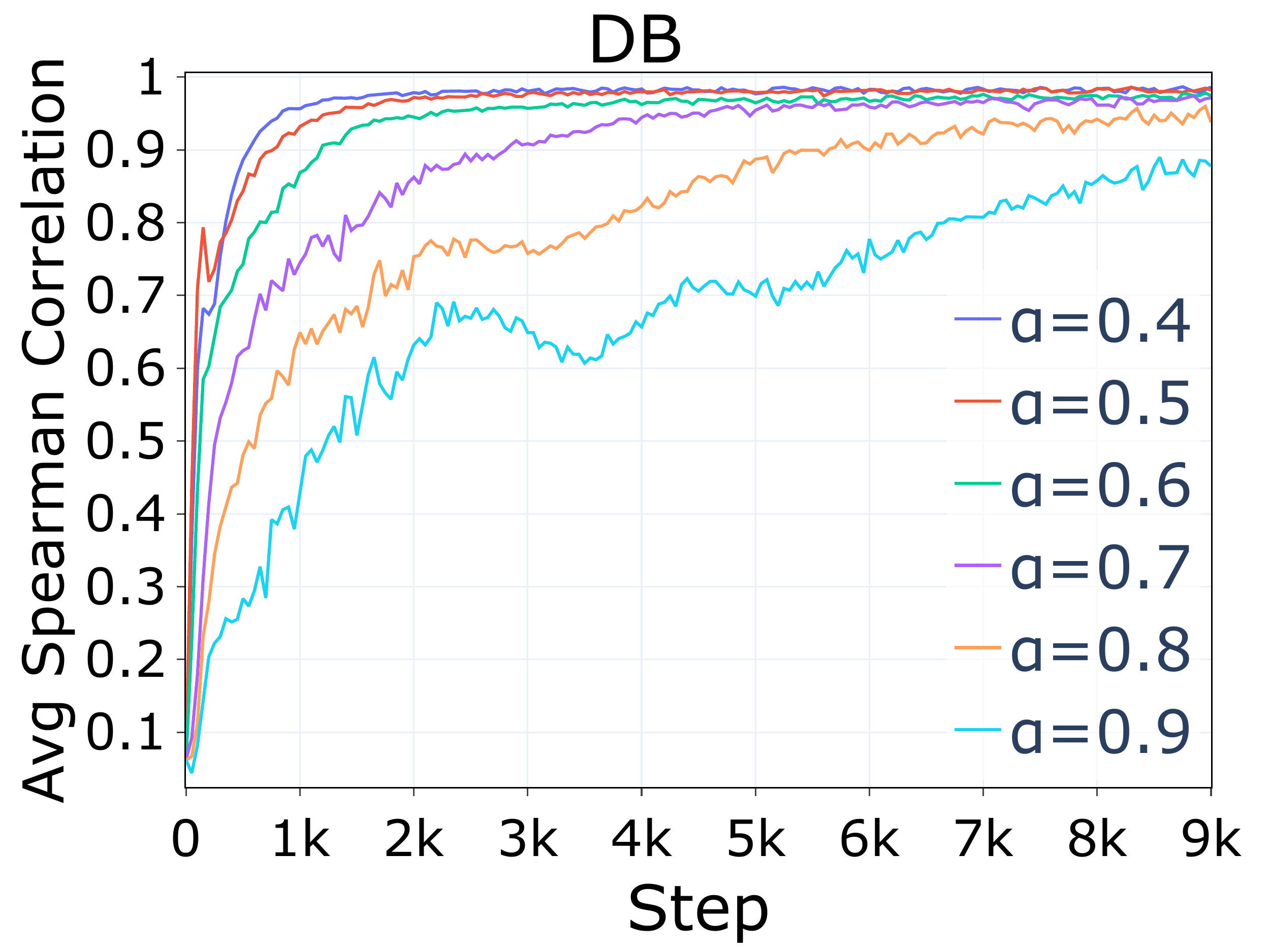} &
    \includegraphics[width=.27\textwidth]{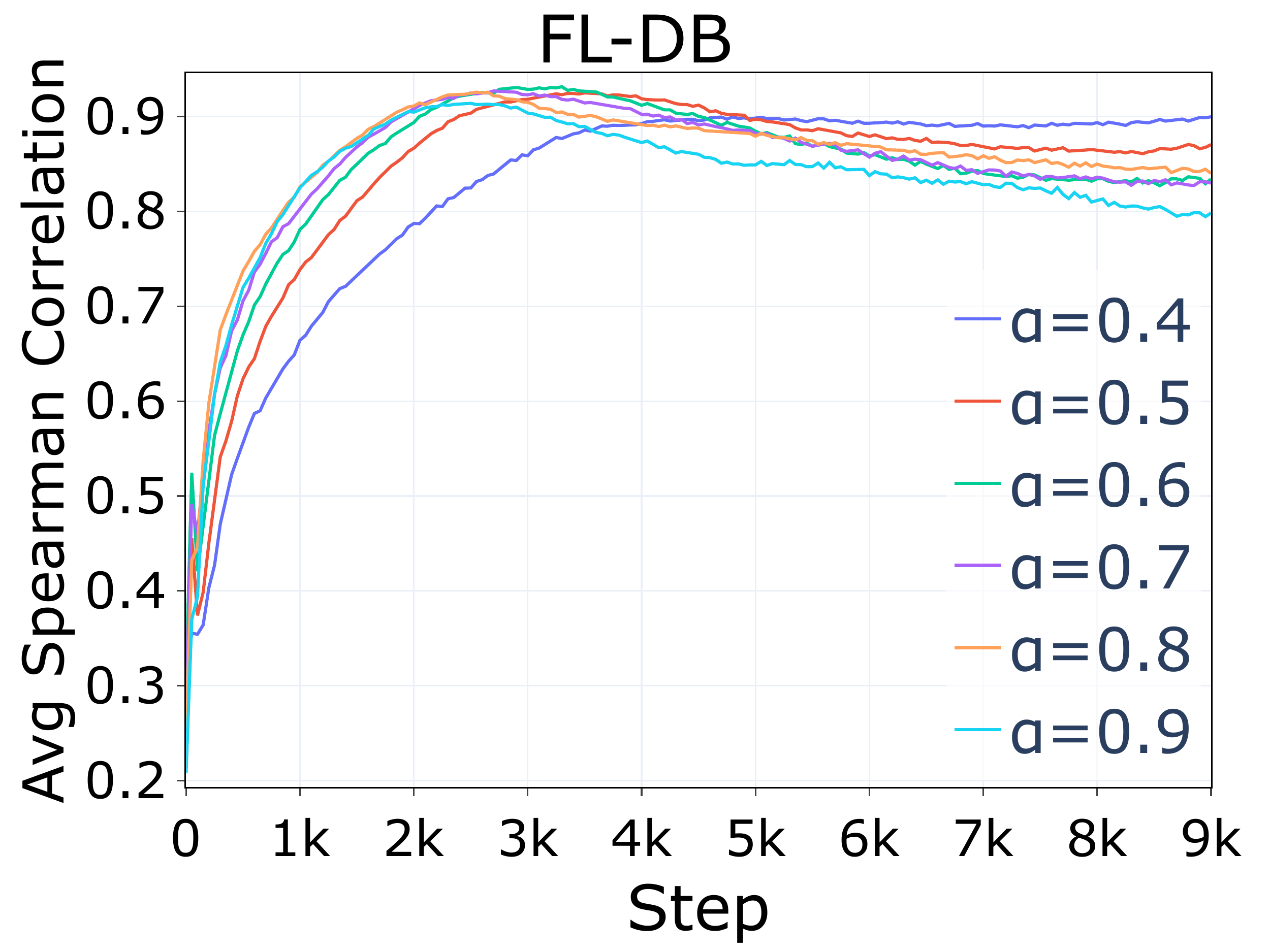} &
    \includegraphics[width=.27\textwidth]{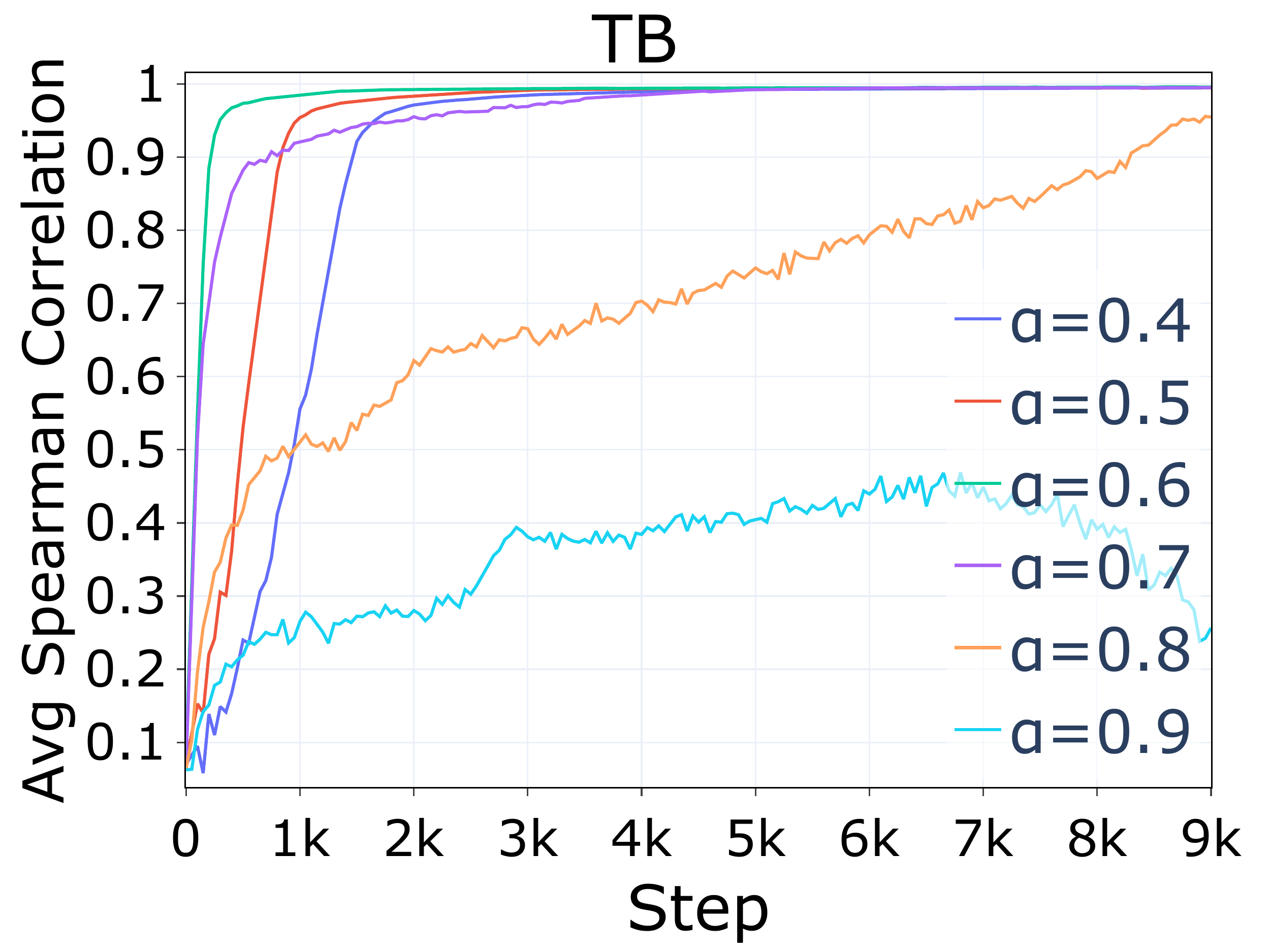}
  \end{tabular}
  \caption{Spearman correlations ($P_F^{\top}$ vs. $R$) for $\alpha$-GFNs under \textbf{(Left)} DB, \textbf{(Center)} FL-DB, and \textbf{(Right)} TB objectives in large sets.}
\label{fig:alpha-gfn spearman correlation}
\end{figure*}
Fixing the value of $\alpha$ may lead to undesirable empirical effects. As shown in Fig.~\ref{fig:alpha-gfn spearman correlation}, certain choices of $\alpha$ largely degrade the reward-fitting ability of the forward policy $P_F$.
To address this, we propose a two-stage scheduling algorithm that retains the benefits of $\alpha$-GFNs while preserving the fitting behavior of vanilla GFNs: (i) start with $\alpha$ far from $0.5$ (e.g., $0.1$--$0.4$ or $0.6$--$0.9$) and keep it fixed for a set number of steps; (ii) gradually anneal $\alpha$ to $0.5$ over the remaining steps. See Alg.~\ref{alg:two-staged-training} for details.
\begin{algorithm}[H]
  \caption{Scheduled Training of $\alpha$-GFNs}
  \label{alg:two-staged-training}
  \footnotesize
  \begin{algorithmic}[1]
    \STATE {\bfseries Input:} total steps $N$, stage-1 steps $N_1$, initial $\alpha_0$, scheduling function $f$\footnotemark
    \STATE Initialize $\alpha \leftarrow \alpha_0$.
    \STATE Select the vanilla objective $\mathcal{L}$ and augment it with $\alpha$ to obtain the $\alpha$-GFN objective $\mathcal{L}_{\alpha}$.
    \STATE Initialize model parameters $\theta$, forward/backward policies $P_F^{\theta}, P_B^{\theta}$, and flow function $F^{\theta}$.
    \FOR{$n = 1$ {\bfseries to} $N$}
      \IF{$n > N_1$}
        \STATE $\alpha \leftarrow f(\alpha_0, n, N_1, N)$
      \ENDIF
      \STATE Update $\theta$ by minimizing $\mathcal{L}_{\alpha}(P_F^{\theta}, P_B^{\theta}, F^{\theta})$.
    \ENDFOR
  \end{algorithmic}
\end{algorithm}

\footnotetext{$f$ is a scheduling function, e.g.,
$f(\alpha_0,n,N_1,N)=0.5+(\alpha_0-0.5)\exp\!\big(-4\cdot \frac{n-N_1}{N-N_1}\big)$.}

\subsection{Flexible Exploration-Exploitation Trade-Off with $\alpha$}
\label{sec:effects of alpha}

How does $\alpha$ contribute to the training of GFlowNets? 
MC theory suggests that convergence rates of $\alpha$-GFN objectives to unique flows may vary exponentially for different $\alpha$ values (see App.~\ref{app:convergence rate}).
Although GFNs are not optimized via MCMC, this property still has significant impact on the behavior of the forward policy.

\begin{figure*}[t]
  \centering
  \setlength{\tabcolsep}{0pt}
  \begin{tabular}{@{}c@{\hspace{17pt}}c@{\hspace{17pt}}c@{}}
    \includegraphics[width=.27\textwidth]{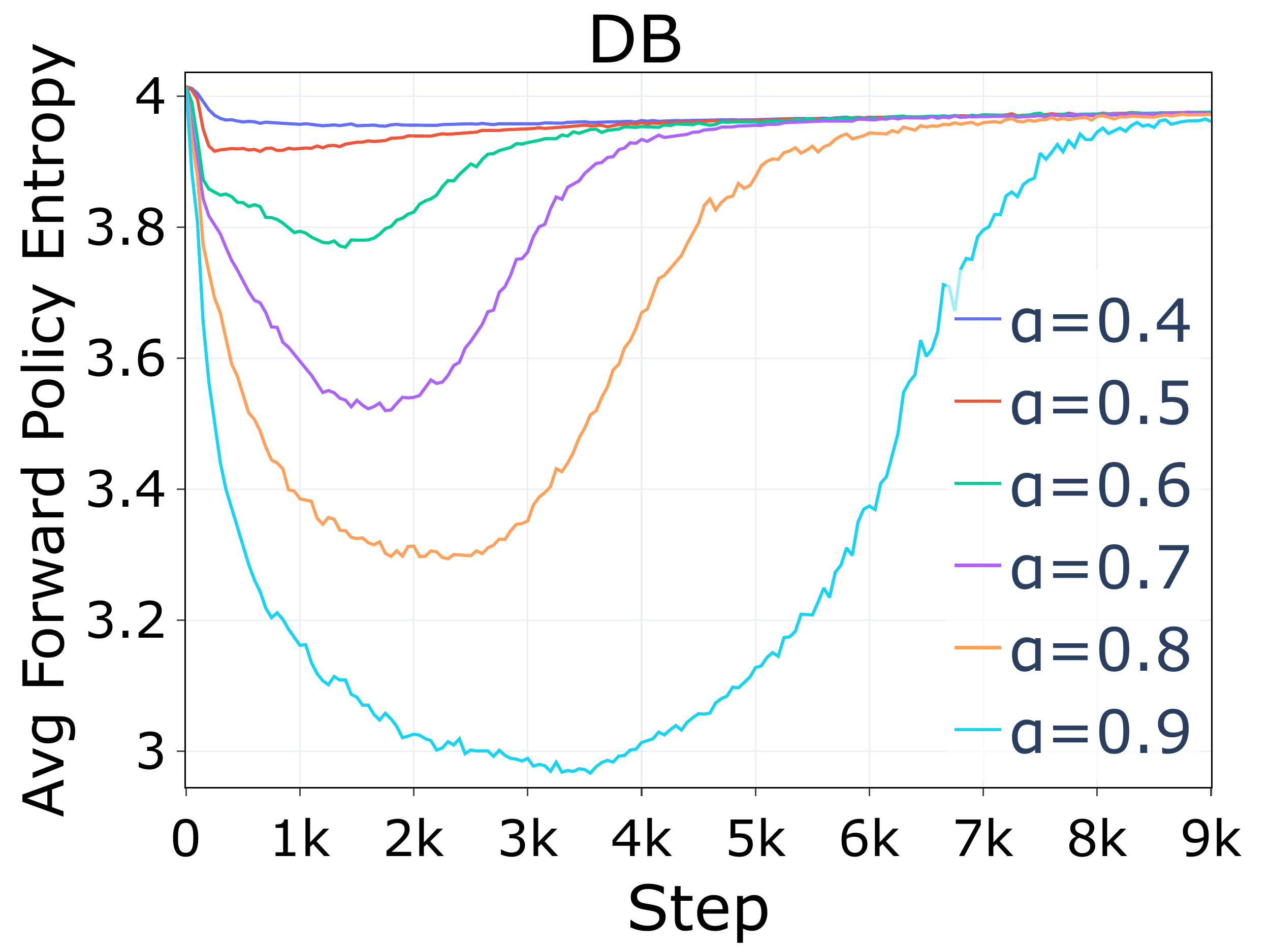} &
    \includegraphics[width=.27\textwidth]{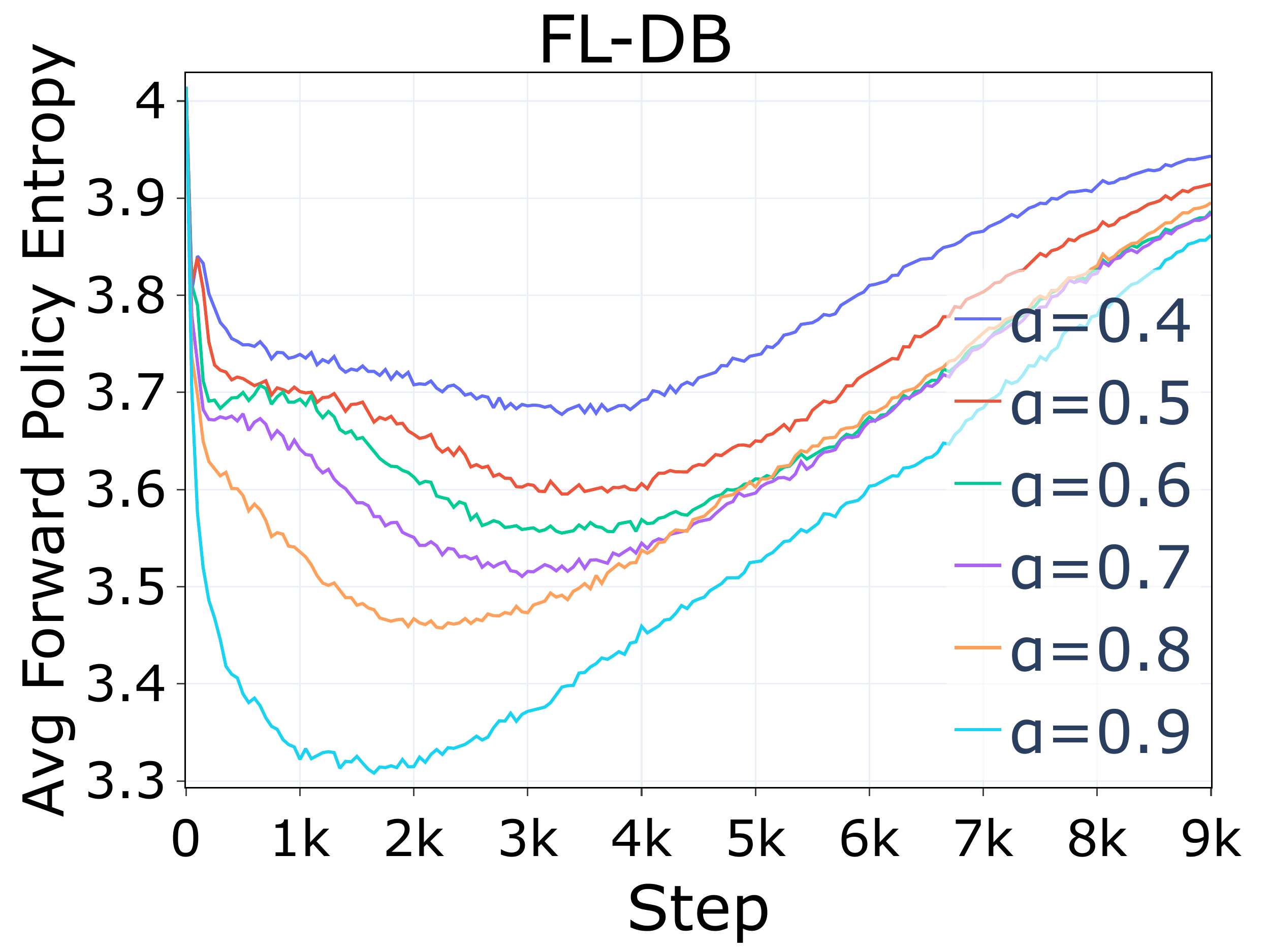} &
    \includegraphics[width=.27\textwidth]{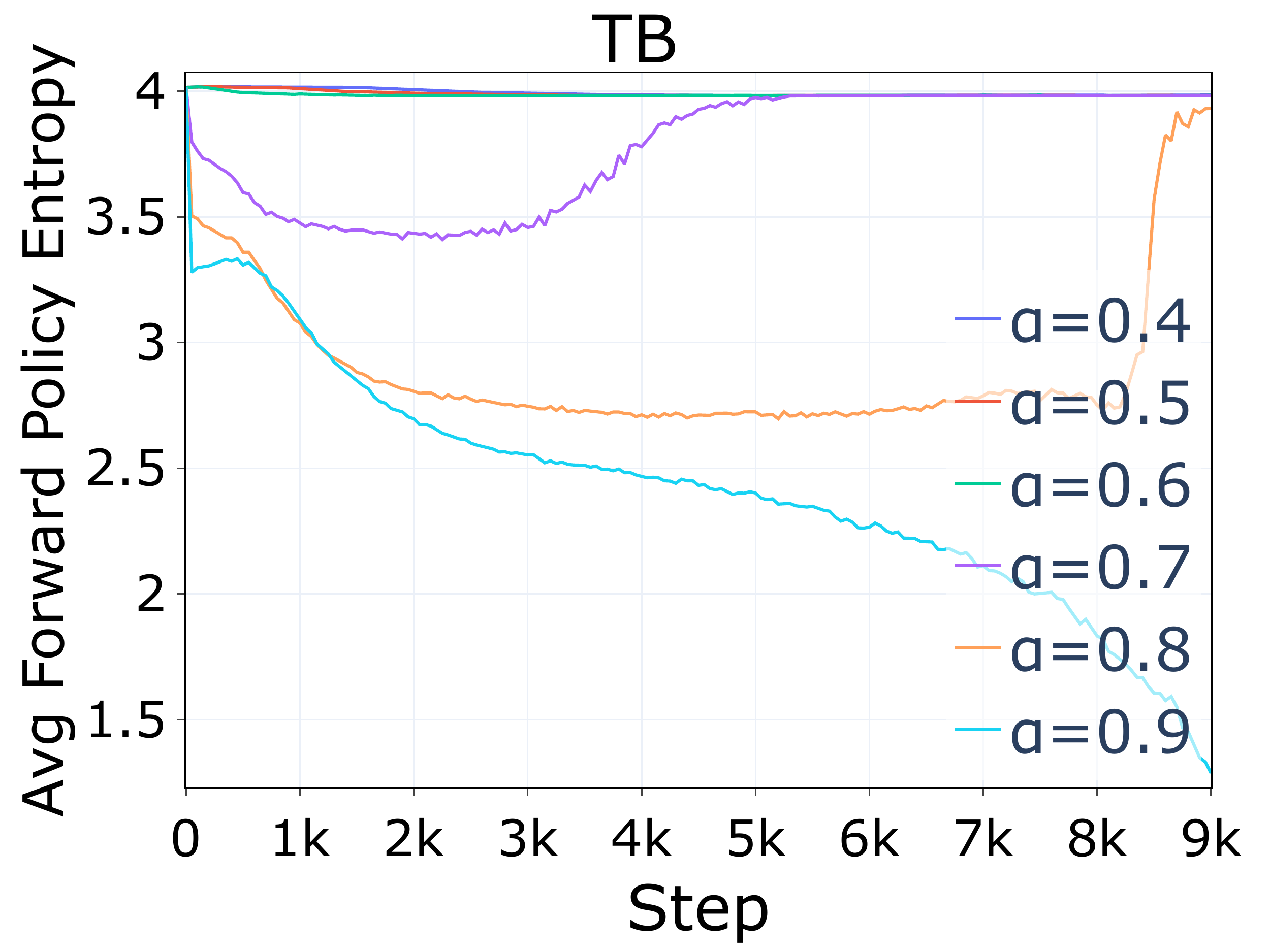} 
  \end{tabular}
  \caption{Entropy of the forward policy $P_F$ for $\alpha$-GFNs under \textbf{(Left)} DB, \textbf{(Center)} FL-DB, and \textbf{(Right)} TB objectives in large sets. Increasing $\alpha$ reduces entropy, indicating a shift toward stronger reward-exploitation, whereas decreasing $\alpha$ promotes exploration.}
\label{fig:alpha-gfn entropy}
\end{figure*}

The hyperparameter $\alpha \in (0,1)$ is the mixing ratio in $P_{\alpha}=\alpha P_F+(1-\alpha)P_B$, which controls the contribution of the forward policy $P_F$. In practice, $\alpha$ scales the training pressure on $P_F$: larger values of $\alpha$ accelerate exploitation of current estimates, while smaller values temper it. Combined with the GFlowNet target $P_F^{\top}(x)\propto R(x)$ and the convergence rates of different MCs, this leads to the following heuristic. When $\alpha>0.5$, exploitation dominates, quickly suppressing low-reward or unseen actions and concentrating mass on high-reward ones. When $\alpha<0.5$, exploitation slows down, which sustains broader exploration and produces a flatter action distribution. Empirically, the entropy dynamics in Fig.~\ref{fig:alpha-gfn entropy} follow this pattern: larger $\alpha$ induces an early drop in per-action entropy, while smaller $\alpha$ maintains higher entropy. This behavior can be explained by examining the gradient more closely: 
\begin{proposition}[Gradient of $\alpha$-GFN objectives, proof in App.~\ref{app:proof for gradient of alpha-gfn objectives}]
\label{prop:gradient of alpha-gfn objectives}
We take SubTB as an example and denote $P_F(\mathfrak{t}')=\prod_{i=1}^m P_F(s_{k+i}\mid s_{k+i-1})$ for simplicity. Since that GFN losses are log-square functions of their targets, the gradient of $\alpha$-SubTB loss can be expressed as a modification of the SubTB loss gradient for $P_F$ as
\begin{equation}
        \frac{\partial L_{\alpha-\text{SubTB}}(\mathfrak{t}')}{\partial P_F(\mathfrak{t'})} = \frac{\partial L_{\text{SubTB}}(\mathfrak{t}')}{\partial P_F(\mathfrak{t'})} + \frac{2m}{P_F(\mathfrak{t}')} \log \frac{\alpha}{1-\alpha}.
\end{equation}
\end{proposition}
For $\alpha>0.5$, the term $\frac{2m}{P_F(\mathfrak{t}')} \log\frac{\alpha}{1-\alpha}$ is larger when $P_F(\mathfrak{t}')$ is small (low reward) and smaller when $P_F(\mathfrak{t}')$ is large (high reward). Consequently, low-reward probabilities decay faster while high-reward ones decay more slowly, sharpening the distribution and strengthening exploitation.
For $\alpha<0.5$, $\log\frac{\alpha}{1-\alpha}<0$ reverses the effect, which yields a flatter distribution and contributes to exploration. In addition, Proposition~\ref{prop:gradient of alpha-gfn objectives} shows that the exploration-exploitation effects scales with $\bigl|m\log\frac{\alpha}{1-\alpha}\bigr|$.

Nevertheless, training $P_F$ under the $\alpha$-GFN objectives may suffer from over-exploitation when $\alpha$ is too large or inefficient credit assignment when $\alpha$ is too small. The former is reflected by a drop in the correlation between $P_F$ and the comparison of reward in Fig.~\ref{fig:alpha-gfn spearman correlation}, and the latter is reflected by low reward of vanilla GFN objectives in Fig.~\ref{fig:main_figure}. 
Thus, we schedule $\alpha$ with Alg.~\ref{alg:two-staged-training} to combine the advantages of different $\alpha$ values and avoid the undesirable effects of a fixed $\alpha$.

\begin{table*}[t]
\caption{\textit{Results on Set Generation.}  $\alpha$-GFNs perform better at reward and modes across all settings by enhancing the reward exploitation of GFlowNets. For all metrics except Spearman correlation, we \textbf{bold} the better result. Standard deviations are presented in \textcolor{gray}{gray}.}
\centering
\setlength{\tabcolsep}{6pt}
\renewcommand{\arraystretch}{1.2}
\resizebox{\linewidth}{!}{%
\begin{tabular}{@{\hspace{.6em}}clcccccccccc@{\hspace{.6em}}}
\toprule
\multicolumn{2}{c}{} &
\multicolumn{2}{c}{\textbf{DB}} & \multicolumn{2}{c}{\textbf{FL-DB}} & \multicolumn{2}{c}{\textbf{SubTB($\lambda$)}} & \multicolumn{2}{c}{\textbf{FL-SubTB($\lambda$)}} & \multicolumn{2}{c}{\textbf{TB}} \\
\cmidrule(lr){3-4}\cmidrule(lr){5-6}\cmidrule(lr){7-8}\cmidrule(lr){9-10}\cmidrule(lr){11-12}
\textbf{Set Size} & \textbf{Metric} & Baseline & Ours & Baseline & Ours & Baseline & Ours & Baseline & Ours & Baseline & Ours \\
\midrule
\multirow{3}{*}{Small} 
& Modes$\uparrow$ & \msl{24.8}{2.6} & \msl{\textbf{55.8}}{13.9} & \msl{83.0}{2.4} & \msl{\textbf{88.6}}{2.2} & \msl{20.0}{4.1} & \msl{\textbf{34.2}}{8.1} & \msl{89.6}{0.9} & \msl{\textbf{90.0}}{0.0} & \msl{23.2}{2.4} & \msl{\textbf{24.8}}{3.3} \\
& Top-1000 R$\uparrow$ & \msl{0.184}{0.001} & \msl{\textbf{0.204}}{0.006} & \msl{0.218}{0.001} & \msl{\textbf{0.220}}{0.002} & \msl{0.179}{0.002} & \msl{\textbf{0.191}}{0.006} & \msl{0.221}{0.000} & \msl{0.221}{0.000} & \msl{0.180}{0.001} & \msl{\textbf{0.184}}{0.002} \\
& Spearman & \msl{0.999}{0.000} & \msl{0.995}{0.001} & \msl{0.992}{0.002} & \msl{0.975}{0.007} & \msl{0.992}{0.005} & \msl{0.997}{0.001} & \msl{0.997}{0.001} & \msl{0.991}{0.002} & \msl{0.999}{0.000} & \msl{0.999}{0.000} \\
\cmidrule(lr){1-12}
\multirow{3}{*}{Medium} 
& Modes$\uparrow$ & \msl{0.0}{0.0} & \msl{\textbf{31.4}}{34.1} & \msl{14.2}{10.1} & \msl{\textbf{118.6}}{45.9} & \msl{0.0}{0.0} & \msl{\textbf{5.2}}{6.8} & \msl{16.0}{8.6} & \msl{\textbf{258.6}}{296.3} & \msl{0.0}{0.0} & \msl{\textbf{499.0}}{427.1} \\
& Top-1000 R$\uparrow$ & \msl{110427}{5337} & \msl{\textbf{552020}}{59210} & \msl{507463}{47471} & \msl{\textbf{638688}}{24711} & \msl{160689}{40231} & \msl{\textbf{434821}}{78121} & \msl{531806}{14738} & \msl{\textbf{655796}}{68335} & \msl{44854}{1481} & \msl{\textbf{703904}}{39195} \\
& Spearman & \msl{0.993}{0.004} & \msl{0.973}{0.005} & \msl{0.968}{0.008} & \msl{0.935}{0.013} & \msl{0.968}{0.007} & \msl{0.956}{0.013} & \msl{0.972}{0.006} & \msl{0.775}{0.031} & \msl{0.998}{0.000} & \msl{0.862}{0.300} \\
\cmidrule(lr){1-12}
\multirow{3}{*}{Large} 
& Modes$\uparrow$ & \msl{0.0}{0.0} & \msl{\textbf{355.8}}{319.8} & \msl{247.6}{147.2} & \msl{\textbf{2239.2}}{169.8} & \msl{0.0}{0.0} & \msl{\textbf{2.0}}{4.5} & 
\msl{118.6}{84.8} & \msl{\textbf{394.4}}{253.6}
& \msl{0.0}{0.0} & \msl{\textbf{591.6}}{210.6} \\
& Top-1000 R$\uparrow$ & \msl{58087}{5217} & \msl{\textbf{624722}}{55337} & \msl{593090}{67865} & \msl{\textbf{768545}}{6904} & \msl{53579}{27674} & \msl{\textbf{159145}}{87727} &
\msl{542497}{54448} & \msl{\textbf{635917}}{57065}
& \msl{11754}{231} & \msl{\textbf{683702}}{32707} \\
& Spearman & \msl{0.984}{0.004} & \msl{0.945}{0.004} & \msl{0.902}{0.019} & \msl{0.847}{0.018} & \msl{0.944}{0.008} & \msl{0.916}{0.017} & 
\msl{0.901}{0.029} & \msl{0.873}{0.031}
& \msl{0.996}{0.000} & \msl{0.996}{0.001} \\
\bottomrule
\end{tabular}
}
\label{tab:set-exp-results}
\end{table*}

\begin{table*}[!b]
\caption{\textit{Results on Bit Sequence Generation.} In terms of number of modes on average, $\alpha$-GFN objectives outperform vanilla GFlowNet objectives across 92\% task settings. For modes, we \textbf{bold} the better result. Standard deviations are presented in \textcolor{gray}{gray}.}
\centering
\setlength{\tabcolsep}{6pt}
\renewcommand{\arraystretch}{1.2}
\resizebox{0.99\linewidth}{!}{
\begin{tabular}{llcccccccccc}
\toprule
\multicolumn{1}{c}{} & \multicolumn{1}{c}{} & \multicolumn{2}{c}{\textbf{DB}} & \multicolumn{2}{c}{\textbf{FL-DB}} & \multicolumn{2}{c}{\textbf{SubTB($\lambda$)}} & \multicolumn{2}{c}{\textbf{FL-SubTB($\lambda$)}} & \multicolumn{2}{c}{\textbf{TB}} \\
\cmidrule(lr){3-4}\cmidrule(lr){5-6}\cmidrule(lr){7-8}\cmidrule(lr){9-10}\cmidrule(lr){11-12}
\textbf{k} & \textbf{Metric} & Baseline & Ours & Baseline & Ours & Baseline & Ours & Baseline & Ours & Baseline & Ours \\
\midrule
\multirow{2}{*}{2} &  Modes$\uparrow$ & \msl{0.60}{0.89} & \msl{\textbf{2.60}}{2.07} & \msl{60.00}{0.00} & \msl{60.00}{0.00} & \msl{\textbf{22.80}}{13.27} & \msl{20.80}{11.39} & \msl{55.40}{6.02} & \msl{\textbf{59.80}}{0.45} & \msl{12.20}{2.39} & \msl{\textbf{16.80}}{3.77} \\
&  Spearman & \msl{0.32}{0.43} & \msl{0.50}{0.04} & \msl{0.63}{0.01} & \msl{0.62}{0.00} & \msl{0.47}{0.06} & \msl{0.51}{0.19} & \msl{0.48}{0.07} & \msl{0.55}{0.00} & \msl{0.47}{0.17} & \msl{0.54}{0.00} \\
\midrule
\multirow{2}{*}{4} &  Modes$\uparrow$ & \msl{9.80}{6.50} & \msl{\textbf{13.00}}{6.36} & \msl{57.60}{1.34} & \msl{\textbf{59.20}}{0.45} & \msl{35.40}{4.16} & \msl{\textbf{40.40}}{2.97} & \msl{58.80}{1.10} & \msl{\textbf{59.40}}{0.55} & \msl{38.00}{2.74} & \msl{\textbf{41.80}}{5.07} \\
&  Spearman & \msl{0.58}{0.00} & \msl{0.58}{0.00} & \msl{0.57}{0.00} & \msl{0.56}{0.00} & \msl{0.58}{0.00} & \msl{0.60}{0.02} & \msl{0.57}{0.00} & \msl{0.58}{0.01} & \msl{0.58}{0.00} & \msl{0.58}{0.00} \\
\midrule
\multirow{2}{*}{6} &  Modes$\uparrow$ & \msl{4.20}{1.92} & \msl{\textbf{5.20}}{1.30} & \msl{31.40}{5.03} & \msl{\textbf{32.00}}{4.06} & \msl{20.60}{3.36} & \msl{\textbf{22.20}}{7.33} & \msl{\textbf{48.20}}{2.39} & \msl{47.80}{1.92} & \msl{21.60}{2.70} & \msl{\textbf{23.40}}{3.71} \\
&  Spearman & \msl{0.55}{0.01} & \msl{0.56}{0.00} & \msl{0.55}{0.00} & \msl{0.55}{0.00} & \msl{0.55}{0.00} & \msl{0.59}{0.09} & \msl{0.56}{0.00} & \msl{0.56}{0.00} & \msl{0.56}{0.00} & \msl{0.56}{0.00} \\
\midrule
\multirow{2}{*}{8} &  Modes$\uparrow$ & \msl{36.40}{1.82} & \msl{\textbf{44.40}}{3.78} & \msl{59.80}{0.45} & \msl{\textbf{60.00}}{0.00} & \msl{58.60}{1.67} & \msl{\textbf{58.60}}{0.89} & \msl{60.00}{0.00} & \msl{\textbf{60.00}}{0.00} & \msl{58.80}{1.30} & \msl{\textbf{59.00}}{0.00} \\
&  Spearman & \msl{0.79}{0.00} & \msl{0.79}{0.00} & \msl{0.73}{0.01} & \msl{0.73}{0.01} & \msl{0.81}{0.00} & \msl{0.81}{0.00} & \msl{0.76}{0.01} & \msl{0.76}{0.01} & \msl{0.81}{0.00} & \msl{0.81}{0.00} \\
\midrule
\multirow{2}{*}{10} &  Modes$\uparrow$ & \msl{6.80}{2.05} & \msl{\textbf{8.20}}{2.39} & \msl{20.80}{5.02} & \msl{\textbf{23.40}}{0.55} & \msl{19.20}{4.92} & \msl{\textbf{21.60}}{3.05} & \msl{35.20}{2.28} & \msl{\textbf{37.40}}{2.79} & \msl{\textbf{24.40}}{3.36} & \msl{23.80}{1.79} \\
&  Spearman & \msl{0.57}{0.01} & \msl{0.57}{0.00} & \msl{0.57}{0.00} & \msl{0.57}{0.00} & \msl{0.57}{0.00} & \msl{0.57}{0.00} & \msl{0.57}{0.00} & \msl{0.57}{0.00} & \msl{0.57}{0.00} & \msl{0.57}{0.00} \\
\bottomrule
\end{tabular}
}
\label{tab:bit-exp-results}
\end{table*}

\begin{table*}[t]
\caption{\textit{Results on Molecule Generation.} $\alpha$-GFNs are better at the number of discovered modes for all objectives. 
Spearman correlations of FL-DB and FL-SubTB($\lambda$) are omitted due to their biased target~\citep{silva2025when}. For all metrics except Spearman correlation, we \textbf{bold} the better result. Standard deviations are presented in \textcolor{gray}{gray}.}
\centering
\small
\setlength{\tabcolsep}{6pt}
\renewcommand{\arraystretch}{1.2}
\resizebox{\linewidth}{!}{
\begin{tabular}{@{\hspace{.6em}}lcccccccccc@{\hspace{.6em}}}
\toprule
\multicolumn{1}{c}{} &
\multicolumn{2}{c}{\textbf{DB}} &
\multicolumn{2}{c}{\textbf{FL-DB}} &
\multicolumn{2}{c}{\textbf{SubTB($\lambda$)}} &
\multicolumn{2}{c}{\textbf{FL-SubTB($\lambda$)}} &
\multicolumn{2}{c}{\textbf{TB}} \\
\cmidrule(lr){2-3}\cmidrule(lr){4-5}\cmidrule(lr){6-7}\cmidrule(lr){8-9}\cmidrule(lr){10-11}
\textbf{Metric} & Baseline & Ours & Baseline & Ours & Baseline & Ours & Baseline & Ours & Baseline & Ours \\
\midrule
Modes$\uparrow$
& \msl{10.00}{3.08} & \msl{\textbf{14.40}}{2.41}
& \msl{9.00}{1.58} & \msl{\textbf{25.00}}{4.36}
& \msl{22.20}{2.77} & \msl{\textbf{26.40}}{2.30}
& \msl{16.00}{3.16} & \msl{\textbf{39.20}}{7.05}
& \msl{{38.40}}{6.11} & \msl{\textbf{40.20}}{5.07} \\
Top-1000 R$\uparrow$
& \msl{\textbf{6.33}}{0.01} & \msl{{6.32}}{0.01}
& \msl{\textbf{6.69}}{0.06} & \msl{{6.47}}{0.14}
& \msl{6.51}{0.01} & \msl{\textbf{6.52}}{0.01}
& \msl{6.48}{0.13} & \msl{\textbf{6.50}}{0.05}
& \msl{6.70}{0.07} & \msl{\textbf{6.73}}{0.07} \\
Spearman
& \msl{0.53}{0.02} & \msl{0.50}{0.05}
& \multicolumn{2}{c}{\textemdash}
& \msl{0.57}{0.07} & \msl{0.59}{0.05}
& \multicolumn{2}{c}{\textemdash}
& \msl{0.22}{0.45} & \msl{0.47}{0.16} \\
\bottomrule
\end{tabular}
}
\label{tab:mols-exp-results}
\end{table*}

\section{Experiments}
\label{sec:expe}
Previous sections have shown how the additional hyperparameter $\alpha$ shapes GFlowNet training dynamics. We now evaluate whether these dynamics yield performance gains across various domains, particularly by enhancing discovery of distinct high-reward samples. 

\subsection{Experimental Setups}

\textbf{Baselines.} We compare our method against three GFlowNet objectives: DB \citep{bengio2021flow}, SubTB($\lambda$) \citep{madan2023learning}, and TB \citep{malkin2022trajectory}. To disentangle the benefits of our framework from the use of intermediate rewards, we also include Forward-Looking (FL) variants: FL-DB and FL-SubTB($\lambda$) \citep{pan2023better}. In our experiments, \textbf{Baselines} correspond to these vanilla objectives (effectively $\alpha=0.5$), while \textbf{Ours} refers to those trained with $\alpha$-GFN objectives where $\alpha \neq 0.5$.

\textbf{Evaluation Metrics.} 
Our primary metric is the number of discovered \textbf{Modes} (unique samples exceeding a reward threshold). We also report \textbf{Top-1000 R} (average reward of the top 1000 distinct samples) and \textbf{Spearman} (the spearman correlation between $P_F^{\top}(x)$
and $R(x)$ in a held-out test set). Results are averaged over 5 random seeds. 
Following \citet{pan2023better}, we separate training and evaluation samples to ensure unbiased assessment. 

\textbf{Benchmarks.} We evaluate the performance of $\alpha$-GFN objectives across three diverse domains.
\begin{itemize}[leftmargin=*, topsep=0pt, noitemsep]
\item \textbf{Set Generation}~\citep{pan2023better}. 
The goal of this task is to generate sets of fixed sizes. 
The maximum capacity of sets $|S|$ and the size of the vocabulary vary from being small, medium to large, and task difficulty increases correspondingly.
The sampling process of a set terminates when the number of elements reaches its maximum capacity. 
The reward function is defined as the accumulation of individual energy exponent $\exp(-\mathcal{E}(e_i))$ of each element $e_i$ in a set, i.e. $R(x) \triangleq \prod_{i=1}^{|S|} \exp(-\mathcal{E}(e_i))$, which equips FL variants with exact intermediate energy and ideal local credits~\citep{pan2023better,jang2023learning}. 
\item \textbf{Bit Sequence Generation}~\citep{tiapkin2024generative}: 
The goal of this task 
is to construct $120$-bit strings by iteratively sampling $k$-bit words ($k \in \{2,4,6,8,10\}$).
The reward $R(x)$ is the negative exponent of the Hamming distance to the nearest of $60$ predefined target modes. For FL variants, a masked Hamming distance is employed to facilitate intermediate credit assignment. Crucially, nearby samples are excluded once a mode is identified. Because the mode count is bounded a priori, it effectively captures the overall generative quality. Hence, following~\citet{tiapkin2024generative}, we focus on Modes and Spearman correlation for evaluation.

\item \textbf{Molecule Generation}~\citep{bengio2021flow}: 
This task aims to design binders for the soluble epoxide hydrolase (sEH) protein by iteratively appending “blocks” from a fixed library onto a growing molecular graph~\citep{jin2018junction}.
Both the reward and FL variants' intermediate energies are computed via a pretrained proxy from~\citet{bengio2021flow}. Despite reward, modes are also filtered by a Tanimoto similarity threshold of 7.
\end{itemize}
Further details and hyperparameter settings are in App.~\ref{app:experimental settings}.

\subsection{Results}

Across all three diverse benchmarks, our empirical results (Tables~\ref{tab:set-exp-results}, \ref{tab:bit-exp-results}, and \ref{tab:mols-exp-results}) demonstrate that $\alpha$-GFN objectives consistently achieve a higher number of discovered modes compared to vanilla GFlowNet baselines. 
This performance gain suggests that by introducing a new dimension of exploration–exploitation flexibility, our framework facilitates more effective mode discovery than vanilla objectives.
\begin{figure*}[!b]
  \centering
  \setlength{\tabcolsep}{0pt}
  \begin{tabular}{@{}c@{\hspace{17pt}}c@{\hspace{17pt}}c@{}}
    \includegraphics[width=.3\textwidth]{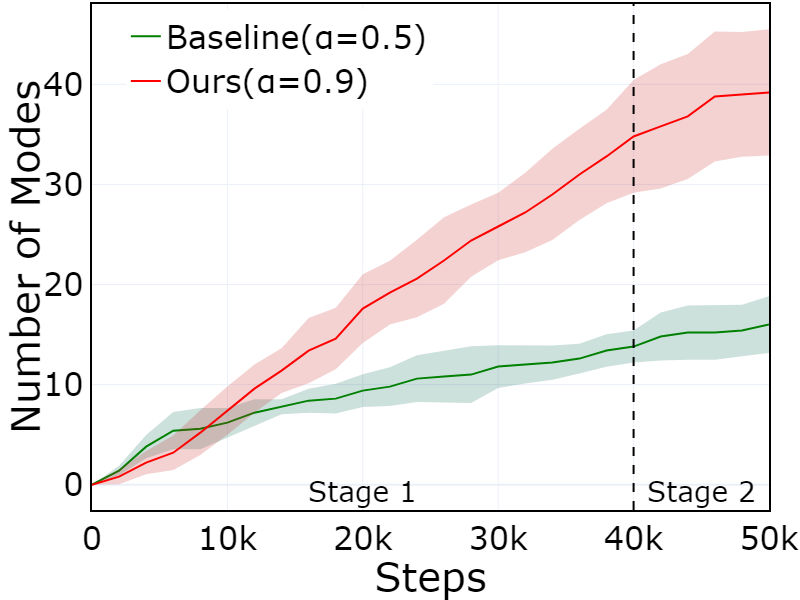} &
    \includegraphics[width=.3\textwidth]{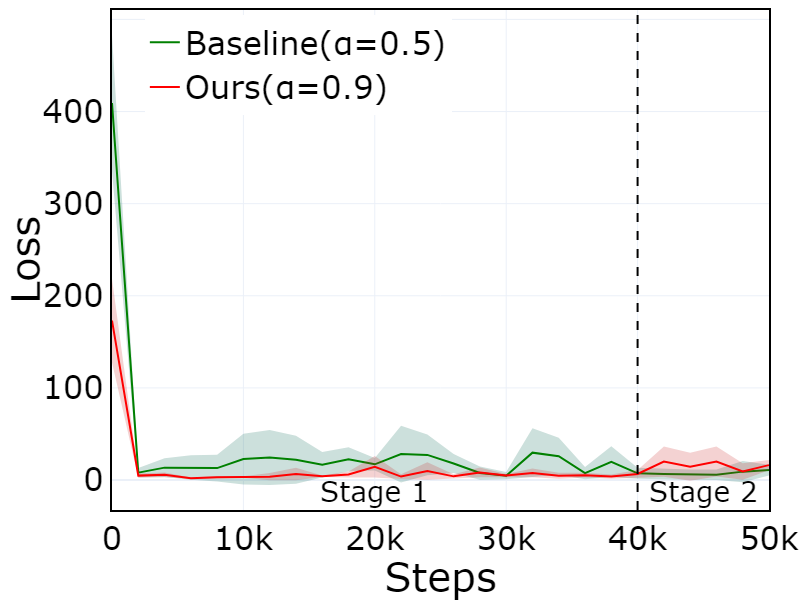} &
    \includegraphics[width=.3\textwidth]{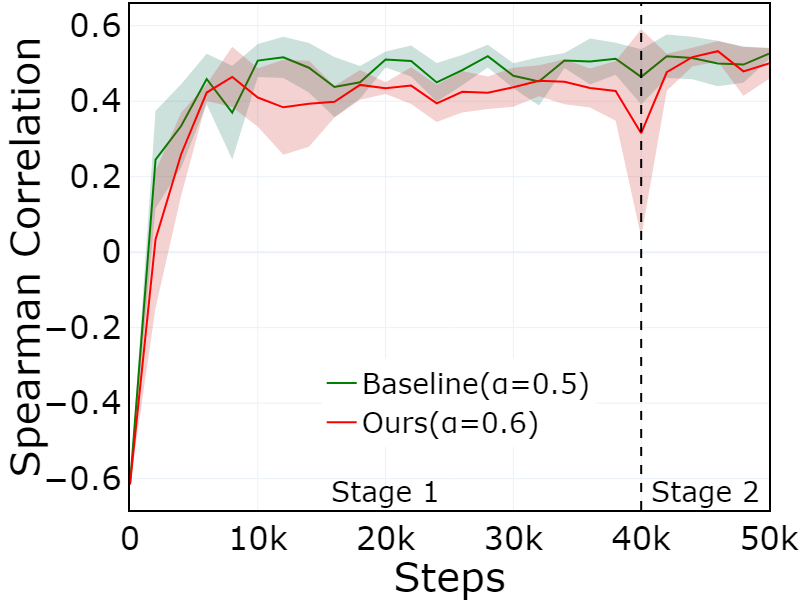} \\
         (a) Modes of $\alpha$-SubTB($\lambda$) & 
         (b) Loss of $\alpha$-SubTB($\lambda$) &
         (c) Spearman correlation of $\alpha$-DB \\
  \end{tabular}
  \caption{A case study on \textbf{(a)} the mode curves of $\alpha$-FL-SubTB, \textbf{(b)} the loss curves of $\alpha$-FL-SubTB($\lambda$) and \textbf{(c)} the Spearman correlation curves of $\alpha$-DB in Molecule Generation during training. The stage 1 and stage 2 in Alg.~\ref{alg:two-staged-training} are marked in the figures.}
\label{fig:stability and annealing}
\end{figure*}

In \textbf{Set Generation}, the results in Table~\ref{tab:set-exp-results} demonstrate that $\alpha$-GFN objectives consistently outperform their vanilla counterparts across all settings. For small sets, $\alpha$-GFNs achieve $125\%$, $6.7\%$, $71\%$, and $6.8\%$ more discovered modes for DB, FL-DB, SubTB($\lambda$), and TB, respectively. While both FL-SubTB($\lambda$) and $\alpha$-FL-SubTB($\lambda$) nearly reach the maximum capacity of 90 modes in this setting, our approach still maintains a slight edge. 
The performance margin becomes even more pronounced as task difficulty increases. In the medium and large set settings, while several vanilla objectives yield zero discovered modes, $\alpha$-GFNs consistently identify a non-trivial number of unique high-reward samples. 
Most notably, $\alpha$-GFNs provide substantial improvements over the state-of-the-art FL baselines: for FL-DB, the mode count increases by $735\%$ on medium sets and $804\%$ on large sets; for FL-SubTB($\lambda$), the gains are even more dramatic at $1516\%$ and $233\%$, respectively. 
These results are accompanied by a substantial rise in average reward, with Top-1000 R improvements ranging from $1.23\times$ to as much as $58.16\times$ in medium and large sets, confirming that the enhanced reward exploitation capability of $\alpha$-GFNs facilitates effective mode discovery.

In \textbf{Bit Sequence Generation} (Table~\ref{tab:bit-exp-results}), $\alpha$-GFNs outperform vanilla GFlowNets in 21 of 25 settings, uncovering up to 8 additional modes on average. Vanilla GFlowNets lead in only 2 settings with a small margin (at most 2 modes), and in the remaining 2 settings both methods achieve the maximum number of modes.
The consistent performance gains across varying word lengths $k$ underscores the robustness of $\alpha$-tuning to action space granularity. These results suggest that the optimal policy mixing often deviates from the vanilla $\alpha=0.5$ case in high-dimensional discrete spaces, indicating that adjusting the trade-off between exploration and exploitation via $\alpha$ is beneficial for maximizing mode discovery in high-dimensional discrete spaces.

Finally, we evaluate $\alpha$-GFNs on a real-world \textbf{Molecule Generation} task. As shown in Table~\ref{tab:mols-exp-results}, incorporating $\alpha \neq 0.5$ yields varying degrees of improvement in mode discovery across all five objectives. Specifically, we observe increases in discovered modes of $44\%$ for DB, $177\%$ for FL-DB, $19\%$ for SubTB($\lambda$), $145\%$ for FL-SubTB($\lambda$), and $5\%$ for TB. These gains are often mirrored by higher average rewards, suggesting that $\alpha$ allows for a more effective balancing of exploration and exploitation. Collectively, these results validate $\alpha$-tuning as a versatile enhancement, demonstrating its capability to consistently improve mode discovery across diverse and challenging domains.

Across almost all experiments, $\alpha$-GFNs maintain Spearman correlations comparable to vanilla GFlowNets, suggesting that Alg.~\ref{alg:two-staged-training} preserves the fundamental property $P_F^{\top}(x) \propto R(x)$. In some cases, $\alpha$-GFNs even show a better fit to the reward distribution. For instance, in Molecule Generation, $\alpha$-TB achieves a twofold improvement in correlation over vanilla TB with significantly lower variance. Notably, Spearman correlation and mode discovery are not strictly coupled. In the case of DB for Molecule Generation, $\alpha$-GFN identifies 4 additional modes despite a marginal 0.03 decrease in correlation. This suggests that even when correlation drops, $\alpha$-GFNs often discover more modes, which aligns with the enhanced flexibility enabled by $\alpha$

App.~\ref{app:additional results} provides more results of the three tasks. 
\begin{table*}[t]
\caption{\textit{Ablation studies on number of modes vs. $\alpha$ in Molecule Generation.} We \textbf{bold} $\alpha \neq 0.5$ entries that discover more modes than the $\alpha =0.5$ baseline. Standard deviations are in \textcolor{gray}{gray}.}
\centering
\small
\setlength{\tabcolsep}{6pt}
\renewcommand{\arraystretch}{1.2}
\resizebox{\linewidth}{!}{%
\begin{tabular}{@{\hspace{.6em}}cccccccccc@{\hspace{.6em}}}
\toprule
\multirow{1}{*}{\textbf{Modes}$\uparrow$} 
& \multicolumn{9}{c}{$\alpha$} \\
\cmidrule(lr){2-10}
\textbf{Objective}  & 0.1 & 0.2 & 0.3 & 0.4 & 0.5 (Baseline) & 0.6 & 0.7 & 0.8 & 0.9 \\
\midrule
\multirow{1}{*}{DB}
& \msl{5.40}{0.89}  & \msl{10.00}{4.18} & \msl{\textbf{11.20}}{4.76} & \msl{\textbf{11.40}}{3.21} & \msl{10.00}{3.08} & \msl{\textbf{14.40}}{2.41} & \msl{\textbf{11.80}}{3.96} & \msl{\textbf{11.00}}{4.90} & \msl{\textbf{13.40}}{3.21} \\
\multirow{1}{*}{FL-DB}
& \msl{5.80}{1.48}  & \msl{5.40}{1.14} & \msl{6.40}{3.13} & \msl{\textbf{9.17}}{2.14} & \msl{9.00}{1.58} & \msl{\textbf{13.20}}{4.66} & \msl{\textbf{14.60}}{6.07} & \msl{\textbf{24.00}}{3.61} & \msl{\textbf{25.00}}{4.36} \\
\multirow{1}{*}{SubTB($\lambda$)}
& \msl{20.20}{5.63} & \msl{18.20}{4.49} & \msl{\textbf{22.80}}{3.77} & \msl{\textbf{24.00}}{6.44} & \msl{22.20}{2.77} & \msl{\textbf{26.40}}{2.30} & \msl{22.00}{5.43} & \msl{\textbf{25.20}}{3.63} & \msl{\textbf{24.00}}{5.79} \\
\multirow{1}{*}{FL-SubTB($\lambda$)}
 & \msl{8.40}{3.21}  & \msl{11.00}{3.94} & \msl{12.00}{4.95} & \msl{12.60}{2.61} & \msl{16.00}{3.16} & \msl{\textbf{18.20}}{5.31} & \msl{\textbf{23.40}}{6.62} & \msl{\textbf{34.20}}{5.67} & \msl{\textbf{39.20}}{7.05} \\
\multirow{1}{*}{TB}
& \msl{19.80}{3.83} & \msl{29.60}{7.30} & \msl{37.00}{7.42} & \msl{34.40}{6.54} & \msl{38.40}{6.11} & \msl{\textbf{40.20}}{5.07} & \msl{36.80}{4.44} & \msl{\textbf{39.20}}{9.81} & \msl{\textbf{60.60}}{32.39} \\
\bottomrule
\end{tabular}
}
\label{tab:ablation-molecule-modes}
\end{table*}

\subsection{Analysis}

\noindent\textbf{Stability and Effects of Scheduling.} 
Although we observe clear performance gains across tasks, the standard deviation especially in the number of modes, sometimes increases, raising stability concerns. To probe this and the effect of annealing $\alpha$ on reward fitting, we conduct a case study on Molecule Generation. As shown in Fig.~\ref{fig:stability and annealing}(a,b), $\alpha$-SubTB($\lambda$) consistently outperforms SubTB($\lambda$) in stage 1 (before annealing) with $\alpha=0.9$ and maintains its advantage in stage 2 (annealing), where exponentially fast annealing induces mild loss oscillations that settle after roughly 8,000 steps, leaving a modest increase in mode variance. Nevertheless, a clear performance gain throughtout the training is still clearly observed. Meanwhile, Fig.~\ref{fig:stability and annealing}(c) shows that annealing restores reward fitting. The Spearman correlation of $\alpha$-DB rebounds from around 0.4 before stage 2 to about 0.5 afterward. In summary, the two-stage schedule preserves the performance gains and recovers reward fitting at the cost of a brief, limited increase in variance in some cases.

\noindent\textbf{Length-Controlling Side Effects.} 
With a `stop' action allowing $P_F$ to choose trajectory length, the trade-off in Prop.~\ref{prop:gradient of alpha-gfn objectives} induces an additional effect. 
As demonstrated in Fig.~\ref{fig:sample length and entropy}, as $\alpha$ is set larger, trajectories are lengthened since stronger exploitation may shift mass from `stop' to constructive actions. 
Intriguingly, Fig.~\ref{fig:sample length and entropy} also shows that average per-action entropy increases with trajectory length, plausibly reflecting accumulated uncertainty over longer horizons. A formal analysis is left to future work.

\begin{figure}[htbp]
    \centering
  \includegraphics[width=.8\linewidth]{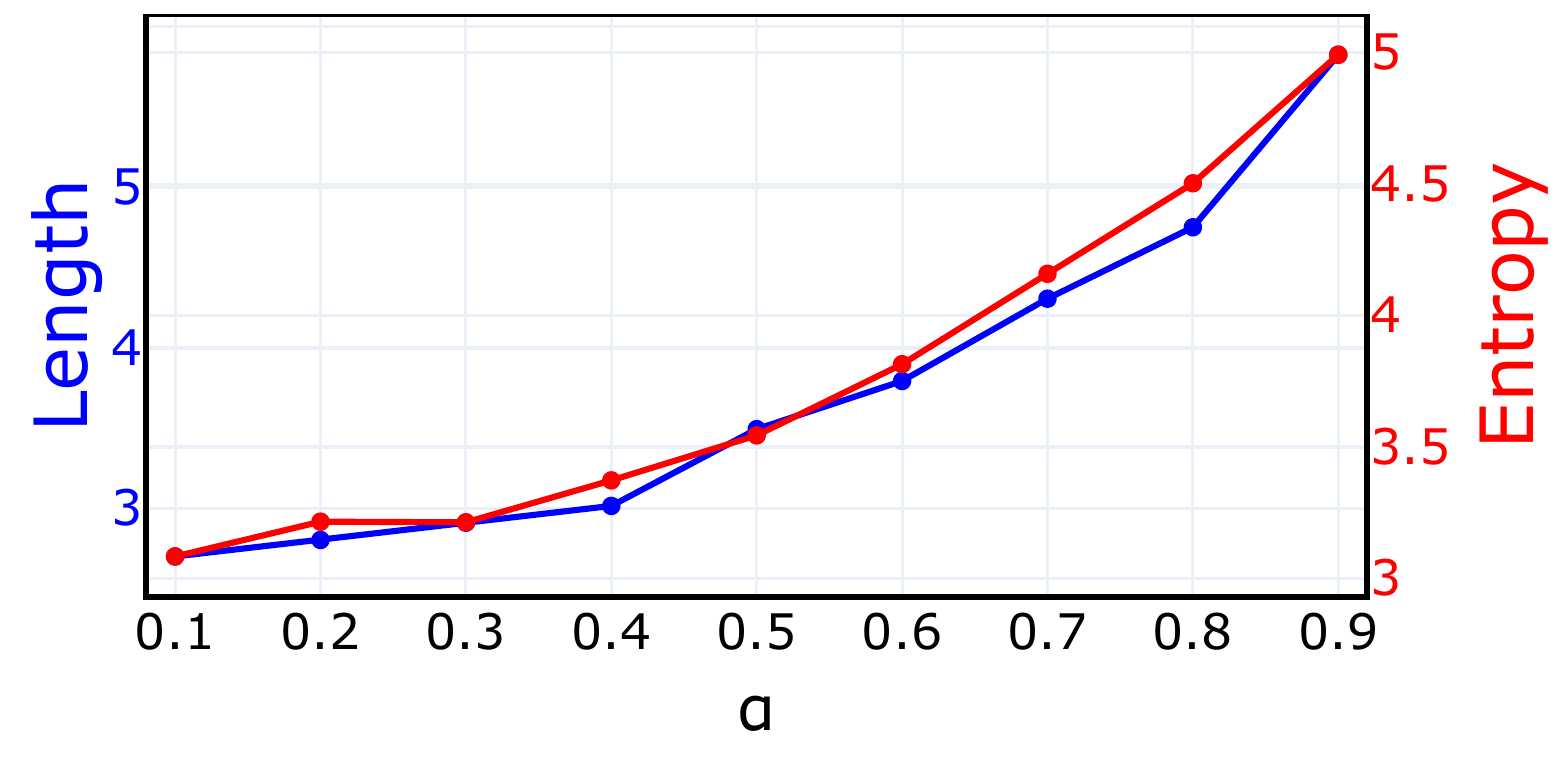}
  \caption{Average sample length and forward policy entropy with $\alpha$-FL-SubTB in Molecule Generation. We observe that both the length and the entropy increases with $\alpha$ in this case.}
  \vspace{-0.4\baselineskip}
\label{fig:sample length and entropy}
\end{figure}

\noindent\textbf{Ablation Studies.} 
We assess the sensitivity of $\alpha$-GFNs via an ablation on the number of modes across $\alpha$ settings for molecule generation. As shown in Table~\ref{tab:ablation-molecule-modes}, performance gains persist even when $\alpha$ is not tuned to its optimal value: for DB, FL-DB, SubTB($\lambda$), FL-SubTB($\lambda$), and TB, adjusting $\alpha$ generally yields more modes. These results indicate that $\alpha$-GFNs exhibit low sensitivity to precise $\alpha$ selection while still delivering consistent improvements.

Additionally, we show in App.~\ref{app:compatibility} that the flexible exploration--exploitation trade-off of $\alpha$-GFNs can be integrated into a broad range of GFlowNet training recipes, including Adaptive Teachers~\citep{kim2024adaptive}, QGFN~\citep{lau2024qgfn}, FlowRL~\citep{zhu2025flowrl}, and reward temperature scaling.
\section{Related Works}
 
\paragraph{GFlowNet theories and connections with Markov chains.}
While GFlowNets were first formalized at the intersection of flow networks and MDPs~\citep{bengio2021flow}, many subsequent developments (e.g., detailed balance conditions~\citep{bengio2023gflownet}) have been strongly influenced by Markov chain (MC) theory, highlighting the value of a precise and rigorous connection to classical MC frameworks.
In particular,~\citet{deleu2023generative} introduced a GFlowNet induced MC under the FM objective using only the forward policy. 
Our work extends this perspective in two ways: first, by explicitly incorporating the backward policy as a key component of GFlowNets and their Markov chain connection; and second, by generalizing the analysis to a broader family of objectives, including DB, SubTB, TB, and their FL variants.
\paragraph{GFlowNet objective design.}
Standard GFlowNet objectives, including DB \citep{bengio2023gflownet}, SubTB \citep{madan2023learning}, and TB \citep{malkin2022trajectory}, are trained with both forward and backward policies to enforce balanced flows, which implicitly treats the two policies symmetrically. Building on this paradigm, temperature conditional GFlowNets \citep{zhang2023robust, zhou2023phylogfn, kim2024learning} scale the reward, forward-looking GFlowNets \citep{pan2023better, jang2023learning} introduce intermediate energies , and \citet{hu2025beyond} modifies the loss forms. While these methods adhere to the flow matching framework, we consider a more general weighting scheme that departs from the strict balance condition by dynamically mixing the forward and backward policies through a tunable parameter, enabling controlled and principled flow imbalance.

\section{Conclusion}

In this work, we uncover the implicit equal weighting of forward and backward policies in GFlowNet objectives through an extended connection to Markov chains.
Building on this finding, we introduce $\alpha$-GFNs, which provide a controllable exploration–exploitation trade-off by mixing forward and backward policies with a hyperparameter $\alpha$. The convergence of these mixed objectives is established via their link to Markov chains.
We further explain the role of different $\alpha$ values through a gradient-based analysis and propose a scheduled training algorithm that combines the benefits of multiple $\alpha$ values.
Experiments across diverse domains show that $\alpha$-GFNs consistently outperform vanilla GFlowNets by discovering more high-reward modes while maintaining sample diversity, highlighting the practical value of $\alpha$.
These results strengthen the link between Markov chain theory and GFlowNet practice and open promising directions for future research.

\section*{Impact Statement}
This paper presents work whose goal is to advance the field of Machine
Learning. There are many potential societal consequences of our work, none
which we feel must be specifically highlighted here.

\section*{Acknowledgements}
This work is sponsored by the National Natural Science Foundation of China (NSFC) grant (No. 62576211) and Fastbox (No.SA0710007).

\bibliography{biblio}
\bibliographystyle{icml2026}

\newpage
\appendix
\onecolumn

\section{GFlowNet Objective Definitions and Related Discussions}

\subsection{Definitions}
\label{app:defs}
In this section, we provide the formal definitions of vanilla GFlowNet training objectives where both the forward policy $P_F$ and the backward policy $P_B$ are used, including the Detailed Balance Loss (DB,~\cite{bengio2023gflownet}, Trajectory Balance Loss (TB,~\cite{malkin2022trajectory} and the convex combination of SubTB, SubTB($\lambda$)~\citep{madan2023learning} along with Forward-Looking (FL) ~\citep{pan2023better} variants of DB (FL-DB) and SubTB($\lambda$) (FL-SubTB($\lambda$)).

\begin{definition}{(DB, \citet{bengio2023gflownet})}
\label{def:DB}
    Given a state flow function $F(\cdot):S \rightarrow \mathbb{R}_{+}$, a forward policy $P_F(\cdot|\cdot): \mathbb{A} \rightarrow [0,1]$, a backward policy $P_B(\cdot|\cdot): \mathbb{A} \rightarrow [0,1]$, Detailed Balance targets at
    \begin{equation}
        F(s)P_F(s'|s) = F(s) P_B(s|s'), \quad \text{for every } (s,s') \in \mathbb{A}
    \end{equation}
    where  $\forall s \in \mathcal{X}, F(s) P_F(s_f|s)=R(s)$.
    And the corresponding loss function is
    \begin{equation}
L_{\text{DB}}(s,s')=\log^2\left(\frac{F(s)P_F(s'|s)}{F(s')P_B(s|s')}\right).
    \end{equation}

\end{definition}

\begin{definition}{(TB, \citet{malkin2022trajectory})}
\label{def:TB}
    Given a state flow function $F(\cdot):S \rightarrow \mathbb{R}_{+}$, a forward policy $P_F(\cdot|\cdot): \mathbb{A} \rightarrow [0,1]$, a backward policy $P_B(\cdot|\cdot): \mathbb{A} \rightarrow [0,1]$, 
    for a complete trajectory $\mathfrak{t}^f=(s_0=s_s,s_1,\dots,s_{n-1}=x,s_n=s_f) \in \mathfrak{T}$ where $\mathfrak{T}^{\text{flow}}$ is the set of complete trajectories in the graph, 
    TB targets at
    \begin{equation}
    \label{eq:tb target}
        F(s_0)P_F(s_f|x)\prod\limits_{i=1}^{n-1} P_F(s_i | s_{i-1})
        =
        R(s_n) \prod\limits_{i=1}^{n-1} P_B(s_{i-1} | s_i)
    \end{equation}
    And the corresponding loss function is
    \begin{equation}
        L_{\text{TB}}(\mathfrak{t}^f)=\log^2\left(\frac{F(s_0) P_F(s_f|x)\prod\limits_{i=1}^{n-1} P_F(s_i | s_{i-1})}{R(s_n) \prod\limits_{i=1}^{n-1} P_B(s_{i-1} | s_i)}\right).
    \end{equation}
    Note that $F(s_0)=F(s_f)=\sum_{x\in \mathcal{X}}R(x)$ and $P_B(x|s_f)=\frac{R(x)}{\sum\limits_{x \in \mathcal{X}}R(x)}$.
\end{definition}
\begin{definition}{(SubTB,~\cite{madan2023learning})}
\label{def:subtb}
    Given a partial trajectory $\mathfrak{t}'=(s_k,s_{k+1},\dots,s_{k+m}) \subset \mathfrak{t}^f \in \mathfrak{T}^{\text{flow}}$, a state flow function $F(\cdot):S \rightarrow \mathbb{R}_{+}$, a forward policy $P_F(\cdot|\cdot): \mathbb{A} \rightarrow [0,1]$, a backward policy $P_B(\cdot|\cdot): \mathbb{A} \rightarrow [0,1]$, SubTB targets at~\eqref{eq:subtb}, and the corresponding loss function is
    \begin{equation}
    \label{eq:subtb loss}
        L_{\text{SubTB}}(\mathfrak{t}')=\log^2
        \left(
        \frac{F(s_k)\prod_{i=1}^m P_F(s_{k+i}|s_{k+i-1})}{F(s_{k+m})\prod_{i=1}^{m}P_B(s_{k+i-1}|s_{k+i})}
        \right).
    \end{equation}
\end{definition}
\begin{definition}{(SubTB($\lambda$), \citet{madan2023learning})}
\label{def:SubTB(lambda)}
    Given a complete trajectory $\mathfrak{t}^f=(s_0=s_s,s_1,\dots,s_{n-1}=x,s_{n}=s_f) \in \mathfrak{T}^{\text{flow}}$, define an extracted subtrajectory $\mathfrak{t}_{i:j}$ to be 
    \begin{equation}
        \mathfrak{t}_{i:j}=(s_i,s_{i+1},\dots,s_j), \qquad 0 \leq i < j \leq n .
    \end{equation}
    Then, SubTB($\lambda$) loss for the complete trajectory $\mathfrak{t}$ is a convex combination of SubTB loss at these partial trajectories, i.e.
    \begin{equation}
        L_{\text{SubTB}}(\mathfrak{t}^f,\lambda)=\frac{\sum\limits_{0 \leq i < j \leq n} \lambda^{j-i}L_{\text{SubTB}}(\mathfrak{t}_{i:j})}{\sum\limits_{0 \leq i < j \leq n} \lambda^{j-i}}
    \end{equation}
    where $\lambda \in (0,+\infty)$.
\end{definition}

Additionally, we consider the Forward Looking (FL)~\citep{pan2023better} variants of DB and SubTB($\lambda$), termed FL-DB and FL-SubTB($\lambda$) respectively. We start with an introduction to the state-level and edge-level energy function.
\begin{definition}{~\citep{pan2023better}}
\label{def:energy function}
    Given a state-level energy function $\mathcal{E}(\cdot): S \rightarrow \mathbb{R}$, the edge-level energy function is an extension of its state-level counterpart, i.e.
    \begin{equation}
        \mathcal{E}(s,s')=\mathcal{E}(s')-\mathcal{E}(s) \quad \text{for all } (s,s') \in \mathbb{A}.
    \end{equation}
\end{definition}
Applying Def.~\ref{def:energy function} to DB and SubTB($\lambda$), the FL variants are obtained:
\begin{definition}{(FL-DB, \cite{pan2023better})}
\label{def:FL-DB}
    Based on the the definition of DB in Def. \ref{def:DB}, with an edge-level energy function $\mathcal{E}(\cdot,\cdot): \mathbb{A} \rightarrow \mathbb{R}$, Forward-Looking Detailed Balance (FL-DB) targets at:
    \begin{equation}
    F(s)P_F(s'|s)=F(s')P_B(s|s')e^{-\mathcal{E}(s,s')}
    \end{equation}
    and the corresponding loss function is
    \begin{equation}
        L_{\text{FL-DB}}(s,s')=\log^2\left(\frac{F(s)P_F(s' | s)}{F(s')P_B(s|s')e^{-\mathcal{E}(s,s')}}\right).
    \end{equation}
\end{definition}

\begin{definition}{(FL-SubTB($\lambda$),~\cite{pan2023better})}
    Based on the definition of SubTB in Def.~\ref{def:subtb}, given a partial trajectory $\mathfrak{t}_{i:j}=(s_i,s_{i+1},\dots,s_j) \subset \mathfrak{t} \in \mathfrak{T}^{\text{flow}}$ and an edge-level energy function $\mathcal{E}(\cdot,\cdot):\mathbb{A} \rightarrow \mathbb{R}$, FL-SubTB targets at
    \begin{equation*}
F(s_k)\prod_{i=1}^m P_F(s_{k+i}\mid s_{k+i-1})
\;=\;
F(s_{k+m}) \prod_{i=1}^{m} P_B(s_{k+i-1}\mid s_{k+i}) \, e^{-\mathcal{E}(s_{k+i-1},s_{k+i})}\, .
\end{equation*}
    and the corresponding loss function is 
    \begin{equation}
        L_{\text{FL-SubTB}}(\mathfrak{t}_{i:j})=\log^2\left(\frac{F(s_i)\prod_{k=1}^{j}P_F(s_{i+k}|s_{i+k-1})}{F(s_j)\prod_{k=1}^{j}P_B(s_{i+k-1}|s_{i+k})}
        \right)
    \end{equation}
    Following Def.~\ref{def:SubTB(lambda)}, FL-SubTB($\lambda$) is a convex combination of FL-SubTB. Given $\lambda \in (0,+\infty)$, FL-SubTB($\lambda$) is 
    \begin{equation}
        L_{\text{FL-SubTB}}(\mathfrak{t}^f,\lambda)=\frac{\sum\limits_{0 \leq i < j \leq n} \lambda^{j-i}L_{\text{FL-SubTB}}(\mathfrak{t}_{i:j})}{\sum\limits_{0 \leq i < j \leq n} \lambda^{j-i}}.
    \end{equation}
\end{definition}

\subsection{Discussions}
To further demonstrate that $\alpha$-GFNs are compatible with vanilla GFlowNet training techniques, 
we formalize FL~\citep{pan2023better} within the Markov chain framework in a matrix form, which suggests FL is a prior to the probability measures:
\begin{theorem}
\label{thm: FL explained}
    FL adds a prior to the probability measures of a GFlowNet, i.e.  $\pi=\Tilde{\pi}\mathcal{E}$, 
    where $\mathcal{E}=diag(e^{-\mathcal{E}(s_0)},\dots,e^{-\mathcal{E}(s_f)})$ is the diagonal matrix constructed with exponent of the energy function $\mathcal{E}(\cdot)$.

\end{theorem}
\begin{proof}
    Assumpation 4.1 and Prop. 4.2 in~\cite{pan2023better} suggests
    \begin{equation*}
        F=\Tilde{F} \mathcal{E}
    \end{equation*}
    is the FL reparameterization of flows, where $F=(F(s_s),\dots,F(s_f))$ is the vector of flows, and $\Tilde{F}=(\Tilde{F}(s_s),\dots,\Tilde{F}(s_f))$ is the reparameterized counterpart. Coupling with the fact that state flows are state-level probability measures~\citep{deleu2023generative}, it is direct that
    \begin{equation*}
        \pi=\Tilde{\pi}\mathcal{E}
    \end{equation*}
    where $\pi$ and $\Tilde{\pi}$ are the probability measures corresponding to $F$ and $\Tilde{F}$ respectively.
\end{proof}
\newpage

\section{Supporting Theoretical Framework}
\label{app:supporting theoretical framework}
The following theoretical part addresses the fundamental results of our paper under the assumption that the state space is finite.
This assumption is reasonable since the state space of GFlowNets are compositional~\citep{bengio2021flow,bengio2023gflownet}.
Throughout, we denote by $P$ the true transition probability of the graph to be learned, which corresponds to the forward policy $P_F$ in GFlowNets. We first summarize the results we prove and then prove the results one by one in subsections.
\begin{enumerate}[label=\textbf{(\roman*)}]
  \item \textbf{From GFN to MC}: We show that GFlowNets can be represented as irreducible Markov Chains.

    In particular, the resulting Markov chain is positive recurrent, admits a unique stationary distribution $\pi$ to which the Markov chain converges from any state.

  \item \textbf{From MC to GFN under constraints}: We show that every reducible Markov chain with transition probability $P$ satisfying a specific finite set of linear constraints is a GFlowNet.

    We call such Markov chains \textbf{GFNMC} (GFlowNet Markov chain).
    Such a Markov chain allows for a detailed balance $\tilde{P}$, which is again a GFNMC sharing the same trajectories (in the reverse direction), the same stationary distribution as well as the same eigenvalues (not necessarily the same eigenvectors).

    Furthermore any $\alpha$-GFNMC is again irreducible, share the same stationary distribution but not necessarily the same eigenvalues.
  \item \textbf{Convergence rate}:
    We show that the periodicity of every GFNMC (GFlowNet Markov Chain) is equal to the greatest common denominator of all the trajectory length.
    It is in particular the case for any reasonable graph with various trajectory length to be aperiodic (periodicity of $1$).

    If not, for instance if all trajectories share the same length, we show that for any $0<\alpha<1$, the periodicity of the $\alpha$-GFN is either $2$ or $1$.

    This is important in terms of convergence as aperiodic Markov chains are ergodic geometric, that is the convergence to the stationary distribution also holds in total variation as exponential rate determined by the second largest eigenvalue.
    In terms of sampling, such results allow for a central limit theorem akin to classical Monte Carlo.
    
    Denote with $\beta$, $\tilde{\beta}$ and $\beta_\alpha$ the second largest eigenvalues.
    While $\beta = \tilde{\beta}$, there is no result as $\beta_\alpha$ as a function of $\beta$ as it is highly non-linear.
    Nevertheless, it might drastically improve by mixing.
    It justifies the approach to take the hyper parameter $\alpha$ as trainable too to achieve a good convergence and a good sampling result.
\end{enumerate}
Prior work~\citep{deleu2023generative} touches on \textbf{(i)}, but it does not fully formalize the Markov-chain-to-GFlowNet (MC$\to$GFN) connection. In particular, key Markov chain details underlying this link are left unspecified, and the remaining aspects, 
\textbf{(ii)} and \textbf{(iii)}, 
have not, to our knowledge, been investigated. We address these gaps and develop a more complete account of the relationship between GFlowNets and Markov chains.
\subsection{GFNs are MCs}
We start with giving a detailed version of notations in Sec.~\ref{sec:preliminaries}.
Following ~\cite{bengio2021flow}, we consider a \emph{directed graph} $(S, \mathbb{A})$ where $S$ is a finite state space and $\mathbb{A}\subseteq S\times S$ is a set of \emph{edges} between states where $s\to s^\prime=(s, s^\prime) \in \mathbb{A}$ denotes an edge.
A \emph{trajectory} is a finite sequence $\mathfrak{t}^f = (s_0, \ldots, s_N)$ where $s_n\to s_{n+1}$ is an edge for any $n\leq N-1$ where $N\geq 1$.
The directed graph is called \emph{acyclic} if every such trajectory satisfies $s_0 \neq s_N$.
A directed acyclic graph is called \emph{pointed} if there exists a \emph{source state} $s_s$ and a \emph{final state} $s_f$ such that for every other state $s$, there exists a trajectory starting from $s_s$, running through $s$ and ending in $s_f$.
Any such finite trajectory starting in $s_s$ and ending in $s_f$ is called \emph{complete}.\footnote{Note that by definition, any state in $S$ is element of a complete trajectory.}
From now one we only consider pointed directed acyclic graphs where the GFlowNet theoretical framework is built~\citep{bengio2023gflownet} and denote by $\mathfrak{T}^{\text{flow}}$ the set of complete trajectories $\mathfrak{t}^f=(s_0, \ldots, s_N)$ where $s_0 = s_s$ and $s_N = s_f$.

To realise the link with Markov chains, let $\mathfrak{T}=S^{\mathbb{N}_0}$ be the set of infinite trajectories $\mathfrak{t} = (s_0, s_1, \ldots)$ endowed with the product $\sigma$-algebra $\mathcal{T} = \otimes S$ where $S=2^S$.
We further denote by $X = (X_t)$ the \emph{canonical process} on $\mathfrak{T}$ where $X_n(\mathfrak{t}) = s_n$ is the $n$-th state of the trajectory $\mathfrak{t}$.
In other terms $X$ is the identity on $\mathfrak{T}$ since $X(\mathfrak{t})=\mathfrak{t}$.\footnote{Note that any Markov chain is about defining a probability measure on the space of trajectories $\mathfrak{T}$ making the canonical process to satisfy the Markov property.}
We finally denote by $\mathcal{T}_n = \sigma(X_m\colon m \leq n)$ the filtration generated by the canonical process.

Using a finite Kolmogorov extension argument, \cite{bengio2021flow} show that any Markovian probability measure $\mathbf{P}^{\text{flow}}$ on $\mathfrak{T}^{\text{flow}}$ is uniquely given by a transition probability $P(s^\prime |s) = P_{ss^\prime}$ for every edge $s \to s^\prime$.
\begin{remark}
  Throughout, we assume that $P(s^\prime | s)>0$ for every edge $s\to s^\prime$.
\end{remark}
This transition probability can be extended to $S\times S$ by defining $P(s_s | s_f) = 1$ and $P(s^\prime | s) = 0$ for any other pair $s\to s^\prime$ which is not an edge.
In other terms the probability of returning to the source state from the final state is equal to $1$.
Such an extension also defines a Markovian probability $\mathbf{P}^{s_s}$ on the space of infinite trajectories with corresponding Markov chain starting at the source state $s_s$. 
The following theorem embed Markovian GFlowNets into a subset of Markov Chains with specific properties.

\begin{theorem}
\label{thm:gfn to mc}
  The Markovian GFlowNet probability $\mathbf{P}^{\text{flow}}$ \textbf{coincides exactly} with $\mathbf{P}^{s_s}$ on the set of complete trajectories $\mathfrak{T}^{\text{flow}}$.
  Furthermore, the resulting Markov Chain is \textbf{irreducible}.
\end{theorem}
In order to state Thm.~\ref{thm:gfn to mc} correctly as $\mathfrak{T}^{\text{flow}}$ is not even a subset of $\mathfrak{T}$, let us consider random times and stopping times. 
\begin{definition}
  A function $\tau\colon \mathfrak{T} \to \mathbb{N}_0 \cup\{\infty\}$ is called a \emph{random time} if $\tau$ is measurable and a \emph{stopping time} if $\{\tau \leq n\}$ is in $\mathcal{T}_n$ for every $n$.
\end{definition}
Typical example of stopping times are the first time a trajectory $\mathfrak{t}$ satisfies a condition.
In particular, we make use of 
\begin{equation*}
  \tau^s(\mathfrak{t}) = \inf\{n\colon X_n(\mathfrak{t})=s\}\quad \text{and}\quad \tau^s_+(\mathfrak{t}) = \inf\{n\geq 1\colon X_n(\mathfrak{t})=s, n>0\},
\end{equation*}
which are the first time and first time $\geq 1$ such that the trajectory $\mathfrak{t}$ visits the state $s$, respectively.
Given a stopping time $\tau<\infty$, we denote by $X_\tau(\mathfrak{t}) = \mathfrak{t}_{\tau(\mathfrak{t})}$ the value of the trajectory at this particular random time $\tau(\mathfrak{t})$, as well as the stopped Markov chain $X^\tau:=(X_{n\wedge \tau})$.
The stopped Markov chain is just a slice of the trajectory $\mathfrak{t}$ between $0$ and $\tau(\mathfrak{t})$ and constant afterwards:
\begin{equation*}
  X^\tau(\mathfrak{t}) =(\underbrace{s_0, \ldots, s_{\tau(\mathfrak{t})}}_{\text{slice before }\tau(\mathfrak{t})}, \underbrace{s_{\tau(\mathfrak{t})}, \ldots}_{\text{constant after}}).
\end{equation*}

Going back to flows, we define
\begin{equation*}
  \sigma(\mathfrak{t}) = 
  \begin{cases}
    N & \text{if } (s_0, \ldots, s_N) \in \mathfrak{T}^{\text{flow}} \text{ is a complete flow trajectory for some }N\\
    0 & \text{otherwise}
  \end{cases}
\end{equation*}
as the function returning the length of the complete flow trajectory if the infinite trajectory starts in $\mathfrak{T}^{\text{flow}}$ and $0$ otherwise.

\begin{lemma}
  The function $\sigma$ is a uniformly bounded random time.
  Furthermore, $\mathfrak{T}^{\text{flow}}$ is in bijection with $\{\mathfrak{t}\colon \sigma(\mathfrak{t})\geq 1\}$ which is a measurable subset of $\mathfrak{T}$.
\end{lemma}
\begin{proof}
  Since the state space is finite, it follows that $\sigma<\# S +1$ and therefore uniformly bounded.
  It follows that the number of elementary flows is finite and therefore $\sigma$ is a finite sum of simple random variables greater than $1$, hence a random variable.
  Finally, by definition $\mathfrak{T}^{\text{flow}}$ is in bijection with $\{\sigma\geq 1\}$ and therefore measurable since 
  $\sigma$ is measurable.
\end{proof}
\begin{remark}
  If the state space is infinite, then elementary flows might be of arbitrary length.
  Hence, the set of elementary flows might be of uncountable cardinality.
  In this case, the measurability argument does not hold without further assumptions.
  Also, even if $\sigma$ is measurable, it is not clear a priori that it is a stopping time, and in general likely not.
\end{remark}
Even if $\sigma$ is not a stopping time, it coincides with the stopping time $\tau^{s_f}$ with probability $1$ for the Markov probability starting from $s_s$.
\begin{proposition}
  Let $\mathbf{P}^{s_s}$ be the Markovian probability starting from the source state $s_s$.
  Then it follows that
  \begin{equation*}
    \mathbf{P}^{s_s}[\sigma = \tau^{s_f}] = 1
\end{equation*}
  and $\mathbf{P}^{s_s}$ coincides with $\mathbf{P}^{\text{flow}}$ in the sense that for any $\mathfrak{t}^f$ in $\mathfrak{T}^{\text{flow}}$ it holds
  \begin{equation*}
    \mathbf{P}^{s_s}[X_{[0:\tau^{s_f}]}=\mathfrak{t}^f] = \mathbf{P}^{\text{flow}}[\mathfrak{t}^f] .
  \end{equation*}
\end{proposition}
\begin{proof}
  It is clear that for any trajectory $\mathfrak{t}$ such that $\sigma(\mathfrak{t})\geq 1$, it holds that $\tau^{s_f}(\mathfrak{t}) = \sigma(\mathfrak{t})$.
  And by the definition of $\mathbf{P}^{s_s}$ from the same transition probability, we get that
  \begin{equation*}
    \mathbf{P}^{s_s}[X_{[0:\tau^{s_f}]}=\mathfrak{t}^f] = \mathbf{P}^{s_s}[X_{[0:\sigma]}=\mathfrak{t}^f]=\mathbf{P}^{\text{flow}}[\mathfrak{t}^f]
  \end{equation*}
  In particular, $\mathbf{P}^{s_s}$ has measure $1$ on the set where $\{\mathfrak{t}\colon \sigma(\mathfrak{t})\geq 1\}$.
\end{proof}
Recall that a Markov chain is called \textbf{irreducible} if for every pair of states $s, s^\prime$ it holds that $\mathbf{P}^s[\tau^{s^\prime}<\infty]=1$ for any two states $s$ and $s^\prime$.
\begin{proposition}
\label{prop:irreducibility}
  The Markov chain resulting from $P$ is irreducible.
\end{proposition}
\begin{proof}
  Any flow trajectory in $\{\sigma \geq 1\}$ starting in $s_s$ will reach $s_f$ with probability $1$ in a bounded amount of steps.
  Furthermore, from $P(s^\prime|s)>0$ for any edge $s\to s^\prime$ shows that each state is visited by a flow trajectory with strict positive probability.
  The assumption that $P(s_s|s_f)=1$ and the fact that $P(s_s|s_s)=0$ shows that starting from any state $s$, the probability to reach $s_f$ and therefore $s_s$ is equal to $1$.
  From $s_s$ to any other state $s^\prime\neq s_s$ is with strict positive probability, hence from strong Markov property, it follows that any state $s^\prime$ is accessible from any state $s$ with strict positive probability, that is $\mathbf{P}^{s}[\tau^{s^\prime}<\infty] = 1$.
\end{proof}
Since the Markov chain with respect to the transition kernel $P$ is defined over a finite discrete state space, it is also \textbf{Harris recurrent} and \textbf{positive recurrent}.

\subsection{MC Constraints to be a GFN}

\begin{remark}
  For an easy definition of pointed directed graphs, it is better to distinguish between source state $s_s$ and final state $s_f$.
  However, since the transition from $s_s$ to $s_f$ is with probability one, from Markov Chain perspective, it is equivalent to identify them both by setting $s_s=s_f=\bar{s}$, which is also the practice in~\cite{deleu2023generative}.
\end{remark}
We consider a transition probability $P$ such that the resulting Markov chain is irreducible.
In particular, it is positive recurrent and Harris recurrent and has a stationary distribution $\pi$ since it is defined over a finite discrete state space.
For any state $s$, $\tau^s$ as well as $\tau^{s}_+$ are finite stopping time with finite expectation.

For this Markov Chain to be a GFlowNet, it is necessary and sufficient that any finite trajectory $\mathfrak{t}^f=(s_0 = \bar{s}, s_1, \ldots, s_N=\bar{s})$ does not contain inner loop.
In other terms for any state $s\neq \bar{s}$, the probability of returning to $s$ is $1$ and it shall pass first through $\bar{s}$ also with probability $1$.

\begin{theorem}
\label{thm:mc to gfn}
  A transition probability $P$ for an irreducible Markov chain is a GFlowNet if and only if
  \begin{equation}\label{eq:cond01}
    \mathbf{P}^s[\tau^{\bar{s}} < \tau^s_+] =1 \quad \text{for every } s\neq \bar{s}.
  \end{equation}

  The condition in \eqref{eq:cond01} translates into 
  \begin{equation*}
    \pi_s (Z_{\bar{s}\bar{s}} - Z_{s\bar{s}}) + \pi_{\bar{s}}(Z_{ss} - Z_{\bar{s}s}) = \pi_{\bar{s}},
  \end{equation*}
  where $Z$ is the fundamental matrix
  \begin{equation*}
    Z := (I - P + \Pi)^{-1} = \sum (P-\Pi)^n
  \end{equation*}
  with $\Pi$ being the matrix with each row equal to $\pi$.
\end{theorem}
Here, we use $Z$ to align with the classical Markov chain theory. \textbf{Note that this $Z$ is not the amount of flows in GFlowNets.}

\begin{proof}
  The first statement \eqref{eq:cond01} is clearly equivalent to the Markov Chain is concentrated onto a set of complete trajectories that corresponds to a pointed acyclic directed graph.
  In particular $\tau^{\bar{s}}_+\leq \# S$.
  Now from \citep[Corollary 2.8]{aldous2002reversible}, it holds that
  \begin{equation*}
    \mathbf{P}^s[\tau^{\bar{s}}<\tau^s_+] = \frac{1}{\pi_s(E^s[\tau^{\bar{s}}] + E^{\bar{s}}[\tau^s])}
  \end{equation*}
  while \citep[Lemma 2.12]{aldous2002reversible} states that
  \begin{equation*}
    \pi_s E^{s}[\tau^{\bar{s}}] =\frac{\pi_s}{\pi_{\bar{s}}} (Z_{\bar{s}\bar{s}} - Z_{s\bar{s}}) \quad \text{and}\quad \pi_s E^{\bar{s}}[\tau^{s}] =(Z_{ss} - Z_{\bar{s}s})
  \end{equation*}
  Together with \eqref{eq:cond01}, it yields
  \begin{equation*}
    \pi_s (Z_{\bar{s}\bar{s}} - Z_{s\bar{s}}) + \pi_{\bar{s}}(Z_{ss} - Z_{\bar{s}s}) = \pi_{\bar{s}}, \quad \text{for all } s\neq \bar{s}.
  \end{equation*}
\end{proof}

Given an irreducible Markov chain with transition probability $P$ and resulting stationary distribution $\pi$ we can define the balanced chain $\tilde{P}$ as
\begin{equation*}
  \tilde{P}_{s s^\prime} = \frac{\pi_{s^\prime}}{\pi_s} P_{s^\prime s}
\end{equation*}

\begin{proposition}
  The Markov chain $\tilde{P}$ is again irreducible with same stationary distribution $\pi$ and eigenvalues as $P$.
  Furthermore, it is also a GFNMC.
\end{proposition}
\begin{proof}
  The first statement is classical.
  Let us show that $\tilde{P}$ is a GFNMC.
  Denoting with $D = \mathrm{diag}(\pi)$, it holds that $\tilde{P} = D^{-1}P^\top D$.
  It is easy to check that $D\Pi D^{-1} = \Pi^\top$.
  It follows that
  \begin{align*}
    Z = (I-\tilde{P} - \Pi)^{-1}  &= (I - D^{-1}P^\top D - D^{-1}\Pi^{\top}D)^{-1}\\
                              &= D^{-1}\left(\left(I - P - \Pi\right)^{\top}\right)^{-1}D \\
                              &= D^{-1}\left(\left(I - P - \Pi\right)^{-1}\right)^{\top}D \\
                              &= D^{-1}Z^\top D
  \end{align*}
  showing that the fundamental matrices satisfy the same balance equation.
  It follows that
  \begin{align*}
    \pi_s \left(\tilde{Z}_{\bar{s}\bar{s}} - \tilde{Z}_{s\bar{s}}\right) + \pi_{\bar{s}}\left(\tilde{Z}_{ss} - \tilde{Z}_{\bar{s}s}\right) 
     &=\pi_s \left(Z_{\bar{s}\bar{s}} - \frac{\pi_{\bar{s}}}{\pi_s}Z_{\bar{s}s}\right) + \pi_{\bar{s}}\left(Z_{ss} - \frac{\pi_s}{\pi_{\bar{s}}}Z_{s\bar{s}}\right)\\ 
     &= \pi_s\left(Z_{\bar{s}\bar{s}} - Z_{s\bar{s}}\right) + \pi_{\bar{s}}\left(Z_{ss} - Z_{\bar{s}s}\right) \\
     &= \pi_{\bar{s}}
  \end{align*}
  showing that $\tilde{P}$ is a GFNMC according to Thm.~\ref{thm:mc to gfn}.
\end{proof}

Taking $0<\alpha<1$, the $\alpha$-GFN given by $\alpha P + (1-\alpha) \tilde{P}$ is clearly irreducible with the same stationary distribution.
However, it is no longer of GFNMC type as some inner loop might be present in trajectories from $\bar{s}$ to $\bar{s}$.
It also do not share the same set of eigenvalues.
Nevertheless, it is closely related to the properties of $\alpha$-GFNs, as demonstrated in the following.
\subsection{Convergence Rate}
\label{app:convergence rate}

Let $P$ be the transition probability of an irreducible Markov chain of GFNMC type.
Let $\pi$ be the stationary distribution, $\tilde{P}$ the balanced chain.
We further denote by $\beta$ the largest modulus of the eigenvalue of $P$ which is not $1$ and $\beta_{\alpha}$ the largest modulus of the eigenvalue of $P_{\alpha}=\alpha P + (1-\alpha)\tilde{P}$.
Let further $d(s) = \mathrm{gcd}\{n\colon P^n_{ss}>0\}$ be the periodicity of the Markov chain.
It is a classical result that for irreducible chains, this periodicity is the same for each state.
Hence, let $d$, $\tilde{d}$ and $d_\alpha$ be the corresponding periodicities of $P$, $\tilde{P}$ and $\alpha P$, respectively.

\begin{proposition}
  It holds that the periodicity of $P$ is equal to the greatest common divisor of the length of all complete trajectories.
  Furthermore $d = \tilde{d}$.

  Finally for $0<\alpha<1$, it holds that 
  \begin{equation*}
    d_\alpha = 
    \begin{cases}
      1 & \text{if } d \text{ is odd}\\
      1 \text{ or } 2 & \text{if } d \text{ is even}
    \end{cases}
  \end{equation*}
\end{proposition}

\begin{proof}
  It is clear that the periodicity of $P$ and therefore $\tilde{P}$ will coincide with the greatest common divisor of all trajectory lengths.

  However for $0<\alpha < 1$, it is then possible to have trajectories $s\to s^\prime \to s$ with strict positive probability (since it can go with strict probability from $s$ to $s^\prime$ with $P$ and strict positive probability from $s^\prime$ to $s$ with $\tilde{P}$).
  Hence the periodicity of $\alpha P + (1-\alpha)\tilde{P}$ divides $d$ as well as $2$ showing the result.
\end{proof}
\begin{remark}
  In the case of GFN, usually with very large graphs, either you have many possible lengths that make the periodicity already $1$, or all trajectories are of the same length, and the periodicity is either $1$ or $2$ for the $\alpha$-GFN.
  Although $P$ in this case have a very large periodicity, $\tilde{P}$ will drastically reduce it to at least $2$ if not $1$.
\end{remark}

Having a periodicity of $1$ is important in terms of convergence, as it is geometric ergodic.
In this case the convergence to the stationary distribution also holds in total variation as exponential rate determined by the second largest eigenvalue and a central limit theorem also holds that ensure correct bounds in terms of sampling akin to Monte Carlo.
The rate of convergence is dominated by the second largest eigenvalue.
However, as eigenvalues are highly non-linear, it is not possible to find a direct relation between $\beta_\alpha$ and $\beta$.
Nevertheless, it is highly possible that mixing with different $\alpha$ can improve the convergence rate depending on the structure of the original Markov Chain. 
\textbf{In particular, setting $\alpha=0.5$ directly yields the discussions corresponding to vanilla GFlowNets.}

\newpage

\section{Proofs}
App.~\ref{app:supporting theoretical framework} enables formal discussions on GFlowNets from the Markov chain perspective. Building on this aspect,  we now take a closer look into the objectives of $\alpha$-GFNs and vanilla GFlowNets.
Before diving into the corresponding theorems and propositions in the main body, we first show an equivalence between flows and probability measures with a generalization of~\cite{deleu2023generative}. Although this equivalence has been used in previous parts following~\cite{deleu2023generative}, we give a generalization for clarity in the following.

By defining
\begin{equation}
\label{eq:flows as probability measures}
    Z_{\text{state}} = \sum_{s \in S} F(s), \quad \pi(s) = \frac{F(s)}{Z_{\text{state}}},
\end{equation}
we directly obtain the probability measures at each state $s \in S$. Note that $\pi(\cdot)$ is a probability measure since $\sum_{s \in S} \pi(s)=1$.

Equipped with~\eqref{eq:flows as probability measures} and App.~\ref{app:supporting theoretical framework}, we now derive the proofs for corresponding statements.

\subsection{Proof for Prop.~\ref{prop:reversibility of equally mixed policy}}
\label{app:proof for reversibility of equally mixed policy}
Since $P_{0.5}=\frac{1}{2}P_F+\frac{1}{2}P_B$, given the partial trajectory $\mathfrak{t}'=(s_k,s_{k+1},\dots,s_{k+m})$, the reversibility of $P_{0.5}$ suggests
\begin{equation*}
    \pi(s_k) \prod_{i=1}^{m} P_{0.5}(s_{k+i}|s_{k+i-1}) 
    =
    \pi(s_{k+m})\prod_{i=1}^{m} P_{0.5}(s_{k+i-1}|s_{k+i})
\end{equation*}
which extends to 
\begin{equation*}
    \pi(s_k) \prod_{i=1}^{m} \left(\frac{1}{2}P_{F}(s_{k+i}|s_{k+i-1})+\frac{1}{2}P_B(s_{k+i}|s_{k+i-1})\right) 
    =
    \pi(s_{k+m})\prod_{i=1}^{m} \left(\frac{1}{2}P_{F}(s_{k+i-1}|s_{k+i})+\frac{1}{2}P_B(s_{k+i-1}|s_{k+i})\right)
\end{equation*}
 By noticing the fact that both $P_F$ and $P_B$ are one-directional within this partial trajectory, i.e. $\forall (s,s') \subset \mathfrak{t}'$, if $P_F(s'|s)>0$, then $P_F(s|s')=0, P_B(s|s')>0, P_B(s'|s)=0$, it follows that
 \begin{equation*}
     (\frac{1}{2})^m \pi(s_k) \prod_{i=1}^{m} P_{F}(s_{k+i}|s_{k+i-1})
     =
     (\frac{1}{2})^m \pi(s_{k+m}) \prod_{i=1}^mP_B(s_{k+i-1}|s_{k+i})
 \end{equation*}
 By eliminating $(\frac{1}{2})^m$ on both sides of the equation, one has
 \begin{equation*}
     \pi(s_k)\prod_{i=1}^{m} P_{F}(s_{k+i}|s_{k+i-1})
     =
    \pi(s_{k+m}) \prod_{i=1}^mP_B(s_{k+i-1}|s_{k+i})     
 \end{equation*}

\subsection{Proof for Thm.~\ref{thm:gflownet objectives and reversibility}}
\label{app:proof for gflownet objectives and reversibility}
\textbf{We first prove the claim that DB, SubTB and TB correspond to edge-level, partial-trajectory-level and trajectory-level reversibility of $P_{0.5}$.} Using~\eqref{eq:flows as probability measures}, we obtain that for DB, 
\begin{equation*}
    \frac{F(s)}{Z_{\text{state}}} P_F(s'|s) = \frac{F(s)}{Z_{\text{state}}} P_B(s|s')
\end{equation*}
which translates into
\begin{equation*}
    \pi(s) P_F(s'|s)=\pi(s')P_B(s|s').
\end{equation*}
Obviously, this equation is the reversibility in Prop.~\ref{prop:reversibility of equally mixed policy} at an edge-level. Similarly, one can obtain the proofs for SubTB and TB by plugging~\eqref{eq:flows as probability measures} into~\eqref{eq:subtb} and~\eqref{eq:tb target} and comparing the equations with Prop.~\ref{prop:reversibility of equally mixed policy}.

\textbf{Next, we show that such objectives converge to unique state flows from the Markov chain perspective.}

For \textbf{DB} and \textbf{SubTB($\lambda$)}, they both account for the edge-level reversibility (see Def.~\ref{def:DB} and Def.~\ref{def:SubTB(lambda)}), whereas reaching the edge-level reversibility is actually equivalent to  the convergence of DB and, particularly, SubTB($\lambda$). The corresponding proof for DB is straightforward since the edge-level reversibility is identically its training target (see Def.~\ref{def:DB}). For SubTB($\lambda$), by noticing that given a complete trajectory $\mathfrak{t}^f=(s_0=s_s,s_{1},\dots,s_{n-1}=x,s_{n+1}=s_f)$, edge-level reversibility suggests
\begin{equation*}
    F(s_i)P_F(s_{i+1}|s_i)=F(s_{i+1})P_B(s_{i+1}|s_i), \quad \text{for all } (s_i,s_{i+1}) \subset \mathfrak{t}^f.
\end{equation*}
Therefore, for every partial trajectory $\mathfrak{t}'=(s_k,s_{k+1},\dots,s_{k+m}) \subset \mathfrak{t}^f$, by using the fact that $\forall 0\leq i \leq k-1,F(s_{i+1})=\frac{F(s_i)P_F(s_{i+1}|s_i)}{P_B(s_{i}|s_{i+1})}$, we have
\begin{equation}
    F(s_k)\prod_{i=1}^{k}P_F(s_{k+i}|s_{k+i-1}) = F(s_{k+m})\prod_{i=1}^kP_B(s_{k+i-1}|s_{k+i})
\end{equation}
which suggests the target of SubTB($\lambda$). 
Hence, it is only required to check whether the edge-level reversibility can be achieved from the Markov chain perspective. 
In fact, this edge-level reversibility is equivalent to the \textbf{detailed balance conditions}~\citep{douc2018markov}. If detailed balance conditions are satisfied, the probability measures are unique if the corresponding finite-state Markov chain is irreducible and positive recurrent~\citep{douc2018markov}.  Given that flows are also unique in GFlowNets~\citep{bengio2023gflownet}, it is required to show whether our MC modeling of GFlowNets achieve unique flows as well. Since flows are unnormalized probability measures, it suffices to show that the Markov transition kernel $P_{0.5}$ yields a irreducible and positive recurrent Markov chain, which is direct from the fact that both the forward policy $P_F$ and the backward policy $P_B$ yield irreducible and positive recurrent chains. 
Therefore, DB and SubTB($\lambda$)'s convergence to unique flows is ensured from a Markov chain viewpoint.

For \textbf{TB}, we relate to the \textbf{Kolmogorov's criterion}~\citep{douc2018markov}, which is equivalent to the detailed balance conditions for a finite discrete Markov chain like GFlowNets. Kolmogorov's criterion suggest that for any loop $(s_0=\bar{s},s_1,\dots,s_{n-1},s_n=\bar{s})$, the reversibility at the loop is achieved, i.e.
\begin{equation}
\label{eq:kolmogorov's criterion}
    \pi(s_0) \prod_{i=1}^n P(s_i|s_{i-1}) = \pi(s_n)\prod_{i=1}^n P(s_{i-1}|s_{i})
\end{equation}
if the corresponding finite discrete Markov chain is irreducible and positive recurrent.
If one mergers the source state $s_s$ and the final state $s_f$ as~\cite{deleu2023generative}, i.e. $s_s=s_f=\bar{s}$, then TB of $P_{0.5}$ targets at
\begin{equation*}
    F(\bar{s})\prod_{i=1}^m P_{0.5}(s_i|s_{i-1}) = F(\bar{s})\prod\limits_{i=1}^m P_{0.5}(s_{i-1}|s_i)
\end{equation*}
which translates into
\begin{equation*}
    \pi(\bar{s})\prod_{i=1}^m P_{0.5}(s_i|s_{i-1}) = \pi(\bar{s})\prod\limits_{i=1}^m P_{0.5}(s_{i-1}|s_i)
\end{equation*}
for the complete trajectory $\mathfrak{t}^f=(s_0=s_s,s_1,\dots,s_{n-1}=x,s_n=s_f)$ since $\pi(s)=\frac{F(s)}{Z_{\text{state}}}$ is the corresponding probability measure and $P_B(x|s_f)=\frac{R(x)}{\sum_{x' \in \mathcal{X}}R(x')}$.
Therefore, the convergence of TB is a necessary condition for Kolmogorov's criterion to hold. Due to the fact that $P_{0.5}$ yields a irreducible and positive recurrent Markov chain, if Kolmogorov's criterion holds, the uniqueness of flows is achieved by the convergence of TB as well. However, since TB is only a necessary condition, this uniqueness is relatively fragile, which potentially contribute to instability of TB training~\citep{madan2023learning}.

\subsection{Proof for Prop.~\ref{prop:convergence of alpha-gfn objectives}}
\label{app:proof for convergence of alpha-gfn objectives}
\textbf{We first prove the claim that $\alpha$-DB, $\alpha$-SubTB and $\alpha$-TB correspond to edge-level, partial-trajectory-level and trajectory-level reversibility of $P_{\alpha}$.} The proof is similar to the proof for Thm.~\ref{thm:gflownet objectives and reversibility}, and we only show the proof for $\alpha$-DB. 
Recall the reversibility of $\alpha$-DB: for any $(s,s') \in \mathbb{A}$
\begin{equation}
\label{eq:alpha-db reversibility}
    \pi(s) P_{\alpha}(s'|s) = \pi(s') P_{\alpha}(s|s').
\end{equation}
Since both $P_F$ and $P_B$ are one-directional,~\eqref{eq:alpha-db reversibility} implies
\begin{equation}
\label{eq:alpha-db reversibility revisited}
    \alpha \pi(s) P_{F}(s'|s)  = (1-\alpha) \pi(s') P_B(s|s').
\end{equation}
On the other hand, recall $\alpha$-DB, which targets at
\begin{equation}
\label{eq:alpha-db target}
    \alpha F(s) P_{F}(s'|s)  = (1-\alpha) F(s') P_B(s|s').
\end{equation}
Plugging $\pi(s)=\frac{F(s)}{Z_{\text{state}}}$ into Equation\ref{eq:alpha-db target} yields
\begin{equation*}
    \alpha \pi(s) P_{F}(s'|s)  = (1-\alpha) \pi(s') P_B(s|s')
\end{equation*}
which is identical to~\eqref{eq:alpha-db reversibility revisited}. The proofs for $\alpha$-SubTB and $\alpha$-SubTB($\lambda$) are similar.

\textbf{Next, we show the convergence of such objectives lead to unique state flows from the Markov chain perspective.} This also follows the proof for Thm.~\ref{thm:gflownet objectives and reversibility}. The convergence of $\alpha$-DB and $\alpha$-SubTB($\lambda$) to unique flows are derived by the detailed balance conditions, irreducibility and positive recurrence of the Markov chain with $P_{\alpha}$, whereas that of $\alpha$-TB follows a necessary condition of Kolmogorov's criterion of $P_{\alpha}$. 

\textbf{Nevertheless, App.~\ref{app:convergence rate} reveals that the convergence rates to unique flows vary for different $\alpha$ values.} Even though this aspect may not be perfectly reflected by loss curves due to the difference between GFlowNet training and MCMC algorithms, the tuning of $\alpha$ still contributes the training of GFlowNets, as suggested in Prop.~\ref{prop:gradient of alpha-gfn objectives}.

\subsection{Proof for Prop.~\ref{prop:gradient of alpha-gfn objectives}}
\label{app:proof for gradient of alpha-gfn objectives}
Following Def.~\ref{def:alpha-gfn objectives}, the loss function of $\alpha$-SubTB at the partial trajectory $\mathfrak{t}=(s_k,s_{k+1},\dots,s_{k+m})$ is
\begin{equation*}
    L_{\alpha-\text{SubTB}}(\mathfrak{t}')=\log^2
    \left(
    \frac{\alpha^m F(s_k)\prod_{i=1}^m P_F(s_{k+i}|s_{k+i-1})}{(1-\alpha)^m F(s_{k+m})\prod_{i=1}^{m}P_B(s_{k+i-1}|s_{k+i})}
    \right).
\end{equation*}
We denote by $P_B(\mathfrak{t}')=\prod_{i=1}^m P_B(s_{k+i-1}\mid s_{k+i1})$. Then, the gradient to $P_F(\mathfrak{t}')=\prod_{i=1}^m P_F(s_{k+i}\mid s_{k+i-1})$ is
\begin{align*}
    \frac{\partial L_{\alpha-\text{SubTB}}(\mathfrak{t}')}{\partial P_F(\mathfrak{t'})} 
    &= \frac{2}{P_F(\mathfrak{t}')} \log \frac{\alpha^m F(s_k)P_F(\mathfrak{t}')}{(1-\alpha)^m F(s_{k+m})P_B(\mathfrak{t}')} \\
    &= \frac{2}{P_F(\mathfrak{t}')} \log \frac{F(s_k)P_F(\mathfrak{t}')}{F(s_{k+m})P_B(\mathfrak{t}')} + \frac{2m}{P_F(\mathfrak{t}')}\log \frac{\alpha}{1-\alpha}.
\end{align*}
Meanwhile, taking a gradient of~\eqref{eq:subtb loss} suggests
\begin{equation*}
    \frac{\partial L_{\text{SubTB}}(\mathfrak{t}')}{\partial P_F(\mathfrak{t'})}  = \frac{2}{P_F(\mathfrak{t}')} \log \frac{F(s_k)P_F(\mathfrak{t}')}{F(s_{k+m})P_B(\mathfrak{t}')}.
\end{equation*}
Therefore, it is direct that
\begin{equation*}
        \frac{\partial L_{\alpha-\text{SubTB}}(\mathfrak{t}')}{\partial P_F(\mathfrak{t'})} = \frac{\partial L_{\text{SubTB}}(\mathfrak{t}')}{\partial P_F(\mathfrak{t'})} + \frac{2m}{P_F(\mathfrak{t}')} \log \frac{\alpha}{1-\alpha}.
\end{equation*}
\newpage

\section{Experiments}
In this section, we present the detailed experimental settings, computational resource usage, ablation studies of our $\alpha$-GFNs along with more numerical results and corresponding analysis to support our findings. All experiments are run on a cluster consisting of NVIDIA RTX3090, NVIDIA RTX4090 and NVIDIA ADA6000 GPUs.

\subsection{Detailed Experimental Setups}
\label{app:experimental settings}

\subsubsection{Set Generation}
\paragraph{Implementation Details.} We implement $\alpha$-DB, $\alpha$-TB, $\alpha$-SubTB($\lambda$), $\alpha$-FL-DB, and $\alpha$-FL-SubTB($\lambda$) based on the open-source code of~\citet{pan2023better}, where setting $\alpha=0.5$ yields the baselines. Most details of this task follow~\citet{pan2023better}, including the models and hyperparameters. Specifically, the vocabulary size (action space) is 30, 80 and 100 for small, medium and large sets, respectively, and the maximum set capacity is capped at 20, 60, and 80. We employ the same intermediate energy function $\mathcal{E}(\cdot)$: energies are sampled uniformly from $[-1,1]$, and exactly $\frac{|S|}{10}$ elements share identical energy values, resulting in multiple optimal solutions. The GFlowNet agent is parameterized by an MLP with two hidden layers of 256 units and LeakyReLU activations. We generate 16 samples per training step and use the Adam optimizer~\citep{kingma2014adam} to optimize all objectives for 10,000 training steps.
For $\alpha$-DB and $\alpha$-FL-DB, the learning rate is set to 0.001. For $\alpha$-TB, we use a learning rate of 0.001 for the MLP parameters and 0.1 for the learnable normalizing constant $Z$ (the total flow), following~\citet{pan2023better}. For $\alpha$-SubTB($\lambda$) and $\alpha$-FL-SubTB($\lambda$), we use the same optimizer and learning rate as $\alpha$-DB, and set $\lambda=0.99$ by default. We implement SubTB($\lambda$) by summing balance residuals over all subtrajectories within each sampled trajectory, with $\lambda^{\text{length}}$ weighting; terminal rewards are injected at the end of subtrajectories (as in the standard SubTB formulation). For FL-SubTB($\lambda$), we follow the forward-looking variant where intermediate rewards are incorporated along the trajectory (i.e., per-step), while keeping the same subtrajectory weighting scheme.
In particular, we notice that the parameter $\epsilon$ in the $\epsilon$-greedy sampling trick varies across different values of $\alpha$ and set sizes (in some cases the sampling policy becomes nearly uniform with $\epsilon=1$). Therefore, we adopt a unified schedule: we start from $\epsilon=1$ and linearly anneal it to 0.05 during training. Models are trained for 10,000 steps, where the first 9,000 steps correspond to stage 1 of Alg.~\ref{alg:two-staged-training}. Throughout the experiments, we fix the initialization of model weights and use different random seeds for sample generation. All results are averaged over 5 random seeds, which are integers from 0 to 4.

\paragraph{Evaluation Metrics.} Following~\cite{pan2023better}, the evaluation is conducted online and independent of training samples. Training samples are generated with the $\epsilon$-greedy policy, while evaluation samples are generated only with the forward policy $P_F$. Details of the three evaluation metrics are presented in what follows:
\begin{itemize}[leftmargin=*, topsep=0pt, noitemsep]
\item \textbf{Modes}: The count of unique samples with rewards exceeding a predefined threshold. This metric evaluates the policy's ability to explore diverse high-reward regions. Following \citet{jang2023learning}, we set the threshold to $0.25$ for small sets. For medium and large sets, we use a threshold of $700,000$. Note that while \citet{pan2023better} introduces these tasks, they do not specify a threshold; thus, our choices are based on standard benchmarks in the literature.
\item \textbf{Top-1000 R}: The average reward of the top $1,000$ unique samples with the highest rewards. This measures the policy's exploitation efficiency in identifying the most optimal candidates.
\item \textbf{Spearman}: The Spearman correlations calculated between the sampling probability $P_F^{\top}$ and the ground-truth reward $R$. We evaluate this on a held-out test set of $1,000$ samples generated by the initialized policy to assess the alignment of the learned distribution with the reward landscape.
\end{itemize}
In addition to these metrics, we report the \textbf{average reward} of all evaluation samples in Figure~\ref{fig:main_figure} to provide a global view of the training progress.

\subsubsection{Bit Sequence Generation} 

\paragraph{Implementation Details.}
Following~\citet{tiapkin2024generative}, we adopt the non-autoregressive version of the Bit Sequence Generation task. Unlike the original tree-structured state space in~\citet{malkin2022trajectory}, this version operates on a DAG-structured state space, which presents a more significant challenge for credit assignment and mode discovery. We implement $\alpha$-DB, $\alpha$-TB, and $\alpha$-SubTB based on the framework of~\citet{tiapkin2024generative}, where $\alpha=0.5$ serves as the baseline.
We additionally implement the forward-looking variants (FL-DB and FL-SubTB) by supplying intermediate rewards via a shaped partial log-reward for incomplete sequences: we treat unfilled slots as a sentinel token and compute the minimum token-level Hamming mismatch to the mode set $M$ while ignoring sentinel positions, yielding $\widetilde{\log R}(s_{\le t})$; we then define the per-step intermediate log-reward as the increment $\log r^{\mathrm{FL}}t=\widetilde{\log R}(s{\le t+1})-\widetilde{\log R}(s_{\le t})$ (with the same reward exponent used for the terminal reward), so that these increments telescope to the final log-reward when the sequence is complete. In FL-DB, we subtract $\log r^{\mathrm{FL}}t$ from each one-step DB residual (and analogously at the terminal step); in FL-SubTB, for every segment $(i,j)$ we subtract the accumulated intermediate reward $\sum{t=i}^{j-1}\log r^{\mathrm{FL}}_t$ inside the SubTB residual, matching the forward-looking SubTB formulation. 
The GFlowNet agent is parameterized by a Transformer~\citep{vaswani2017attention} with 3 hidden layers, 64-dimensional hidden states, and 8 attention heads.

All models are trained for 50,000 steps using the Adam optimizer~\citep{kingma2014adam} with a batch size of 16. The learning rate is set to $2 \times 10^{-3}$ for MLP parameters and $10^{-3}$ for the learnable normalizing constant $Z$ in TB-based objectives. For SubTB($\lambda$), we set $\lambda=1.9$. To stabilize training, we apply gradient clipping with a norm of 20. The reward function is augmented as $\tilde{R}(x)=R(x)^{\beta}$ with $\beta=2$. During training, we employ an $\epsilon$-noisy strategy that mixes the forward policy $P_F$ with a uniform distribution using $\epsilon=0.001$. As in the Set Generation task, we adopt a two-stage training procedure where the first 40,000 steps correspond to Stage 1 of Alg.~\ref{alg:two-staged-training}. All results are averaged over 5 random seeds (0--4) with fixed weight initializations.

\paragraph{Evaluation Metrics.} Similar to the setup in Set Generation, we separate training and evaluation: training samples use the $\epsilon$-greedy policy, while evaluation is conducted online using only the forward policy $P_F$. The evaluation metrics are detailed below:
\begin{itemize}[leftmargin=*, topsep=0pt, noitemsep]
\item \textbf{Modes}: The number of unique modes discovered (out of a maximum of 60). A mode is considered identified if a generated sample's Hamming distance to a predefined mode in $M$ is less than 30~\citep{malkin2022trajectory}. To ensure a strict count, once a specific mode from $M$ is identified, it is marked as "found"; subsequent samples falling within the same distance threshold are not counted as additional modes.
\item \textbf{Spearman}: The Spearman correlation between the estimated generation probabilities $P_{\theta}(x)$ and the ground-truth rewards $R(x)$, evaluated on the predefined test set from~\citet{malkin2022trajectory}. To approximate the generation probability $P_{\theta}(x)$, we employ the Monte Carlo estimator from~\citet{zhang2022generative}:\begin{equation*}P_{\theta}(x) \approx \frac{1}{N} \sum_{i=1}^N \frac{P_F(\tau^i)}{P_B(\tau^i | x)},\end{equation*}where $N=10$ and trajectories $\tau^i$ are sampled from the backward policy $P_B$ (which is fixed as uniform).
\end{itemize}

\subsubsection{Molecule Generation}
\paragraph{Implementation Details.} This task aims to generate molecular binders for the soluble epoxide hydrolase (sEH) protein. We implement $\alpha$-variants of DB, TB, SubTB($\lambda$), and their forward-looking (FL) extensions based on the frameworks of~\citet{pan2023better} and~\citet{tiapkin2024generative}. The GFlowNet agent is parameterized by a Message Passing Neural Network (MPNN) with 10 convolution steps, following the configuration in~\citet{tiapkin2024generative}. We use OrderedGNN~\citep{song2023ordered} to ensure reproducible message-passing computations.All models are trained for 50,000 steps using the Adam optimizer~\citep{kingma2014adam} with a batch size of 4. The learning rate is set to $5\times10^{-4}$ for all parameters, including the learnable $\log Z$ (initialized at 30) in TB objectives. For SubTB($\lambda$), we set $\lambda=0.99$. We apply a gradient clipping norm of 2 for TB to stabilize training. The reward is augmented as $\tilde{R}(x)=R(x)^{\beta}$ with $\beta=4$. During training, an $\epsilon$-greedy strategy with $\epsilon=0.05$ is used. As with previous tasks, we adopt a two-stage training procedure (Alg.~\ref{alg:two-staged-training}), where the first 40,000 steps constitute Stage 1. To ensure reproducibility, we replace the timestamp-based sampling in the original codebase with fixed random seeds (0--4) for all 5 runs.

\paragraph{Evaluation Metrics.} Following the protocol in Set Generation, evaluation is performed online using only the forward policy $P_F$, independent of the $\epsilon$-greedy training samples. A total of 200,000 molecules are generated for both training and evaluation per objective. The evaluation metrics are detailed below:
\begin{itemize}[leftmargin=*, topsep=0pt, noitemsep]
\item \textbf{Modes}: The count of unique molecules that satisfy both a reward threshold ($R > 7$) and a diversity constraint (Tanimoto similarity $< 0.7$).
\item \textbf{Top-1000 R}: The average reward of the top 1000 unique samples with the highest rewards, reflecting the model's reward exploitation capability.
\item \textbf{Spearman}: The Spearman correlation between the generation probabilities and ground-truth rewards on a predefined test set to assess whether the model learns to match the reward distribution.
\end{itemize}

\subsection{Additional Results}
\label{app:additional results}

\paragraph{Sample Diversity.}
In addition to the mode discovery results presented in the main text, we provide the Top-1000 Similarity for the Set and Molecule Generation tasks as a supplemental metric. This is calculated as the average pairwise similarity among the top 1,000 unique samples with the highest rewards.
Specifically, the similarity metrics for each task are implemented as follows:
\begin{itemize}[leftmargin=*, topsep=0pt, noitemsep]
\item \textbf{Set Generation:} We employ the average Jaccard similarity to quantify the element overlap among the sets. 
\item \textbf{Molecule Generation:} We calculate the average Tanimoto similarity of the molecules. 
\end{itemize}
Results are shown in Table~\ref{tab:similarity}. Across the evaluated settings, the similarity metrics of our methods remain at a similar level to the baselines, despite the improvements in reward and mode discovery.

In Set Generation, we observe a slight increase in similarity scores as the task scale grows. This reflects the models' focus on high-reward regions, which can lead to more concentrated samples. However, when considering the increased number of modes found, these similarity values (remaining between 0.7 and 0.9) suggest that the models are still identifying a variety of high-quality solutions. The results indicate that the performance gains do not result in a significant collapse of sample variety.

In Molecule Generation, the similarity scores are consistent with the baseline results. For certain objectives like FL-DB and FL-SubTB($\lambda$), our methods yield similarity values, such as 0.50, that are slightly lower than or equal to the baseline values. These metrics suggest that the $\alpha$-GFN objectives can improve reward-related metrics while retaining structural diversity. Overall, the analysis indicates that the observed improvements in other metrics do not come at the cost of a significant loss in sample variety.

In summary, these results suggest that our methods remain consistent with standard GFlowNet objectives in terms of sample similarity, even while exploring high-reward regions more effectively.
\begin{table*}[t]
\caption{\textit{Top-1000 Similarity for Set and Molecule Generation.} We report the average Jaccard similarity for sets and Tanimoto similarity for molecules among the top 1,000 unique high-reward samples. The results illustrate the similarity levels of our methods relative to the baselines. Standard deviations are shown in \textcolor{gray}{gray}.}
\centering
\setlength{\tabcolsep}{6pt}
\renewcommand{\arraystretch}{1.2}
\resizebox{\linewidth}{!}{%
\begin{tabular}{@{\hspace{.6em}}lcccccccccc@{\hspace{.6em}}}
\toprule
\multicolumn{1}{c}{Top-1000 Similarity} &
\multicolumn{2}{c}{\textbf{DB}} &
\multicolumn{2}{c}{\textbf{FL-DB}} &
\multicolumn{2}{c}{\textbf{SubTB($\lambda$)}} &
\multicolumn{2}{c}{\textbf{FL-SubTB($\lambda$)}} &
\multicolumn{2}{c}{\textbf{TB}} \\
\cmidrule(lr){2-3}\cmidrule(lr){4-5}\cmidrule(lr){6-7}\cmidrule(lr){8-9}\cmidrule(lr){10-11}
\textbf{Task} &
Baseline & Ours &
Baseline & Ours &
Baseline & Ours &
Baseline & Ours &
Baseline & Ours \\
\midrule

Set Generation, Small Sets&
\msl{0.69}{0.00} & \msl{0.72}{0.01} & \msl{0.73}{0.00} & \msl{0.74}{0.00} & \msl{0.69}{0.00} & \msl{0.70}{0.01} & \msl{0.74}{0.00} & \msl{0.74}{0.00} & \msl{0.69}{0.00} & \msl{0.69}{0.00} \\
\cmidrule(lr){1-11}

Set Generation, Medium Sets &
\msl{0.71}{0.00} & \msl{0.80}{0.01} & \msl{0.78}{0.01} & \msl{0.81}{0.01} & \msl{0.73}{0.01} & \msl{0.84}{0.01} & \msl{0.78}{0.00} & \msl{0.87}{0.01} & \msl{0.69}{0.00} & \msl{0.86}{0.01} \\
\cmidrule(lr){1-11}

Set Generation, Large Sets &
 \msl{0.76}{0.00} & \msl{0.88}{0.01} & \msl{0.85}{0.01} & \msl{0.87}{0.00} & \msl{0.79}{0.02} & \msl{0.87}{0.02} &
 \msl{0.85}{0.01} & \msl{0.86}{0.01}
 & \msl{0.74}{0.00} & \msl{0.87}{0.01} \\
\cmidrule(lr){1-11}

Molecule Generation & 
\msl{0.56}{0.00} & \msl{{0.56}}{0.00} &
\msl{0.61}{0.03} & \msl{0.50}{0.05} &
\msl{{0.56}}{0.00} & \msl{0.56}{0.01} &
\msl{0.54}{0.10} & \msl{0.50}{0.00} &
\msl{0.60}{0.03} & \msl{0.61}{0.04} \\

\bottomrule
\end{tabular}
}
\label{tab:similarity}
\end{table*}

\paragraph{Length-controlling side-effects.}
Fig.~\ref{fig:length_alpha} illustrates how the parameter $\alpha$ influences the average sample length. Since lengths are fixed in the Set and Bit Sequence Generation tasks, we focus on Molecule Generation as a representative case for variable-length scenarios. We observe that for forward-looking (FL) variants, the generated sample length correlates positively with $\alpha$, showing a increase as $\alpha$ grows. The underlying mechanism of this correlation remains an open question and is left for future investigation.
\begin{figure}[htbp]
  \centering
  \setlength{\tabcolsep}{0pt}
  \begin{tabular}{@{}c@{\hspace{0pt}}c@{\hspace{0pt}}c@{\hspace{0pt}}c@{\hspace{0pt}}c@{}}
    \includegraphics[width=.19\textwidth]{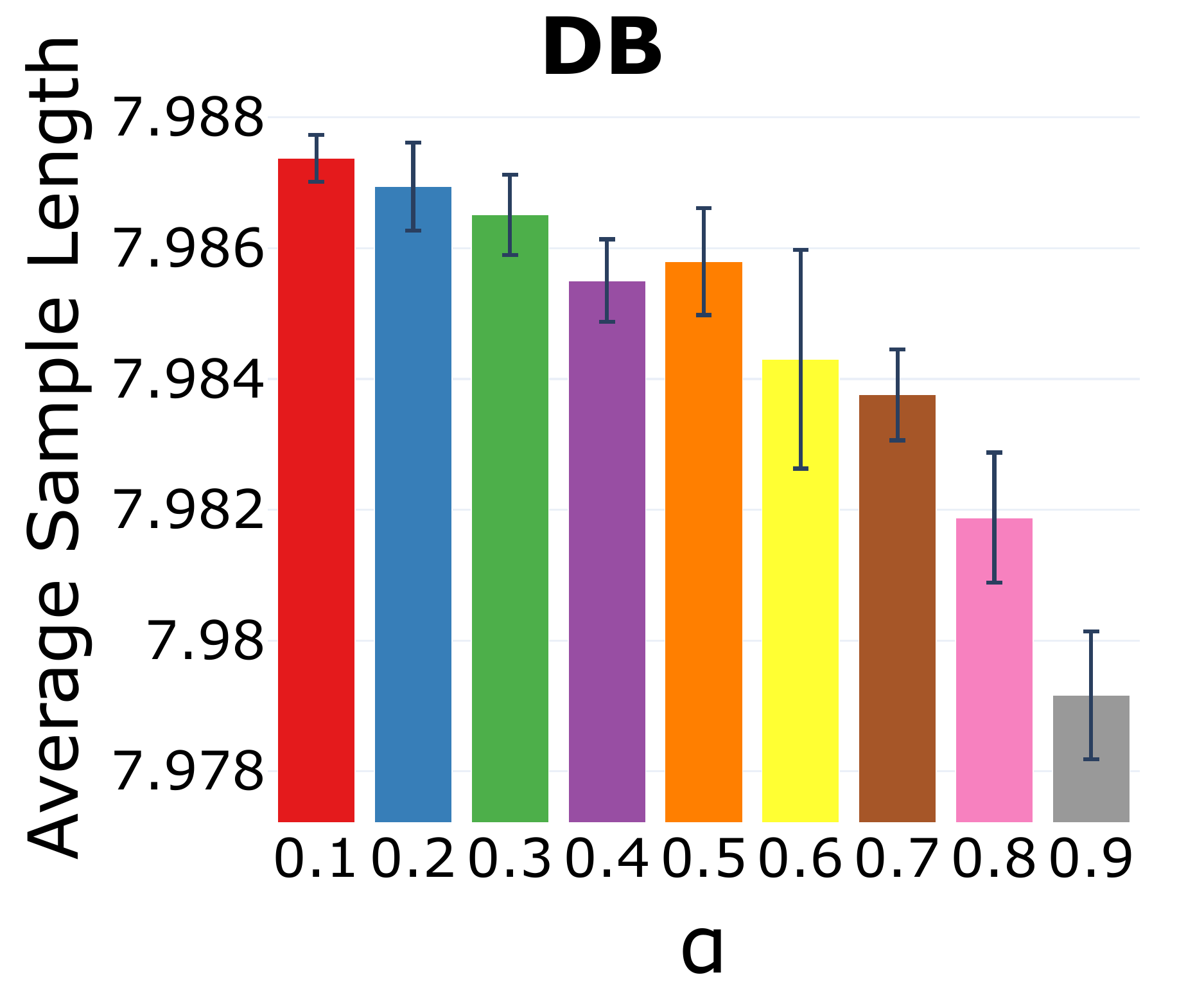} &
    \includegraphics[width=.19\textwidth]{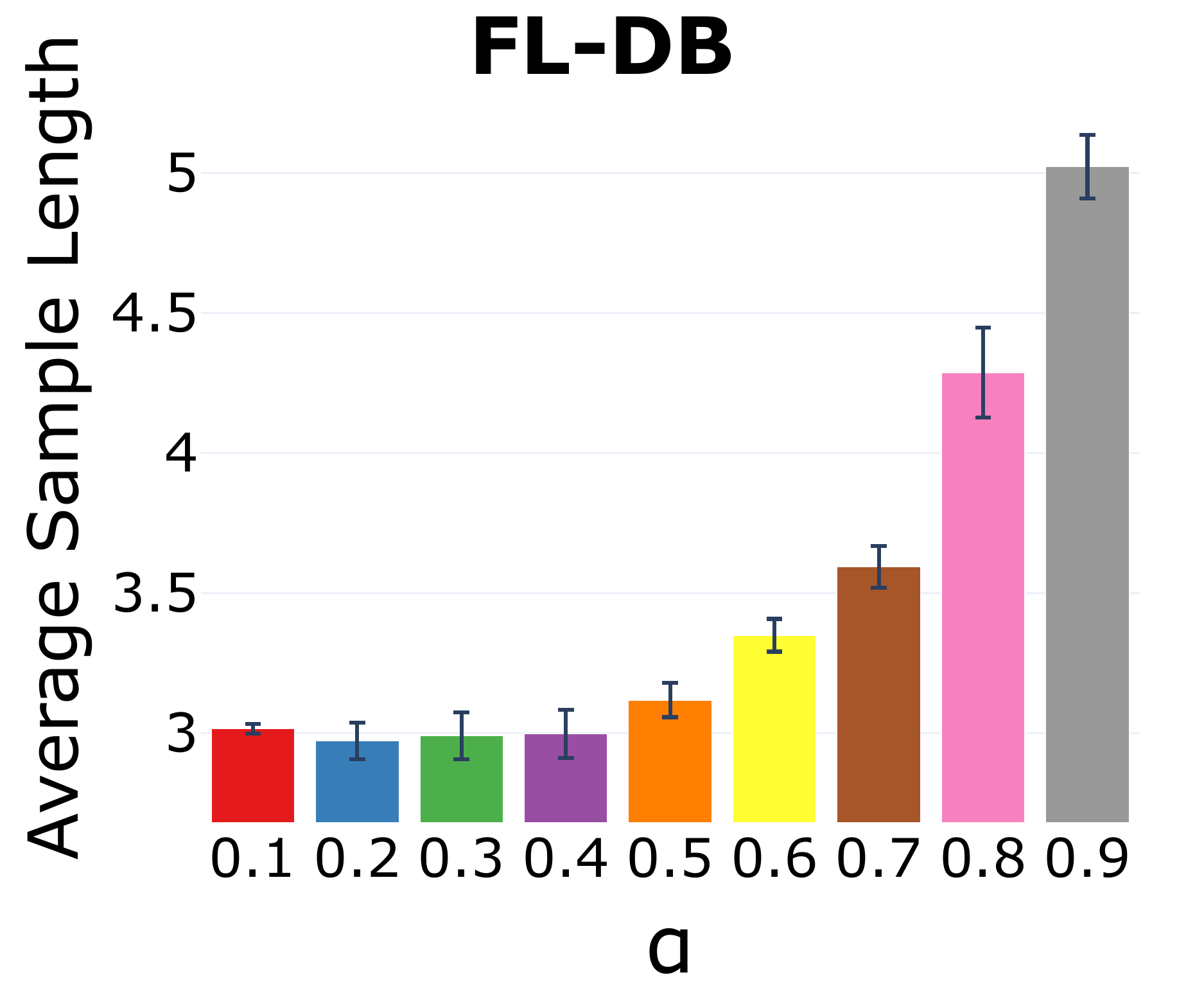} &
        \includegraphics[width=.19\textwidth]{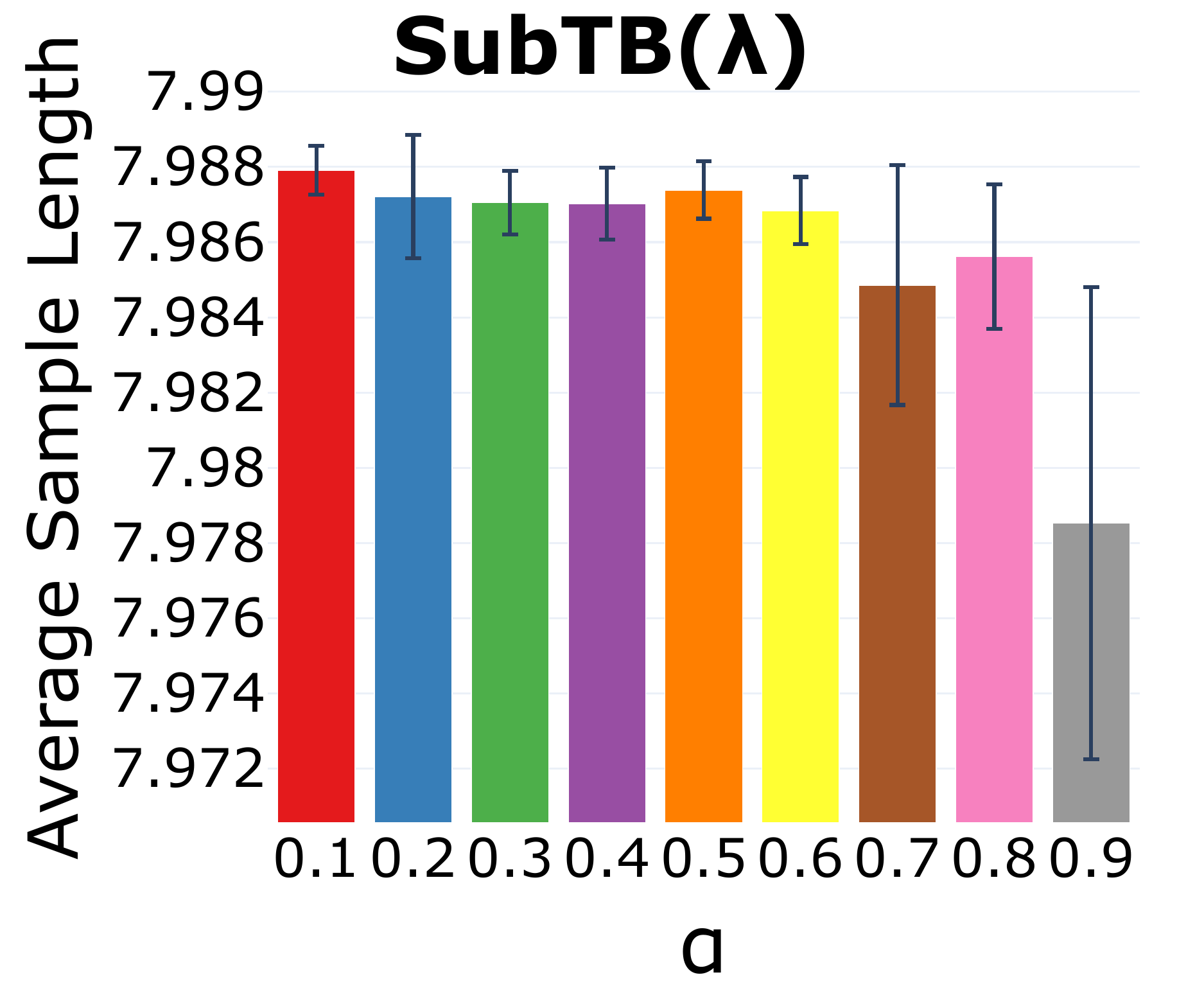} &
    \includegraphics[width=.19\textwidth]{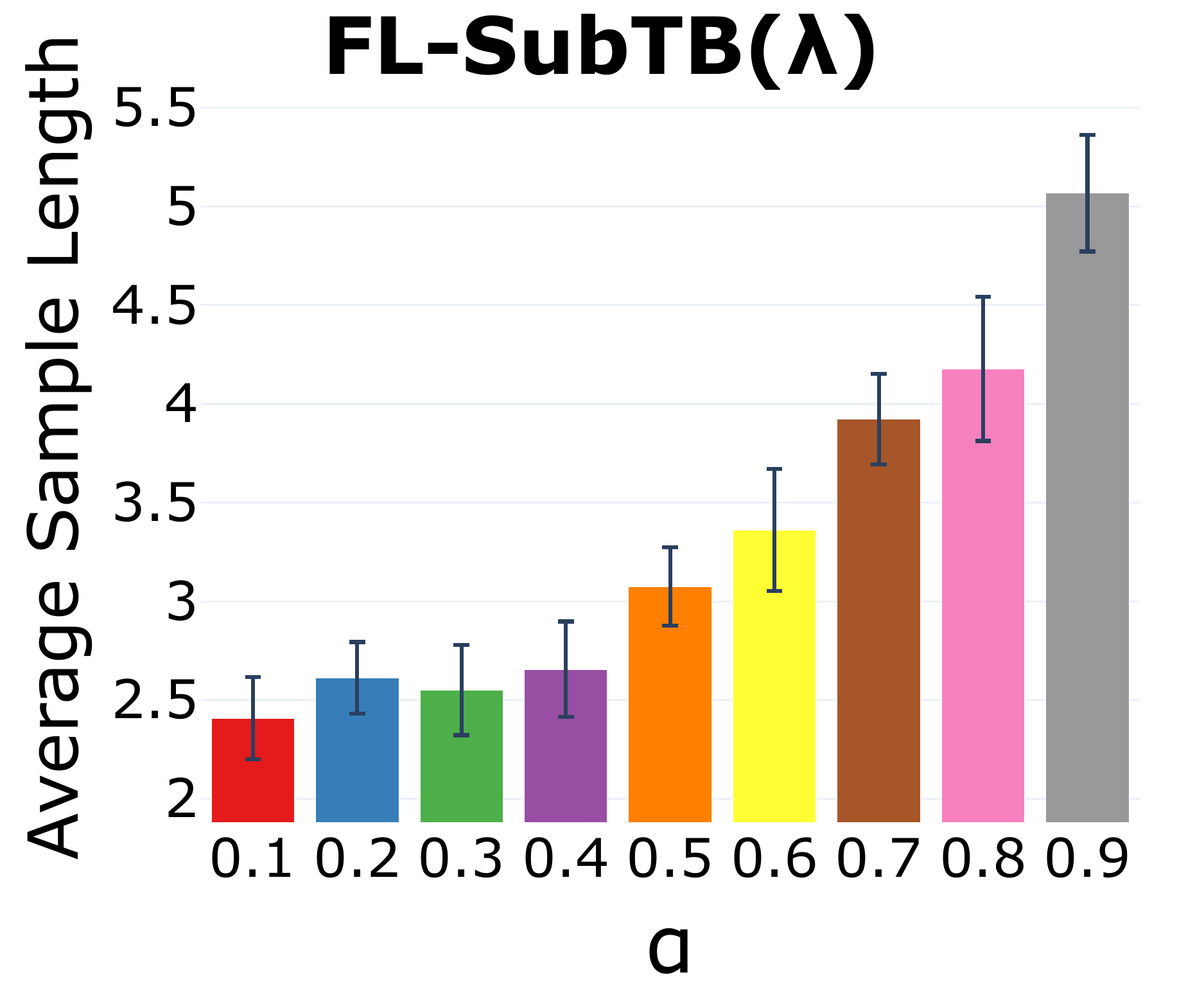} &
    \includegraphics[width=.19\textwidth]{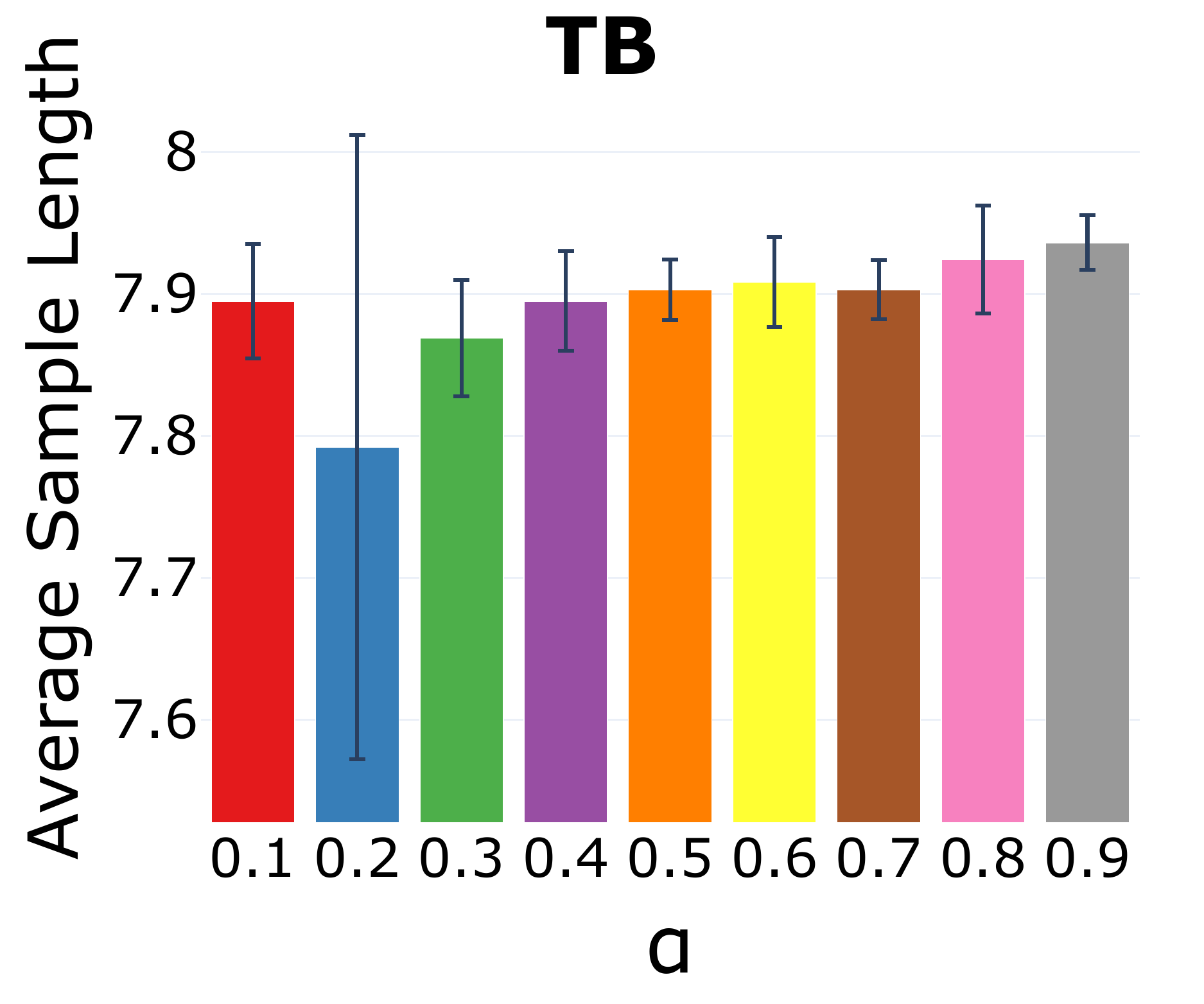} \\
  \end{tabular}
  \caption{\textit{Average Sample Length vs $\alpha$ in Molecule Generation.}}
  \label{fig:length_alpha}
\end{figure}

\paragraph{Dynamics of Scheduled Training.} 
To characterize the behavior of the scheduled scheme (Alg.~\ref{alg:two-staged-training}), we track the evolution of key performance metrics across training steps. Detailed plots are provided in Figs.~\ref{fig:set_metric_mean_top_1000_R}--\ref{fig:set_metric_forward_policy_entropy_eval} for Set Generation, Figs.~\ref{fig:bit_metric_modes}--\ref{fig:bit_metric_loss} for Bit Sequence Generation (with k=4), and Figs.~\ref{fig:mols_metric_num_modes_eval}--\ref{fig:mols_metric_current_loss} for Molecule Generation.

\begin{figure}[htbp]
  \centering
  \setlength{\tabcolsep}{0pt}
   \begin{tabular}{@{}c@{\hspace{0pt}}c@{\hspace{0pt}}c@{\hspace{0pt}}c@{\hspace{0pt}}c@{}}
    \includegraphics[width=0.2\textwidth]{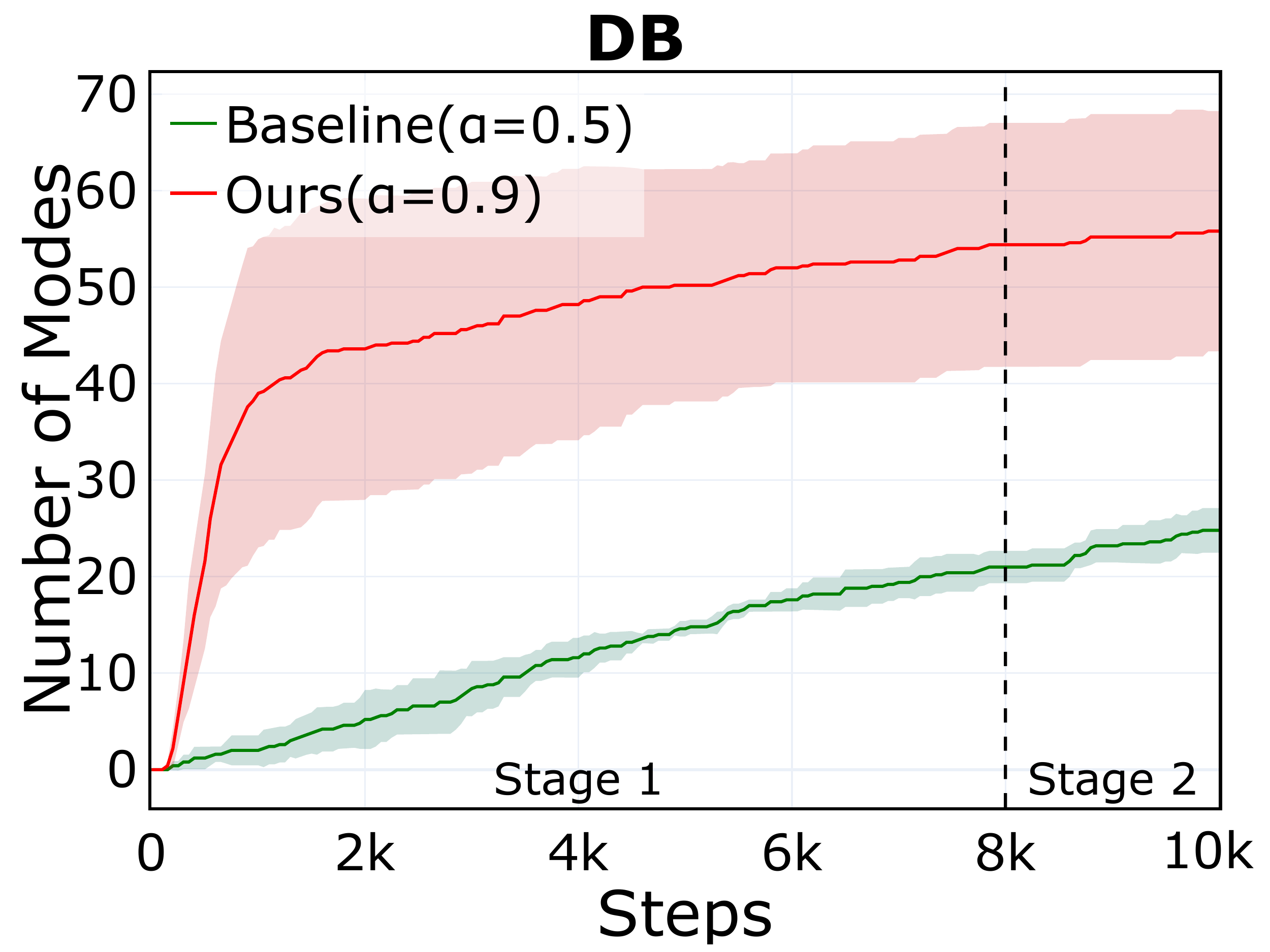} &
    \includegraphics[width=0.2\textwidth]{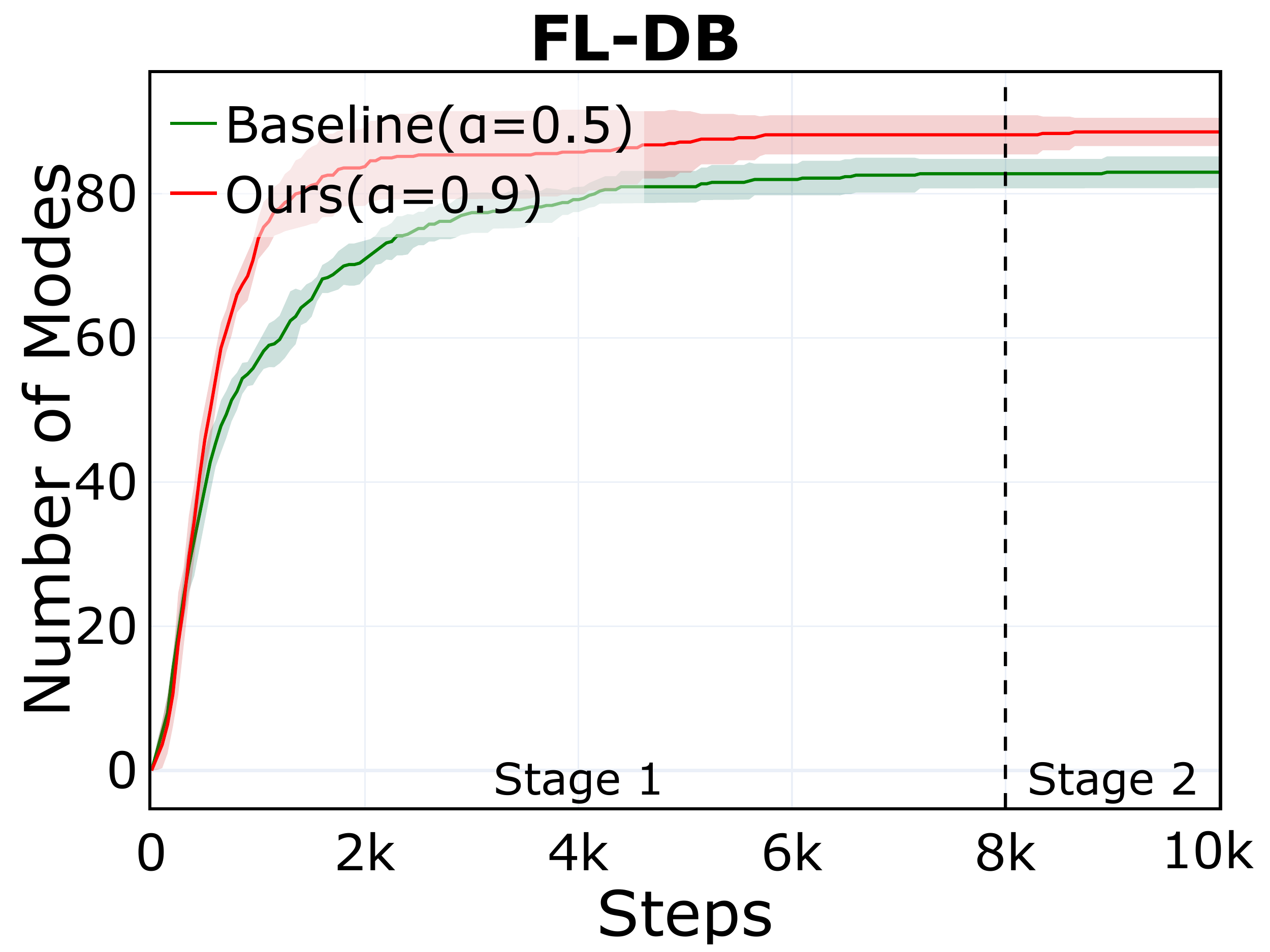} &
    \includegraphics[width=0.2\textwidth]{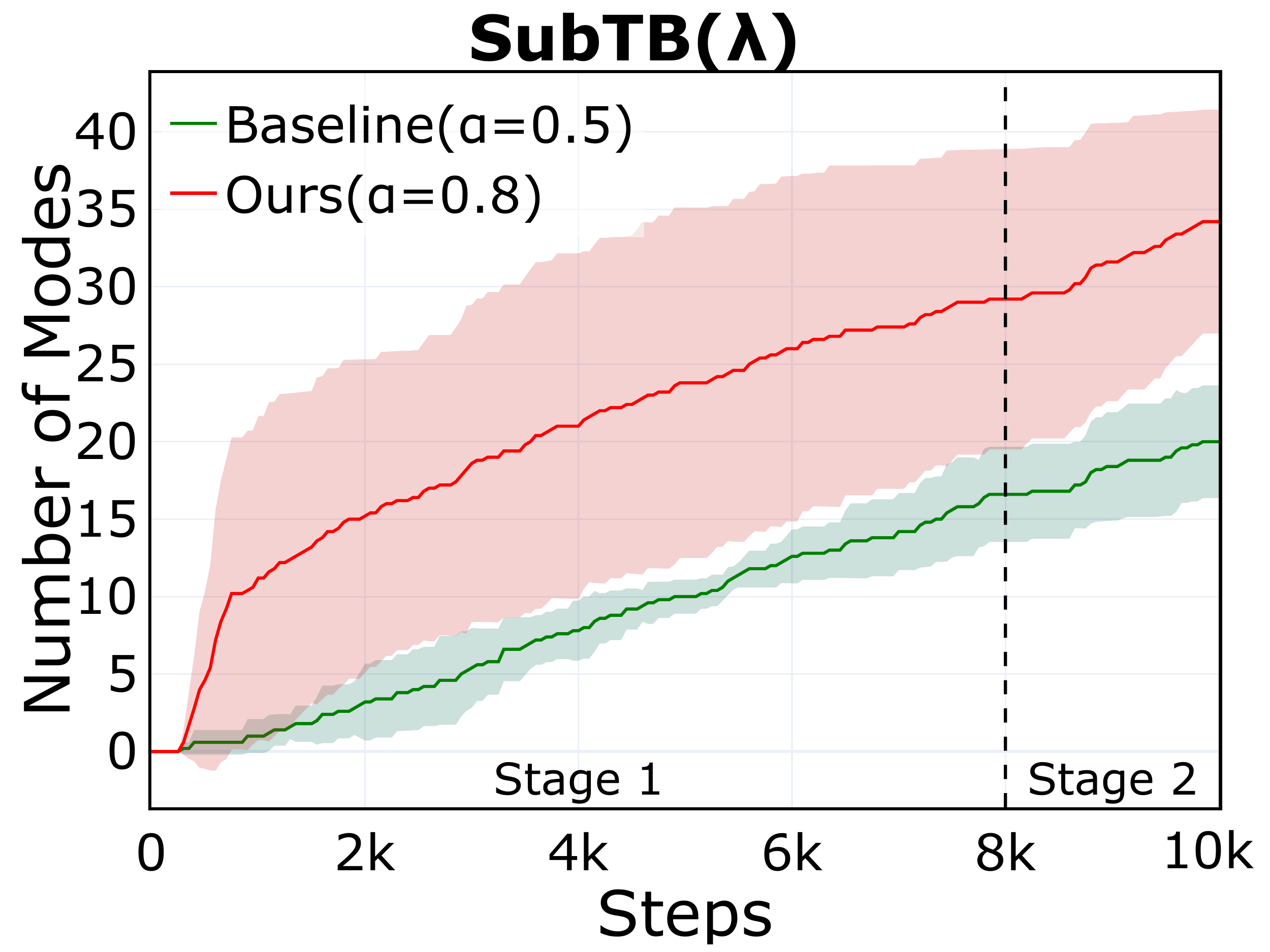} & 
    \includegraphics[width=0.2\textwidth]{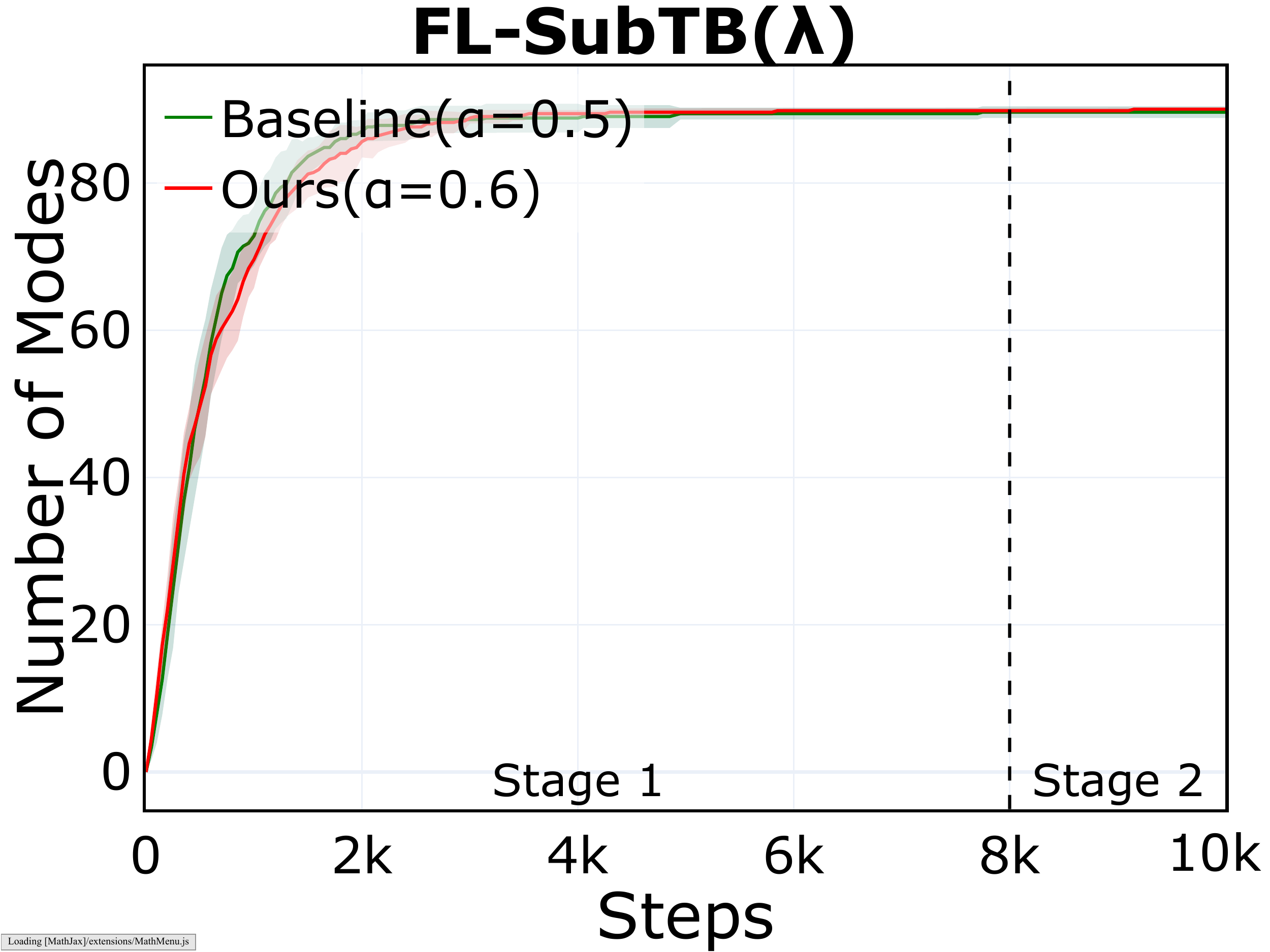} &
    \includegraphics[width=0.2\textwidth]{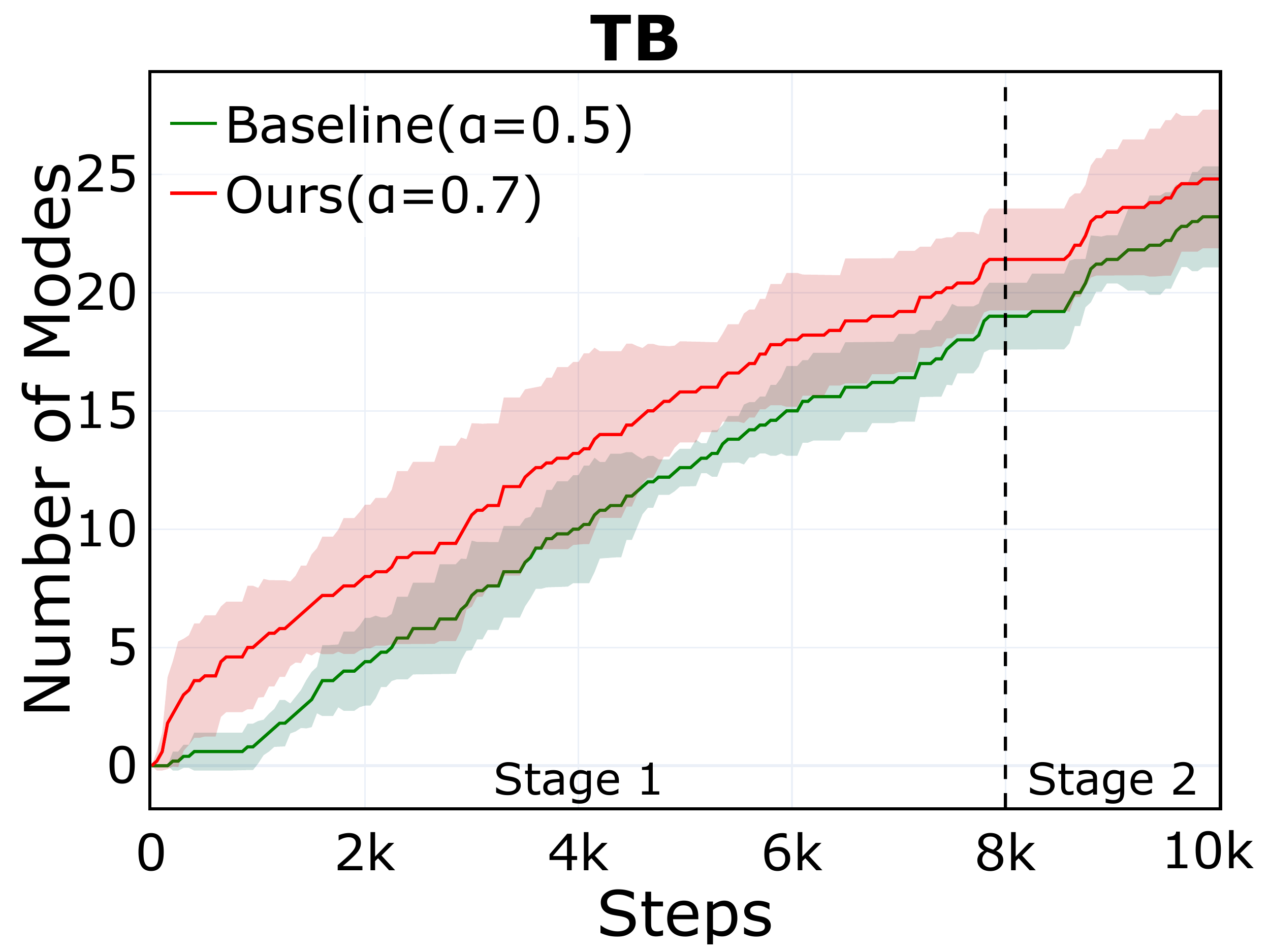} \\
    \small{(a)} DB (small) &
    \small{(b)} FL-DB (small) &
    \small{(c)} SubTB($\lambda$) (small) &
    \small{(d)} FL-SubTB($\lambda$) (small) &
    \small{(e)} TB (small) \\
    \includegraphics[width=0.2\textwidth]{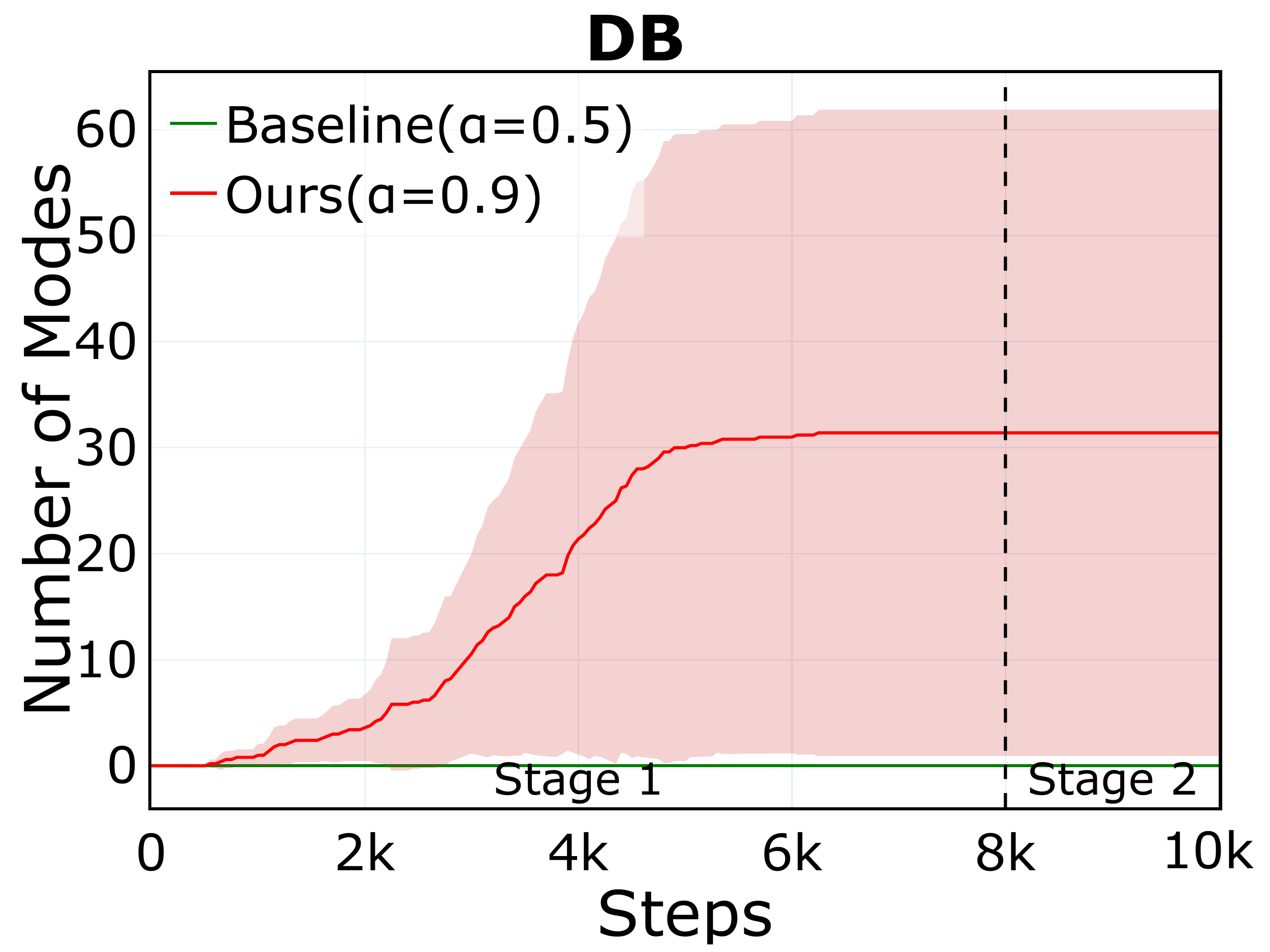} &
    \includegraphics[width=0.2\textwidth]{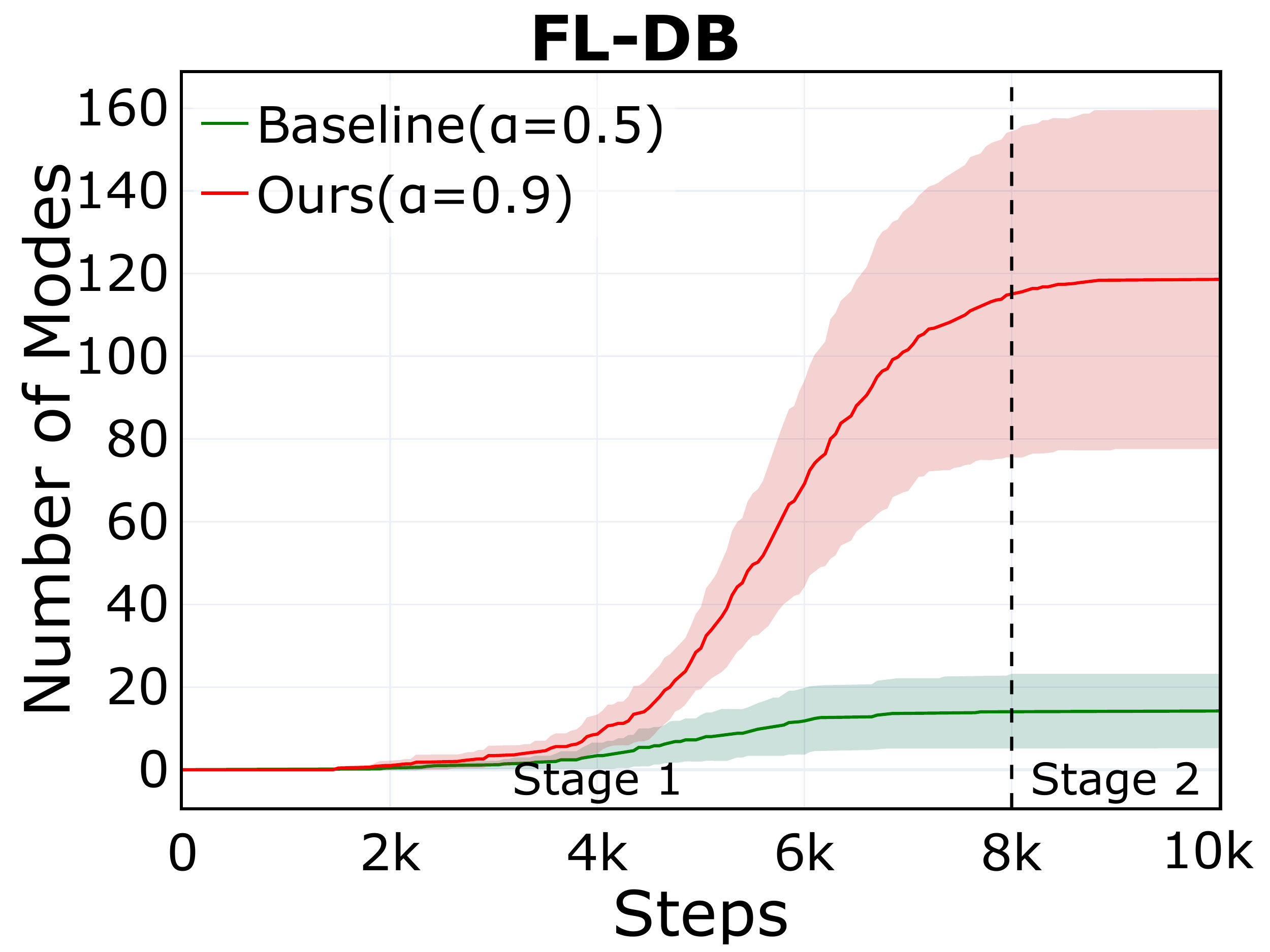} &
     \includegraphics[width=0.2\textwidth]{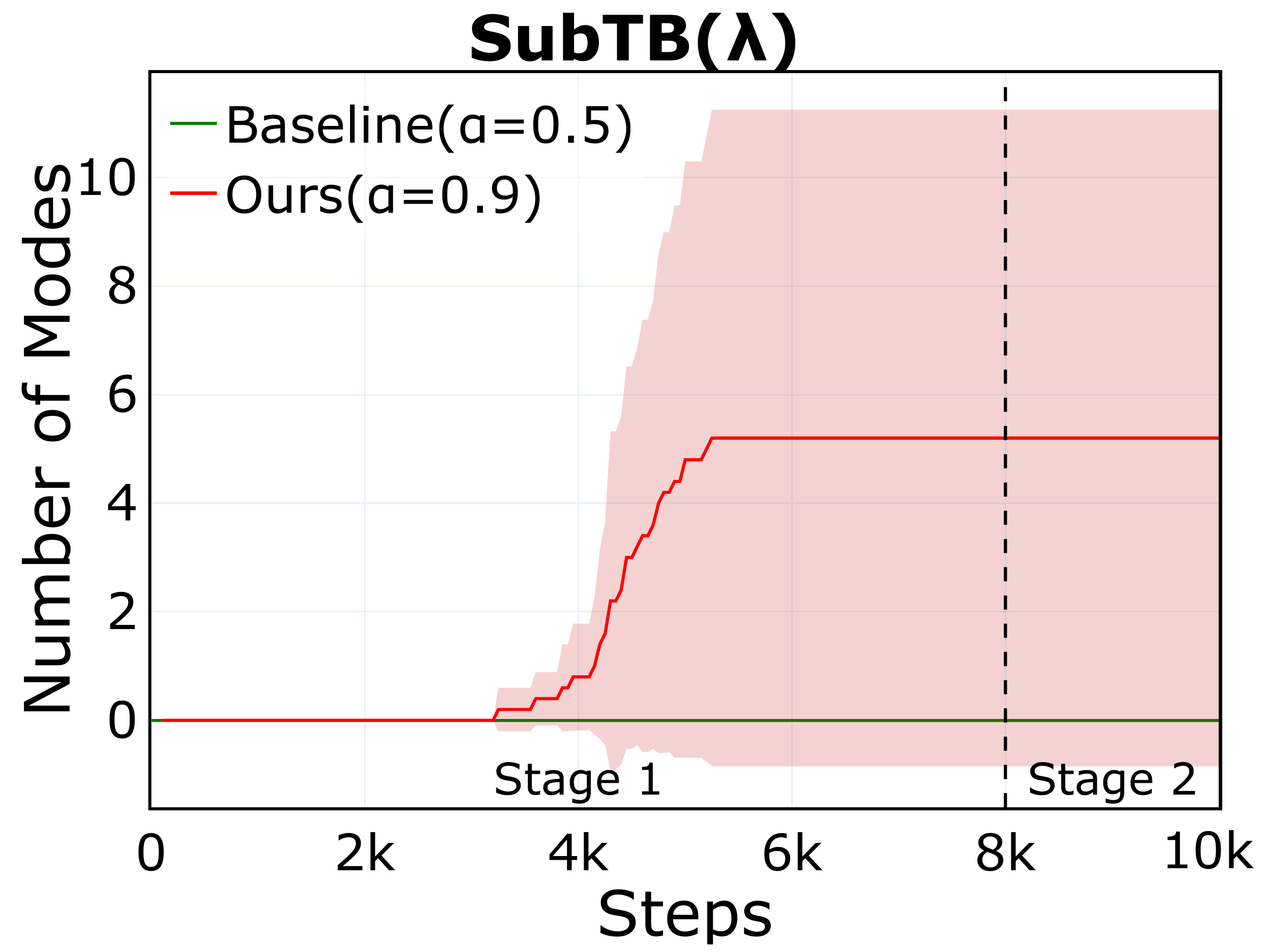}&
    \includegraphics[width=0.2\textwidth]{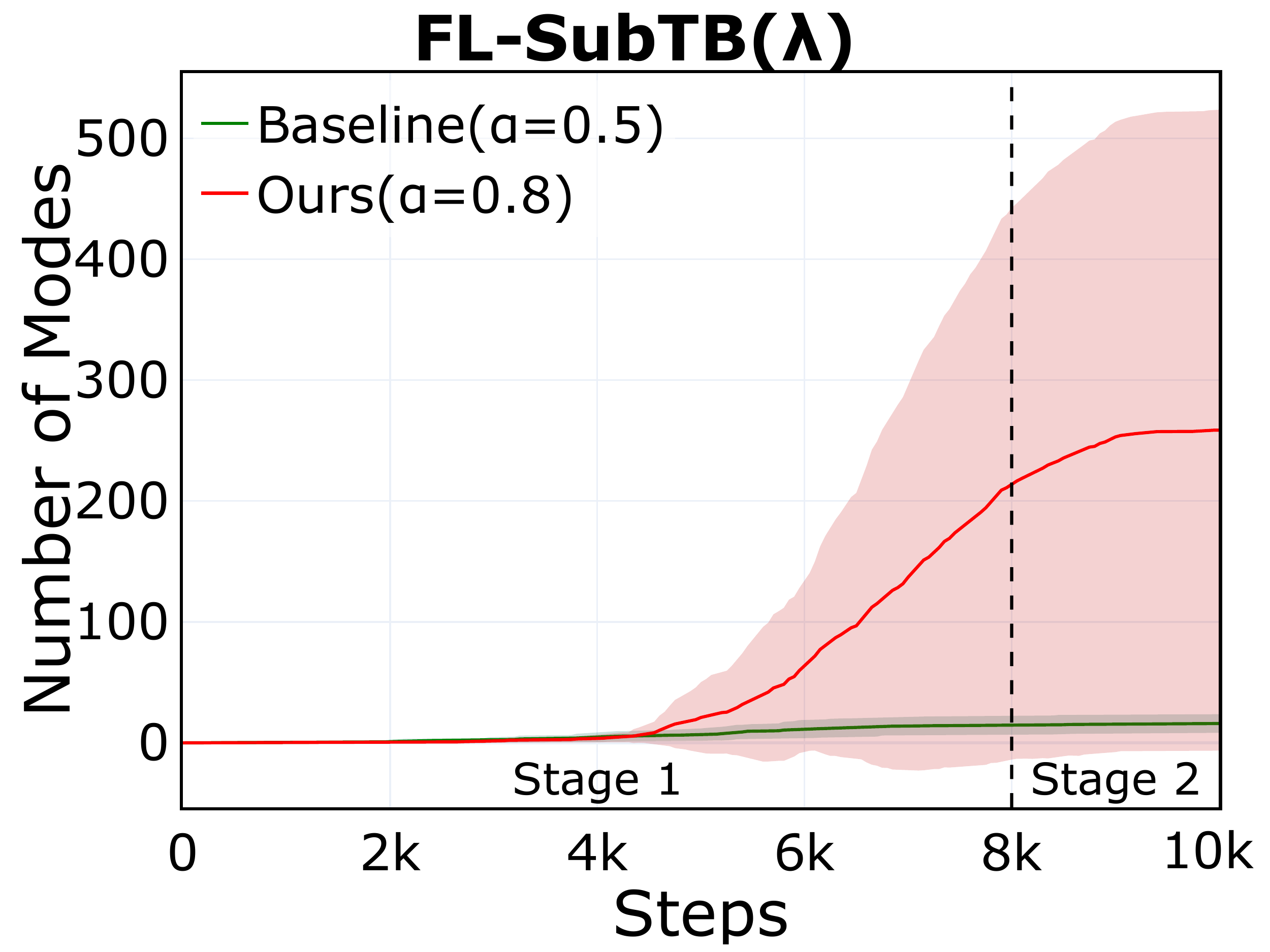} &
    \includegraphics[width=0.2\textwidth]{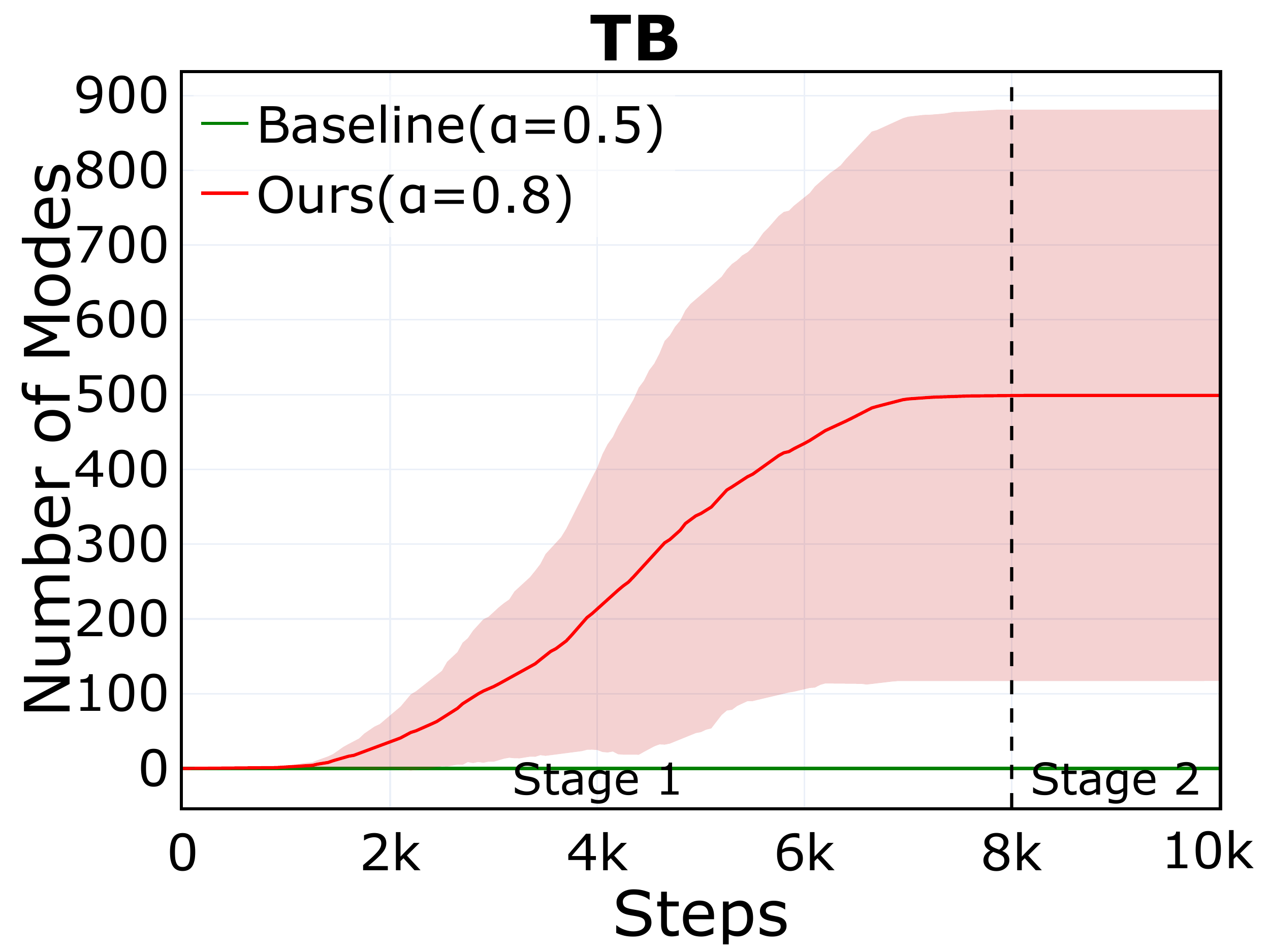} \\
    \small{(f)} DB (medium) &
    \small{(g)} FL-DB (medium) &
    \small{(h)} SubTB($\lambda$) (medium) &
    \small{(i)} FL-SubTB($\lambda$) (medium) &
    \small{(j)} TB (medium) \\
    \includegraphics[width=0.2\textwidth]{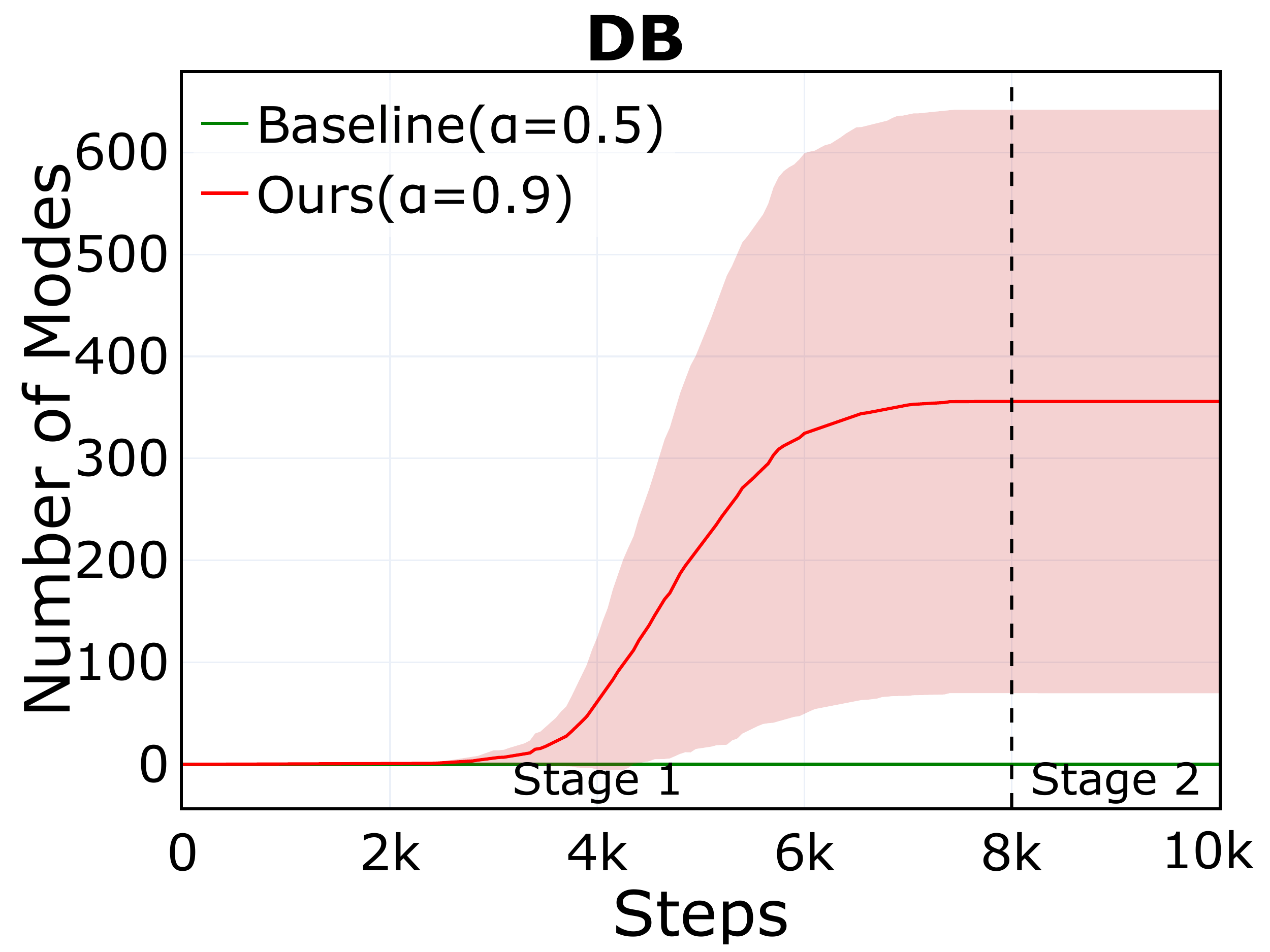} &
    \includegraphics[width=0.2\textwidth]{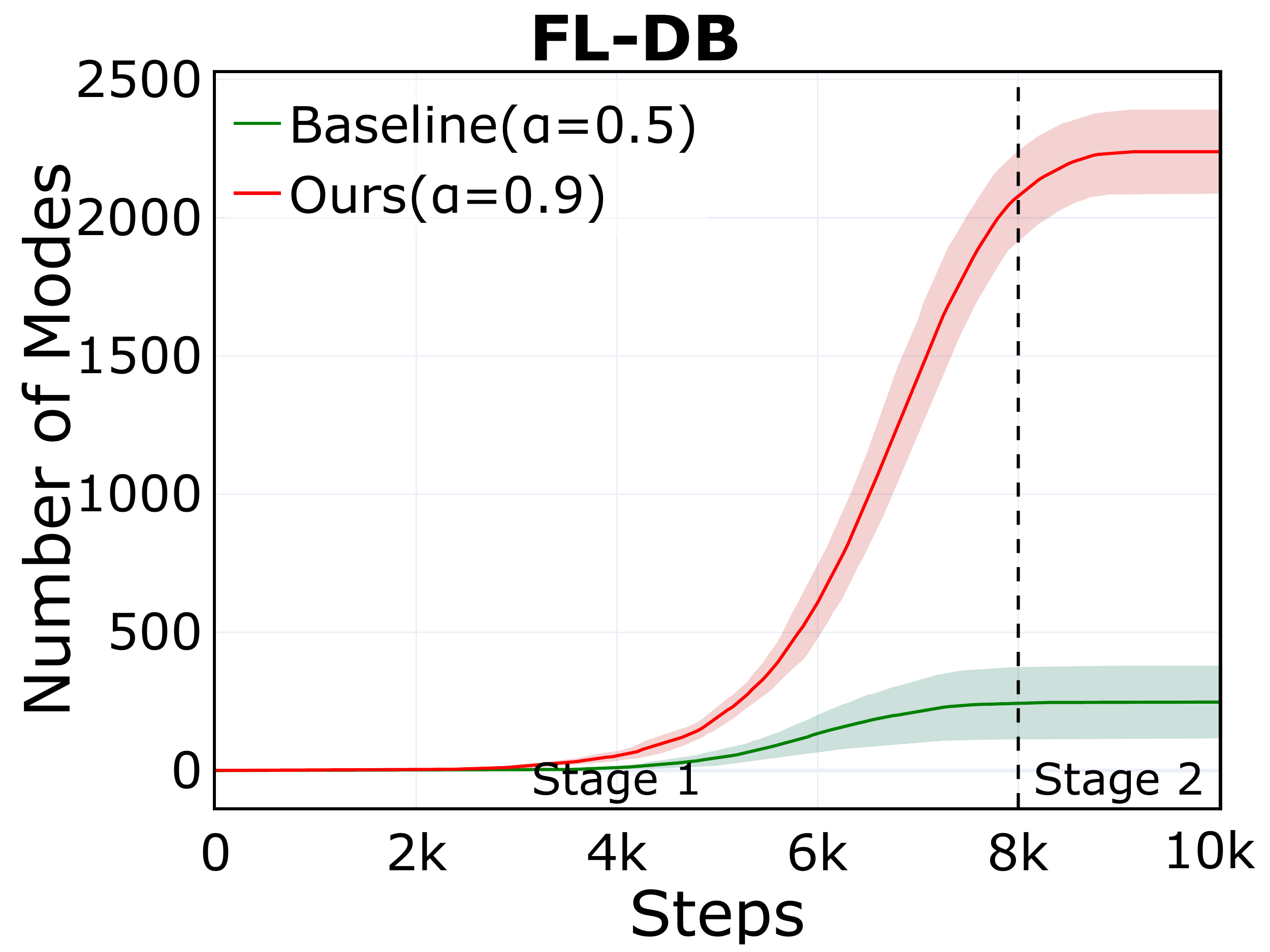} &
     \includegraphics[width=0.2\textwidth]{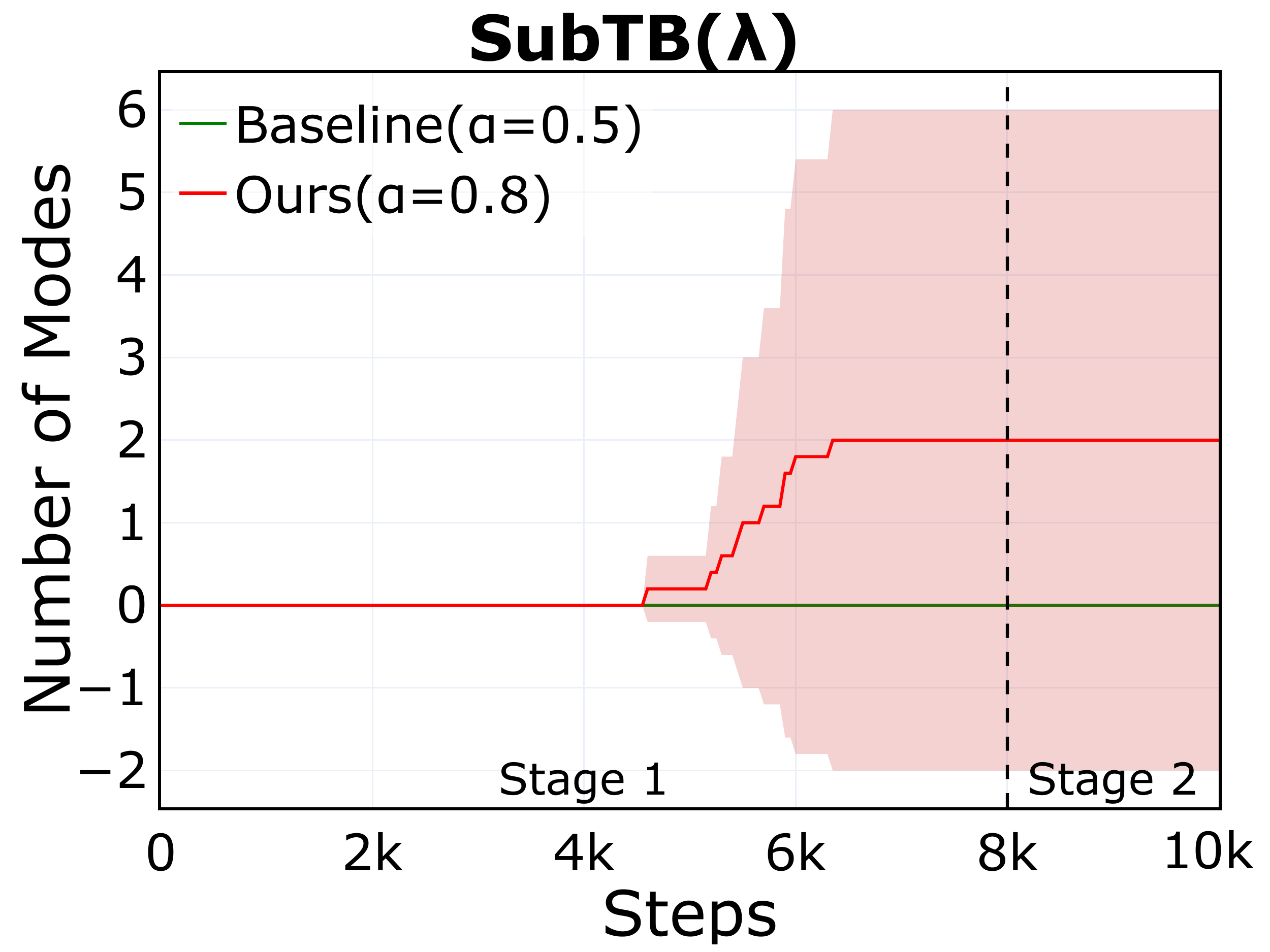} &
     \includegraphics[width=0.2\textwidth]{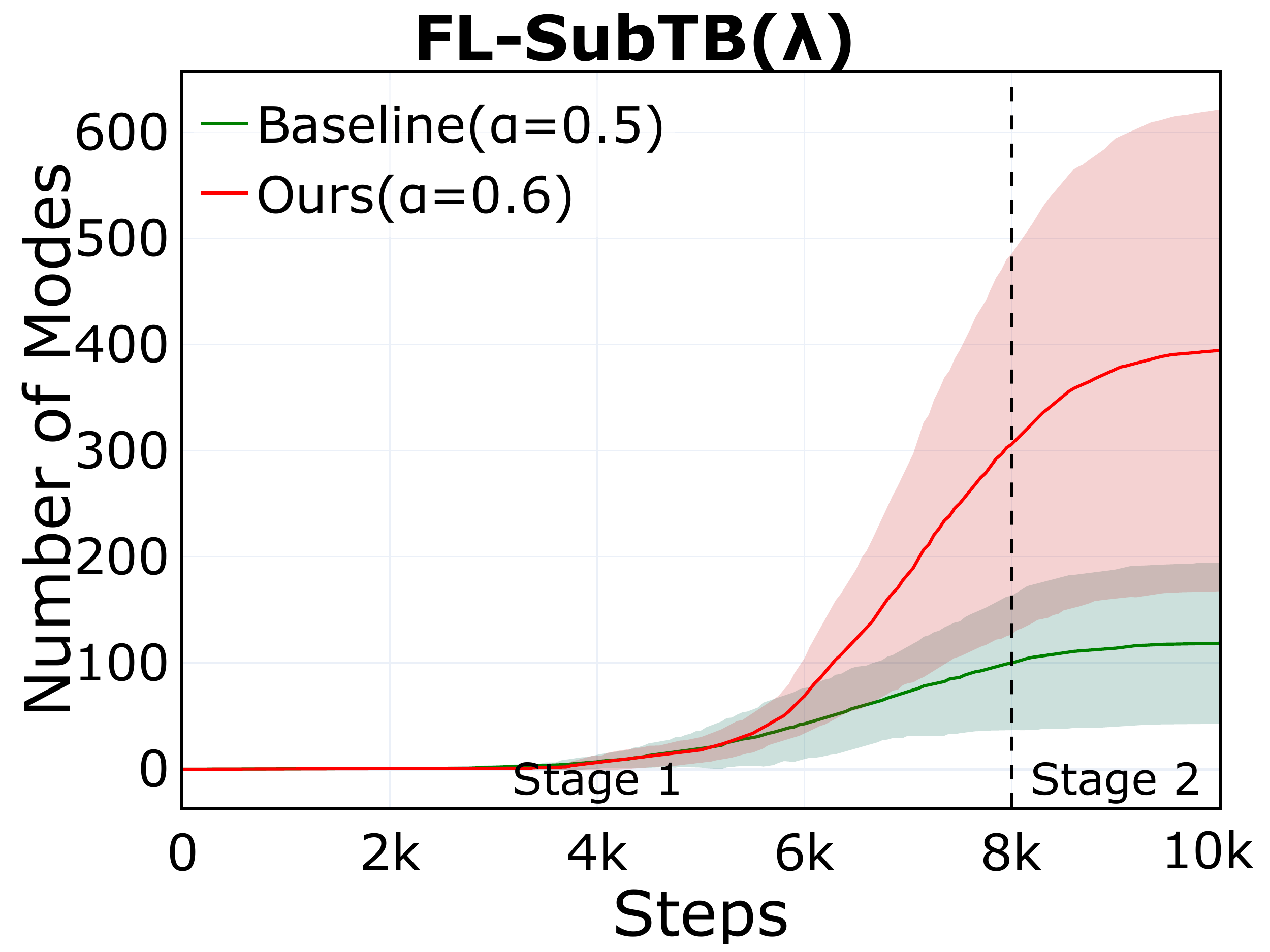}&
    \includegraphics[width=0.2\textwidth]{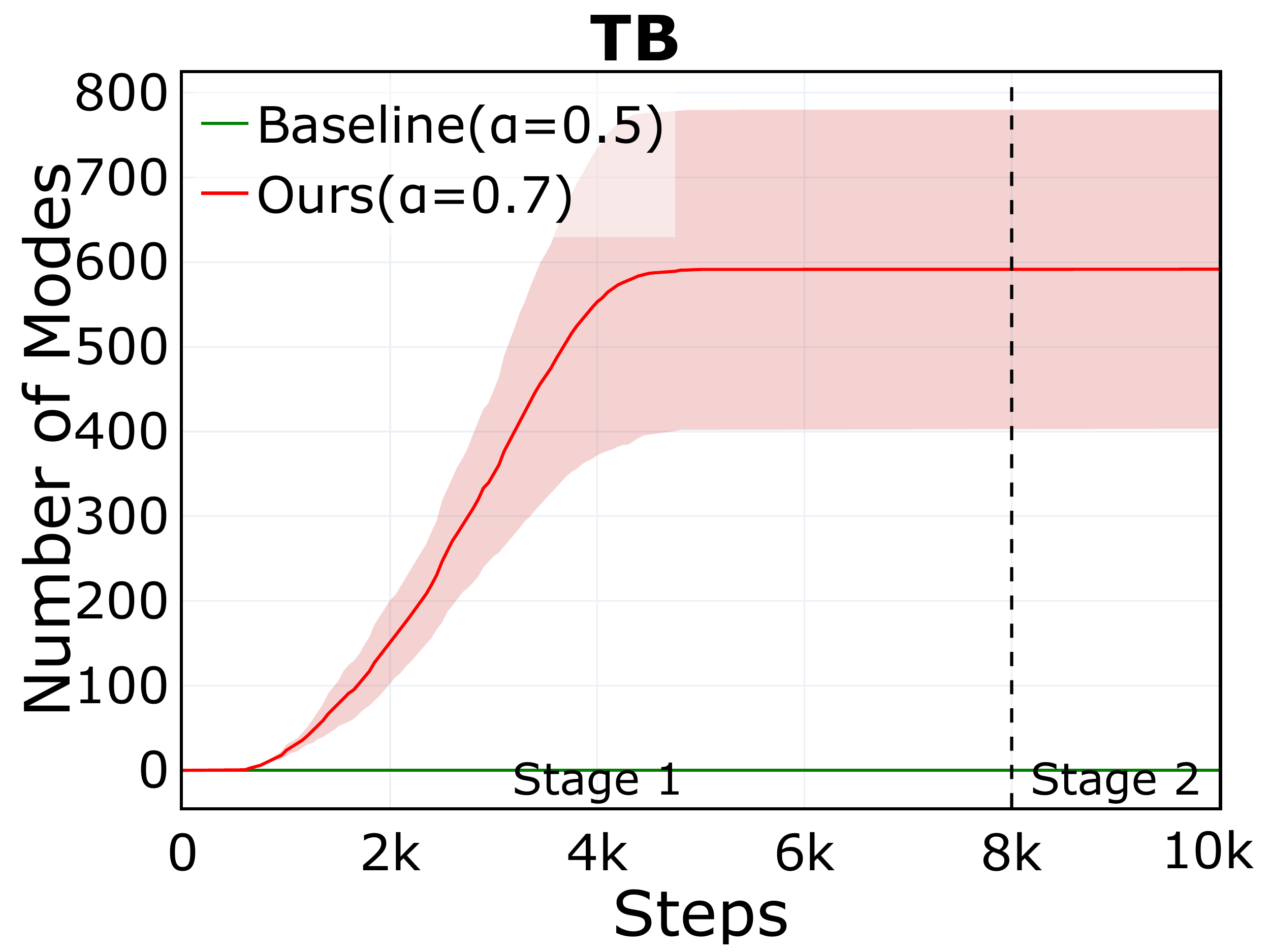} \\
    \small{(k)} DB (large) &
    \small{(l)} FL-DB (large) &
    \small{(m)} SubTB($\lambda$) (large) &
    \small{(n)} FL-SubTB($\lambda$) (large) &
    \small{(o)} TB (large) \\
  \end{tabular}
  \caption{\textbf{Number of Modes} vs Training Steps in \textbf{Set Generation} across different objectives and set sizes.}
  \label{fig:set_metric_modes}
\end{figure}

\begin{figure}[htbp]
  \centering
  \setlength{\tabcolsep}{0pt}
 \begin{tabular}{@{}c@{\hspace{0pt}}c@{\hspace{0pt}}c@{\hspace{0pt}}c@{\hspace{0pt}}c@{}}
    \includegraphics[width=0.2\textwidth]{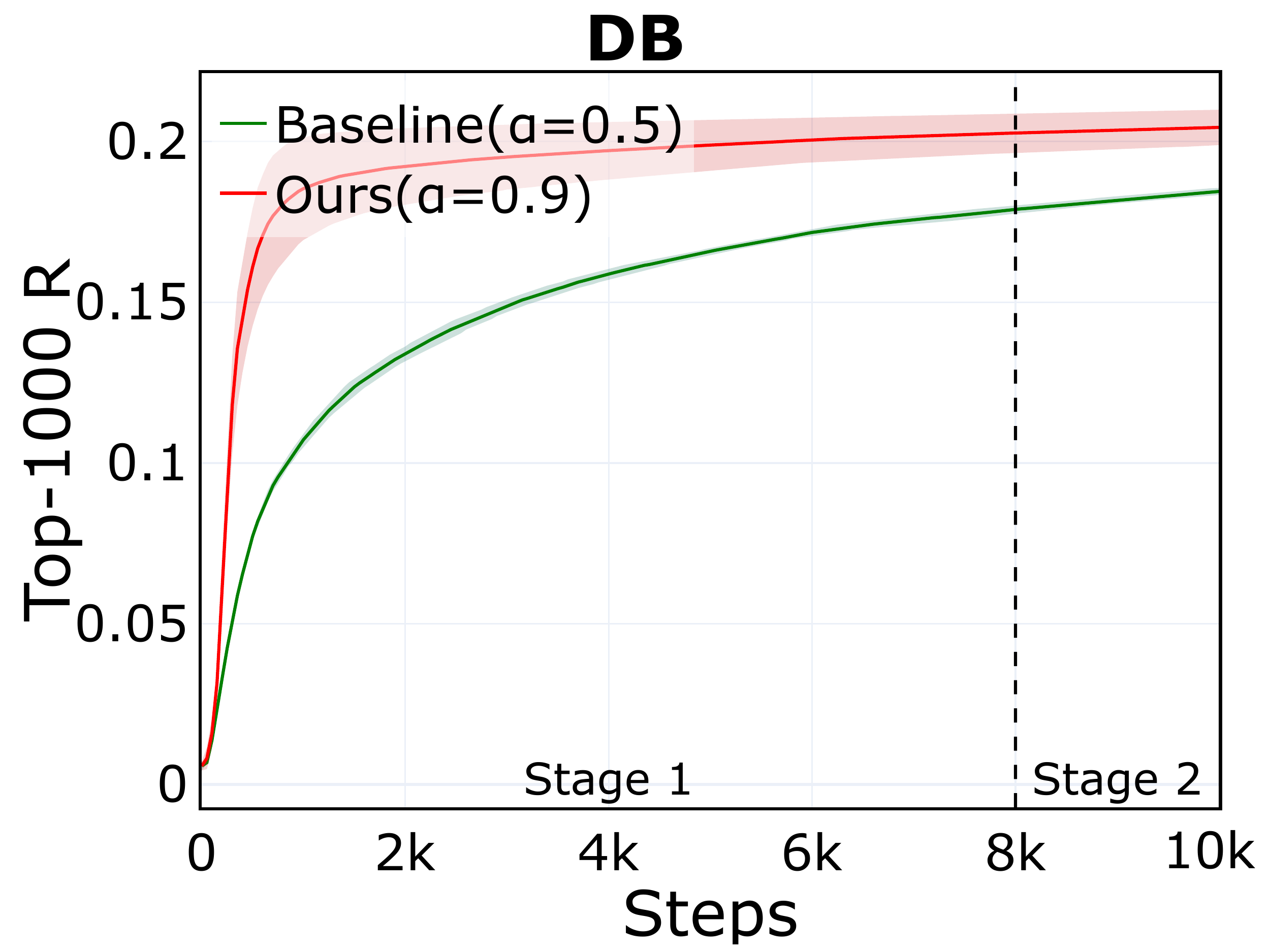} &
    \includegraphics[width=0.2\textwidth]{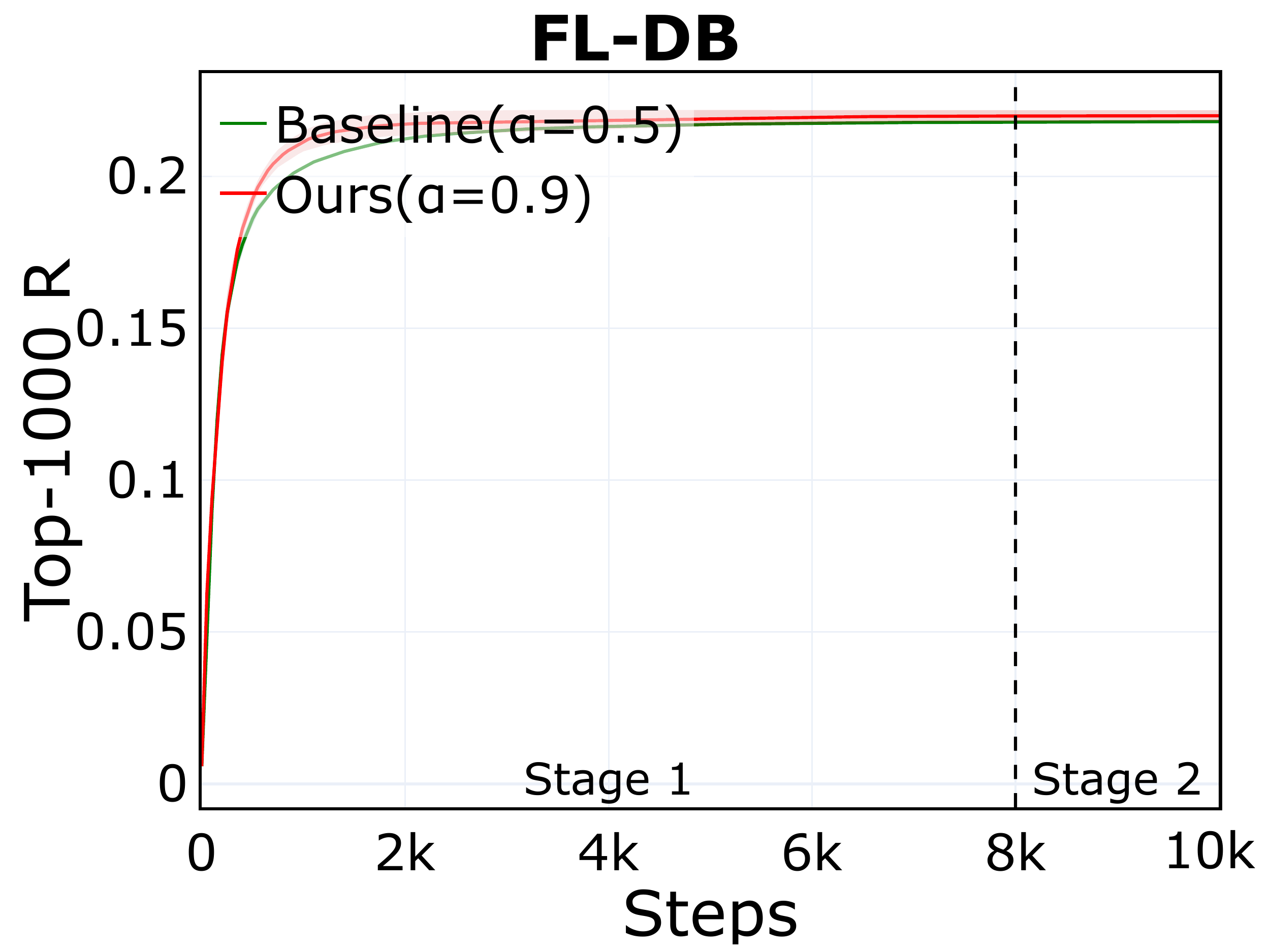} &
    \includegraphics[width=0.2\textwidth]{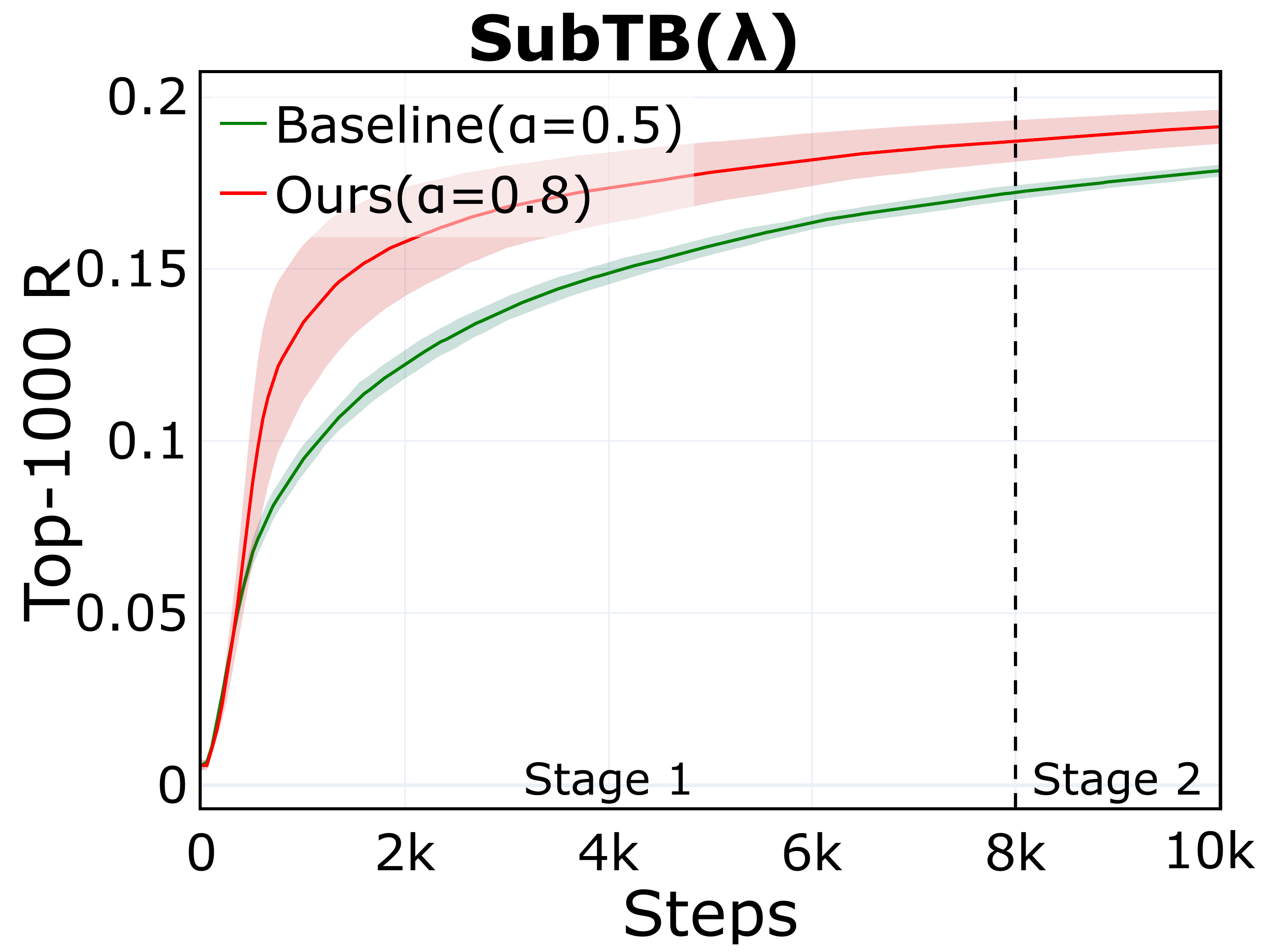} & 
    \includegraphics[width=0.2\textwidth]{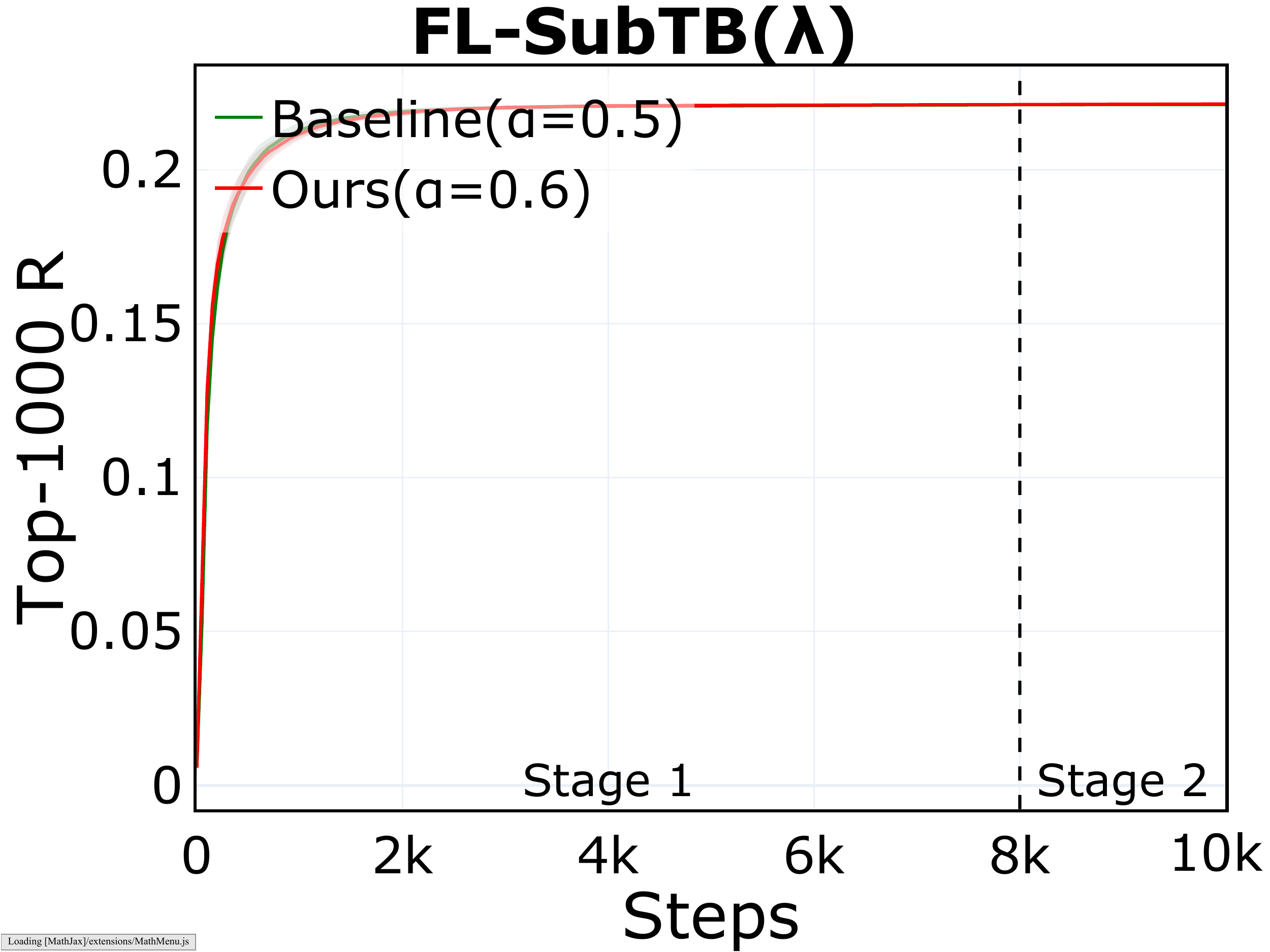} &
    \includegraphics[width=0.2\textwidth]{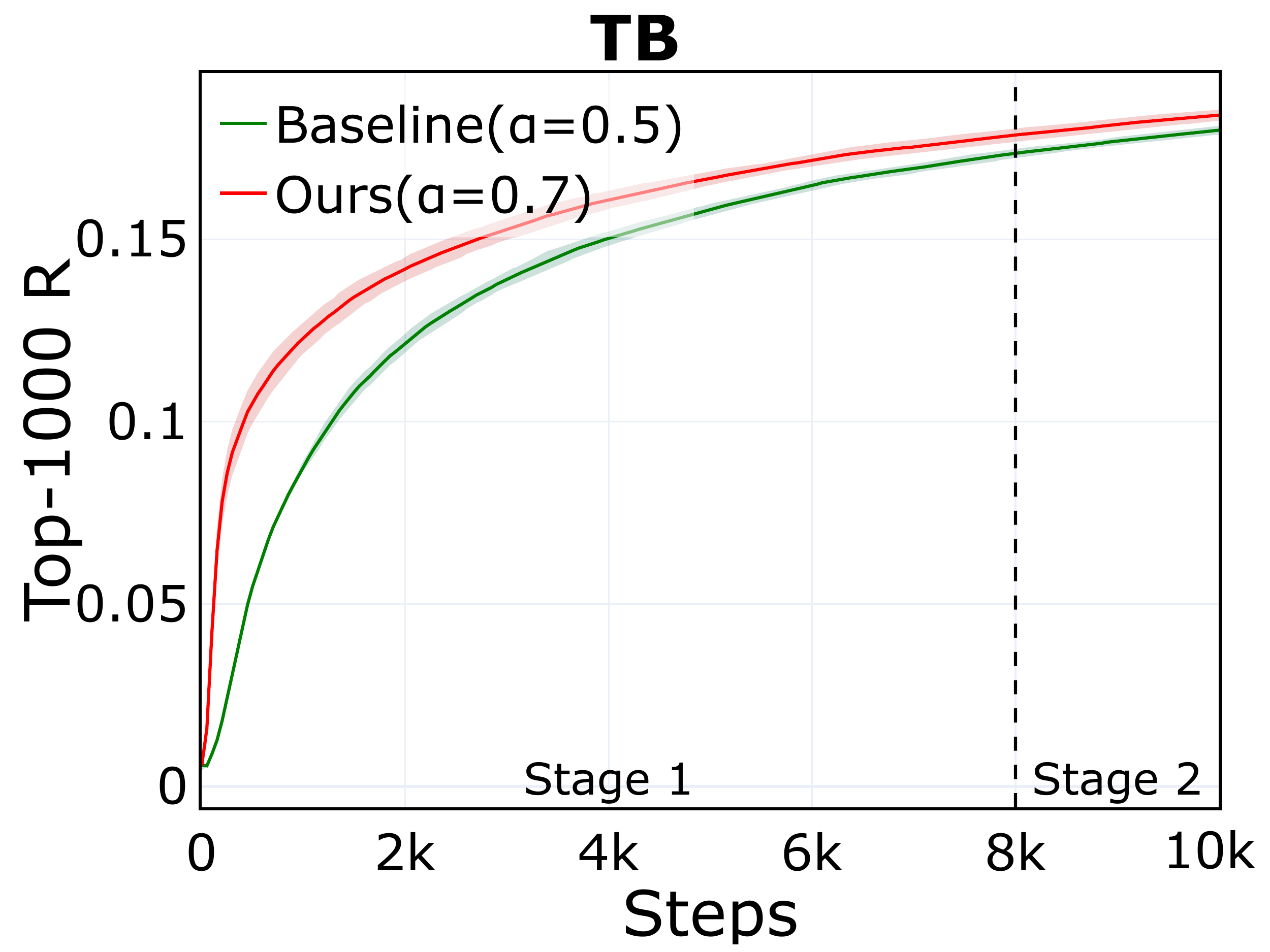} \\
    \small{(a)} DB (small) &
    \small{(b)} FL-DB (small) &
    \small{(c)} SubTB($\lambda$) (small) &
    \small{(d)} FL-SubTB($\lambda$) (small) &
    \small{(e)} TB (small) \\
    \includegraphics[width=0.2\textwidth]{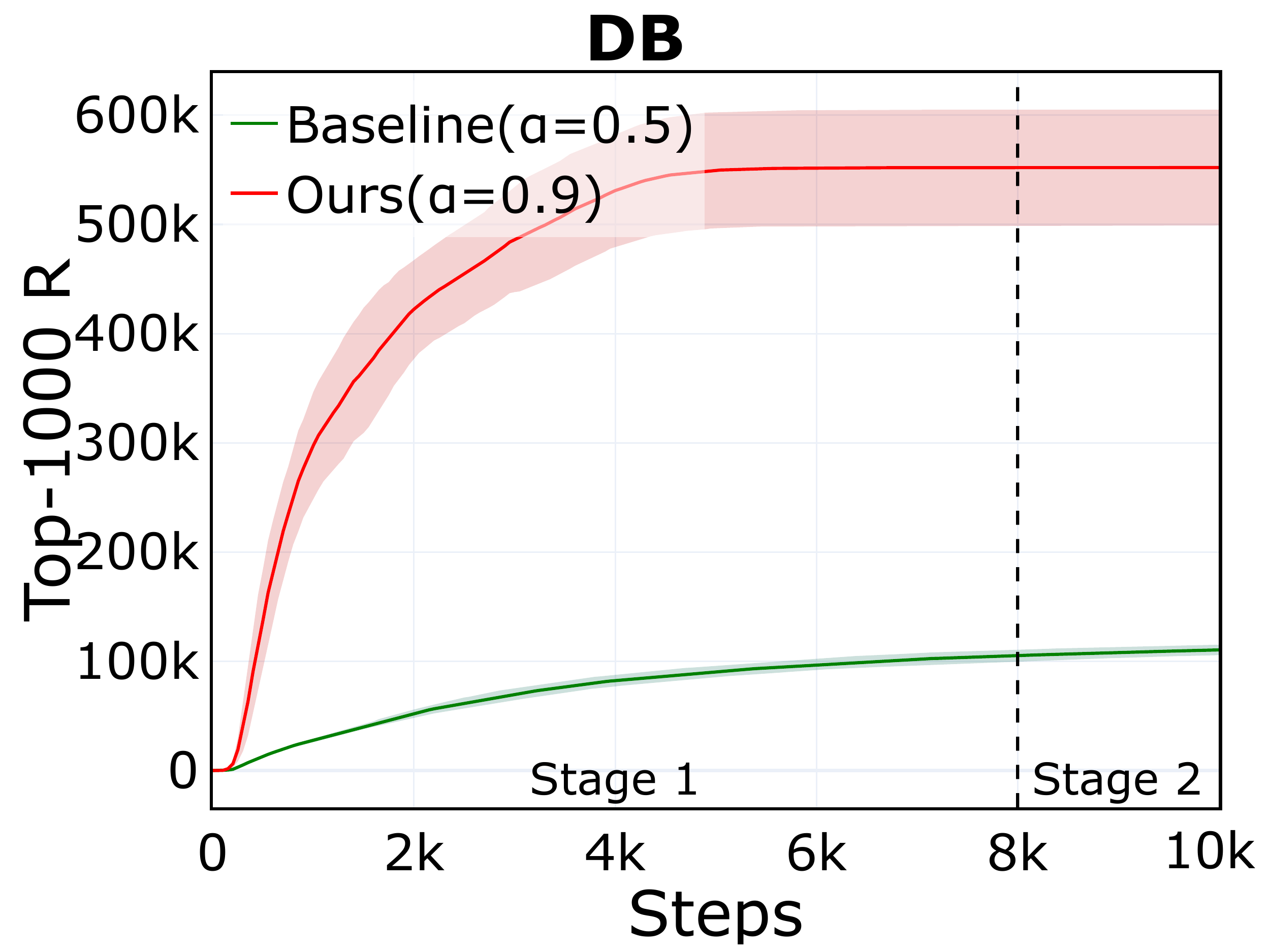} &
    \includegraphics[width=0.2\textwidth]{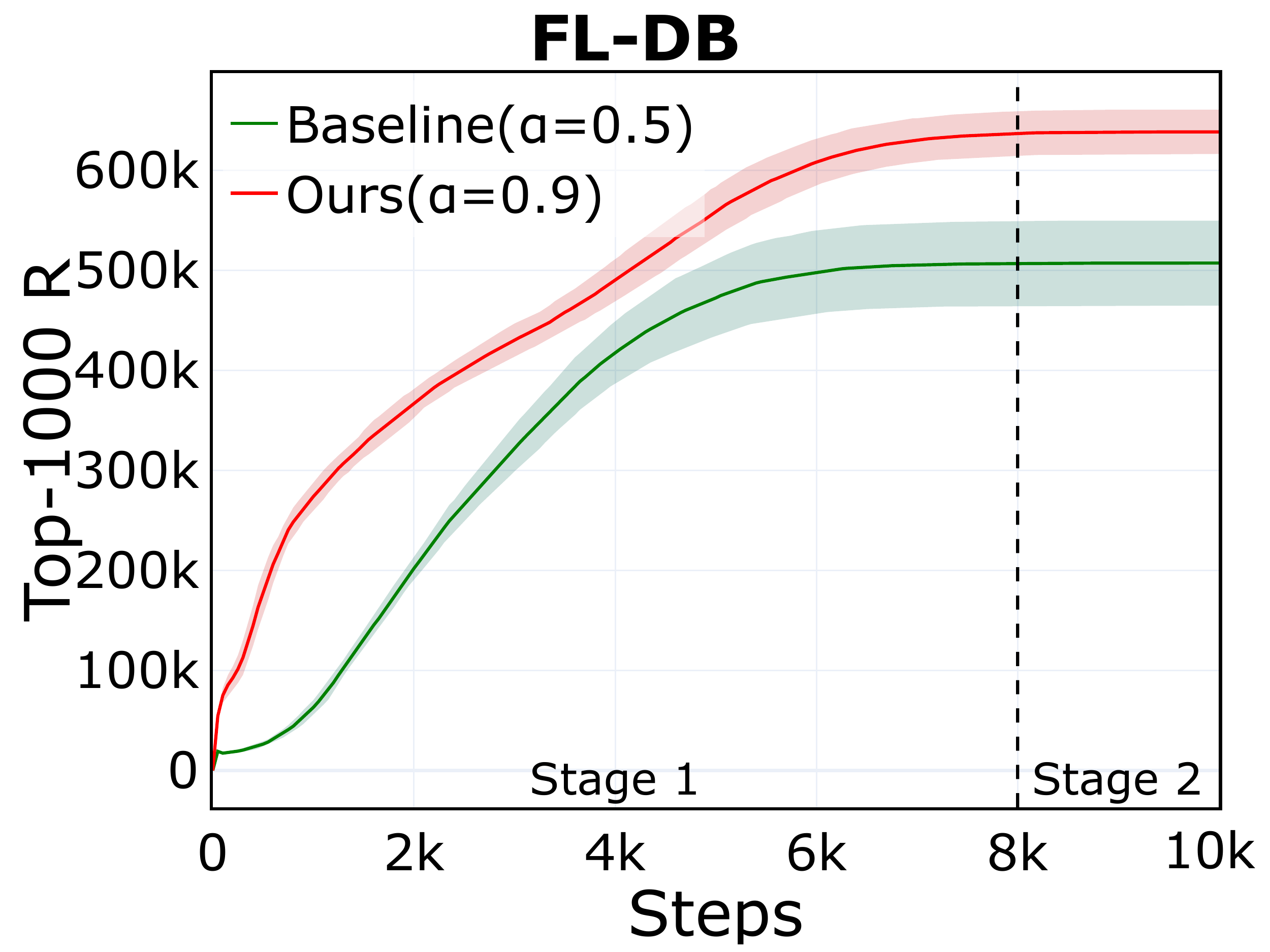} &
     \includegraphics[width=0.2\textwidth]{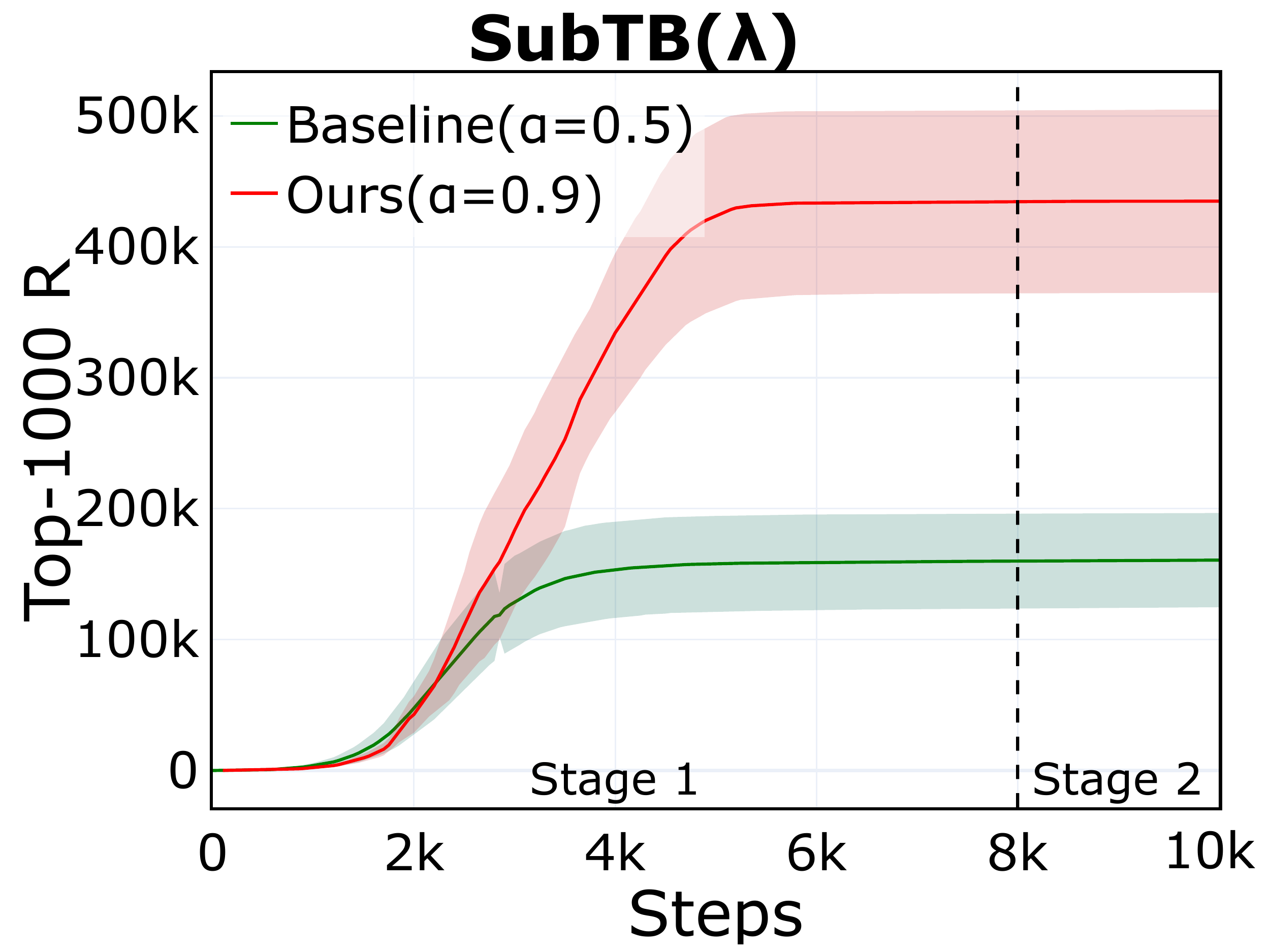}&
    \includegraphics[width=0.2\textwidth]{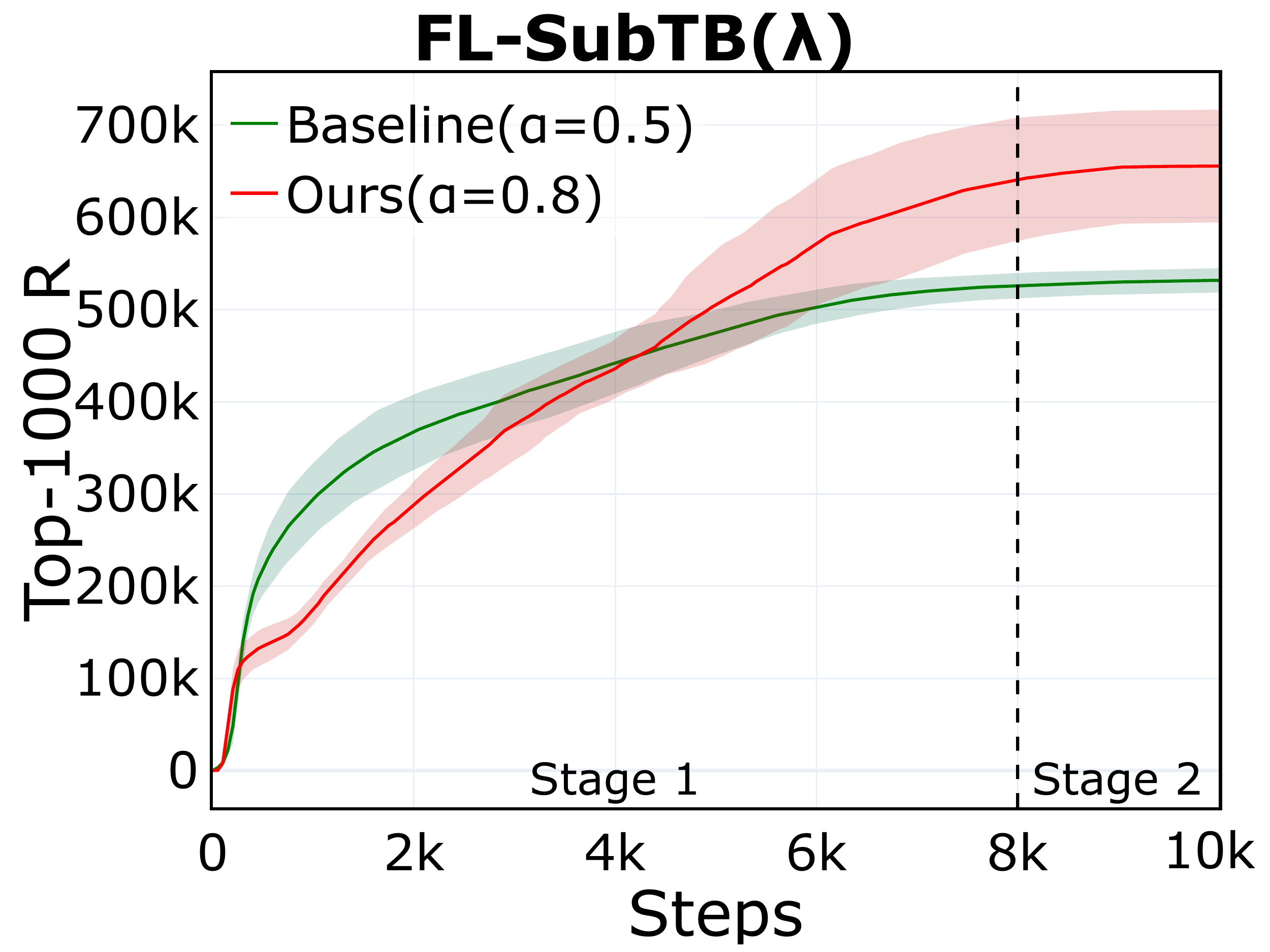} &
    \includegraphics[width=0.2\textwidth]{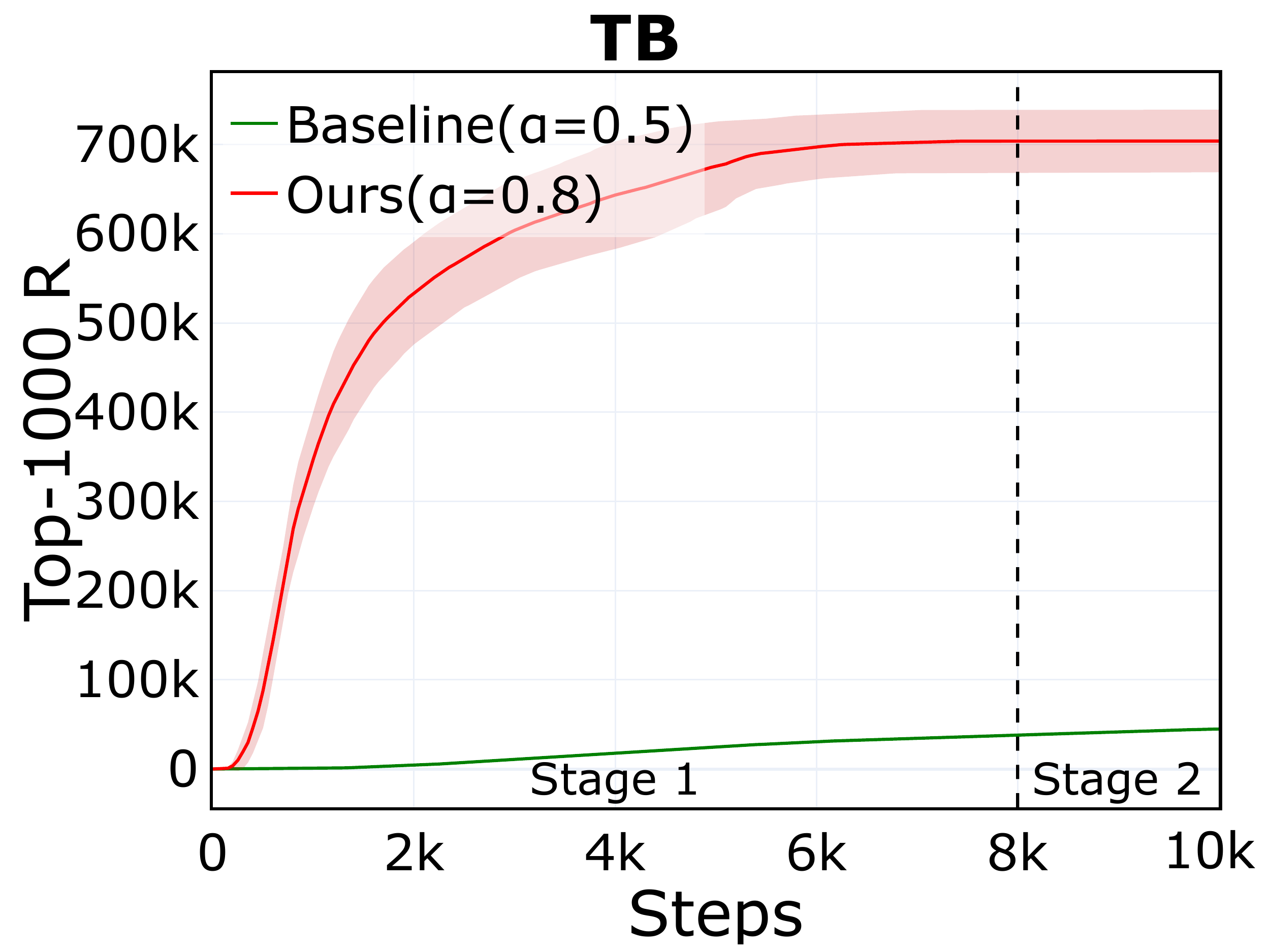} \\
    \small{(f)} DB (medium) &
    \small{(g)} FL-DB (medium) &
    \small{(h)} SubTB($\lambda$) (medium) &
    \small{(i)} FL-SubTB($\lambda$) (medium) &
    \small{(j)} TB (medium) \\
    \includegraphics[width=0.2\textwidth]{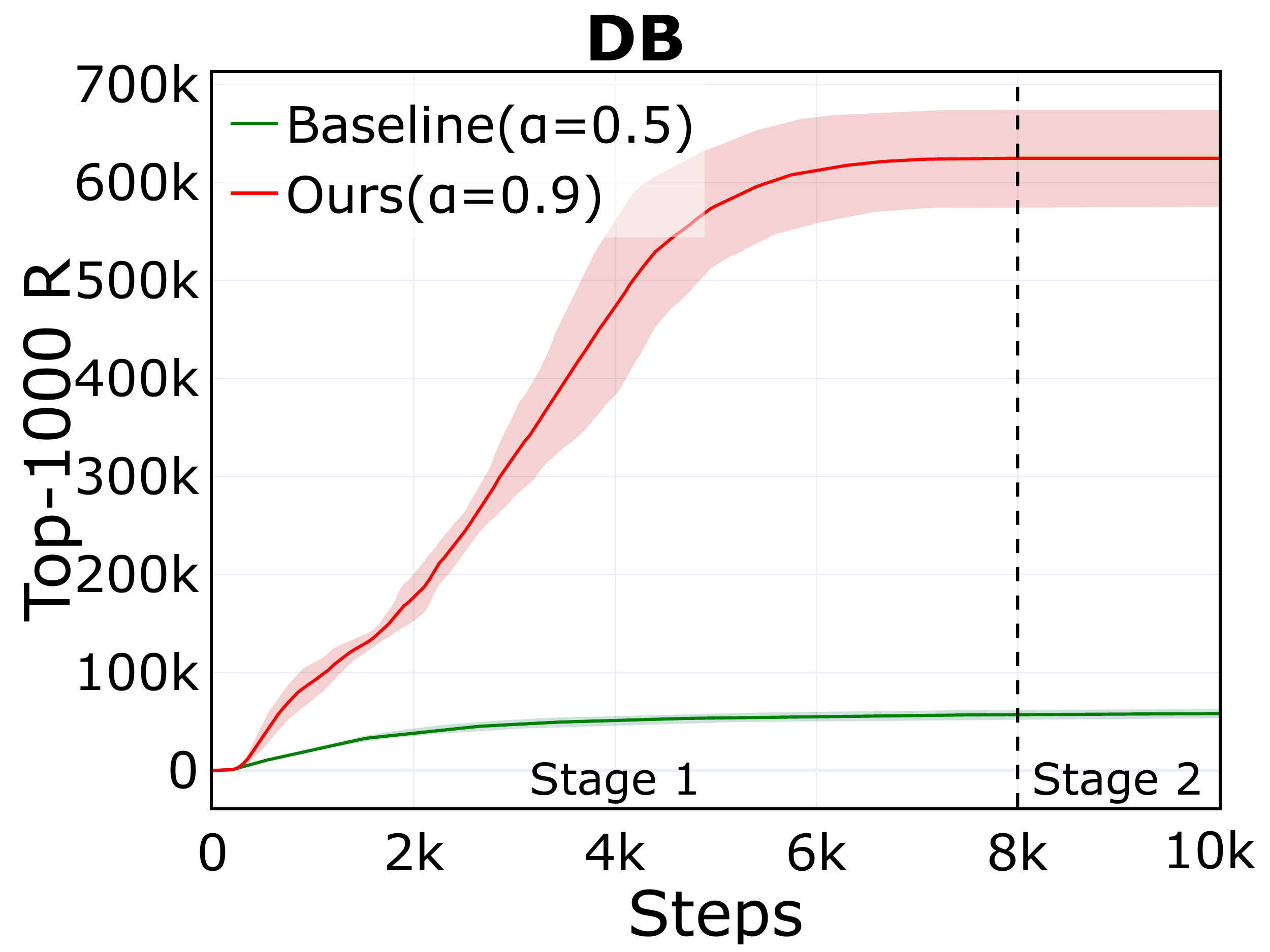} &
    \includegraphics[width=0.2\textwidth]{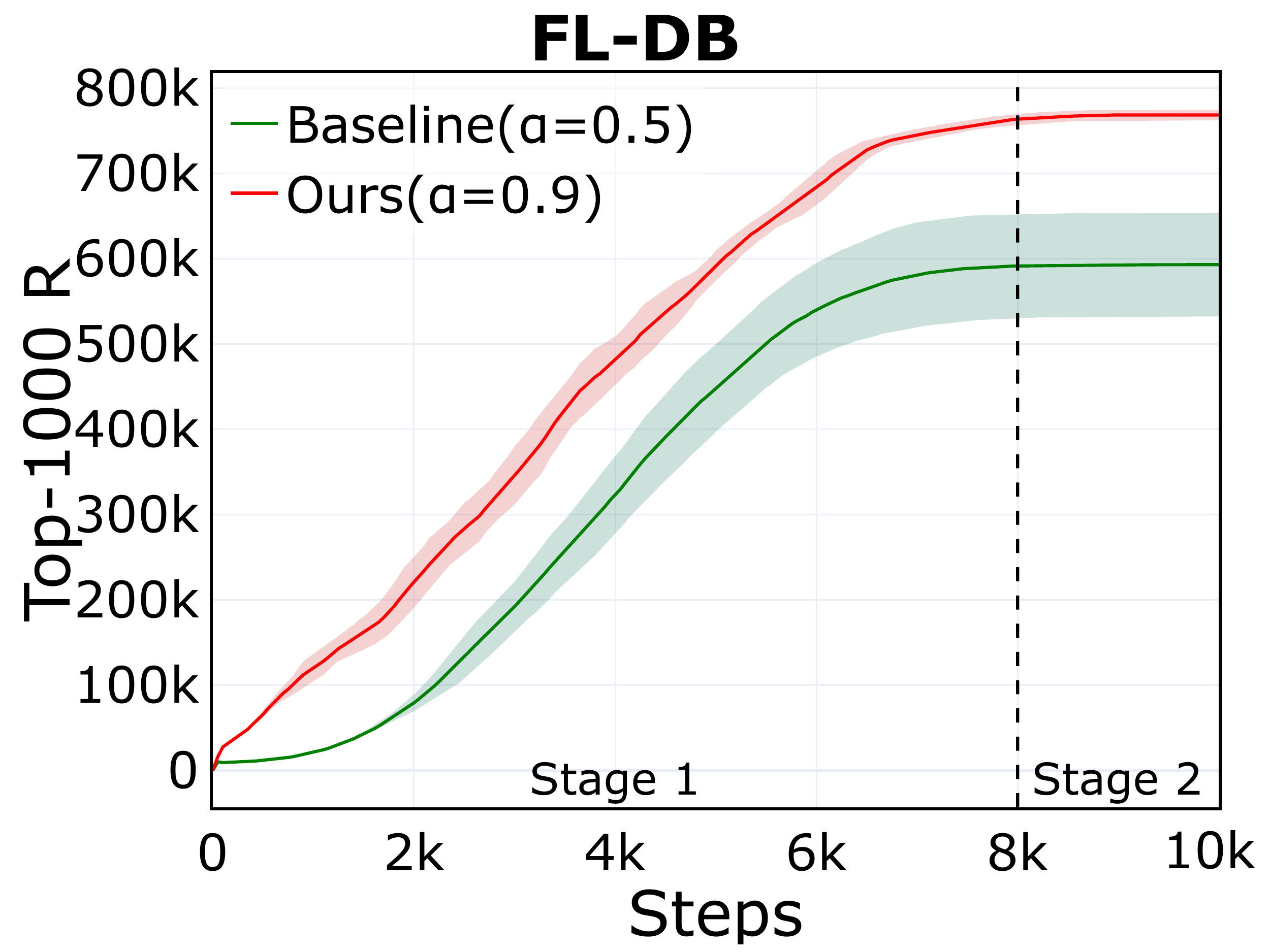} &
     \includegraphics[width=0.2\textwidth]{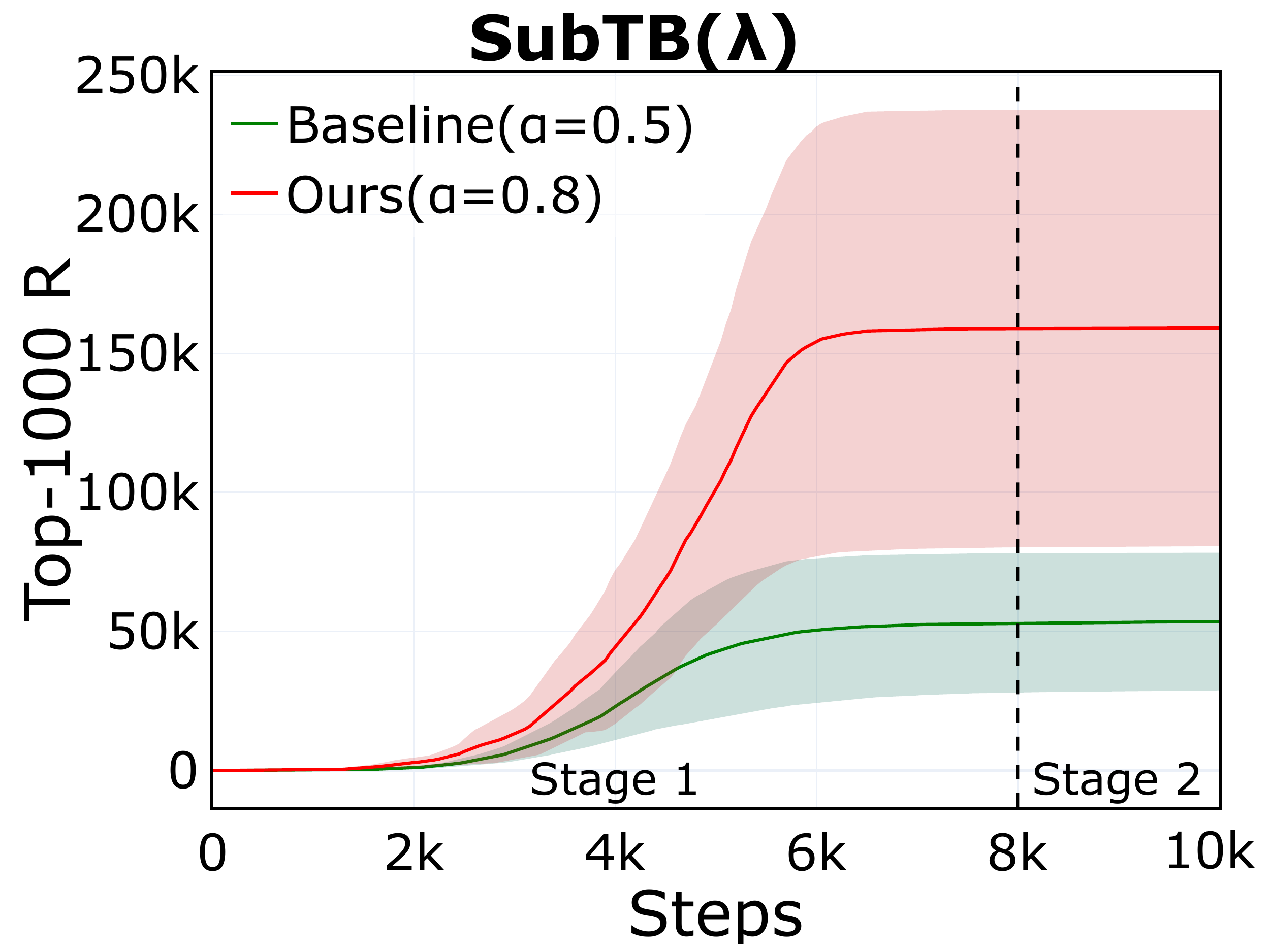} &
     \includegraphics[width=0.2\textwidth]{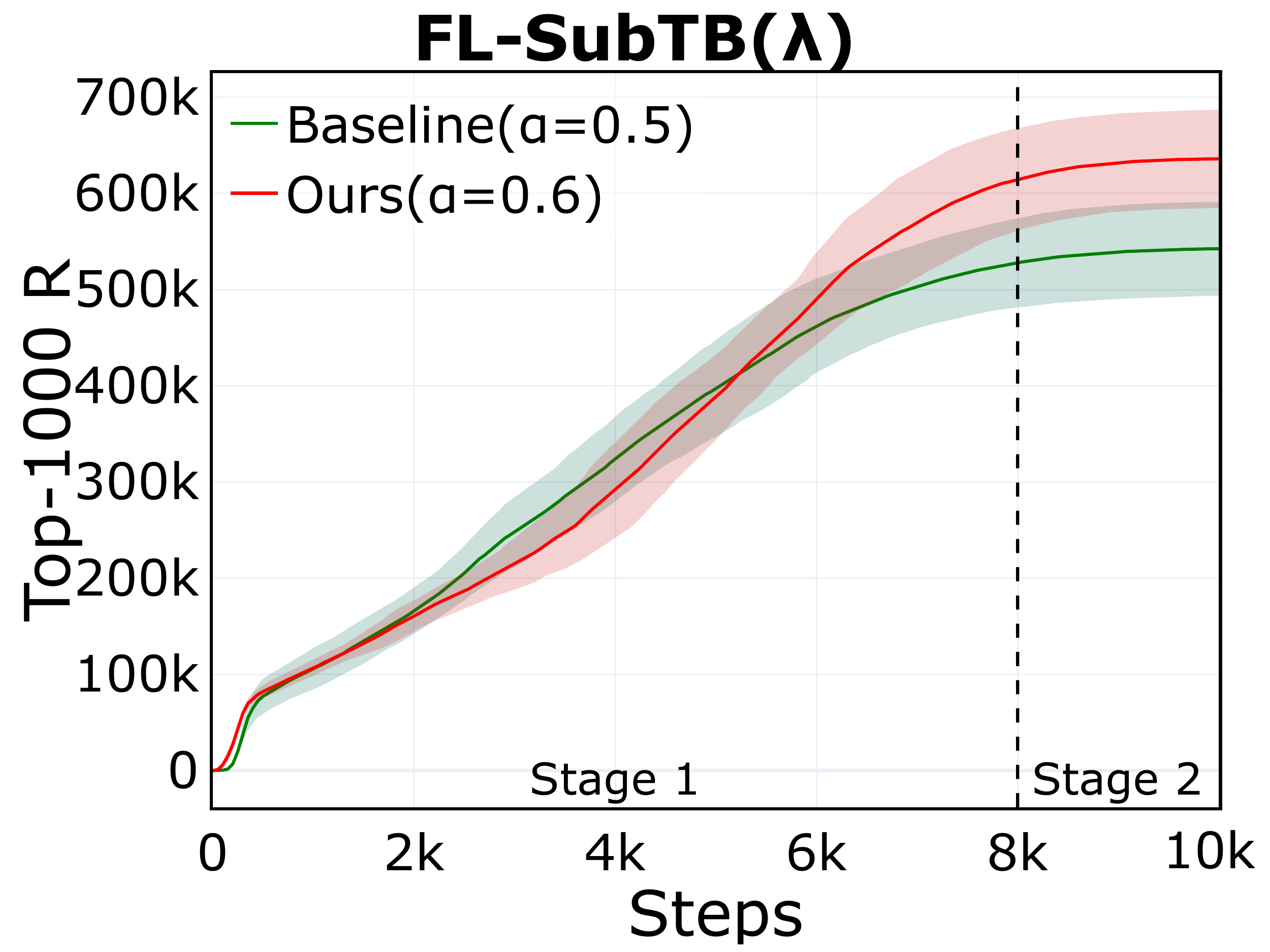}&
    \includegraphics[width=0.2\textwidth]{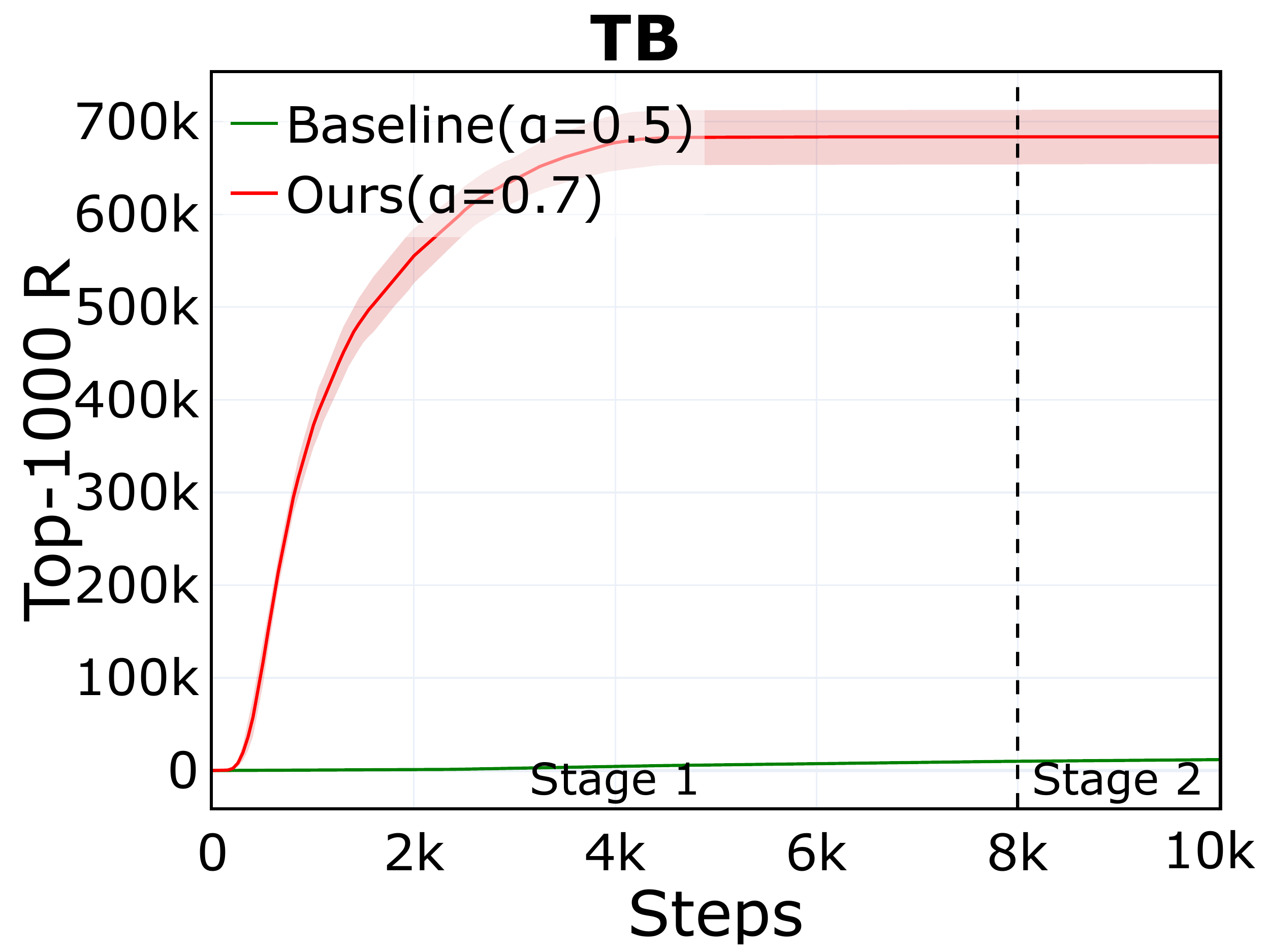} \\
    \small{(k)} DB (large) &
    \small{(l)} FL-DB (large) &
    \small{(m)} SubTB($\lambda$) (large) &
    \small{(n)} FL-SubTB($\lambda$) (large) &
    \small{(o)} TB (large) \\
  \end{tabular}
  \caption{\textbf{Top-1000 R} vs Training Steps in \textbf{Set Generation} across different objectives and set sizes.}
  \label{fig:set_metric_mean_top_1000_R}
\end{figure}

\begin{figure}[htbp]
  \centering
  \setlength{\tabcolsep}{0pt}
  \begin{tabular}{@{}c@{\hspace{0pt}}c@{\hspace{0pt}}c@{\hspace{0pt}}c@{\hspace{0pt}}c@{}}
    \includegraphics[width=0.2\textwidth]{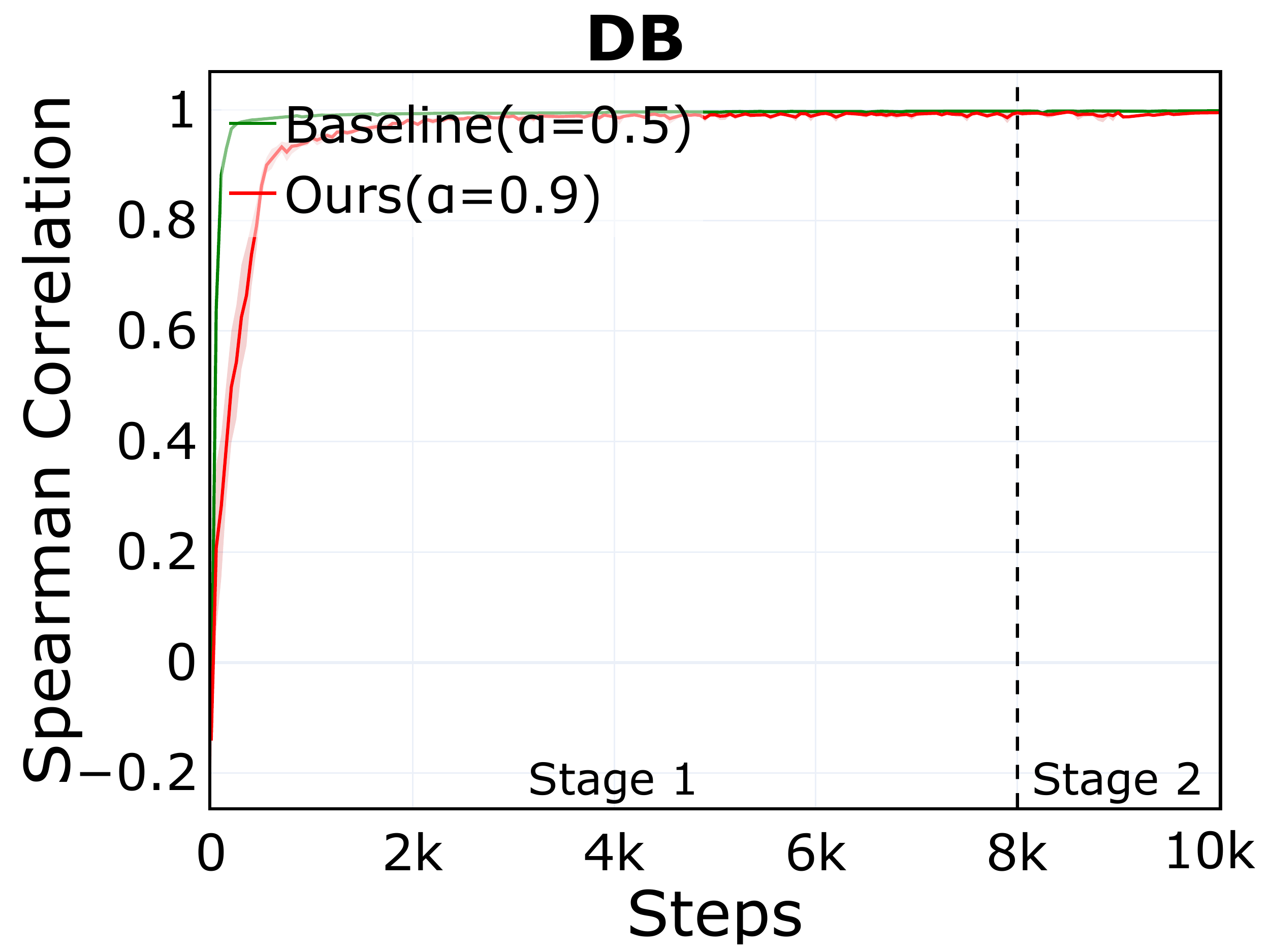} &
    \includegraphics[width=0.2\textwidth]{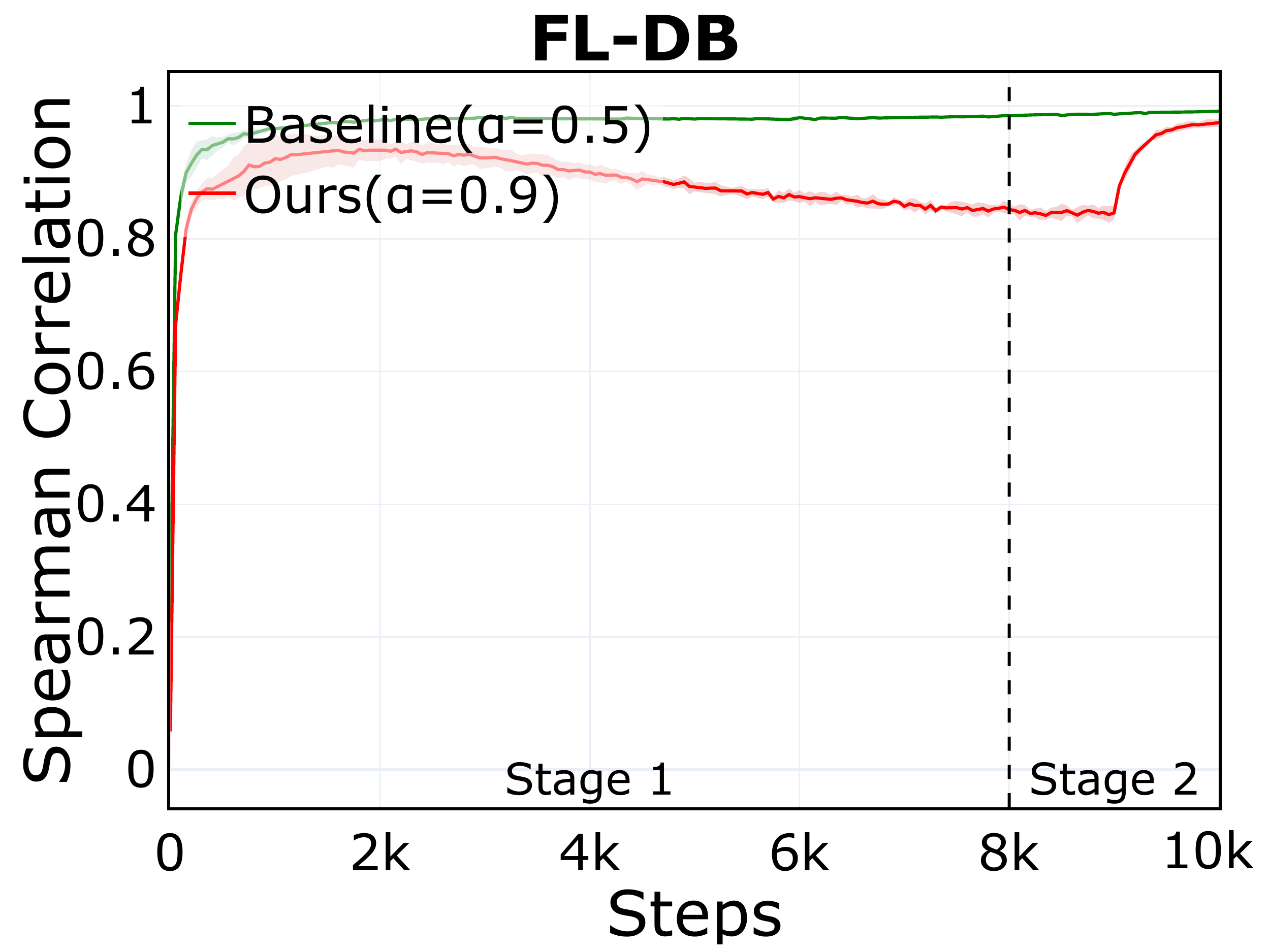} &
    \includegraphics[width=0.2\textwidth]{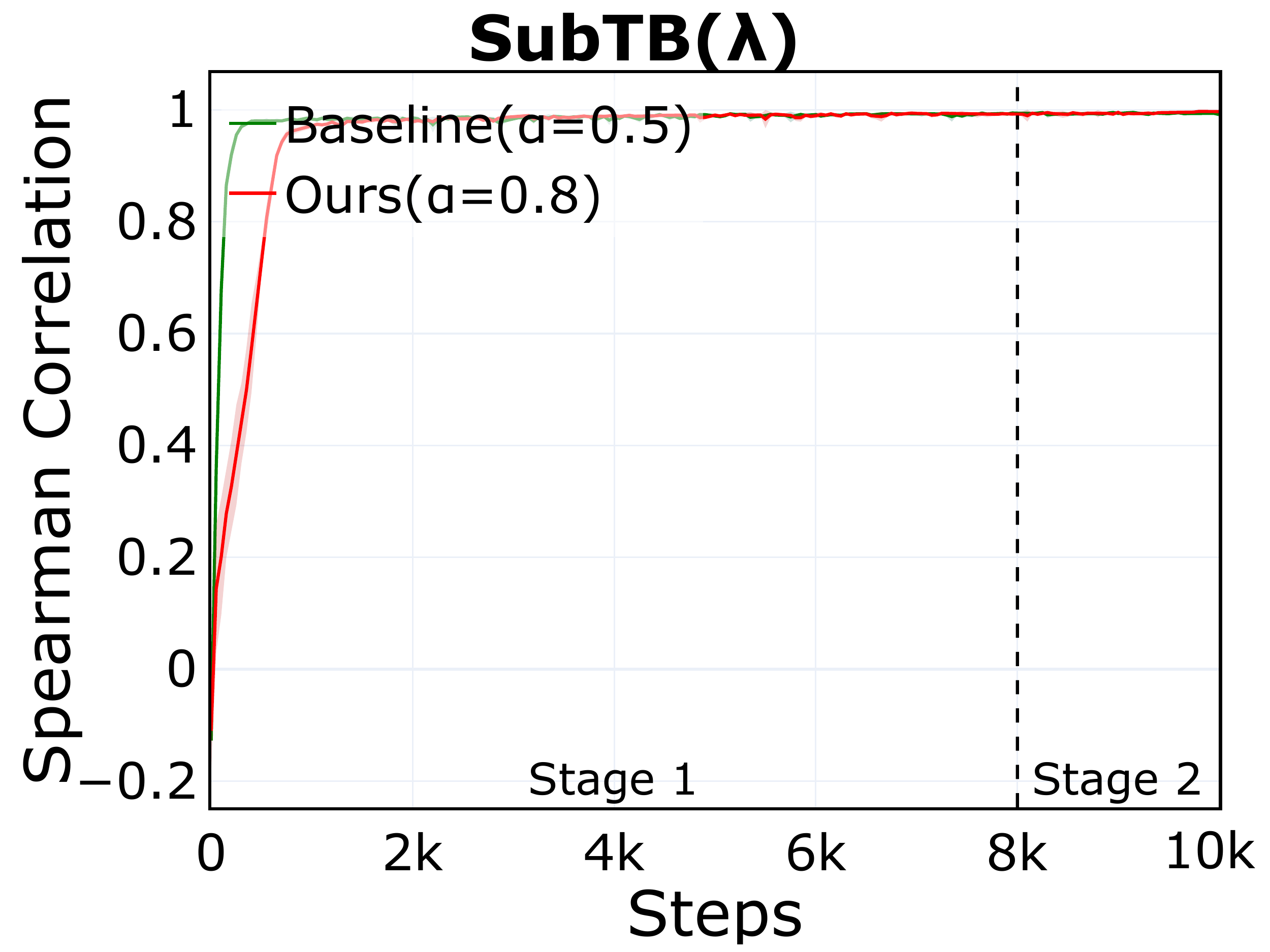} & 
    \includegraphics[width=0.2\textwidth]{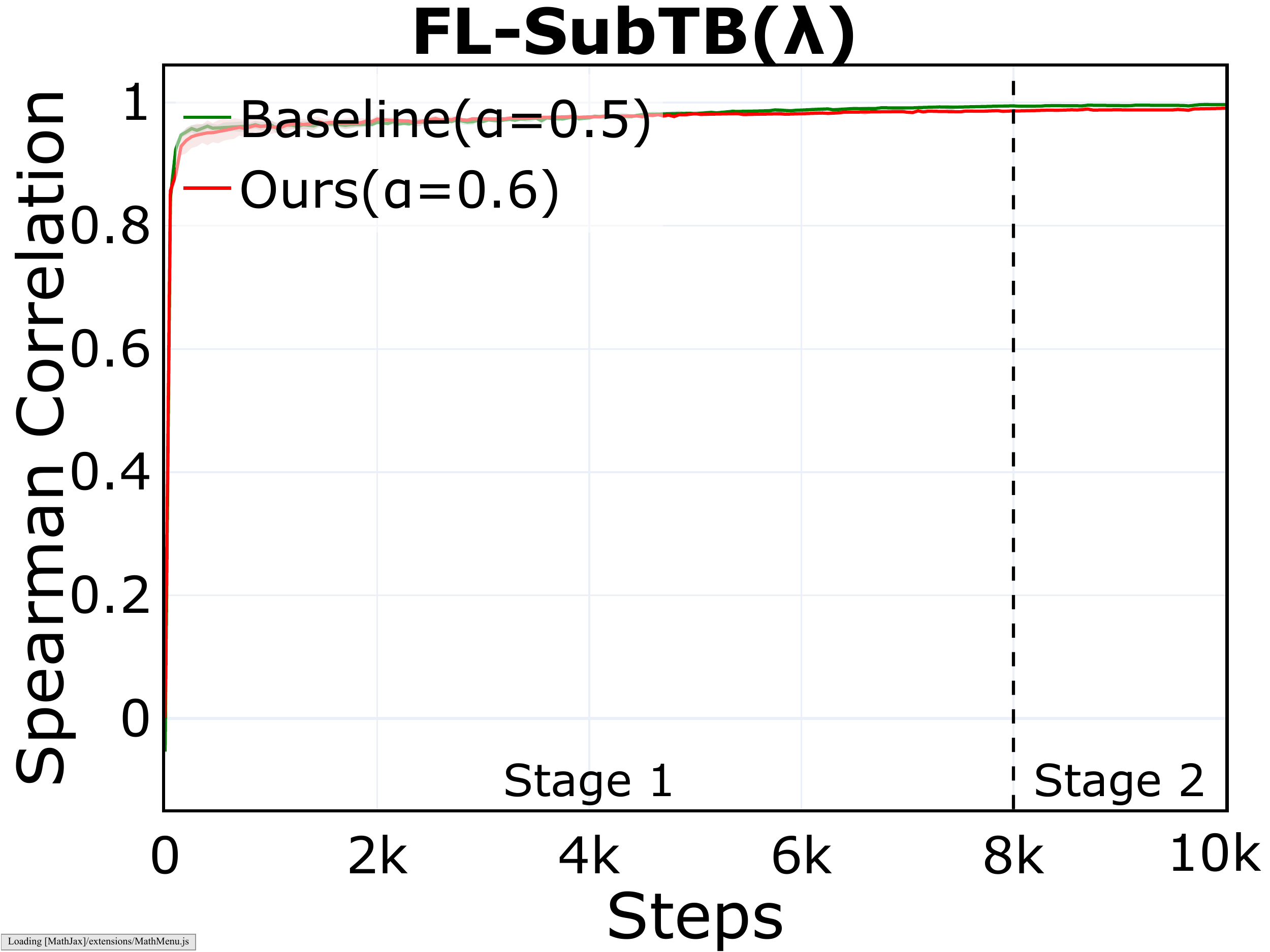} &
    \includegraphics[width=0.2\textwidth]{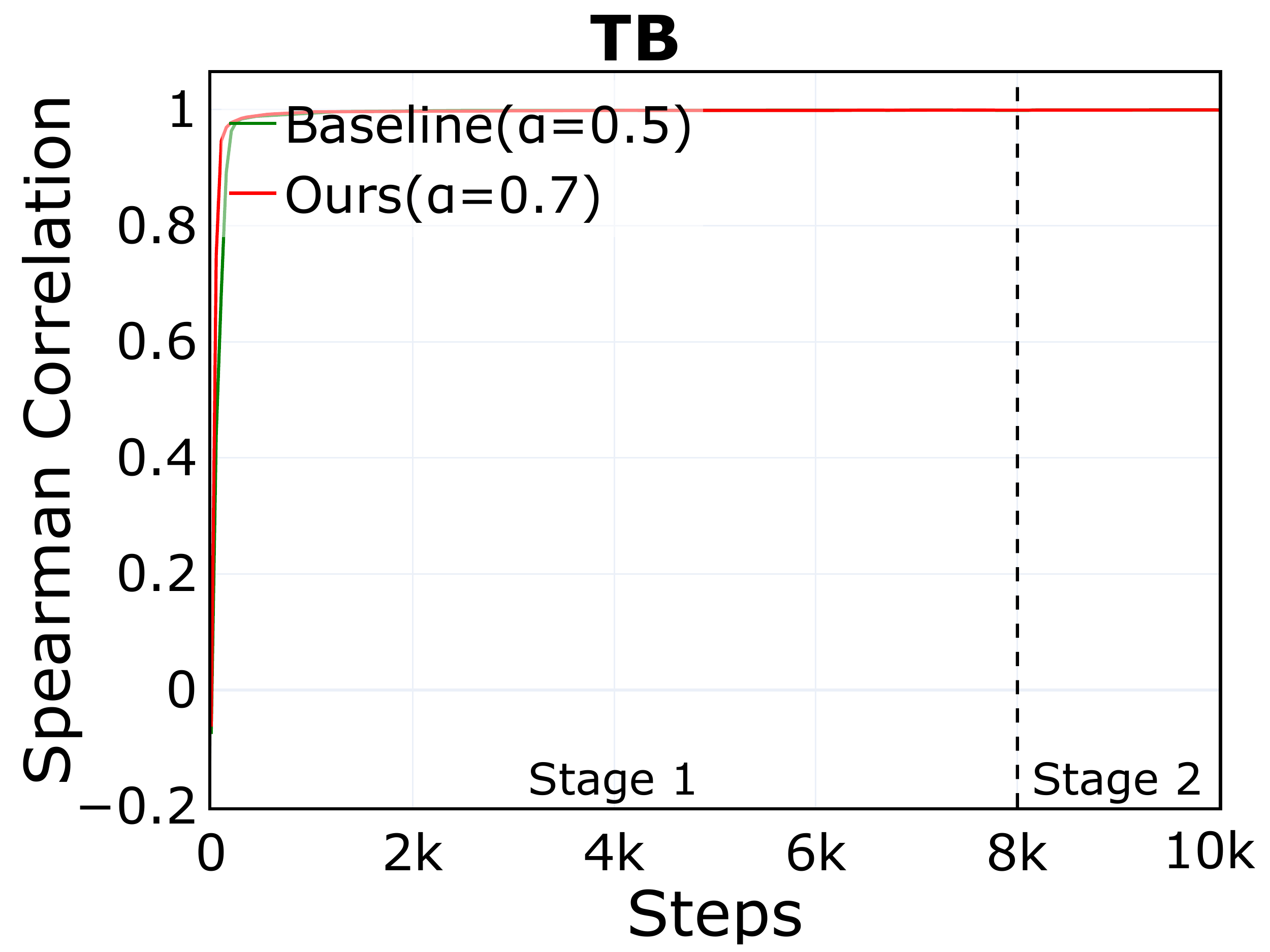} \\
    \small{(a)} DB (small) &
    \small{(b)} FL-DB (small) &
    \small{(c)} SubTB($\lambda$) (small) &
    \small{(d)} FL-SubTB($\lambda$) (small) &
    \small{(e)} TB (small) \\
    \includegraphics[width=0.2\textwidth]{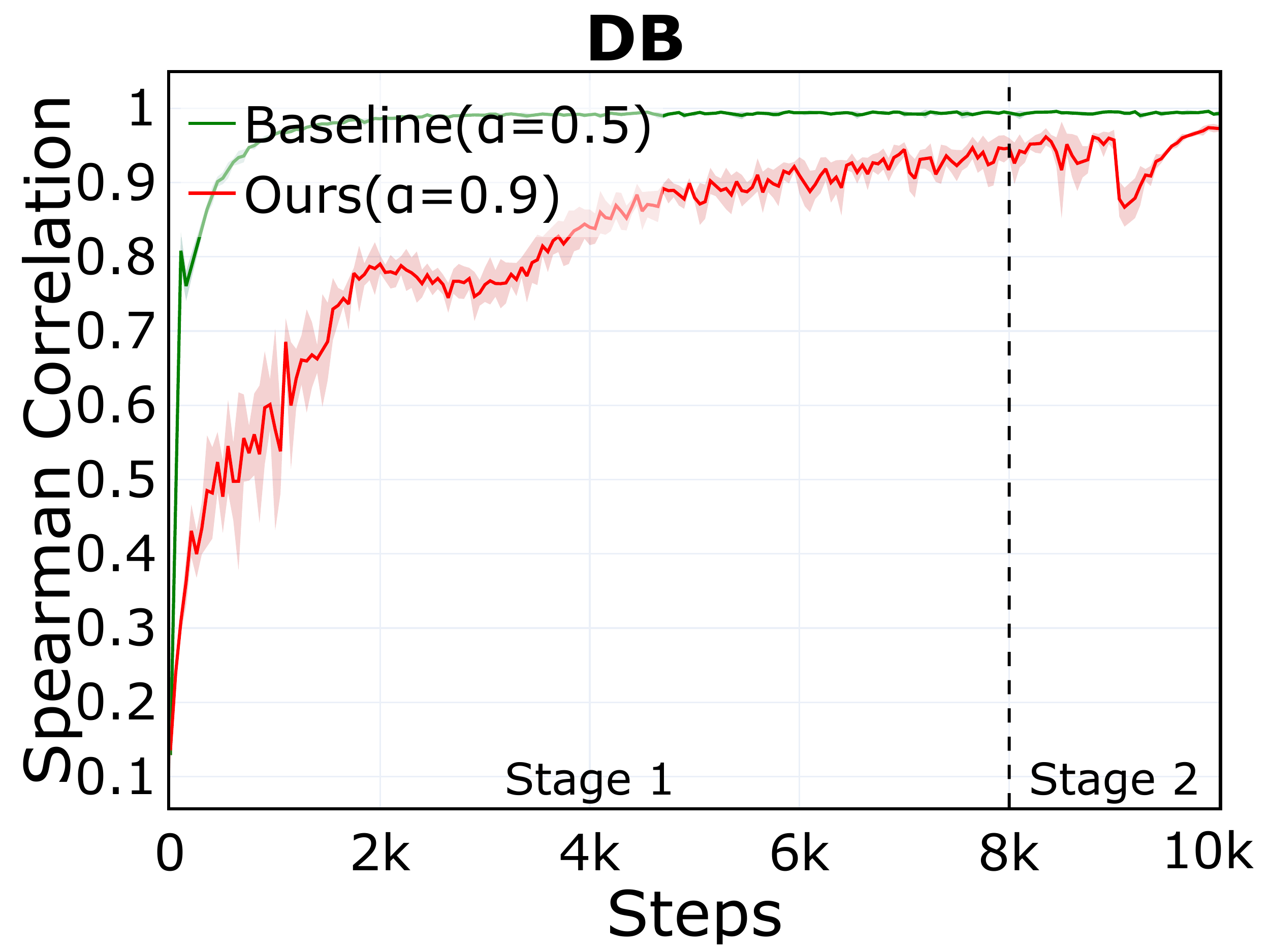} &
    \includegraphics[width=0.2\textwidth]{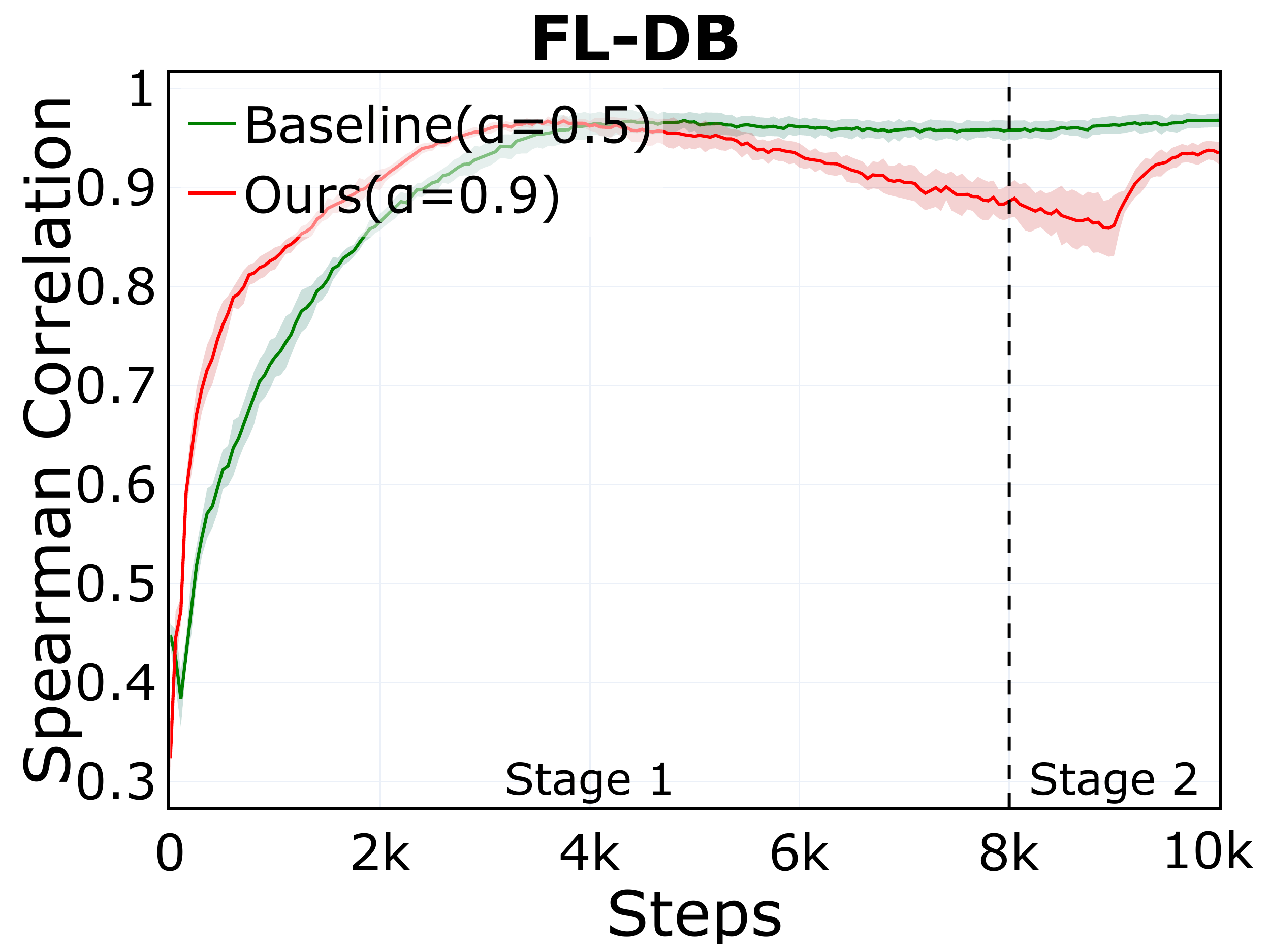} &
     \includegraphics[width=0.2\textwidth]{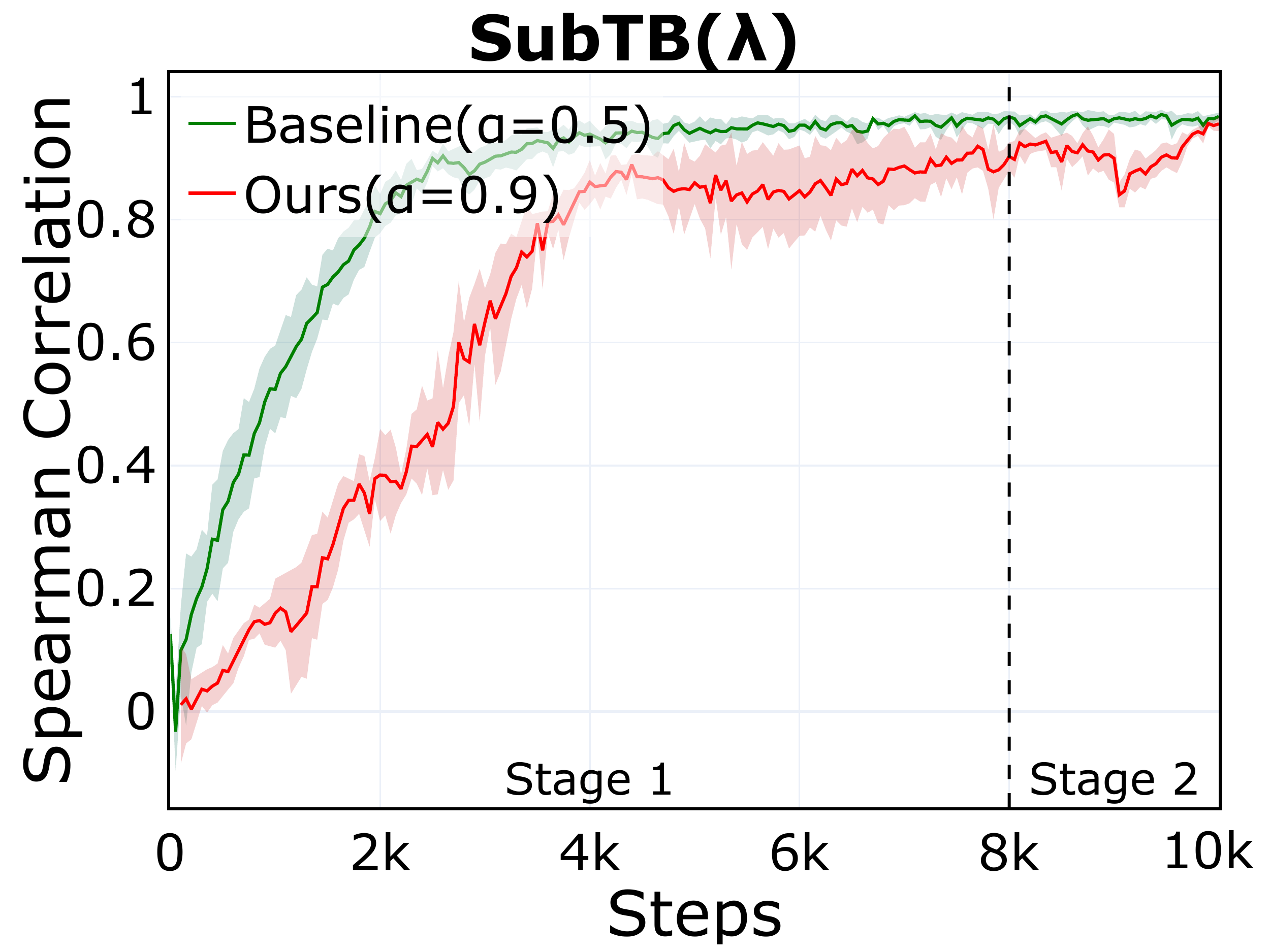}&
    \includegraphics[width=0.2\textwidth]{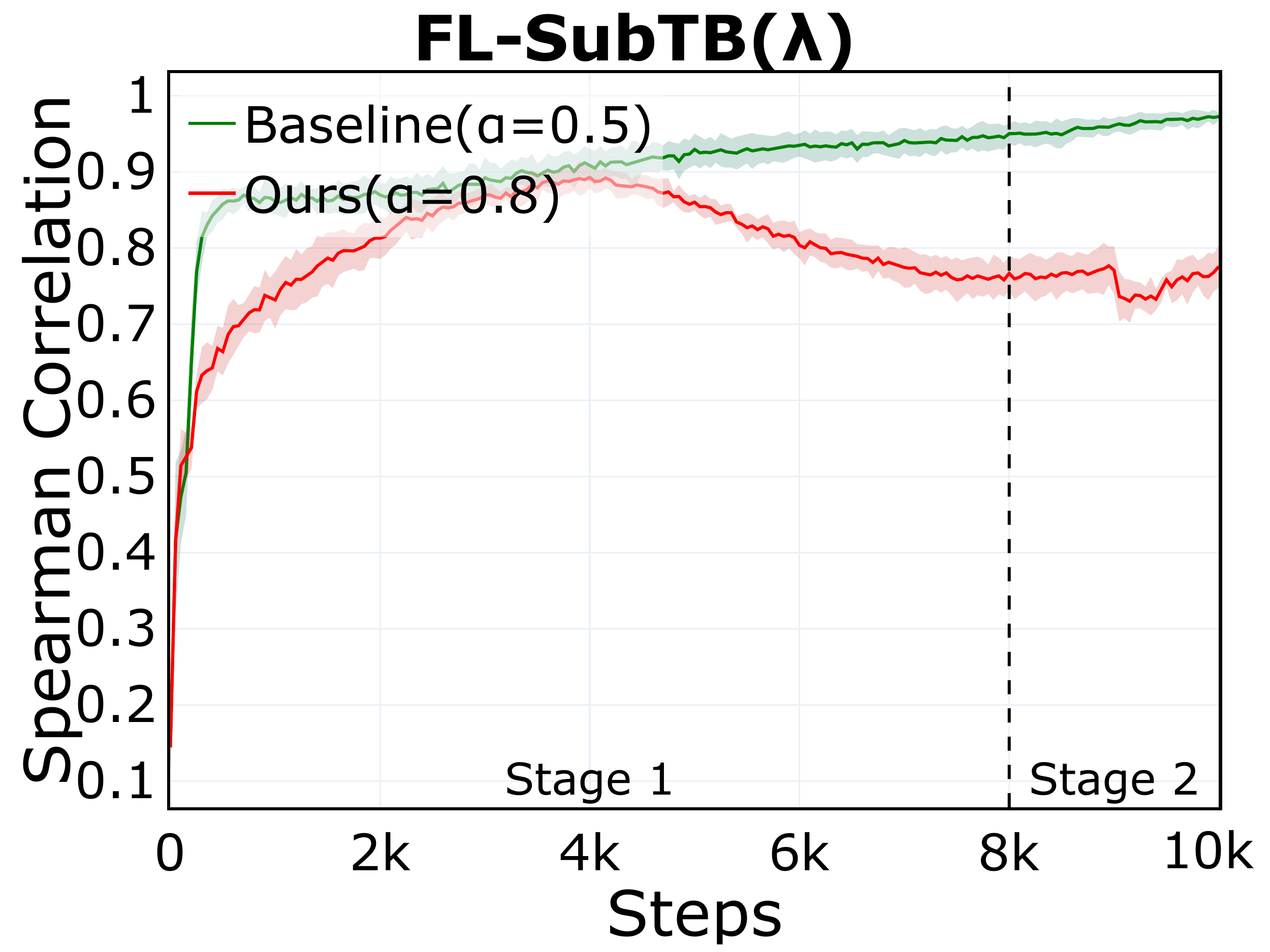} &
    \includegraphics[width=0.2\textwidth]{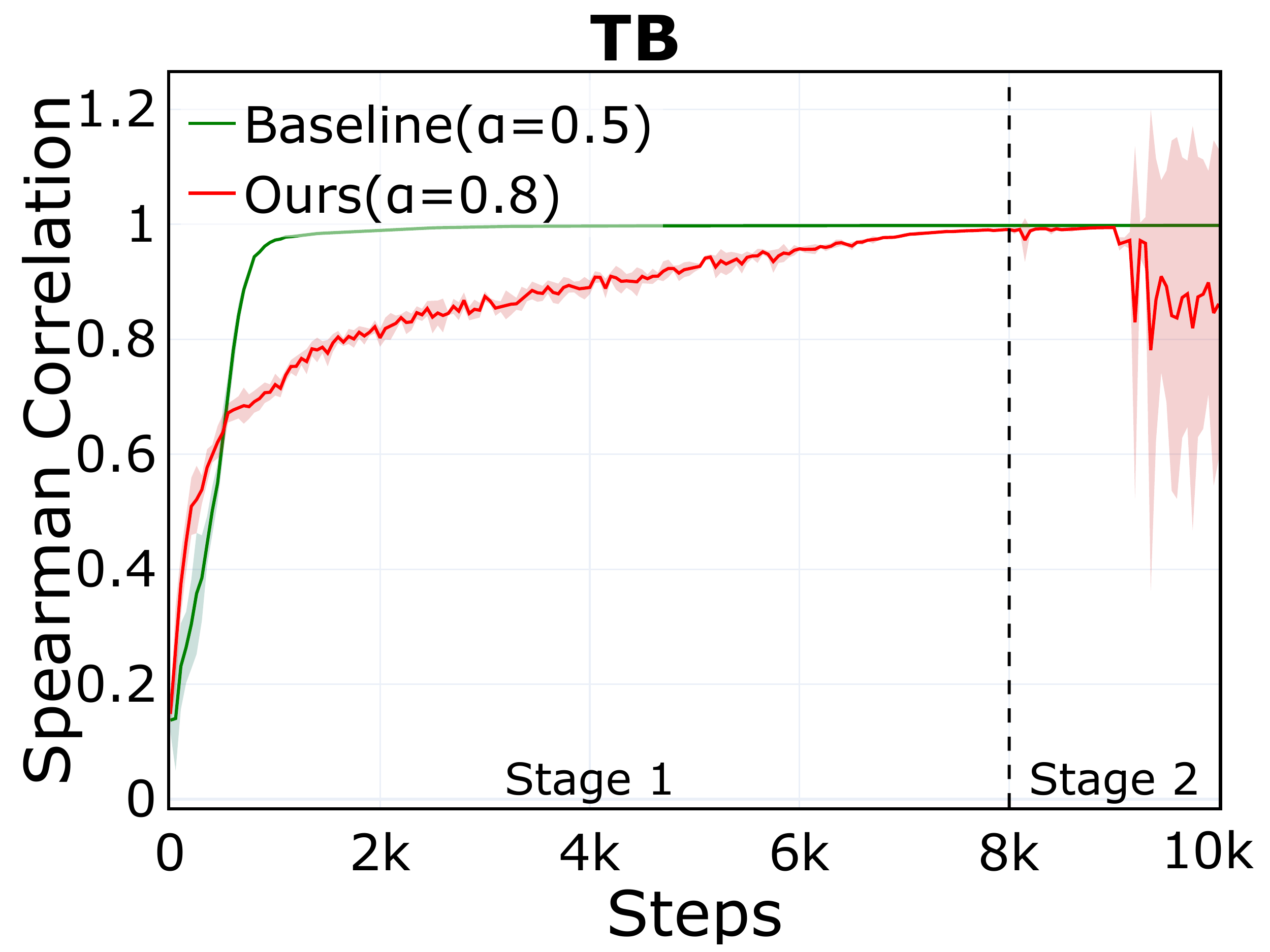} \\
    \small{(f)} DB (medium) &
    \small{(g)} FL-DB (medium) &
    \small{(h)} SubTB($\lambda$) (medium) &
    \small{(i)} FL-SubTB($\lambda$) (medium) &
    \small{(j)} TB (medium) \\
    \includegraphics[width=0.2\textwidth]{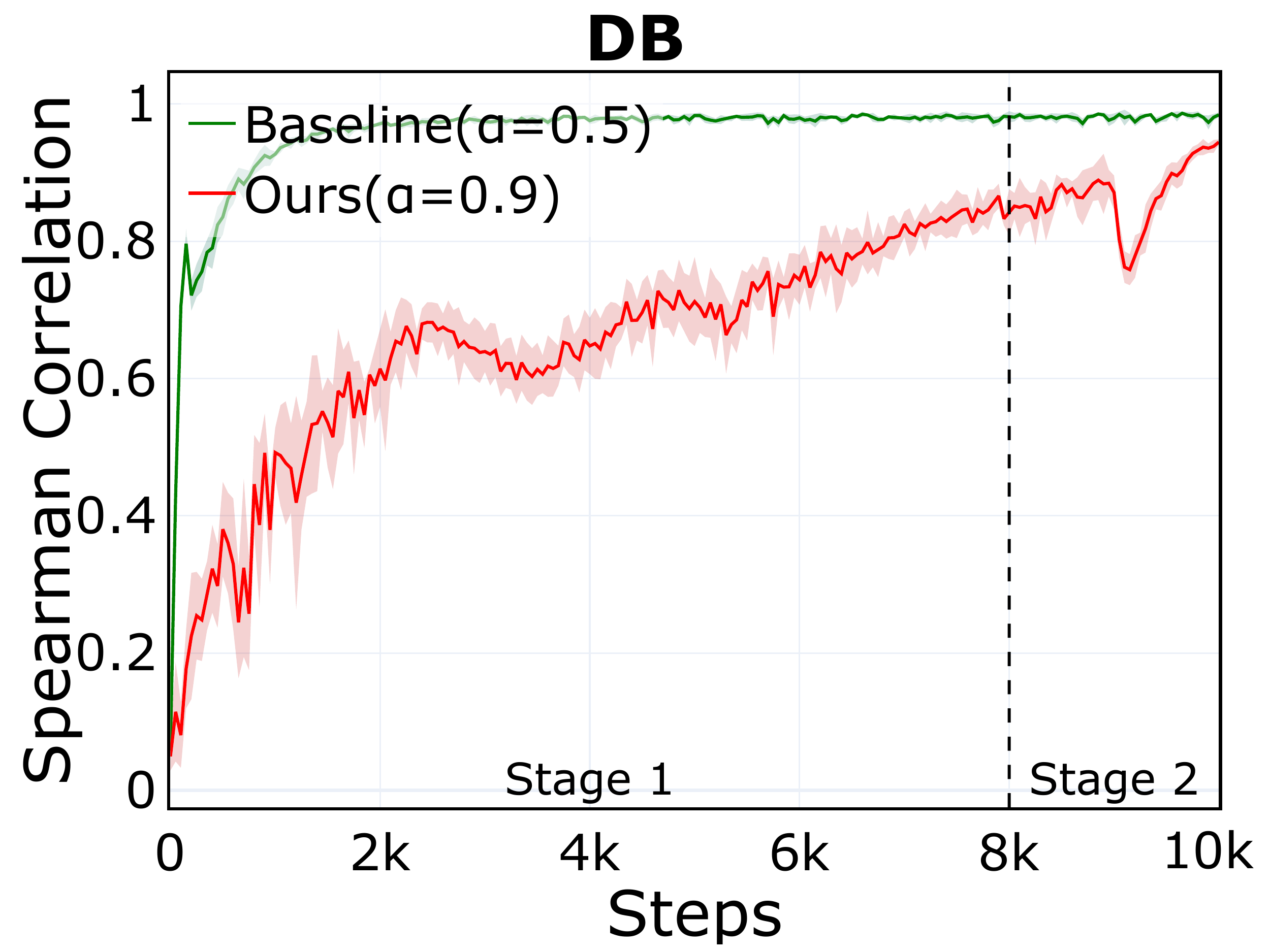} &
    \includegraphics[width=0.2\textwidth]{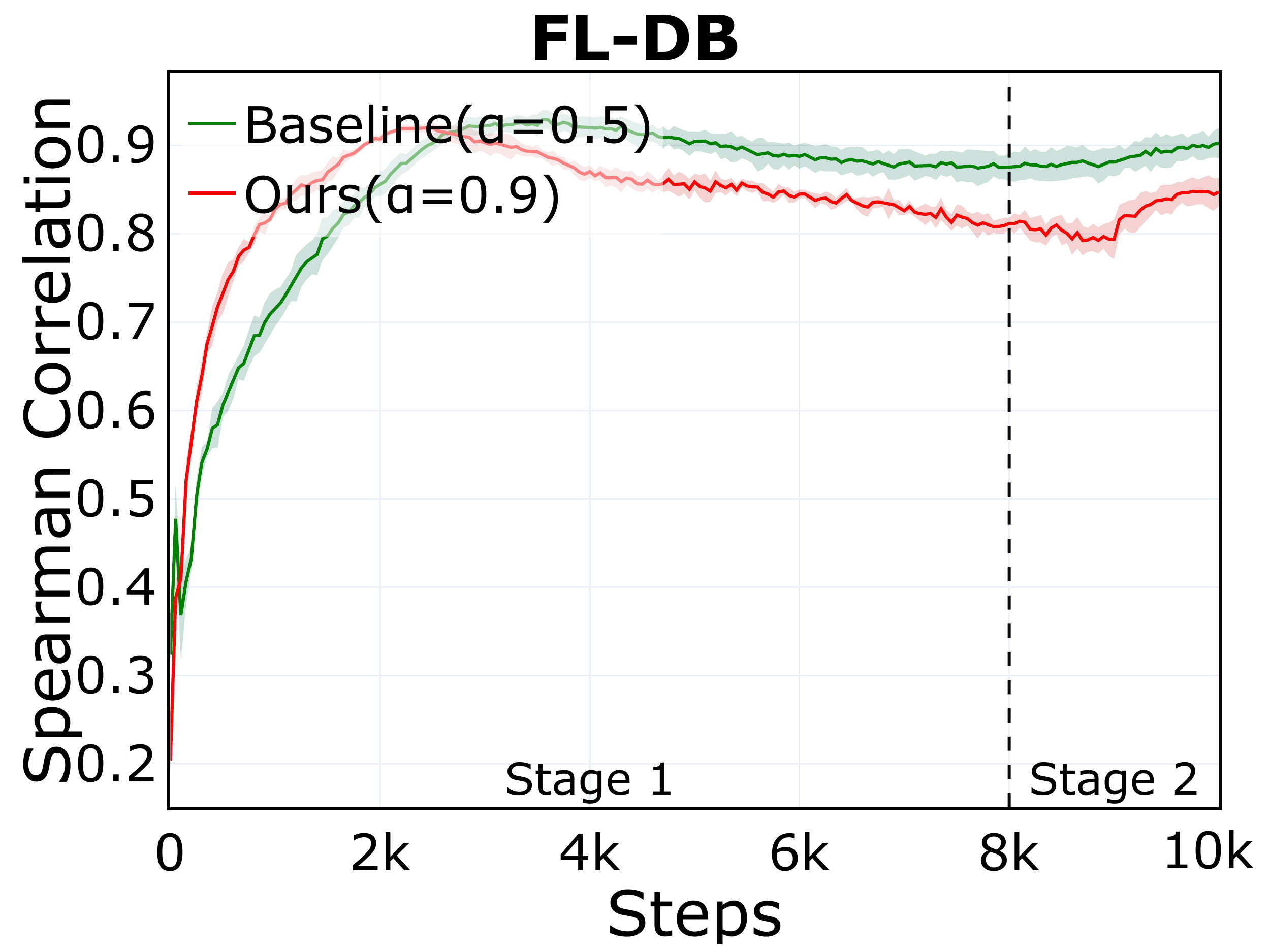} &
     \includegraphics[width=0.2\textwidth]{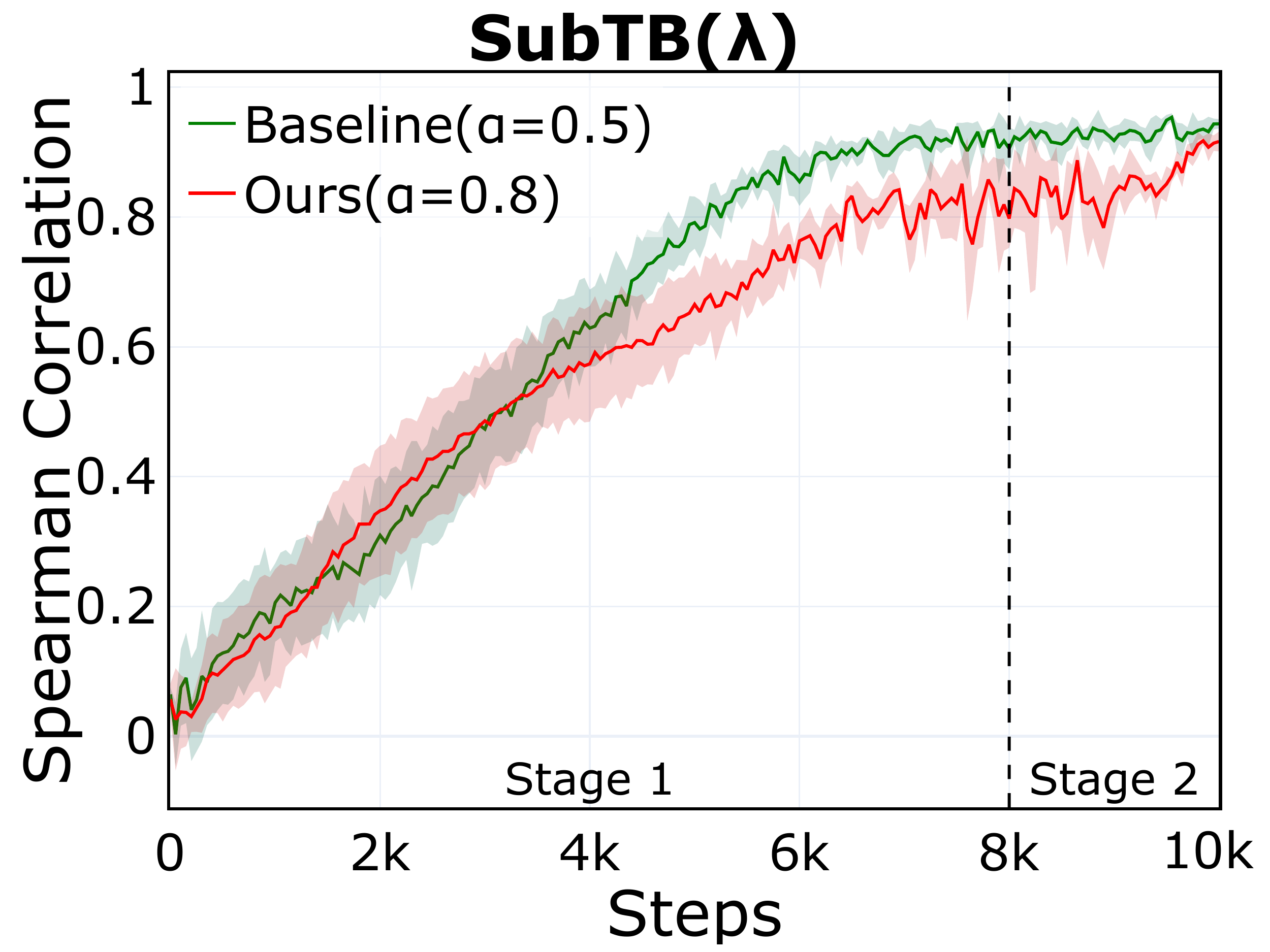} &
     \includegraphics[width=0.2\textwidth]{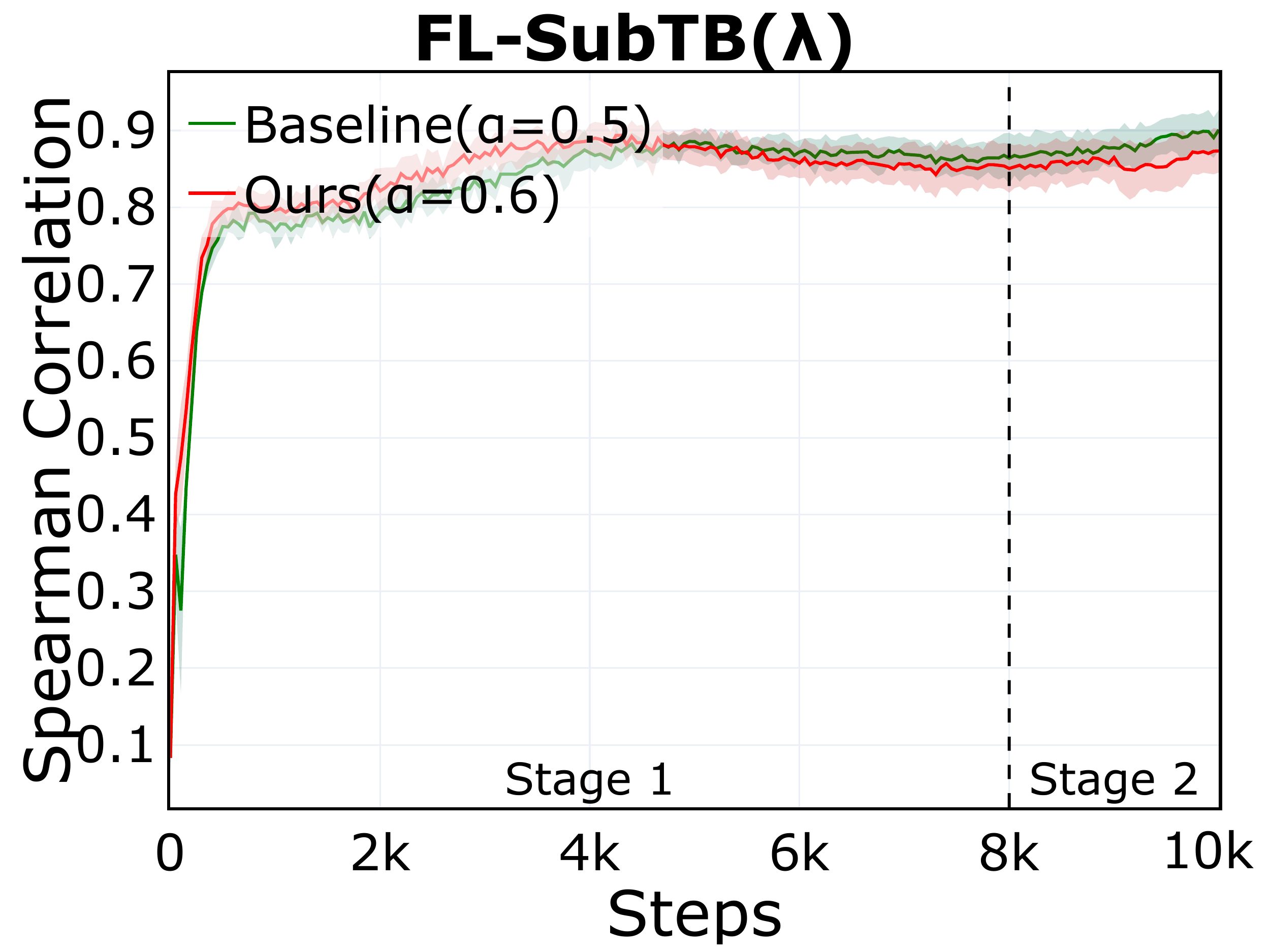}&
    \includegraphics[width=0.2\textwidth]{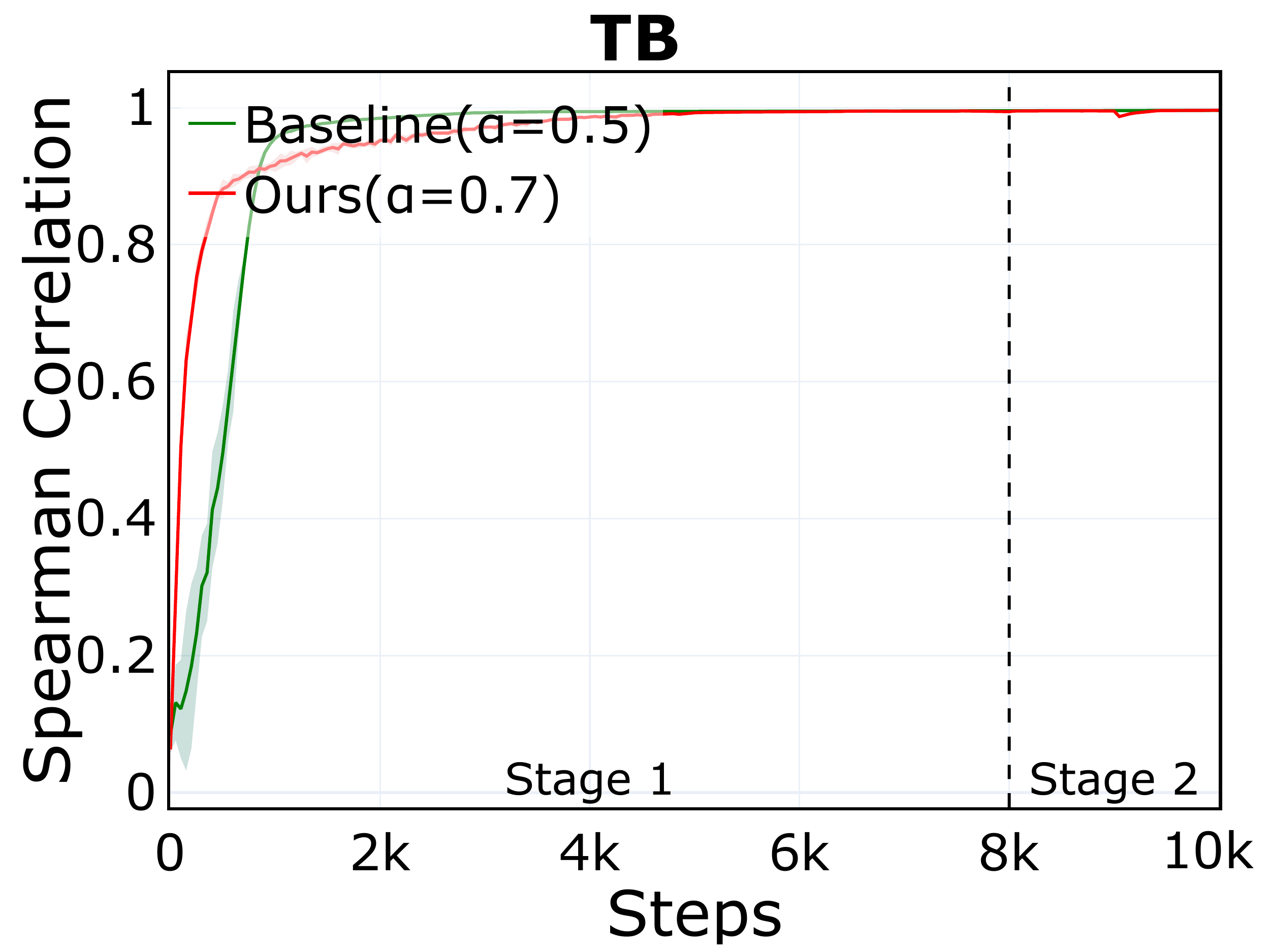} \\
    \small{(k)} DB (large) &
    \small{(l)} FL-DB (large) &
    \small{(m)} SubTB($\lambda$) (large) &
    \small{(n)} FL-SubTB($\lambda$) (large) &
    \small{(o)} TB (large) \\
  \end{tabular}
  \caption{\textbf{Spearman Correlation} vs Training Steps in \textbf{Set Generation} across different objectives and set sizes.}
  \label{fig:set_metric_spearman_corr_test}
\end{figure}

\begin{figure}[htbp]
  \centering
  \setlength{\tabcolsep}{0pt}
  \begin{tabular}{@{}c@{\hspace{0pt}}c@{\hspace{0pt}}c@{\hspace{0pt}}c@{\hspace{0pt}}c@{}}
    \includegraphics[width=0.2\textwidth]{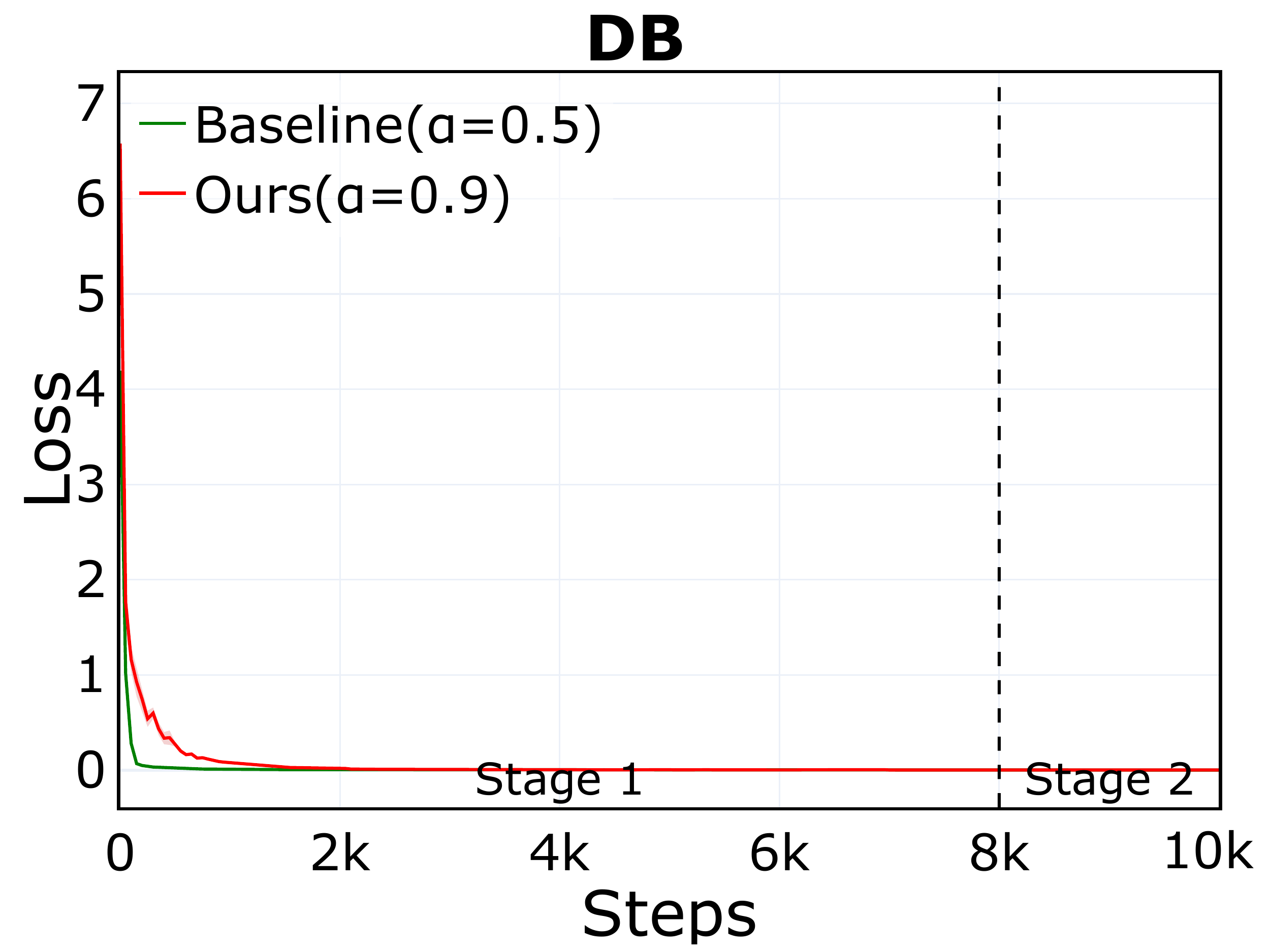} &
    \includegraphics[width=0.2\textwidth]{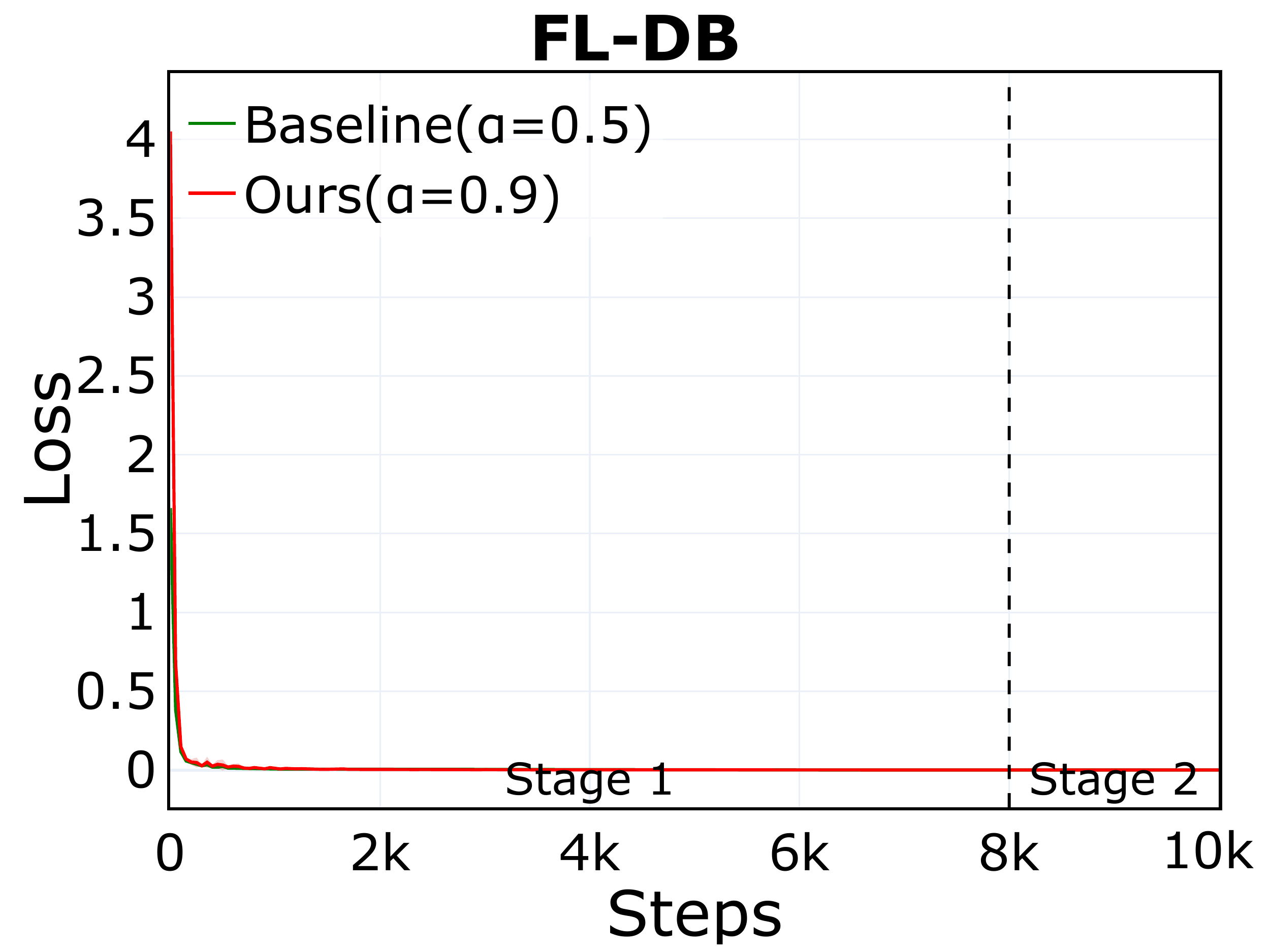} &
    \includegraphics[width=0.2\textwidth]{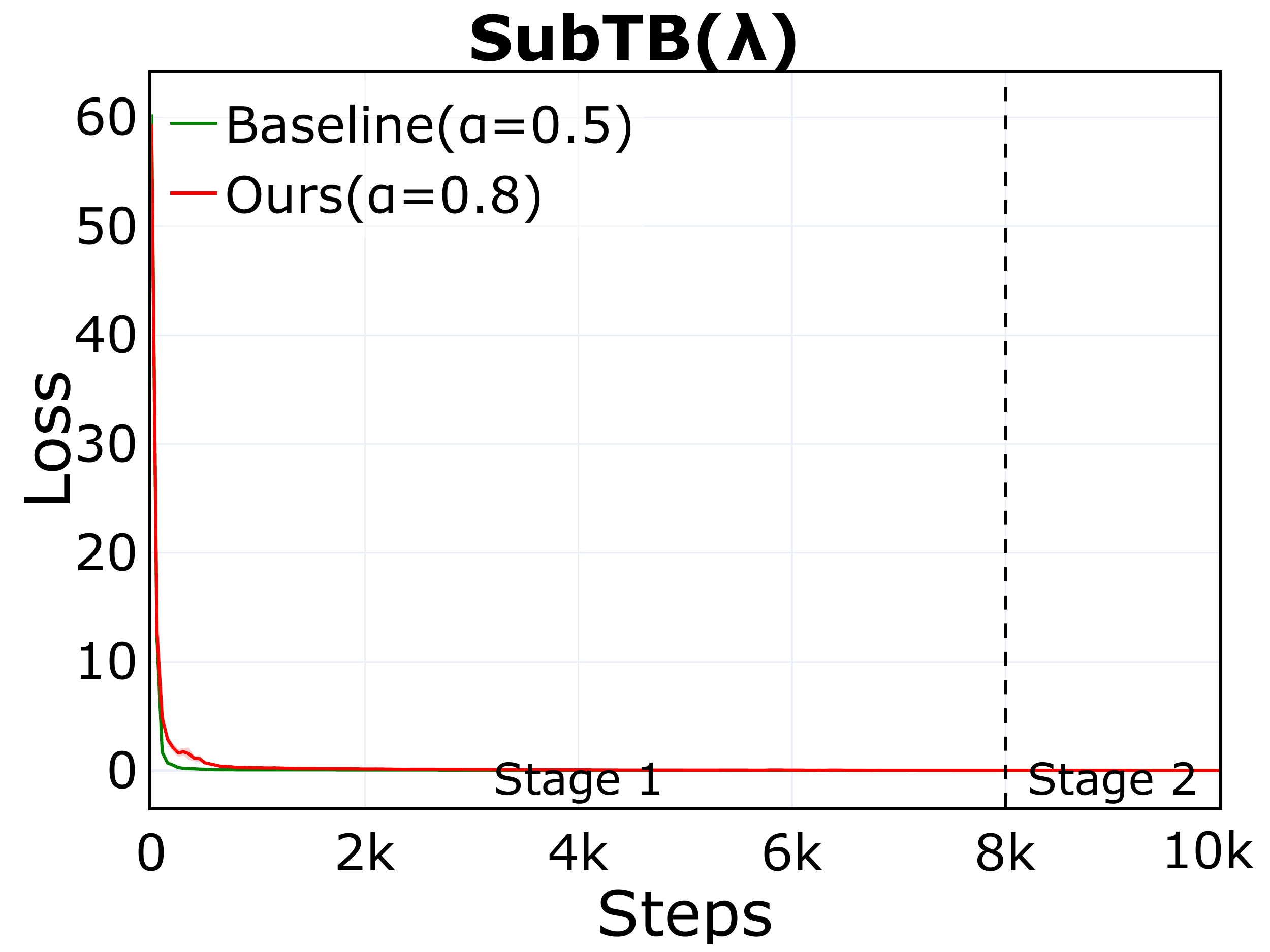} & 
    \includegraphics[width=0.2\textwidth]{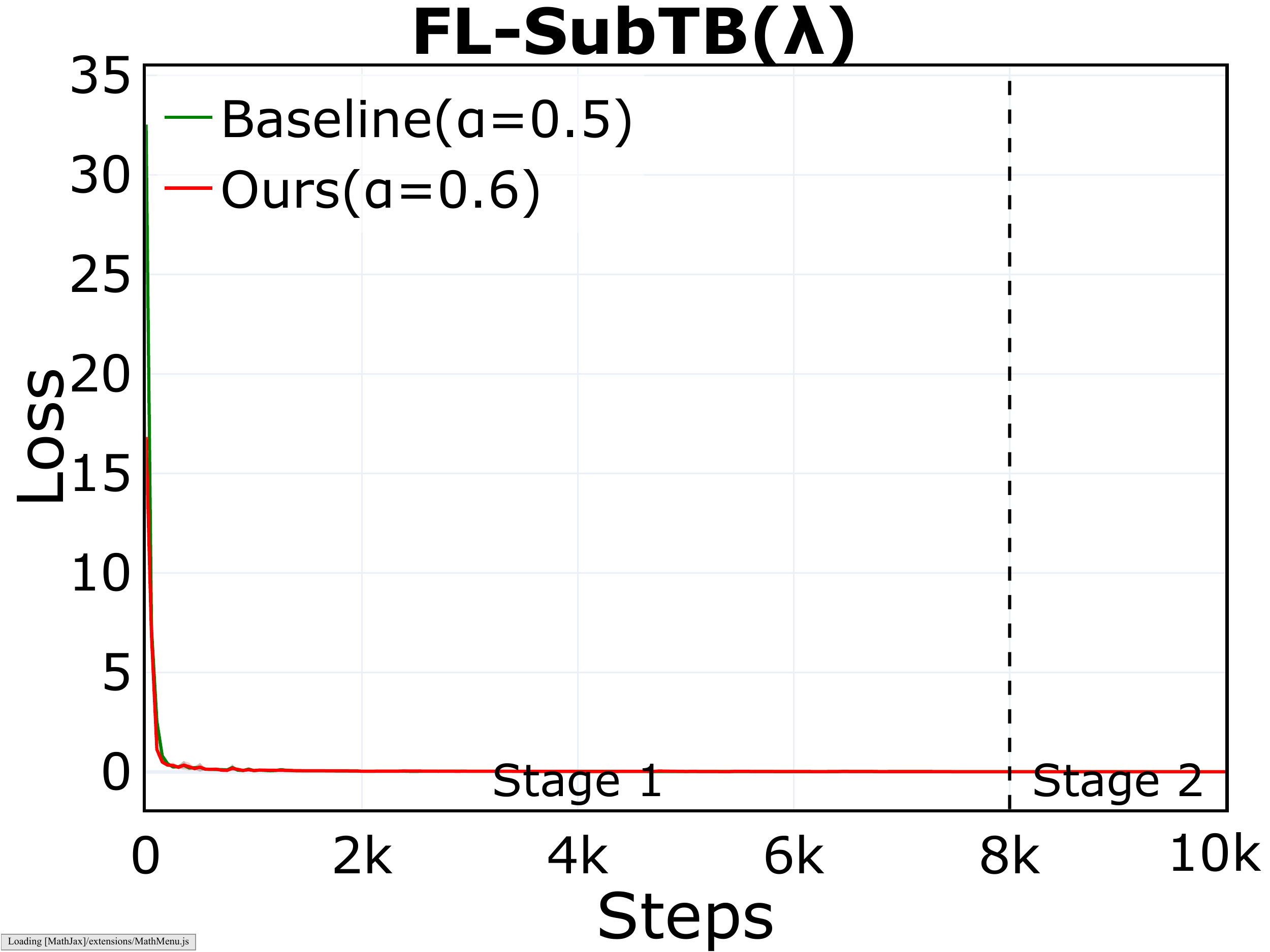} &
    \includegraphics[width=0.2\textwidth]{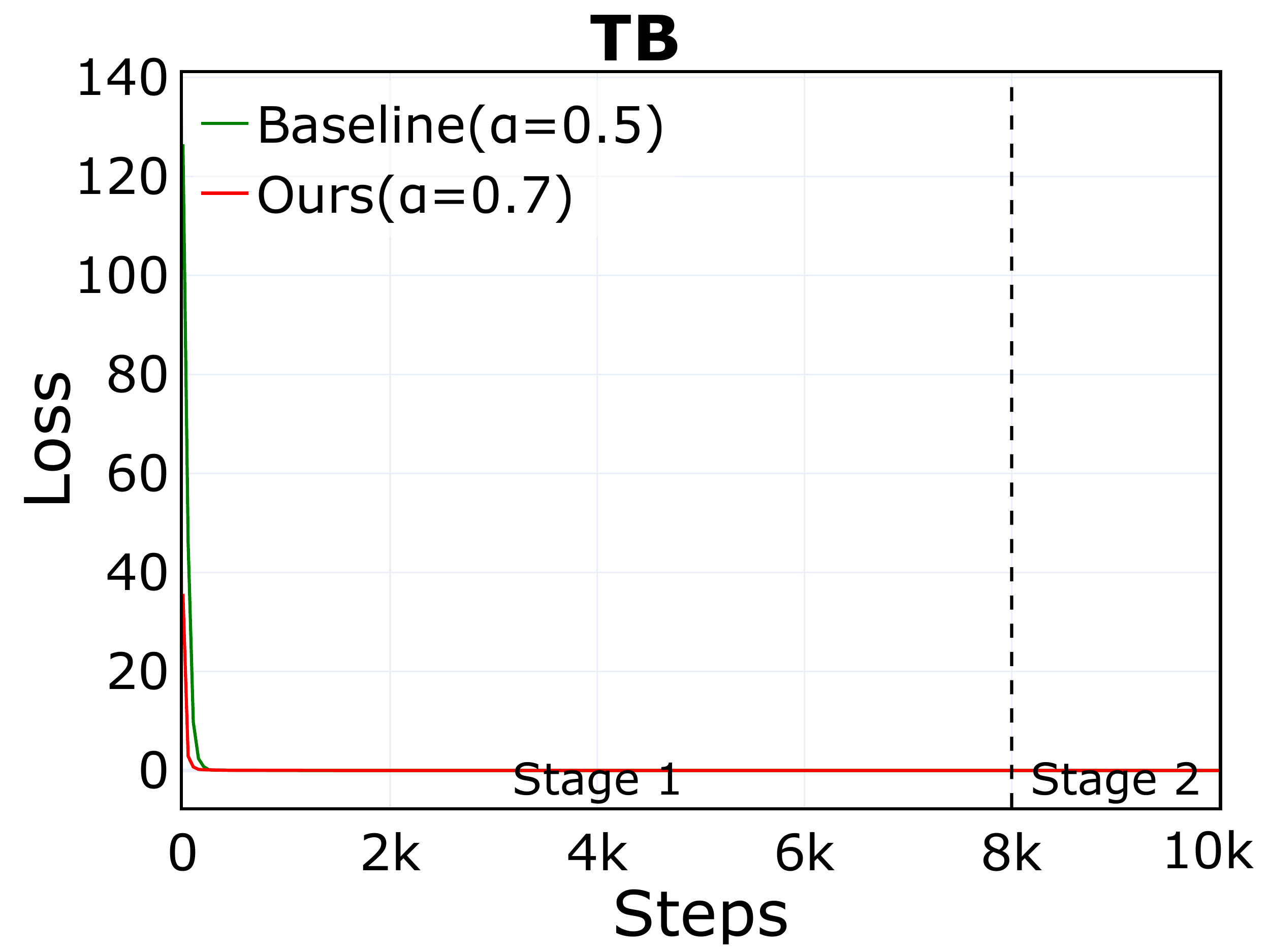} \\
    \small{(a)} DB (small) &
    \small{(b)} FL-DB (small) &
    \small{(c)} SubTB($\lambda$) (small) &
    \small{(d)} FL-SubTB($\lambda$) (small) &
    \small{(e)} TB (small) \\
    \includegraphics[width=0.2\textwidth]{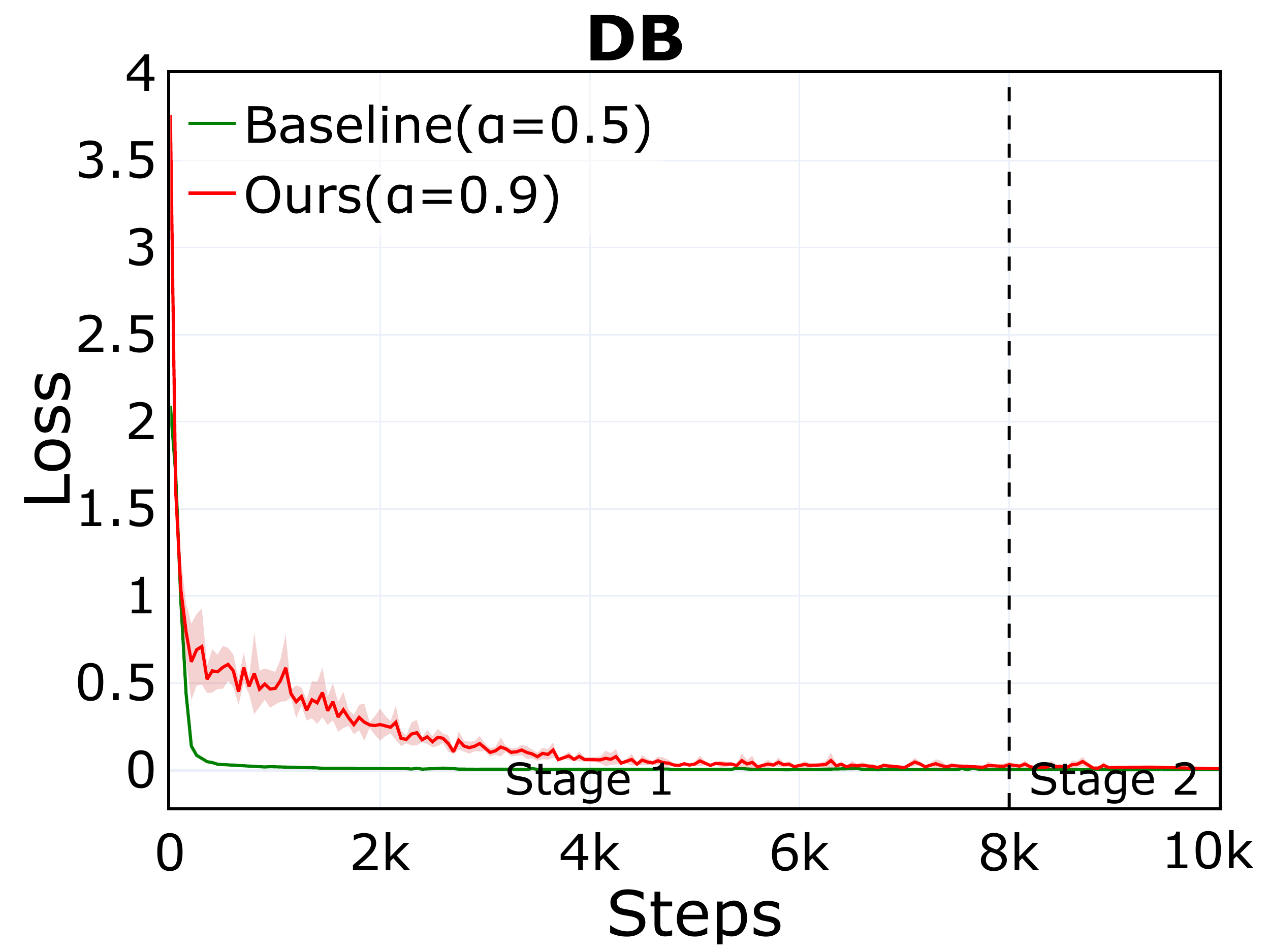} &
    \includegraphics[width=0.2\textwidth]{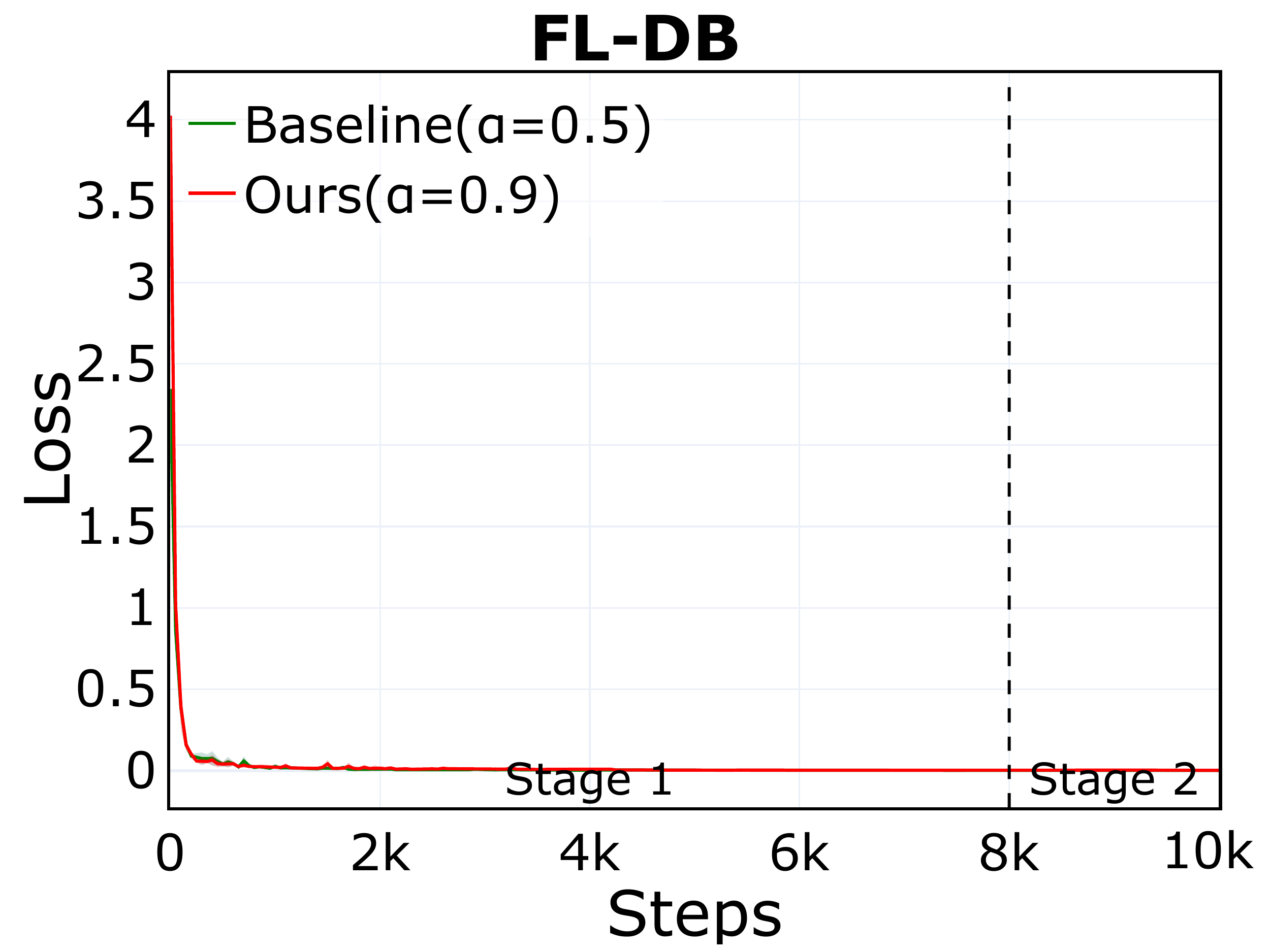} &
     \includegraphics[width=0.2\textwidth]{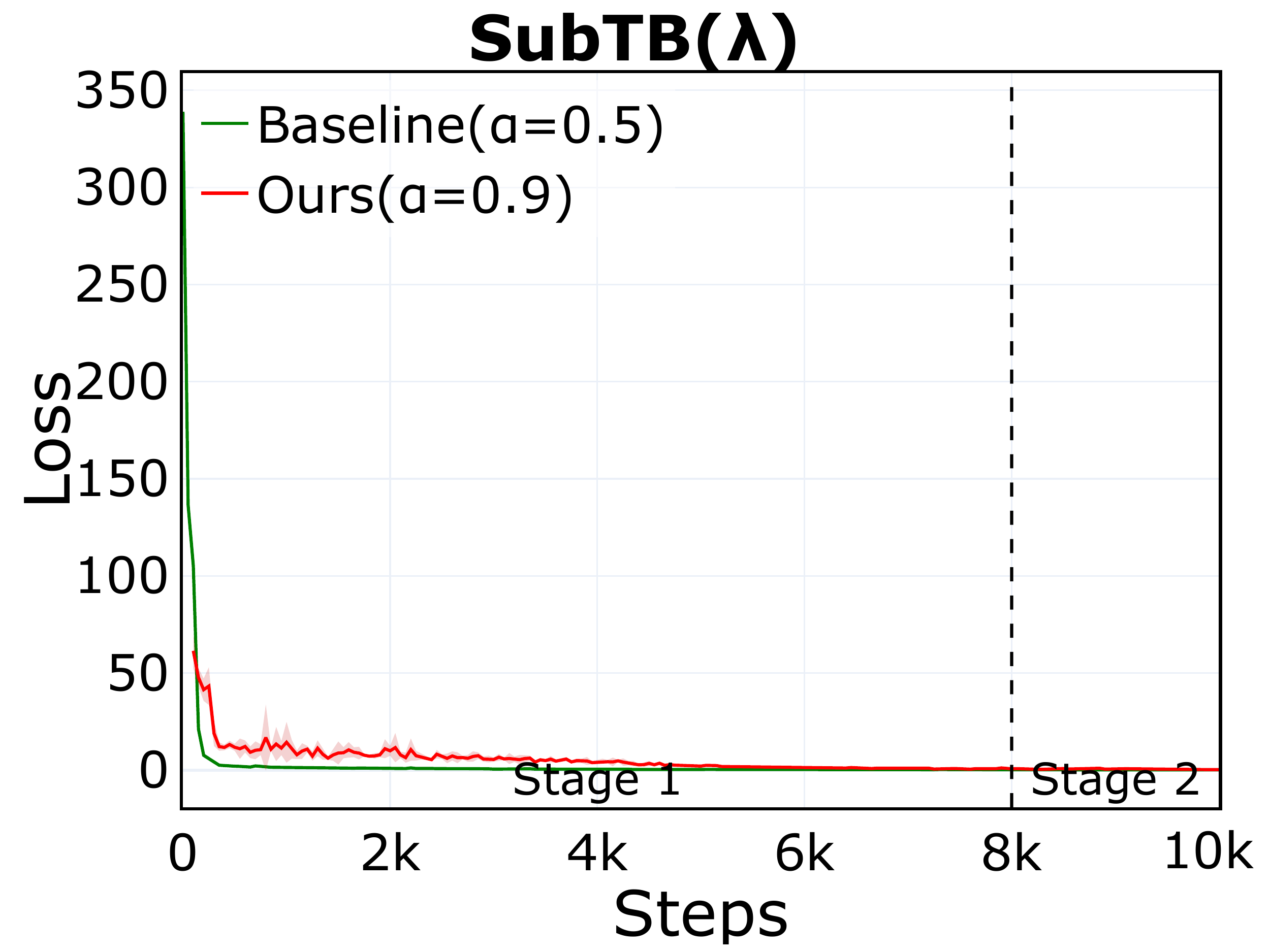}&
    \includegraphics[width=0.2\textwidth]{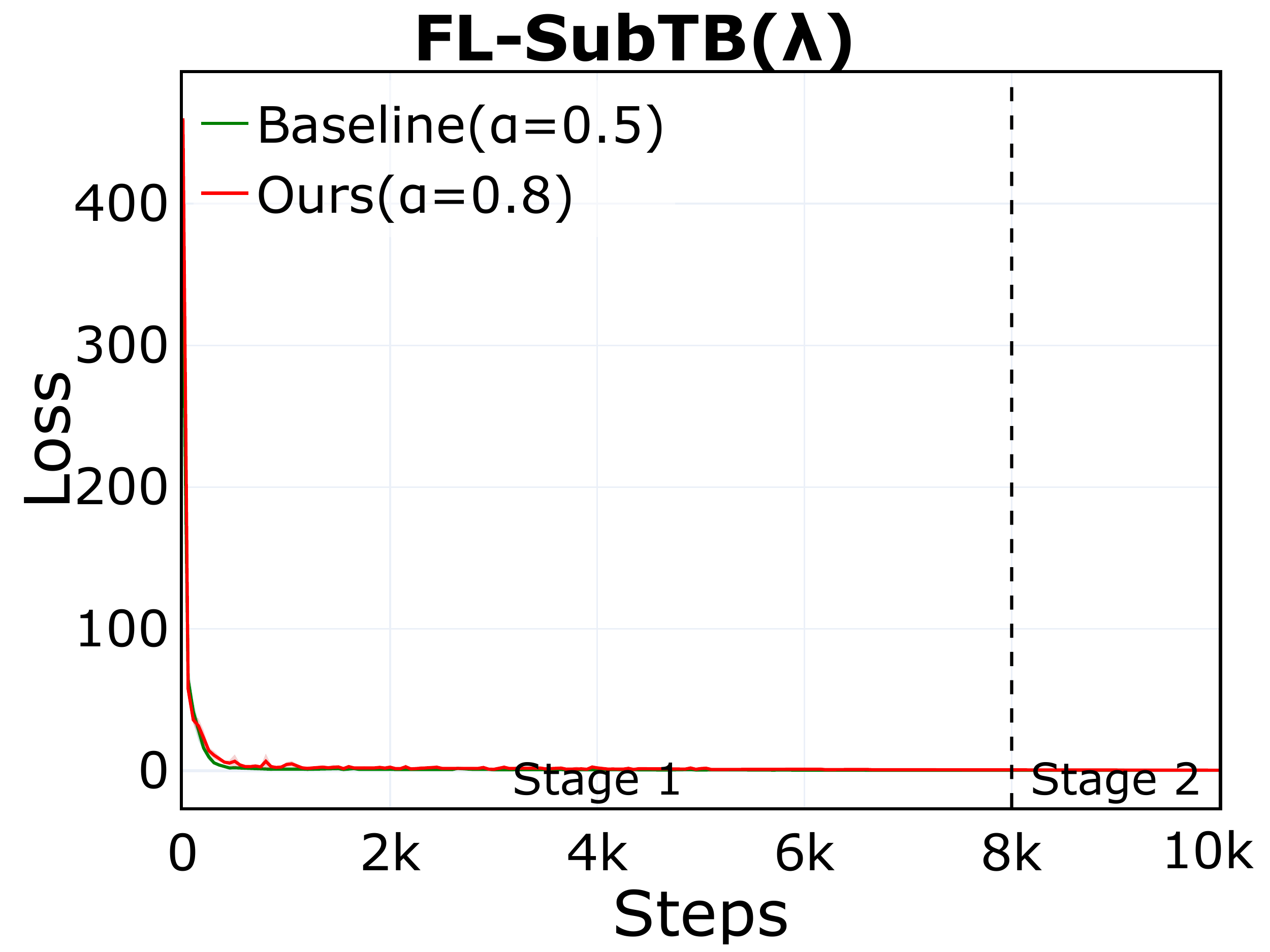} &
    \includegraphics[width=0.2\textwidth]{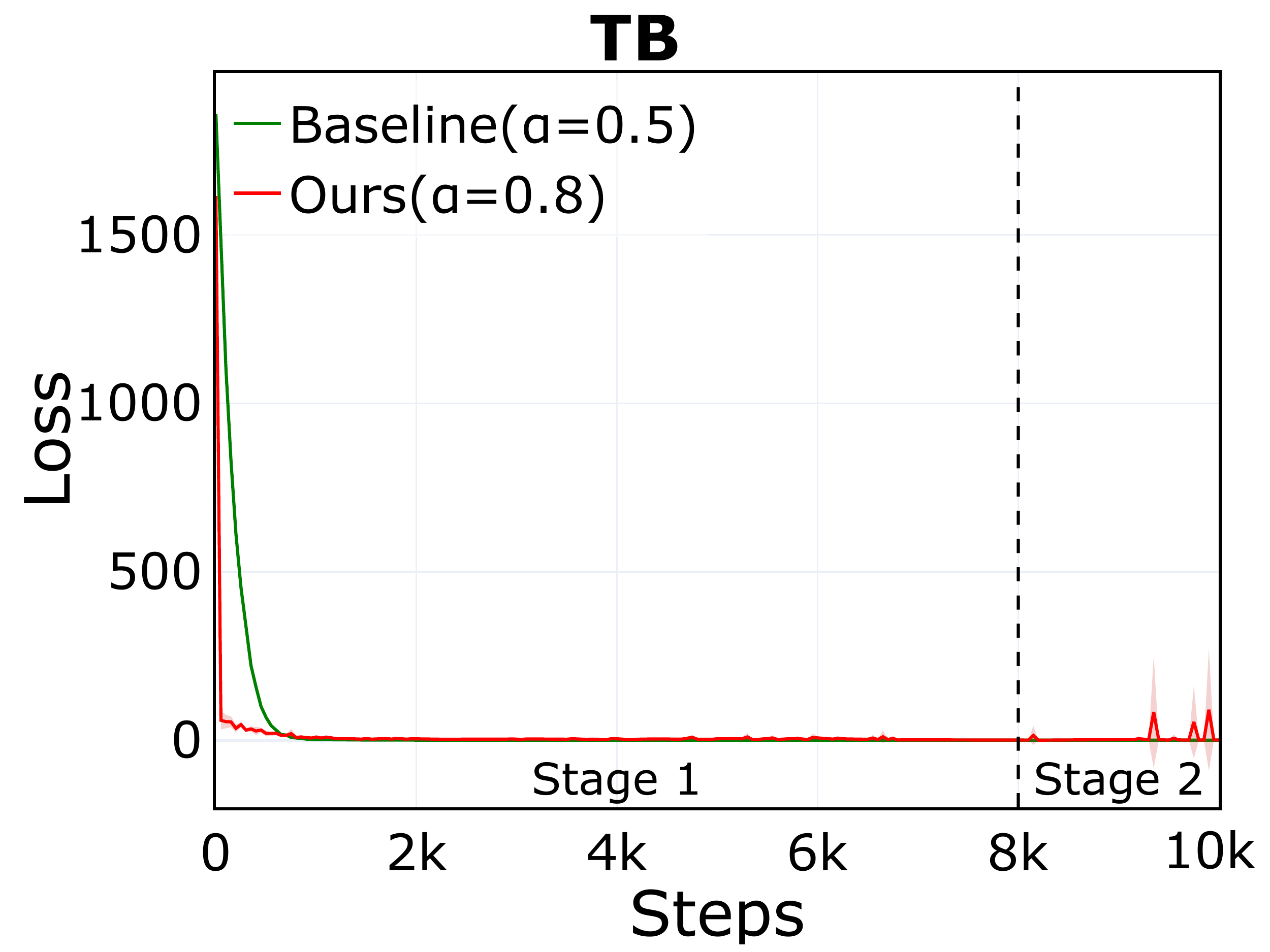} \\
    \small{(f)} DB (medium) &
    \small{(g)} FL-DB (medium) &
    \small{(h)} SubTB($\lambda$) (medium) &
    \small{(i)} FL-SubTB($\lambda$) (medium) &
    \small{(j)} TB (medium) \\
    \includegraphics[width=0.2\textwidth]{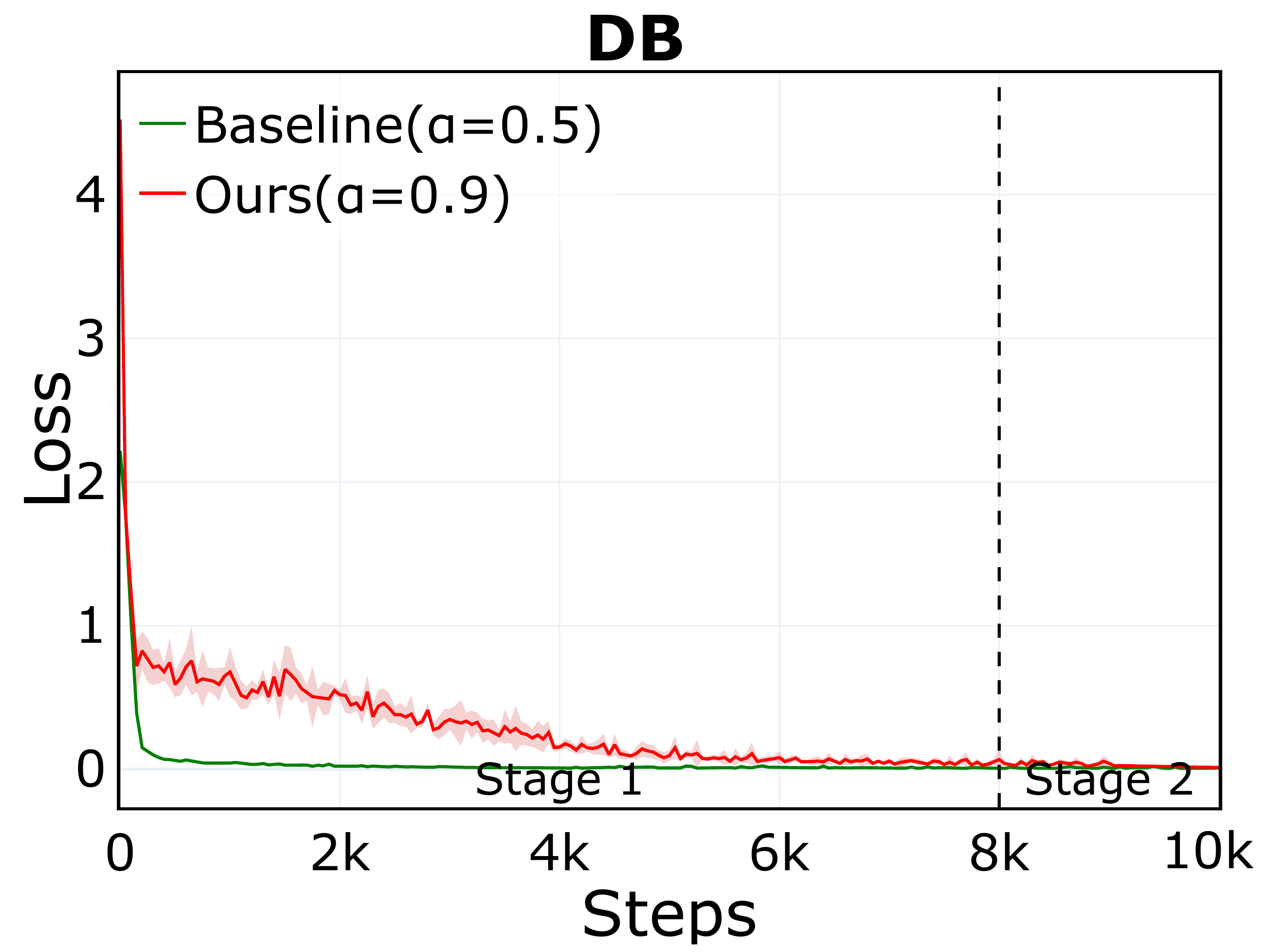} &
    \includegraphics[width=0.2\textwidth]{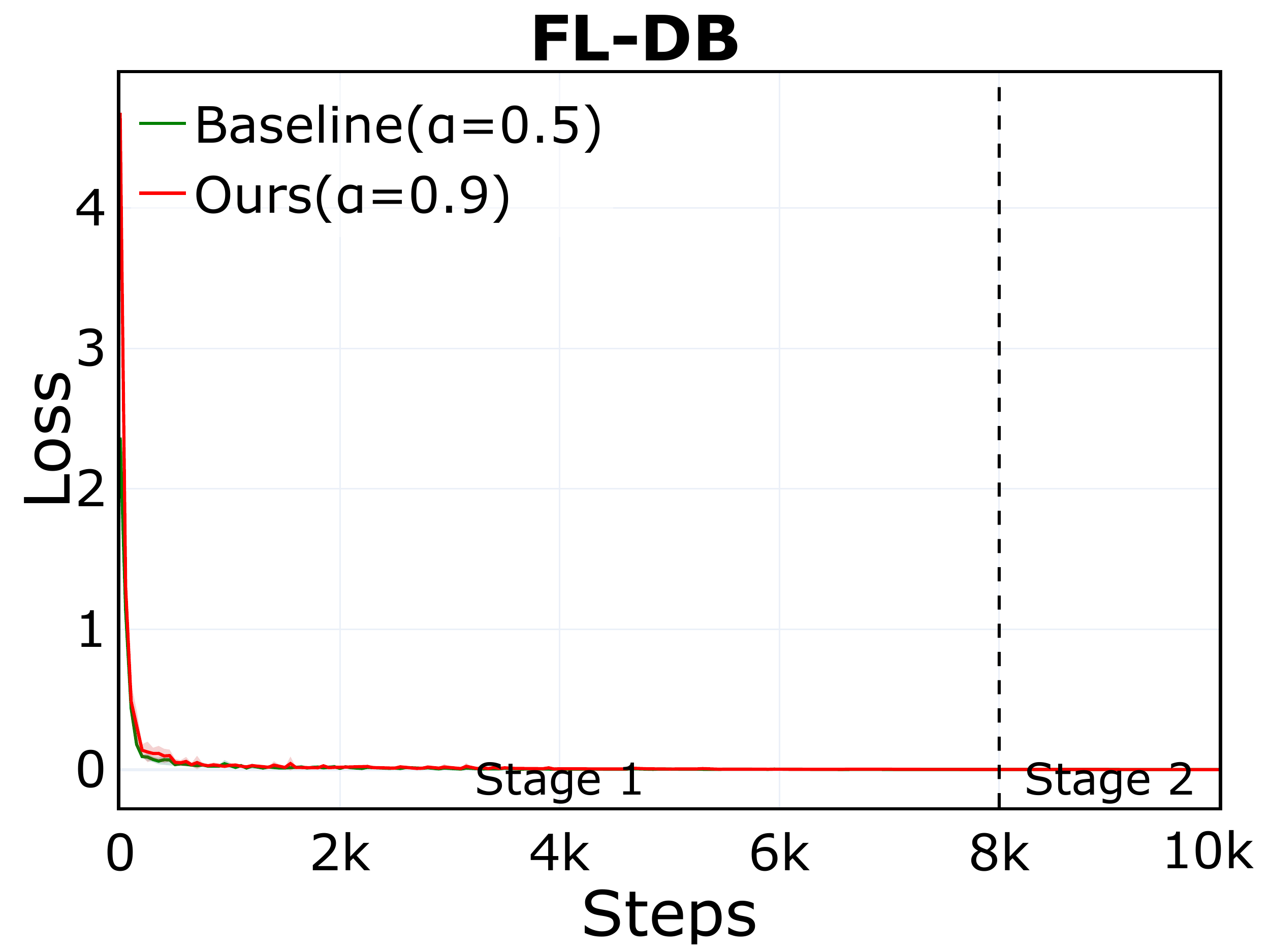} &
     \includegraphics[width=0.2\textwidth]{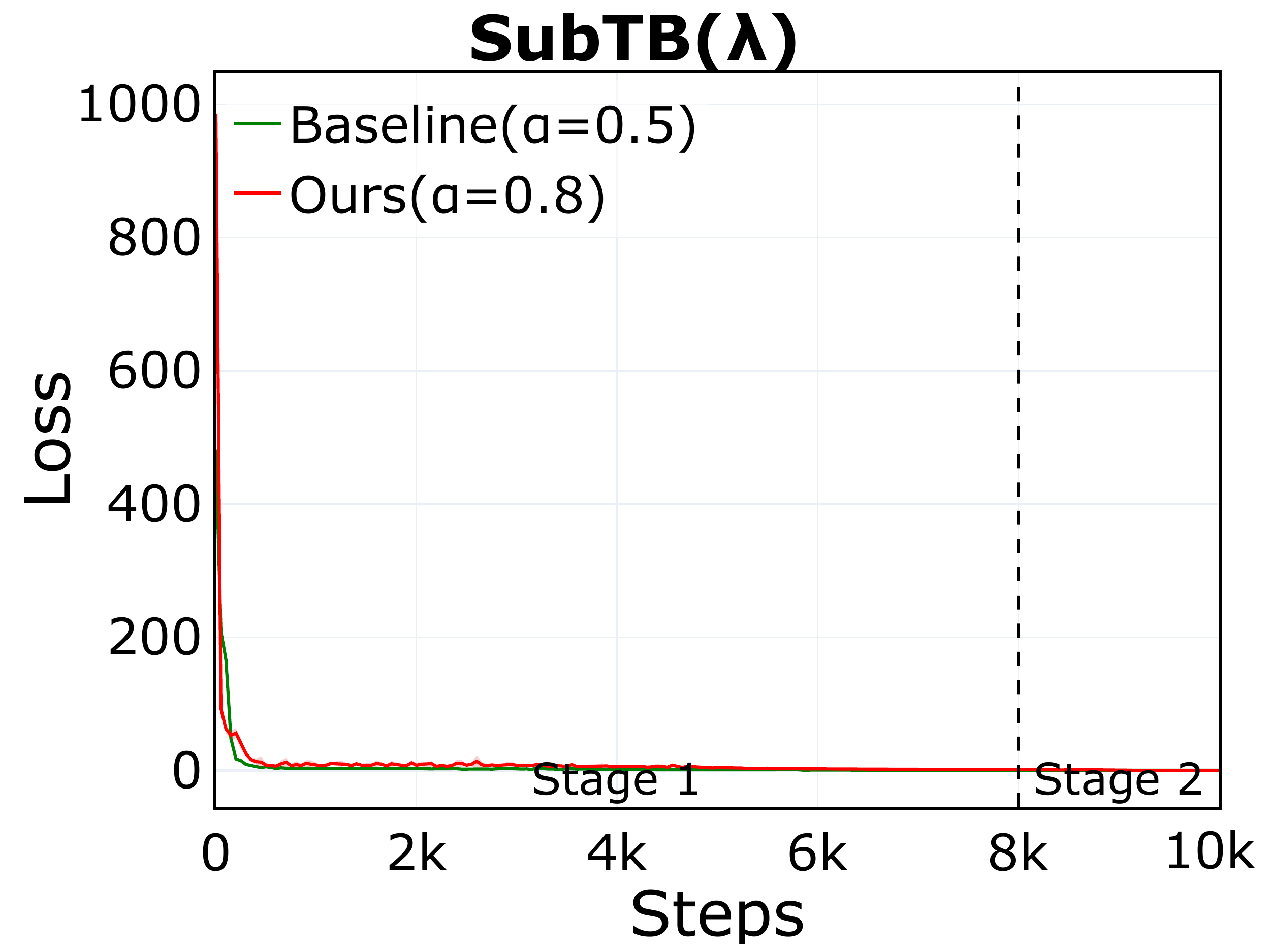} &
     \includegraphics[width=0.2\textwidth]{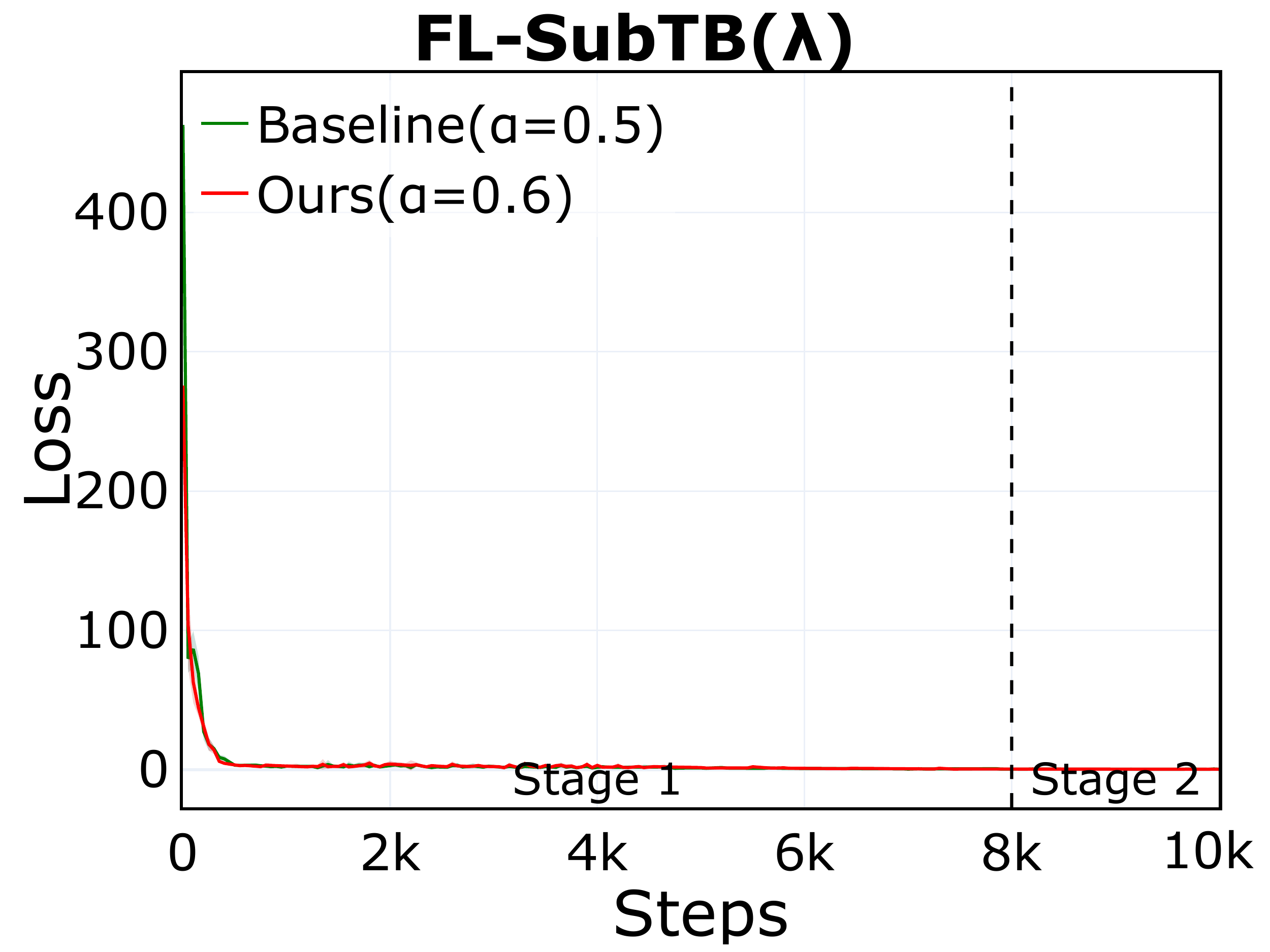}&
    \includegraphics[width=0.2\textwidth]{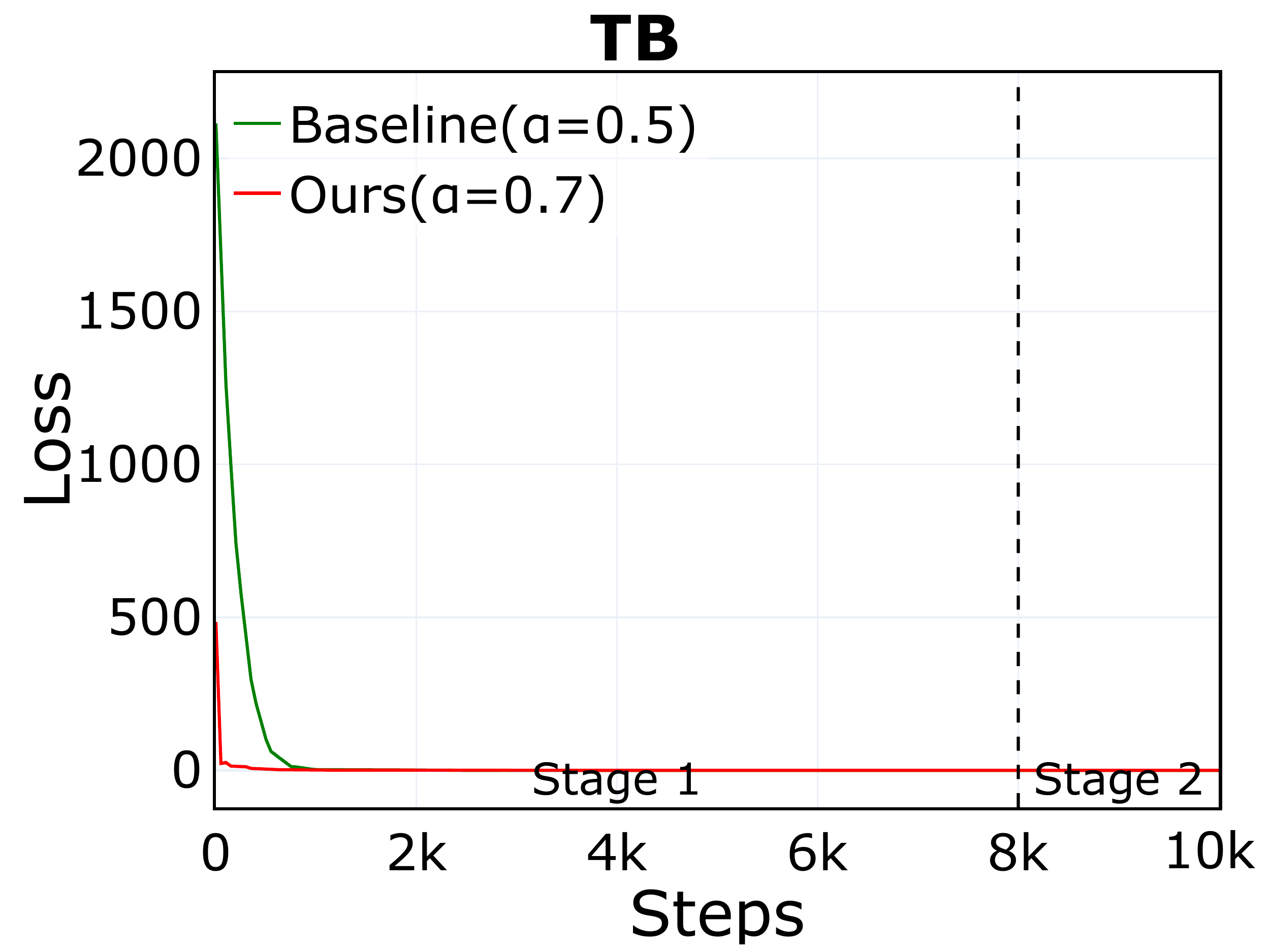} \\
    \small{(k)} DB (large) &
    \small{(l)} FL-DB (large) &
    \small{(m)} SubTB($\lambda$) (large) &
    \small{(n)} FL-SubTB($\lambda$) (large) &
    \small{(o)} TB (large) \\
  \end{tabular}
  \caption{\textbf{Loss} vs Training Steps in \textbf{Set Generation} across different objectives and set sizes.}
  \label{fig:set_metric_loss}
\end{figure}

\begin{figure}[htbp]
  \centering
  \setlength{\tabcolsep}{0pt}
\begin{tabular}{@{}c@{\hspace{0pt}}c@{\hspace{0pt}}c@{\hspace{0pt}}c@{\hspace{0pt}}c@{}}
    \includegraphics[width=0.2\textwidth]{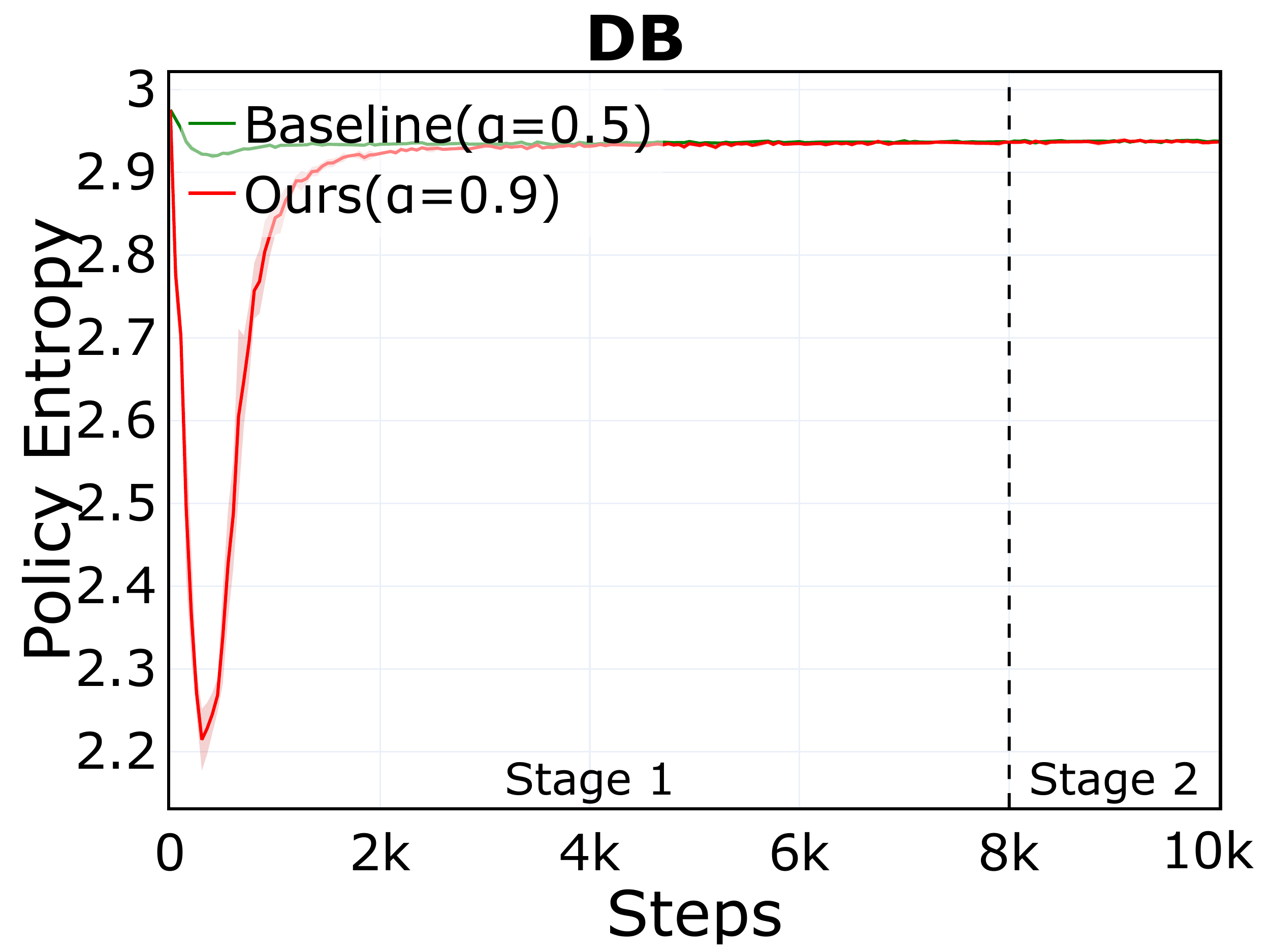} &
    \includegraphics[width=0.2\textwidth]{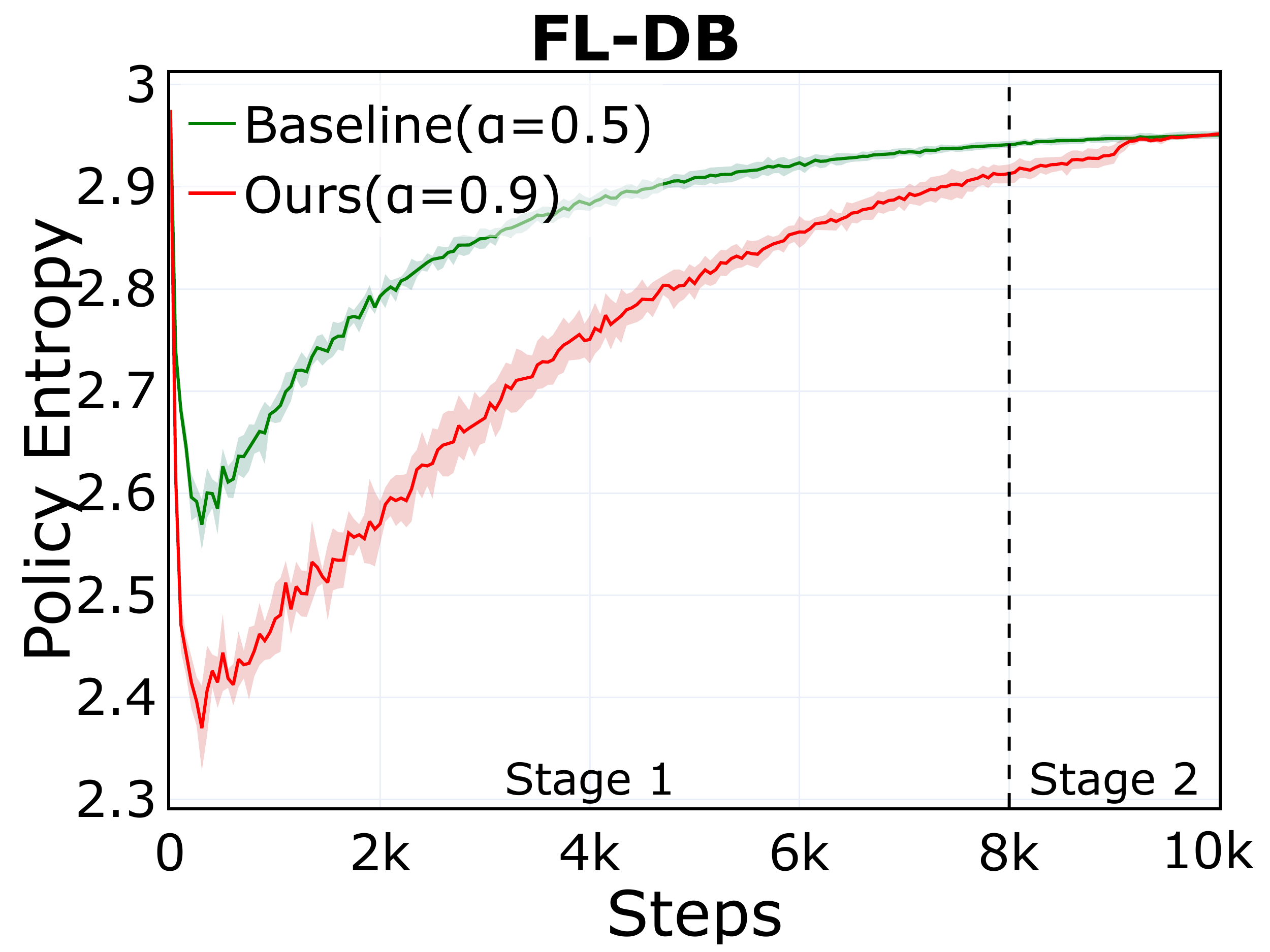} &
    \includegraphics[width=0.2\textwidth]{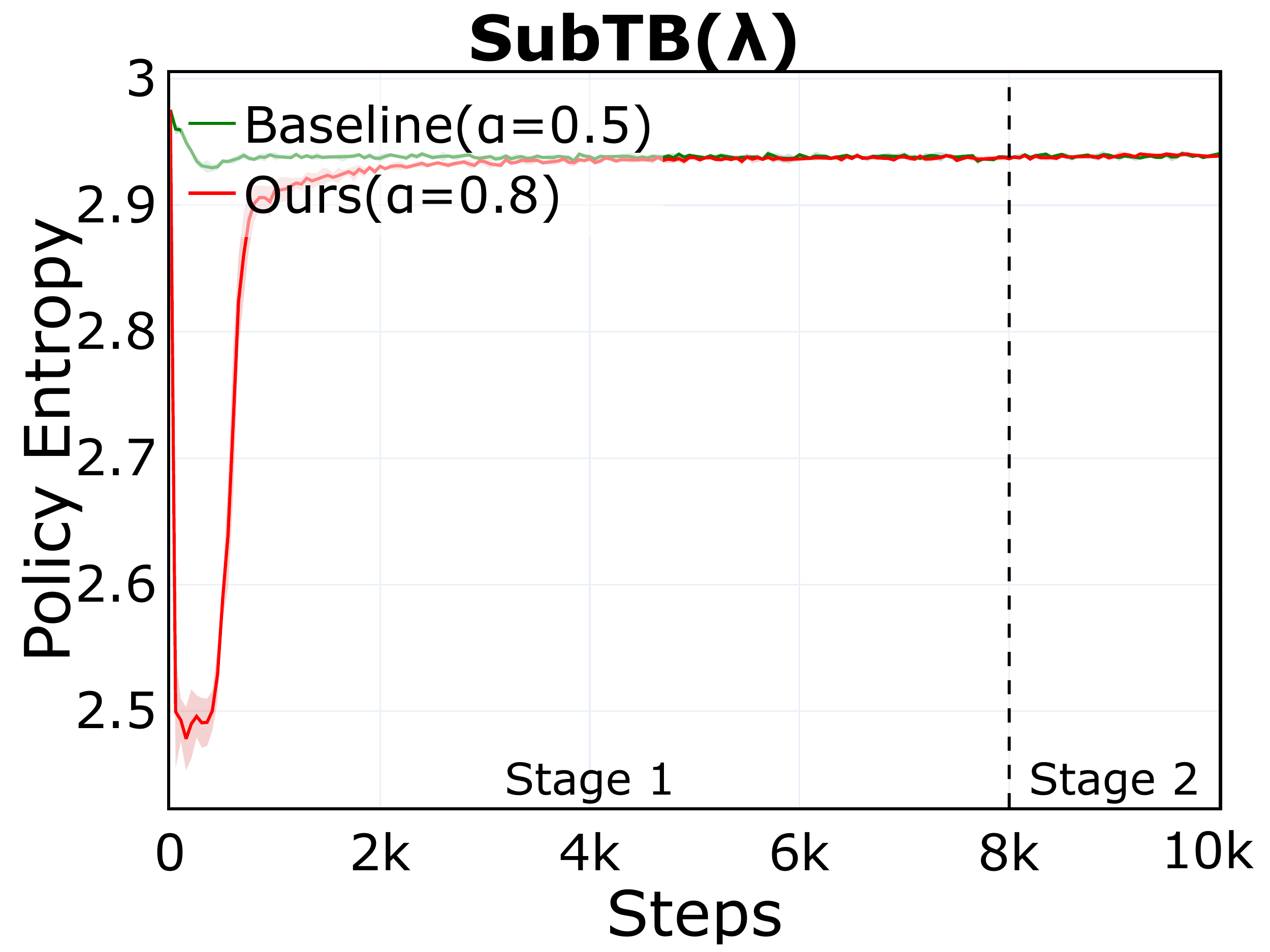} & 
    \includegraphics[width=0.2\textwidth]{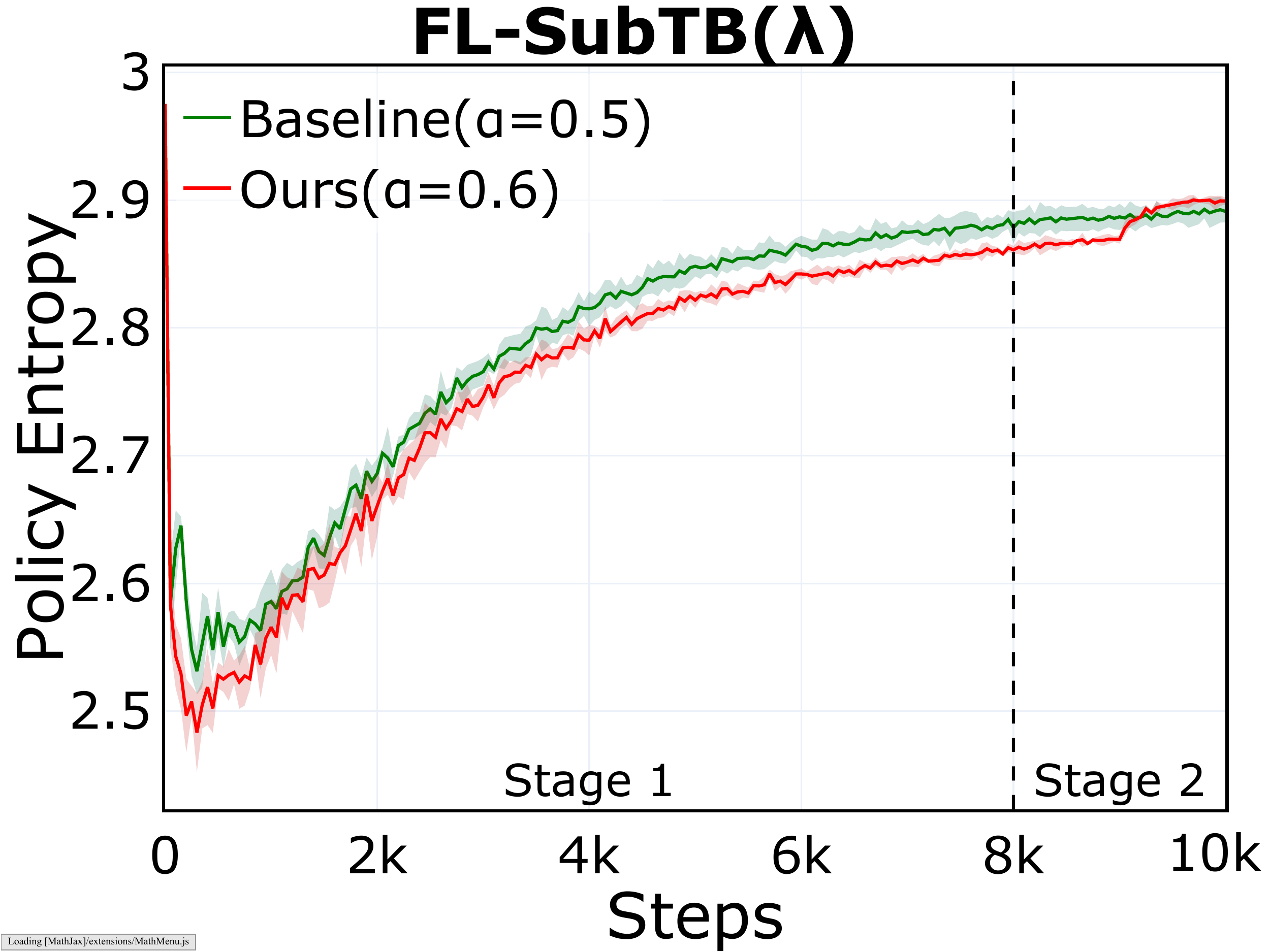} &
    \includegraphics[width=0.2\textwidth]{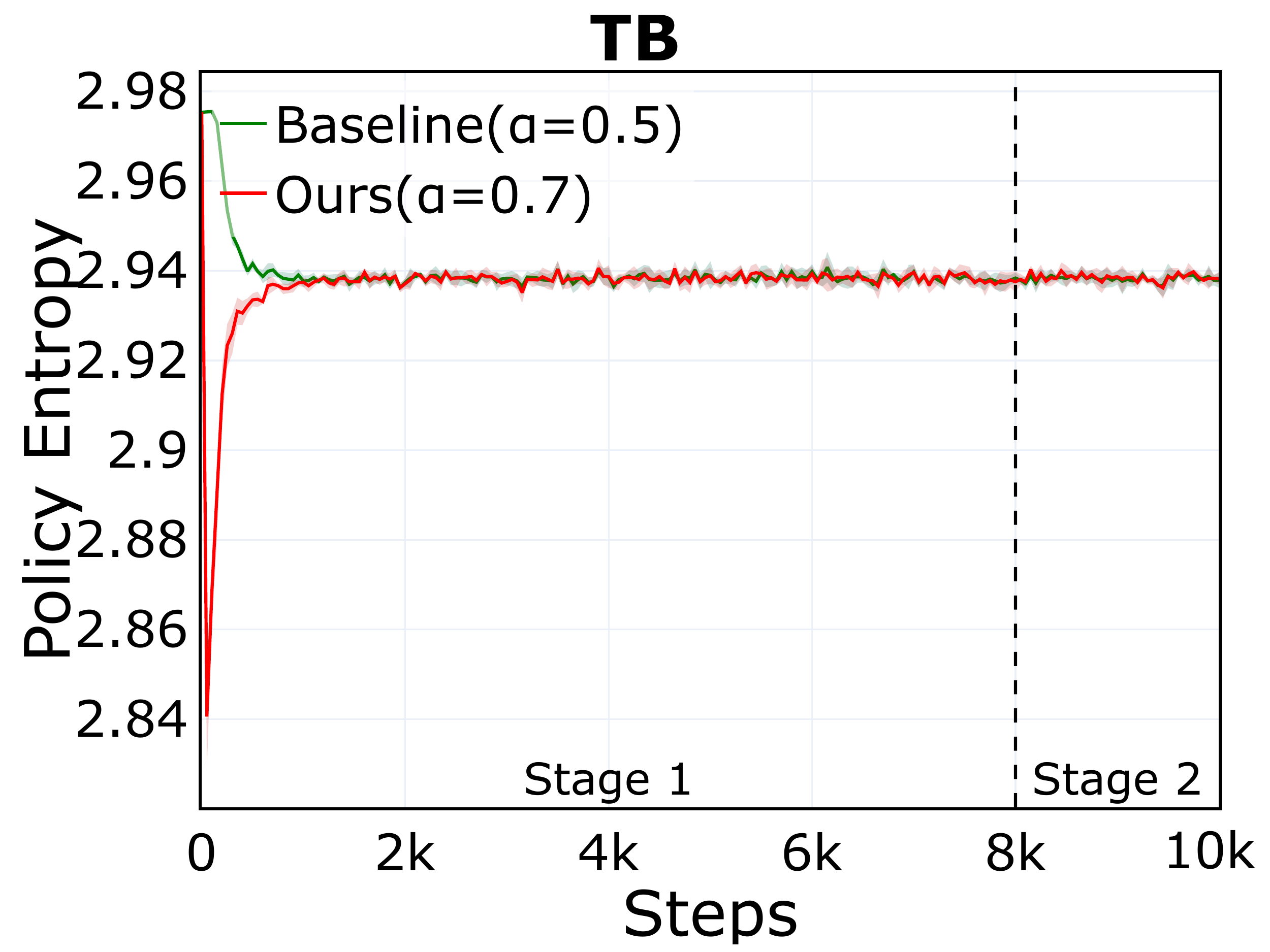} \\
    \small{(a)} DB (small) &
    \small{(b)} FL-DB (small) &
    \small{(c)} SubTB($\lambda$) (small) &
    \small{(d)} FL-SubTB($\lambda$) (small) &
    \small{(e)} TB (small) \\
    \includegraphics[width=0.2\textwidth]{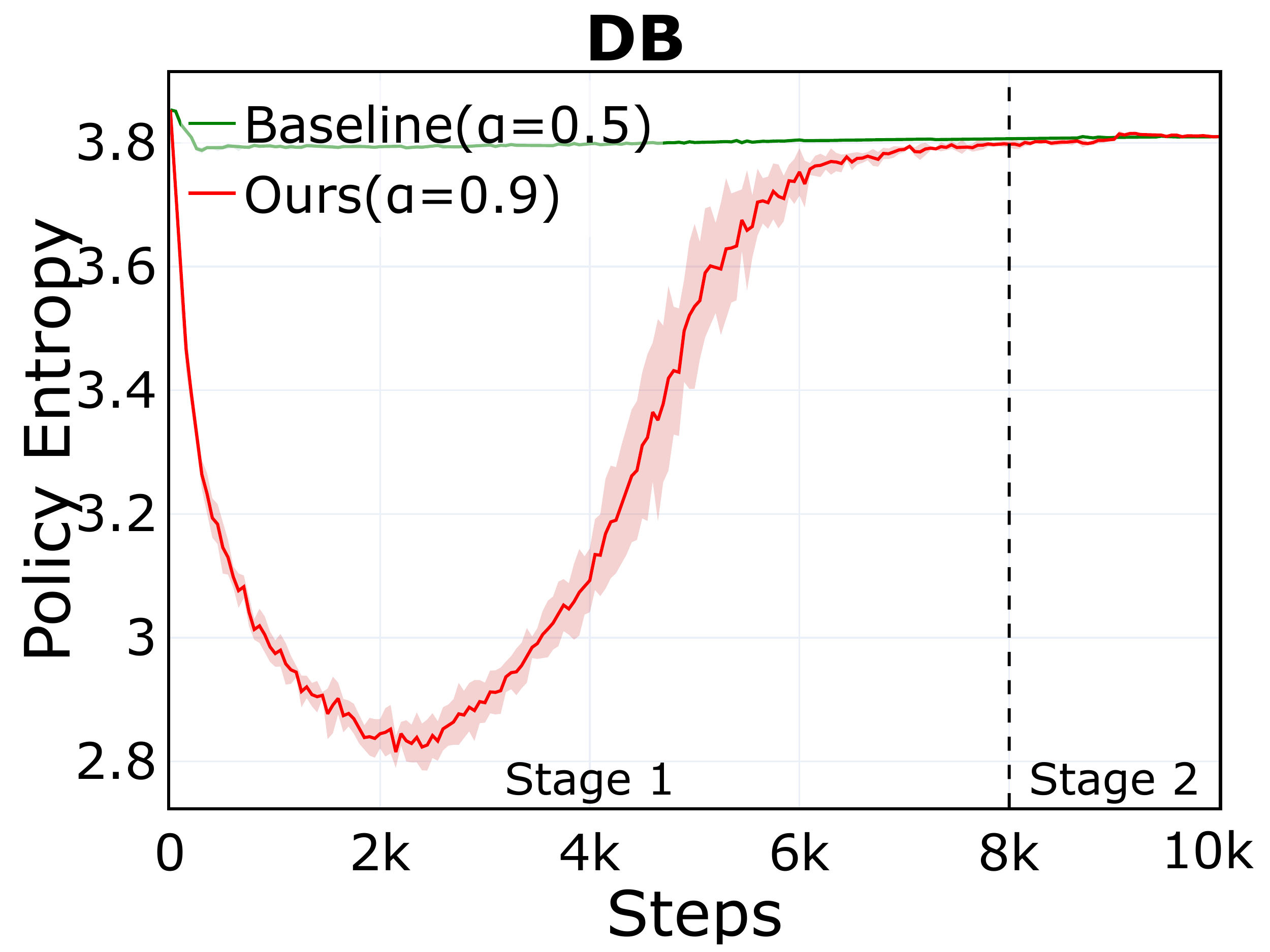} &
    \includegraphics[width=0.2\textwidth]{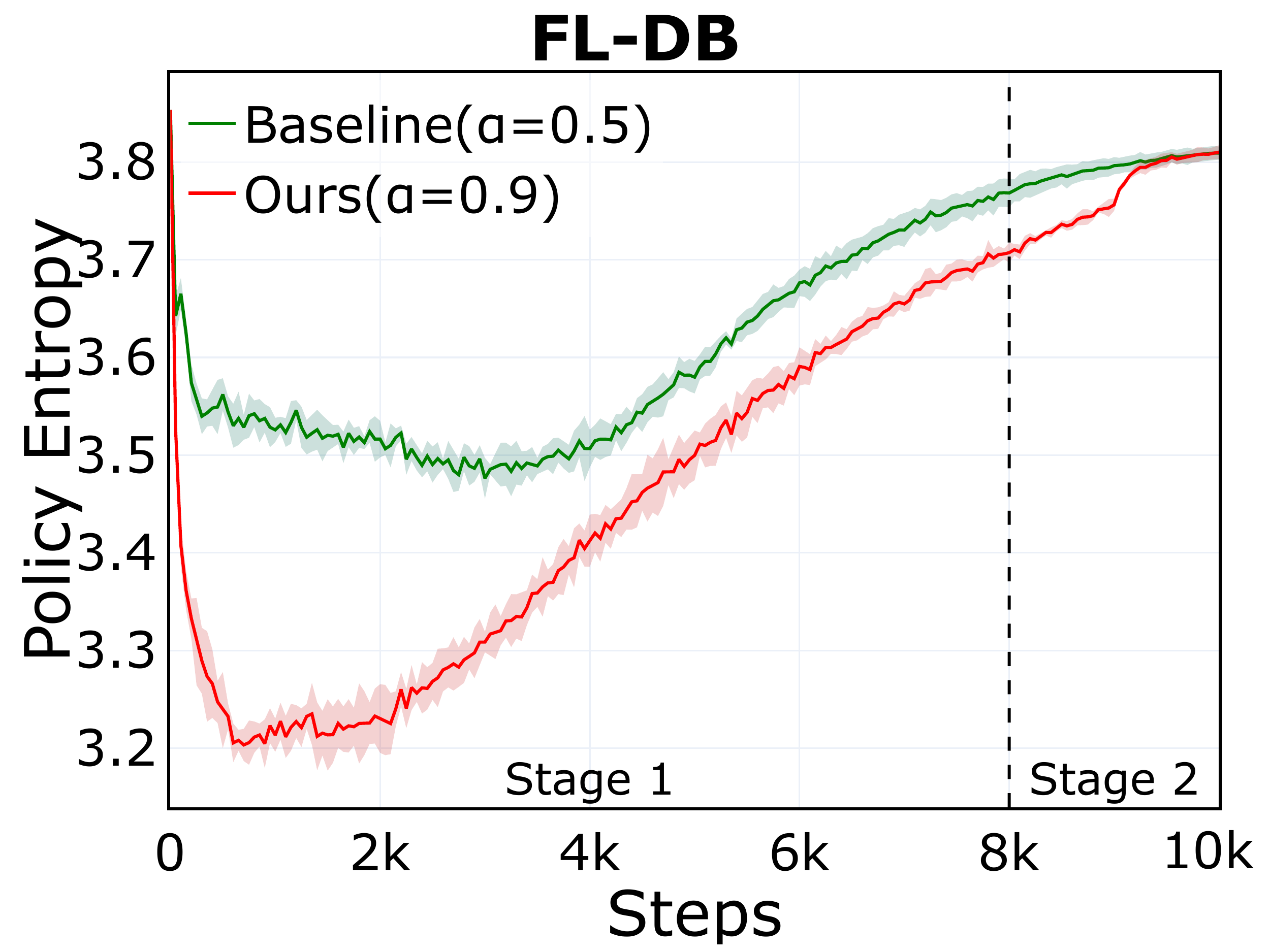} &
     \includegraphics[width=0.2\textwidth]{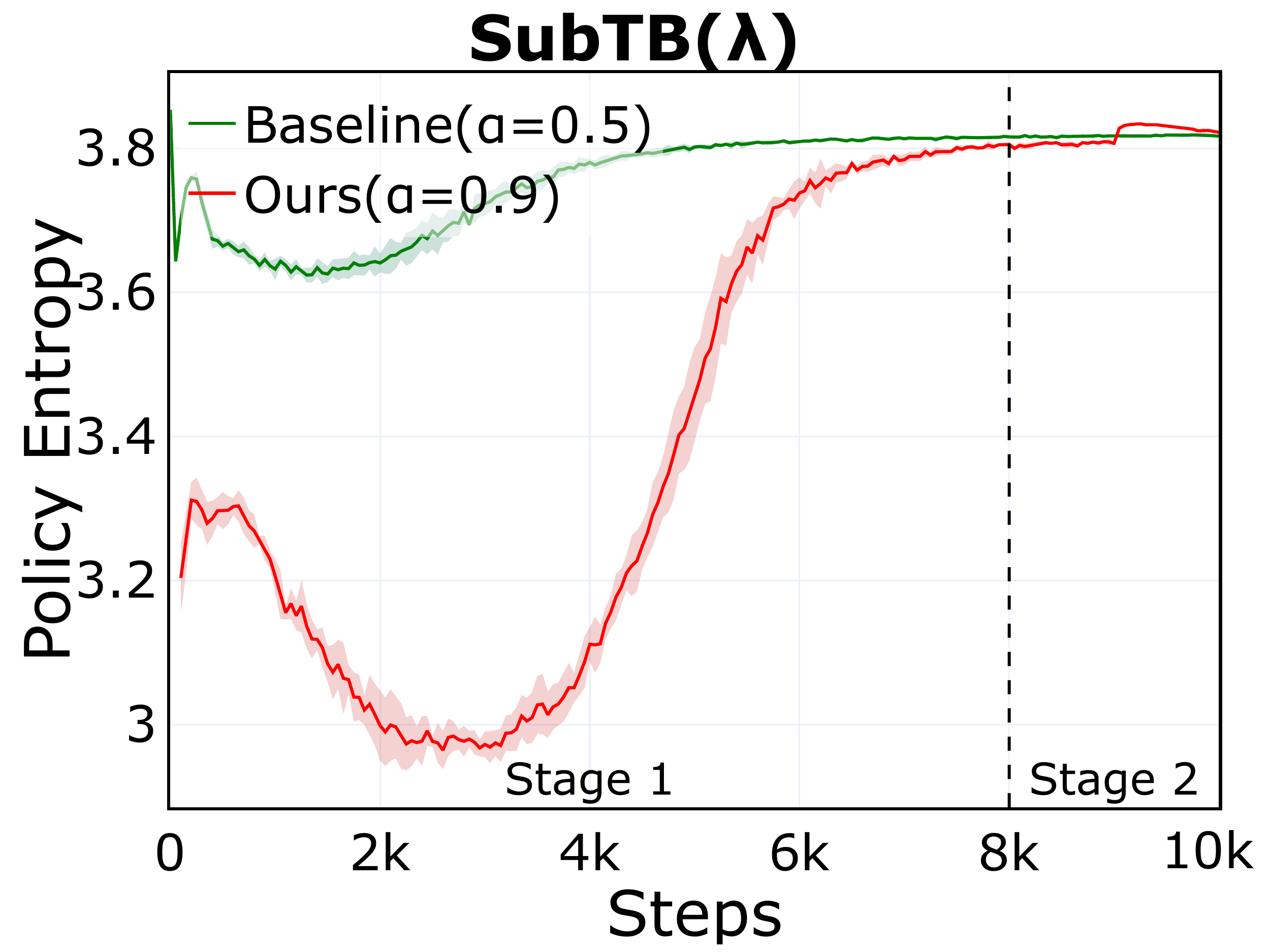}&
    \includegraphics[width=0.2\textwidth]{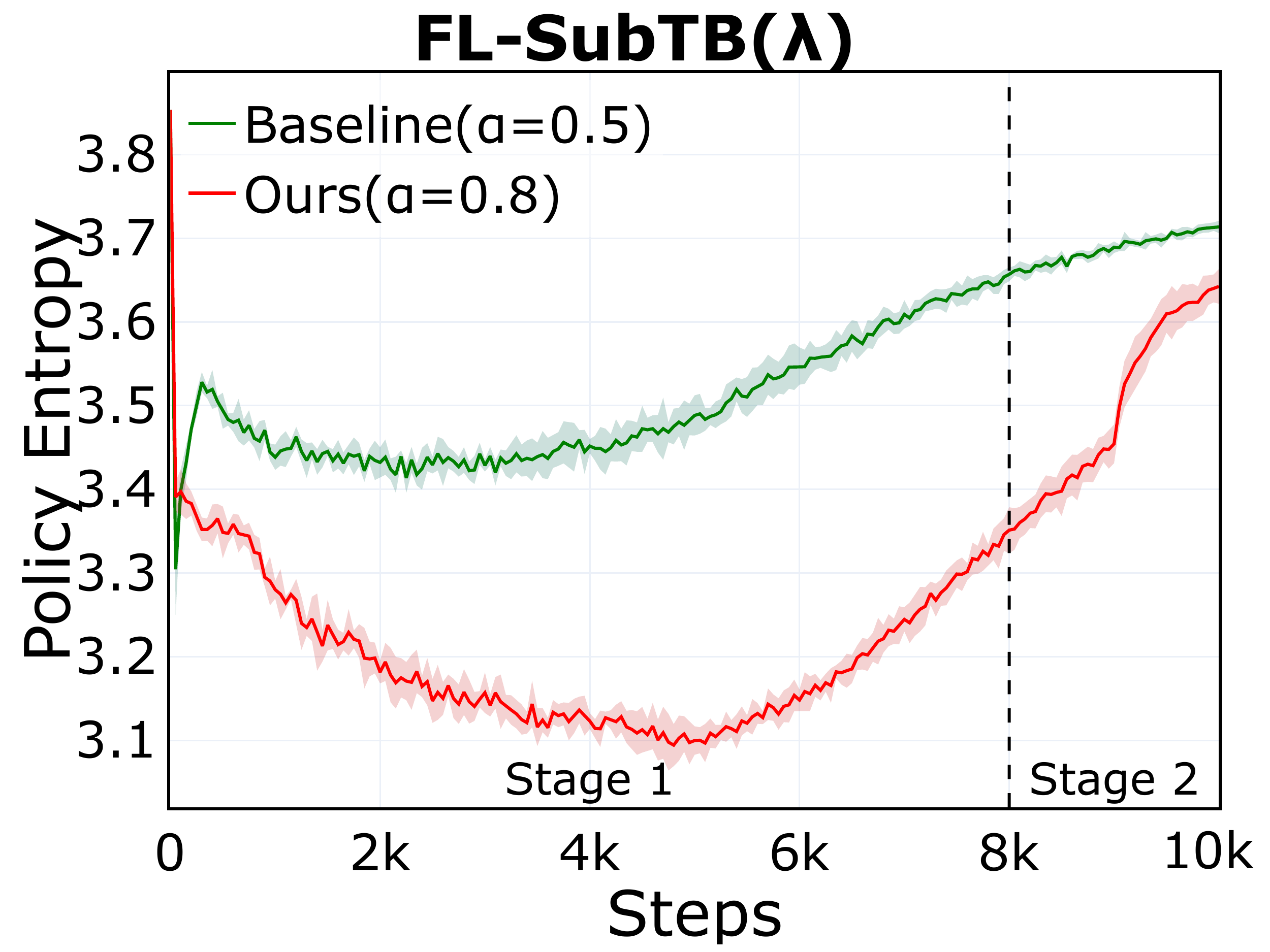} &
    \includegraphics[width=0.2\textwidth]{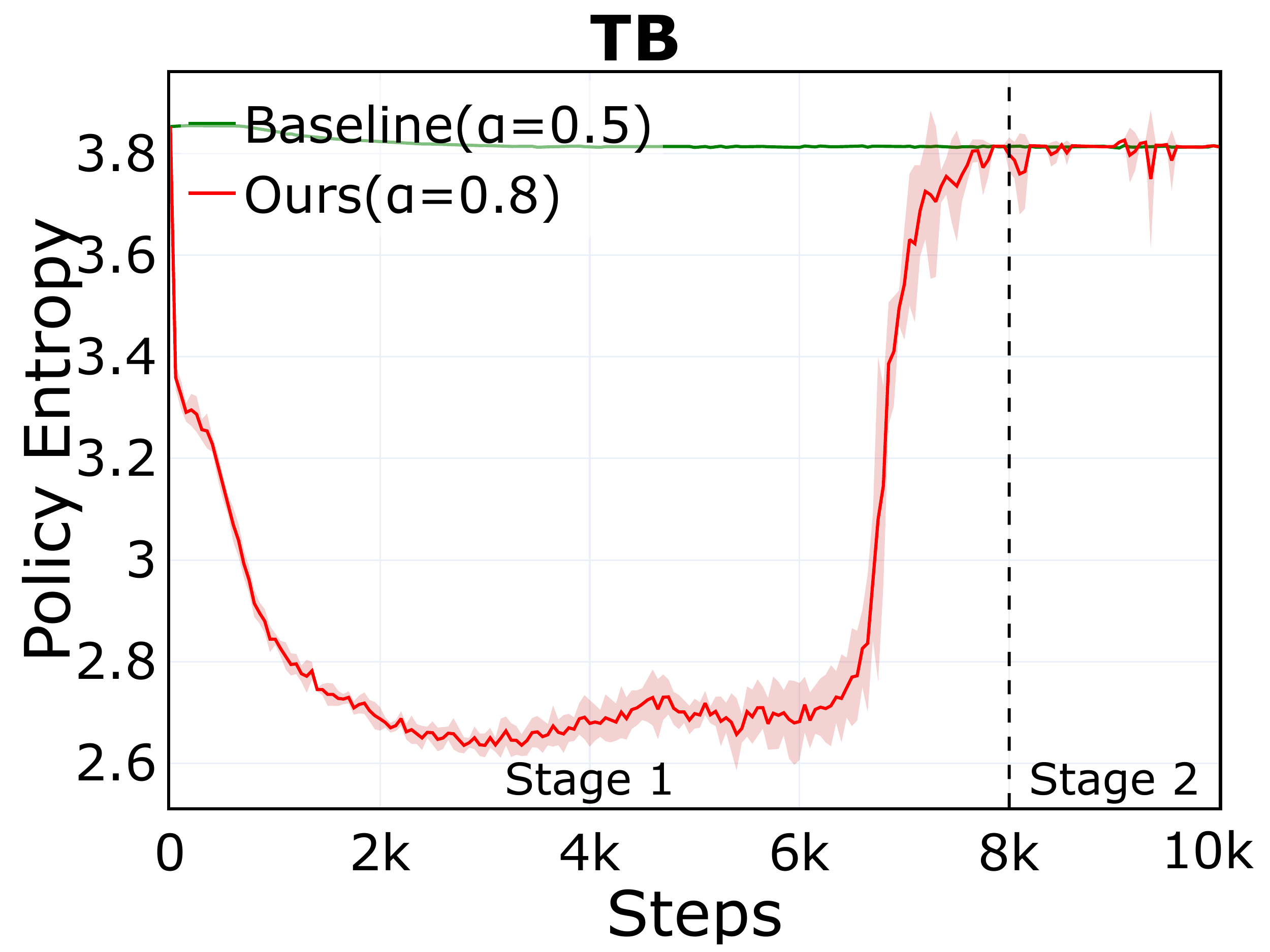} \\
    \small{(f)} DB (medium) &
    \small{(g)} FL-DB (medium) &
    \small{(h)} SubTB($\lambda$) (medium) &
    \small{(i)} FL-SubTB($\lambda$) (medium) &
    \small{(j)} TB (medium) \\
    \includegraphics[width=0.2\textwidth]{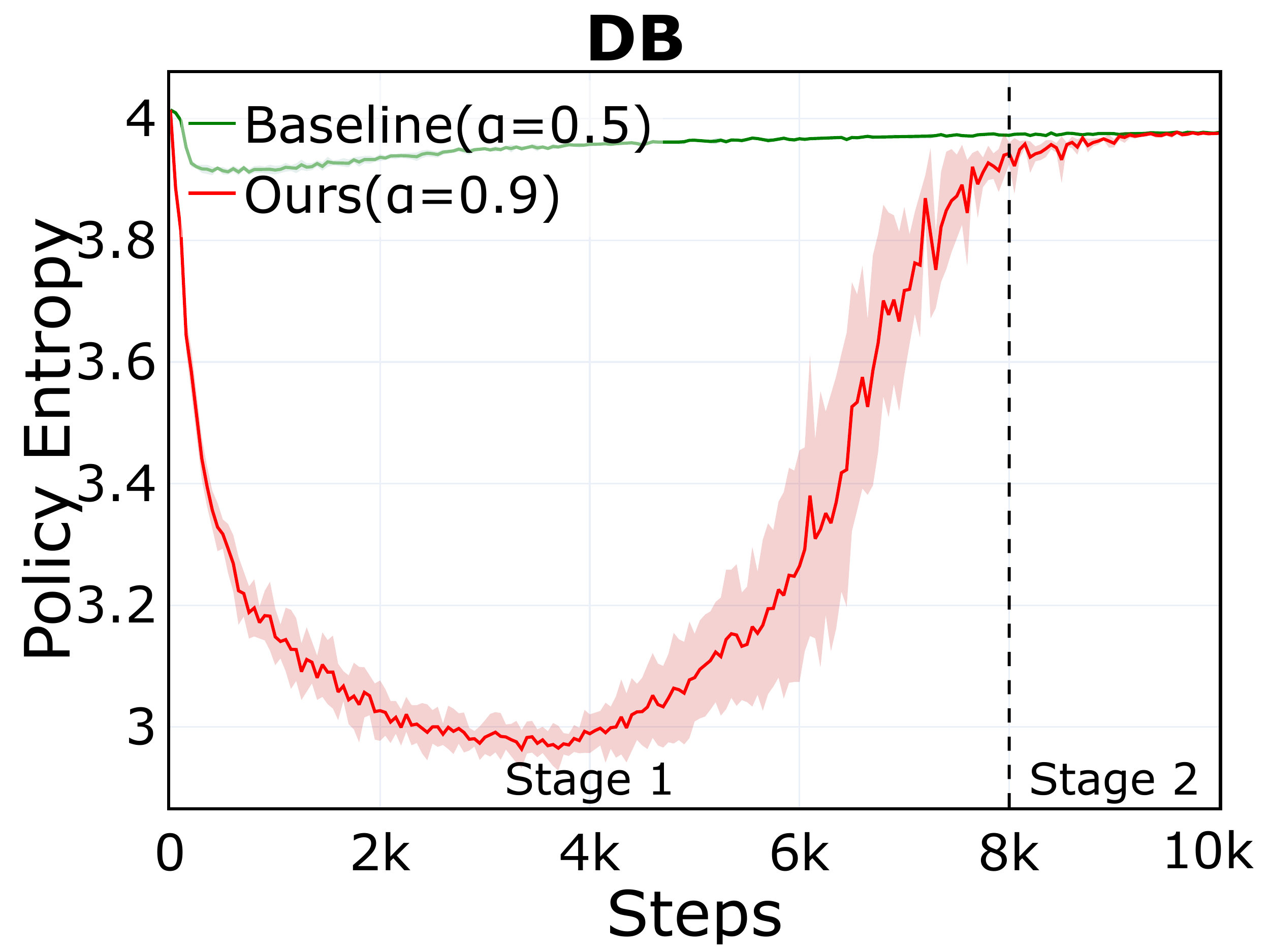} &
    \includegraphics[width=0.2\textwidth]{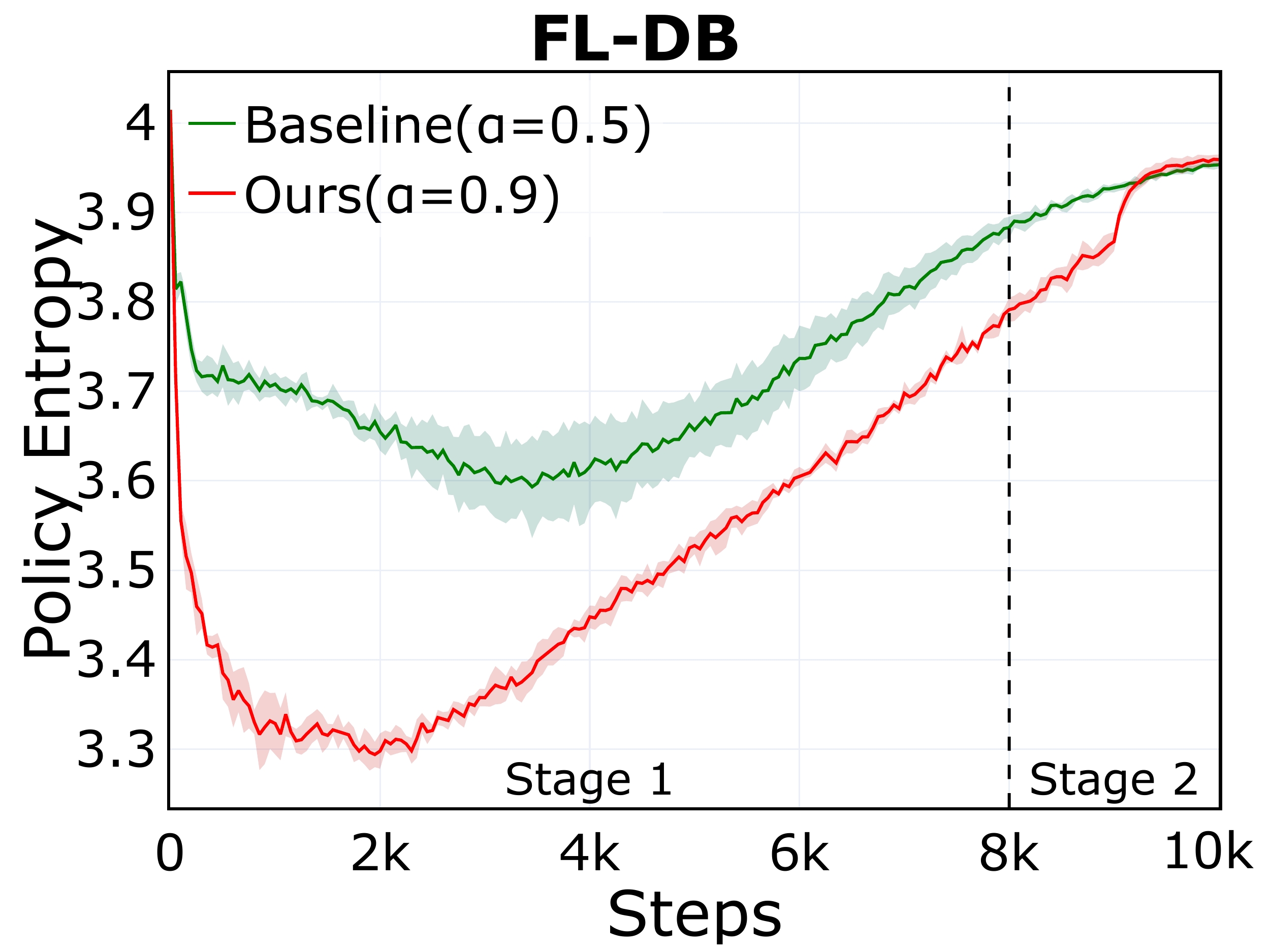} &
     \includegraphics[width=0.2\textwidth]{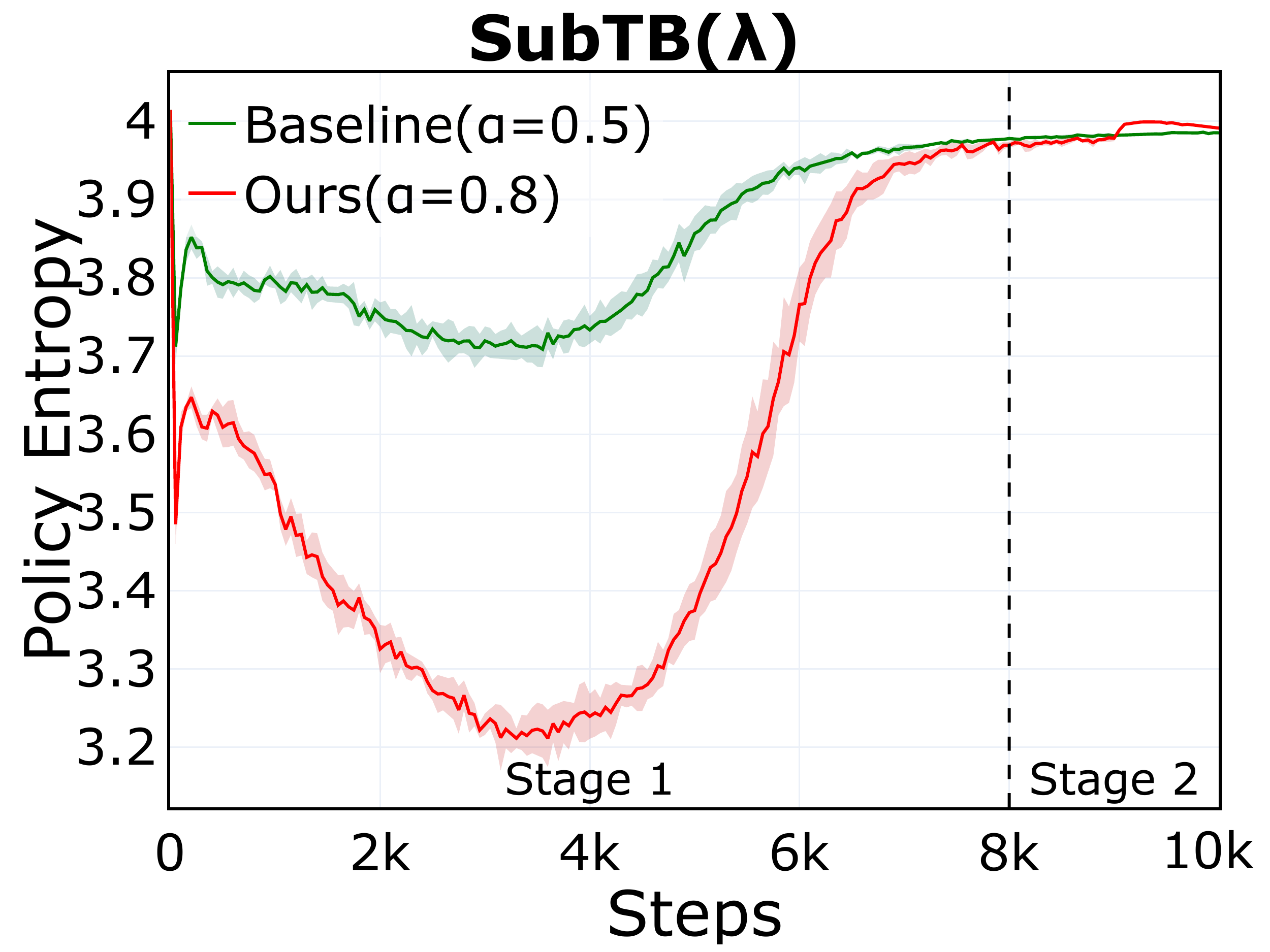} &
     \includegraphics[width=0.2\textwidth]{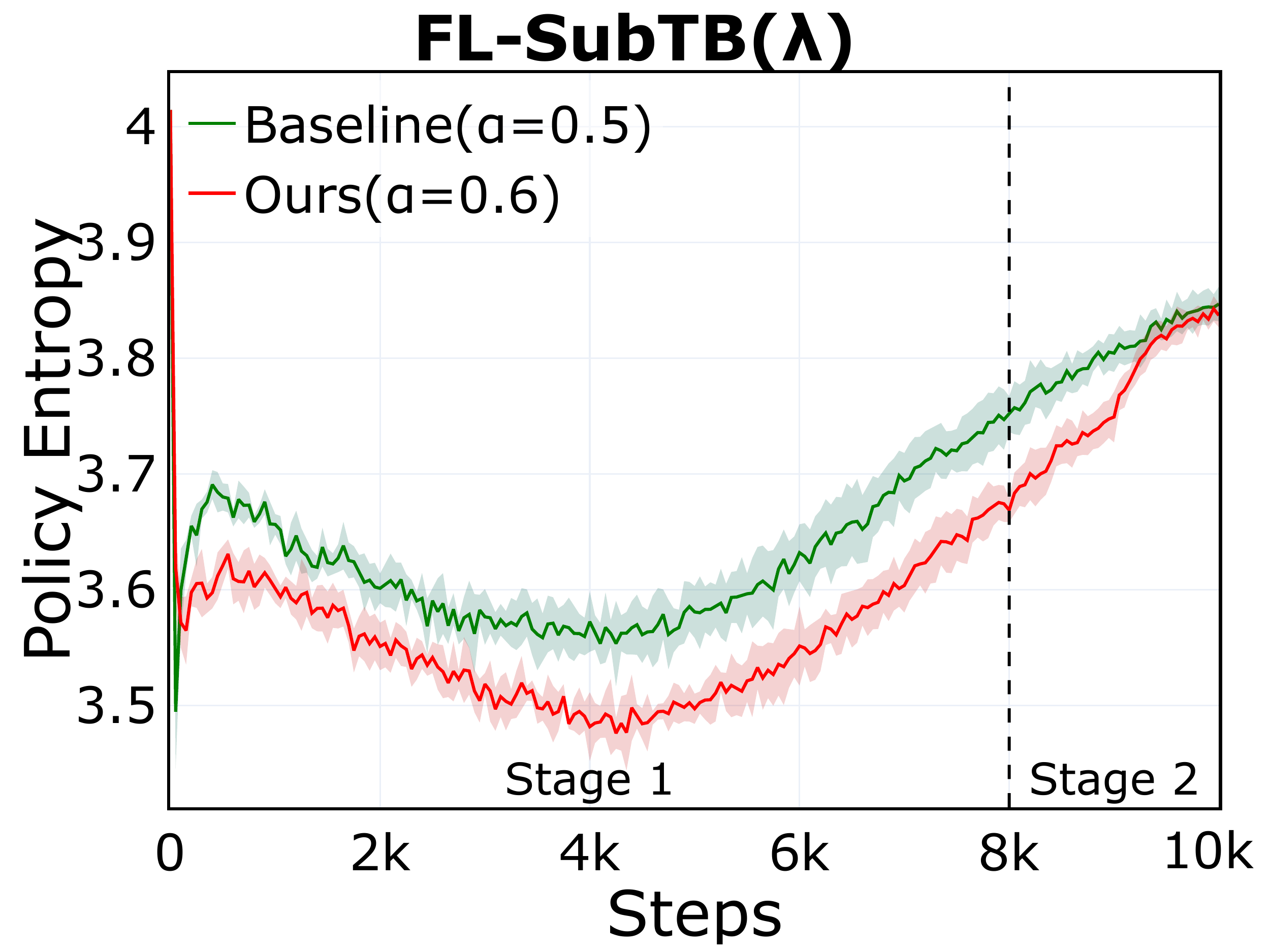}&
    \includegraphics[width=0.2\textwidth]{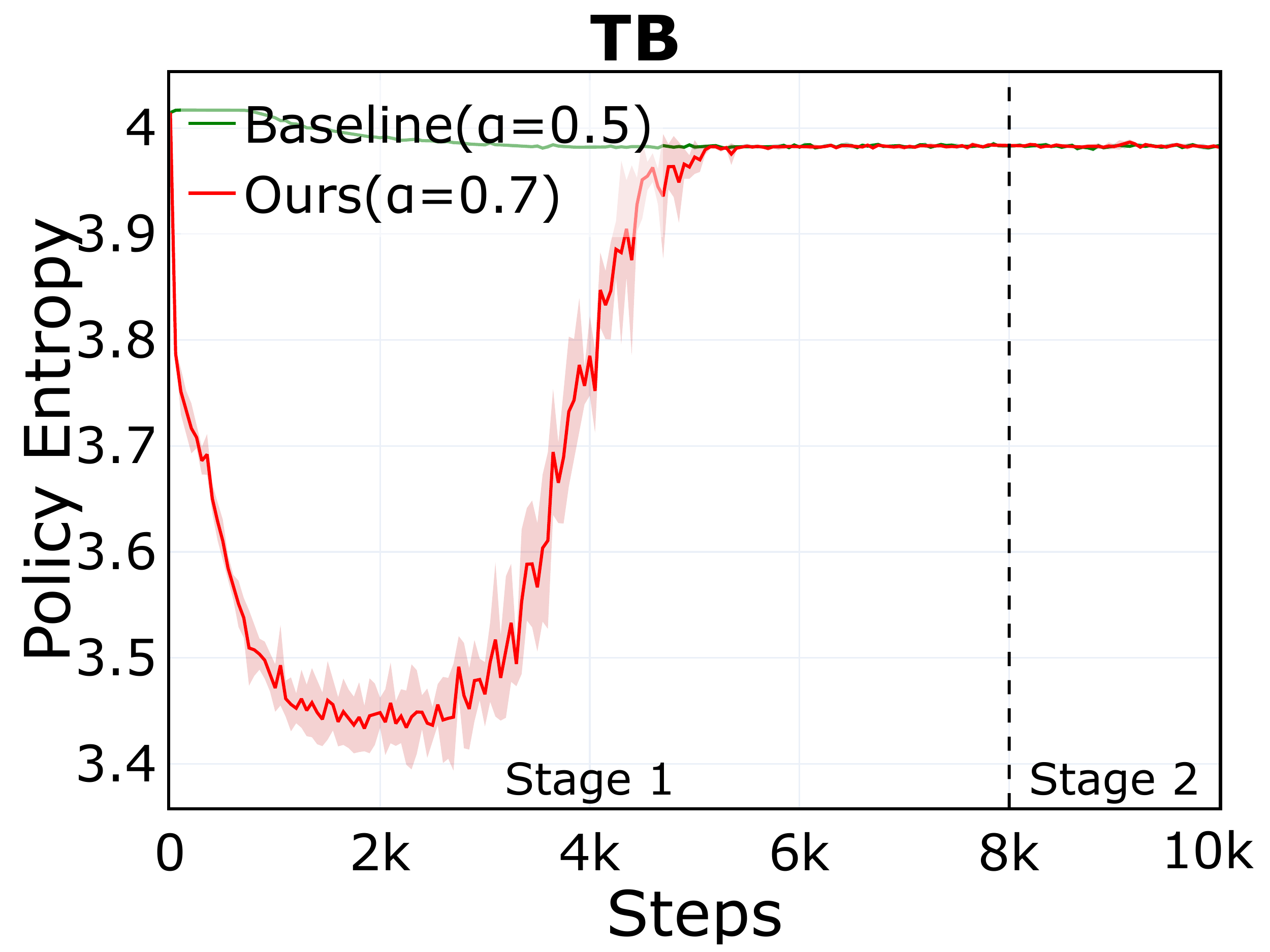} \\
    \small{(k)} DB (large) &
    \small{(l)} FL-DB (large) &
    \small{(m)} SubTB($\lambda$) (large) &
    \small{(n)} FL-SubTB($\lambda$) (large) &
    \small{(o)} TB (large)
  \end{tabular}
  \caption{\textbf{Forward Policy Entropy} vs Training Steps in \textbf{Set Generation} across different objectives and set sizes.}
  \label{fig:set_metric_forward_policy_entropy_eval}
\end{figure}

\begin{figure}[htbp]
  \centering
  \setlength{\tabcolsep}{0pt}
  \begin{tabular}{@{}c@{\hspace{0pt}}c@{\hspace{0pt}}c@{\hspace{0pt}}c@{\hspace{0pt}}c@{}}
    \includegraphics[width=.2\textwidth]{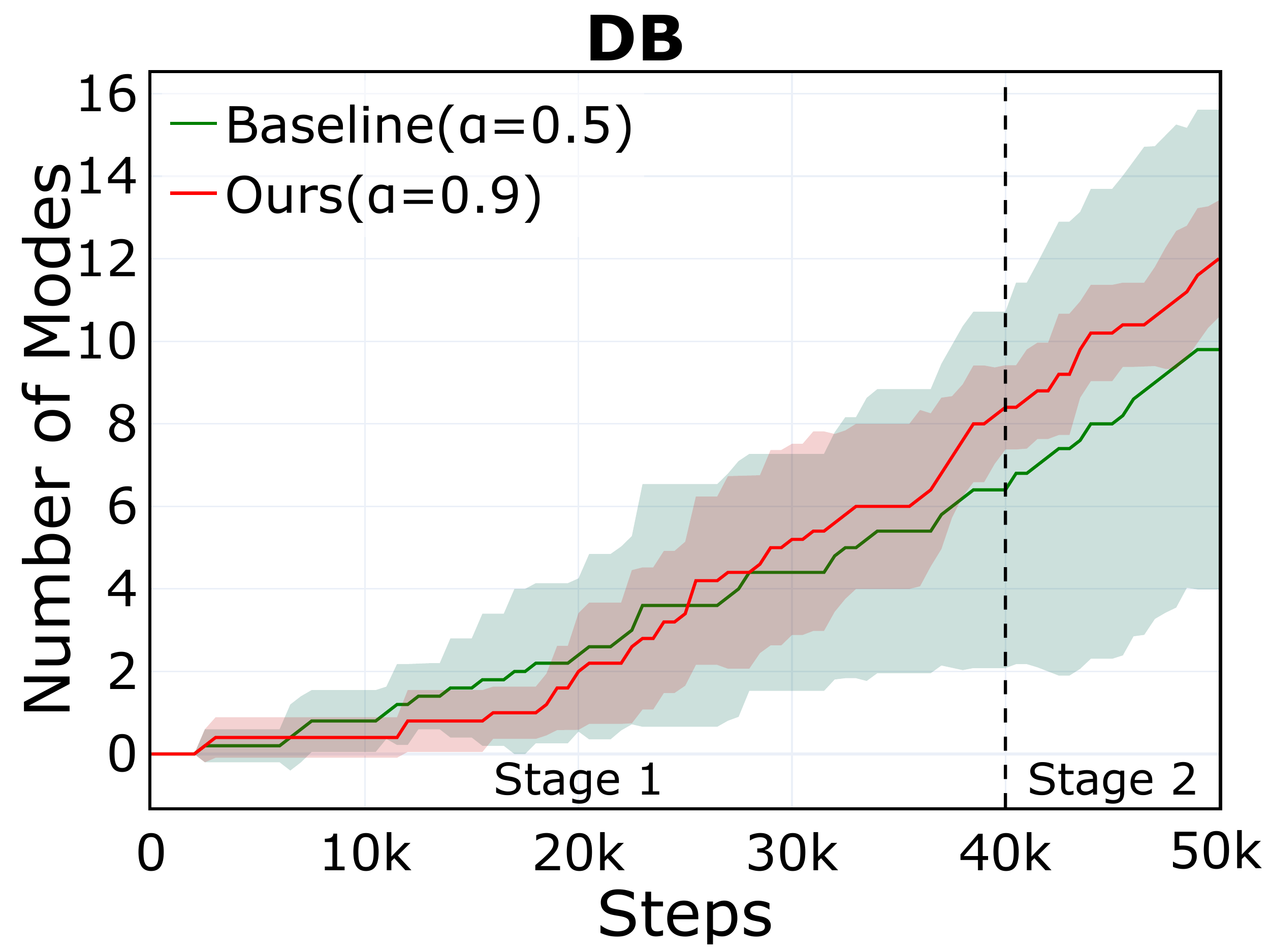} &
    \includegraphics[width=.2\textwidth]{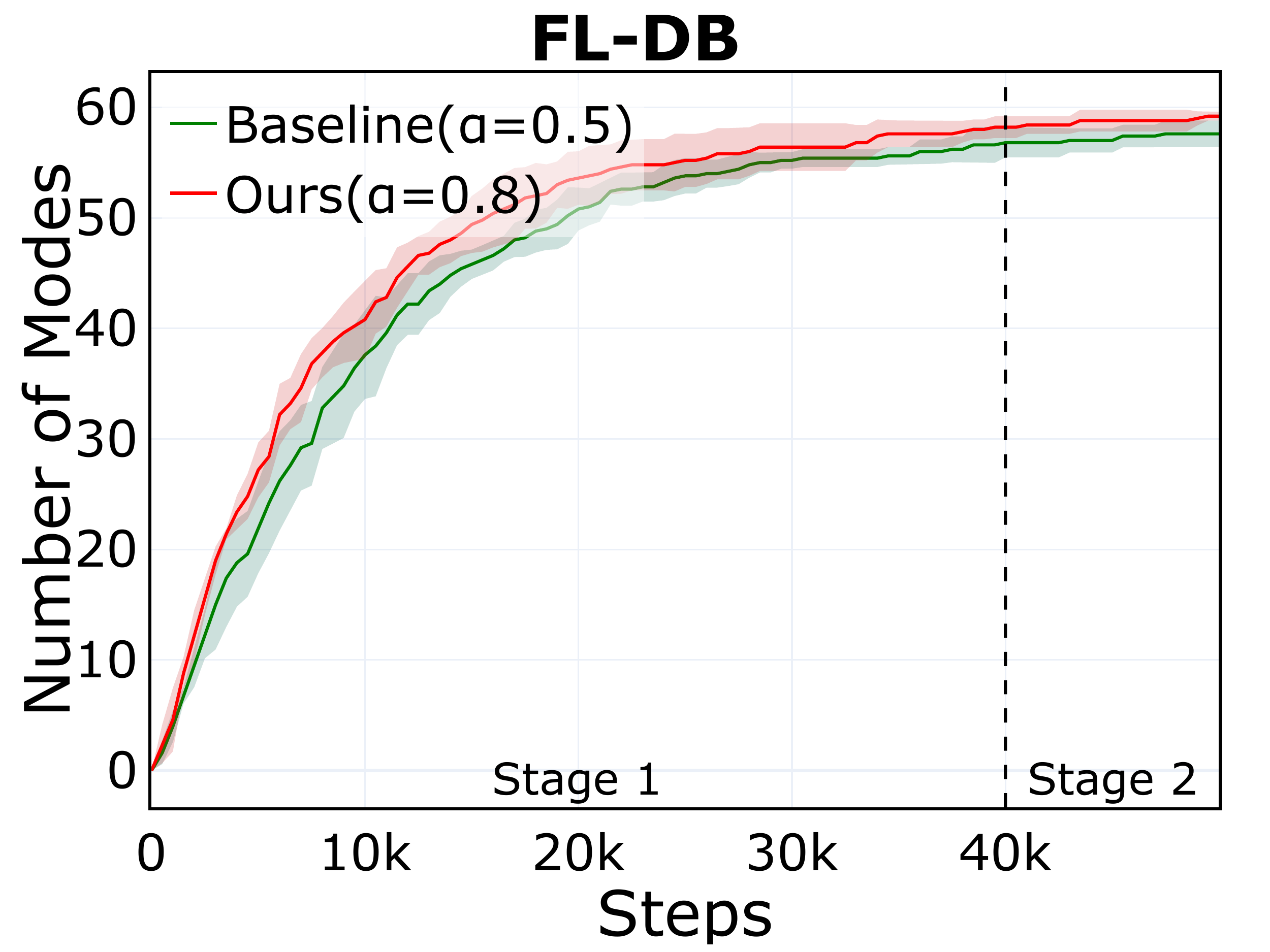} &
    \includegraphics[width=.2\textwidth]{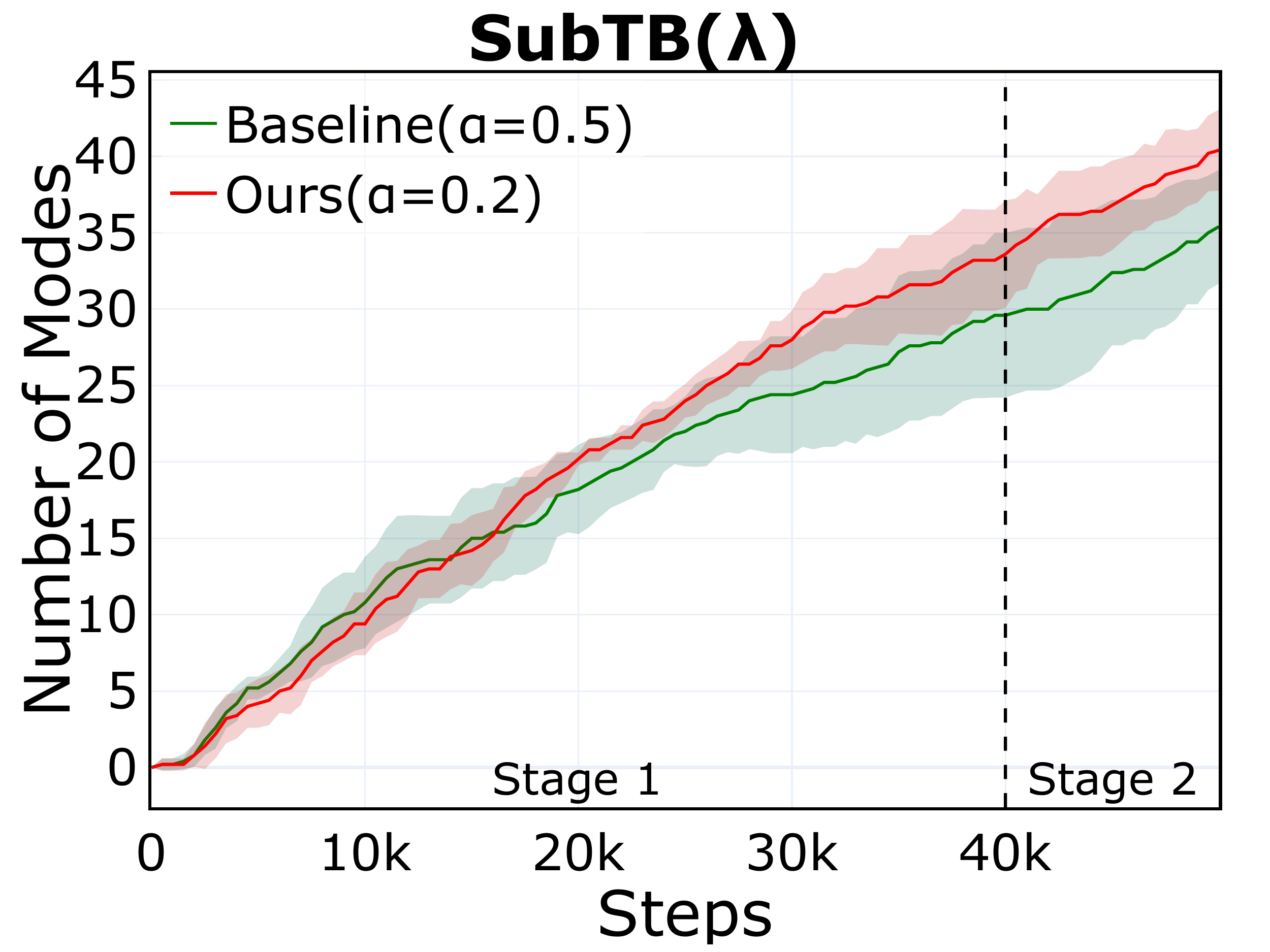} &
    \includegraphics[width=.2\textwidth]{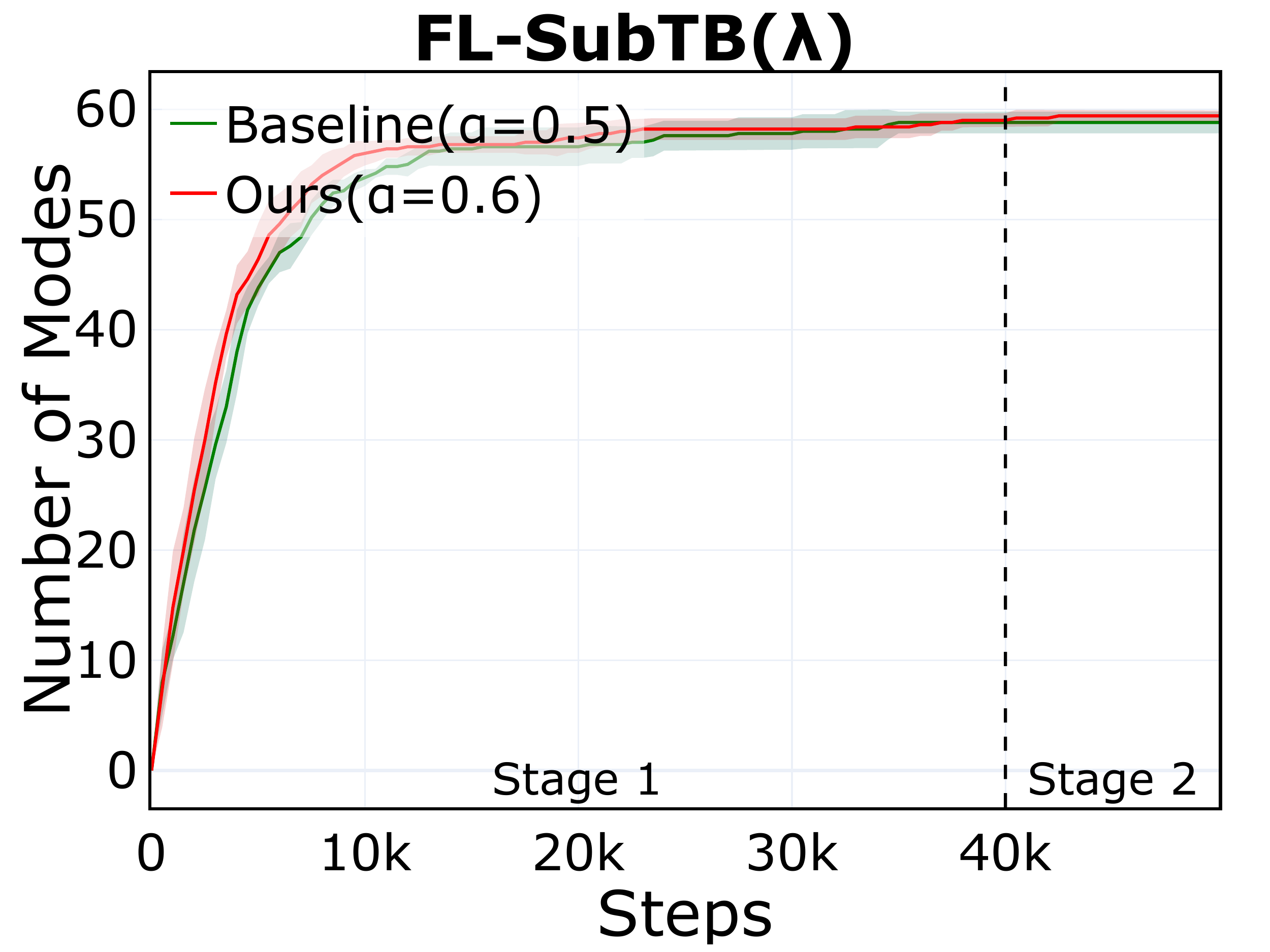} &
    \includegraphics[width=.2\textwidth]{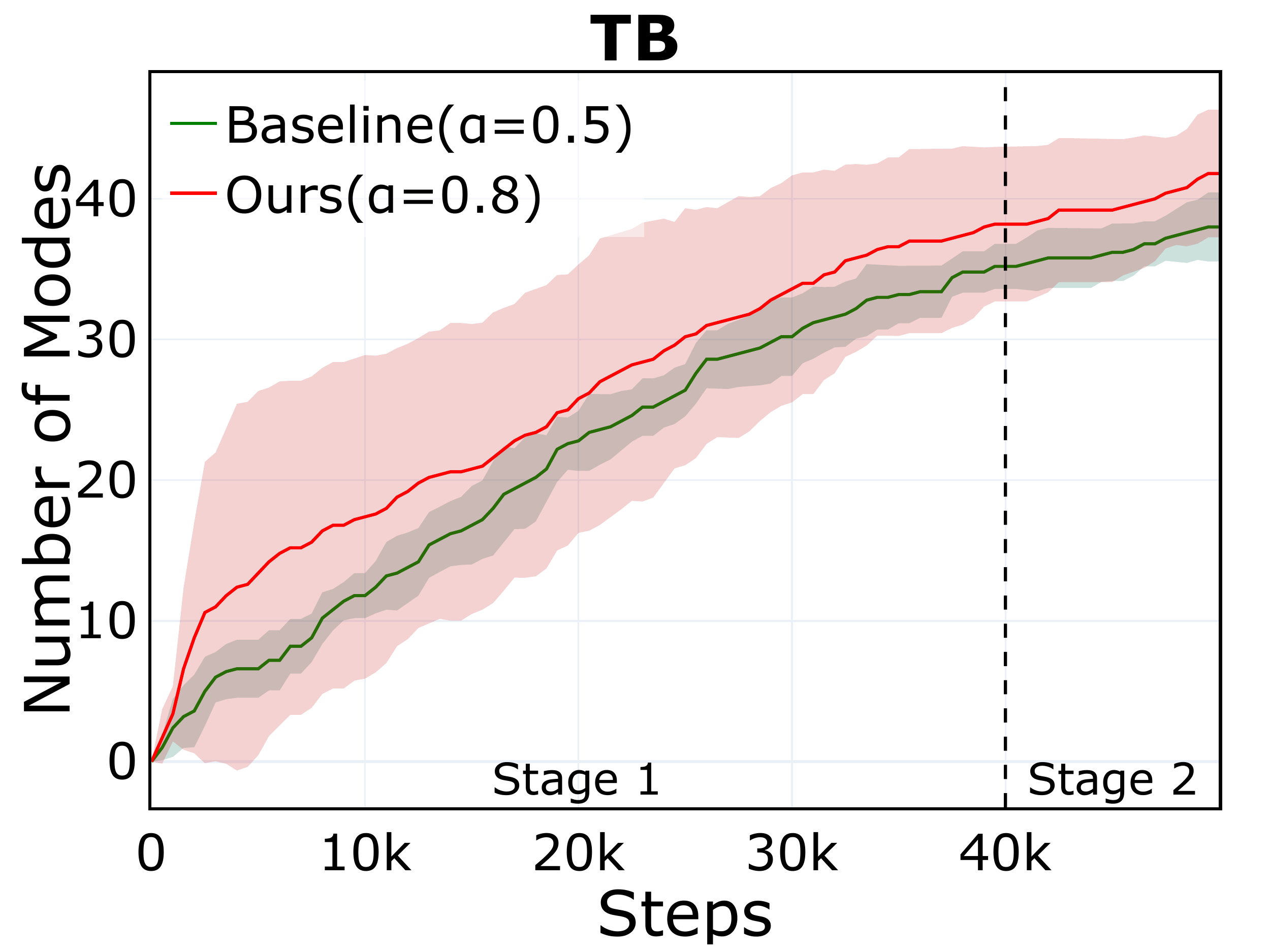} \\
  \end{tabular}
  \caption{\textbf{Number of Modes} vs Training Steps in \textbf{Bit Sequence Generation} across different objectives.}
  \label{fig:bit_metric_modes}
\end{figure}

\begin{figure}[htbp]
  \centering
  \setlength{\tabcolsep}{0pt}
  \begin{tabular}{@{}c@{\hspace{0pt}}c@{\hspace{0pt}}c@{\hspace{0pt}}c@{\hspace{0pt}}c@{}}
    \includegraphics[width=.2\textwidth]{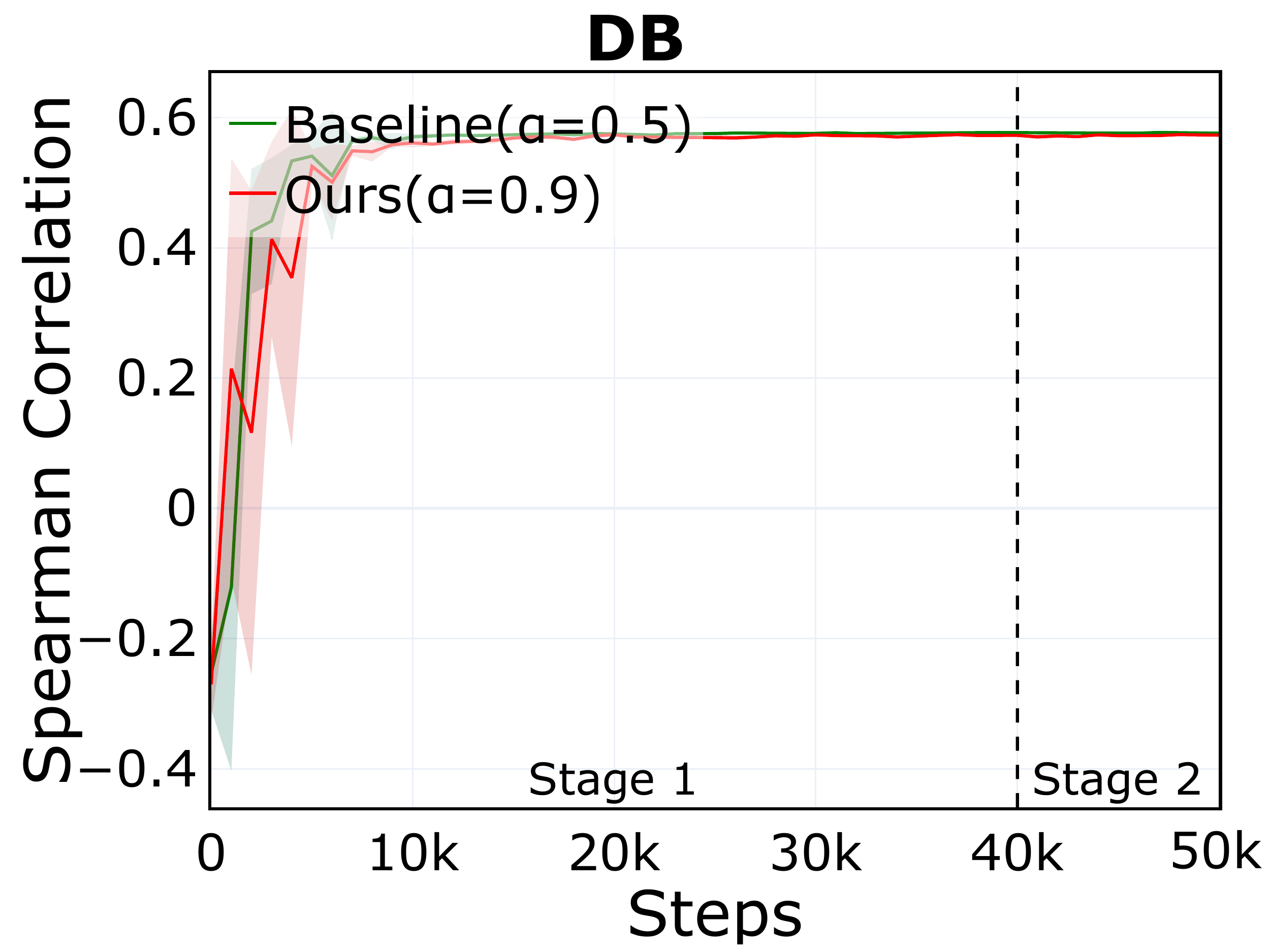} &
    \includegraphics[width=.2\textwidth]{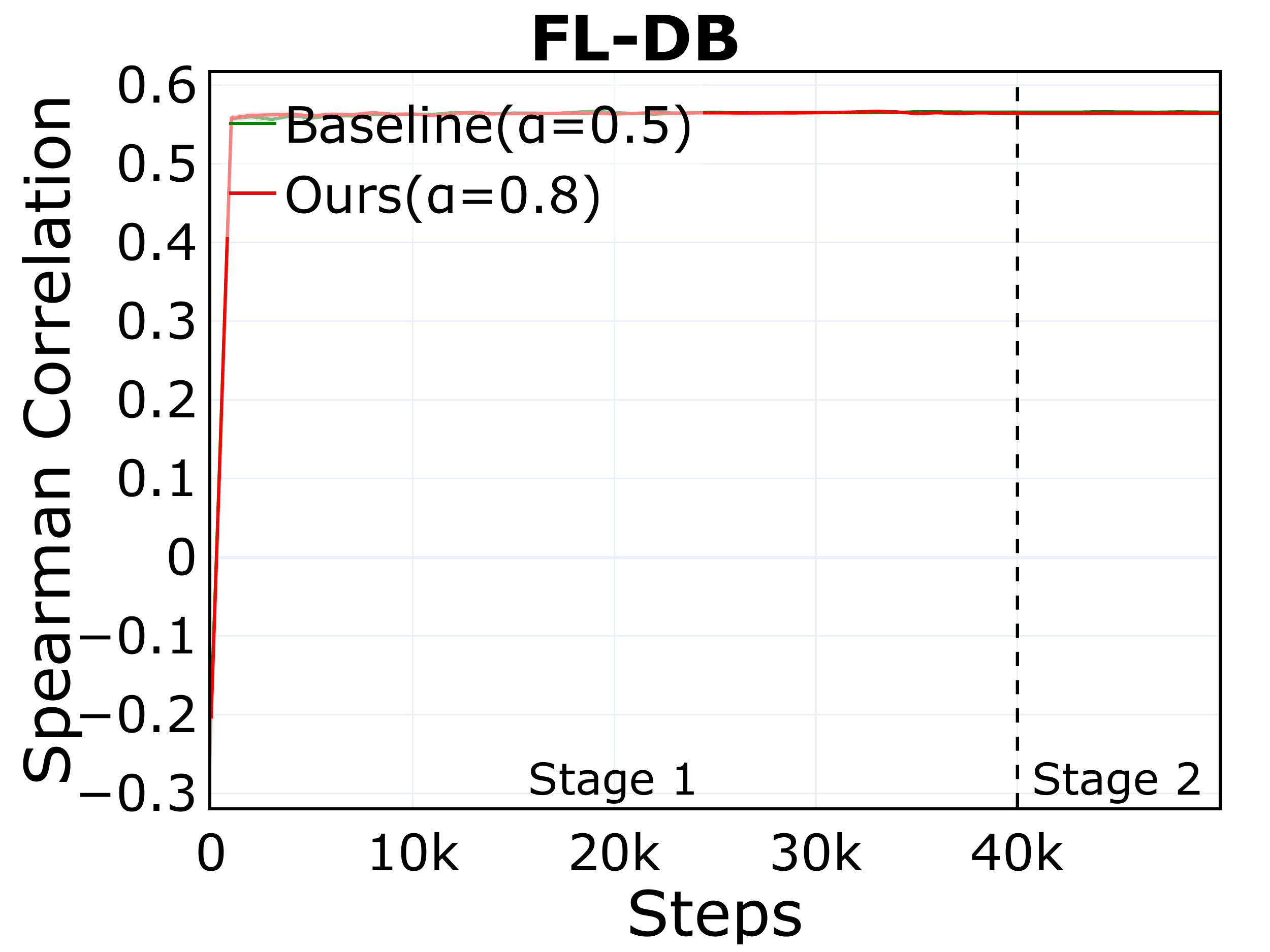} &
    \includegraphics[width=.2\textwidth]{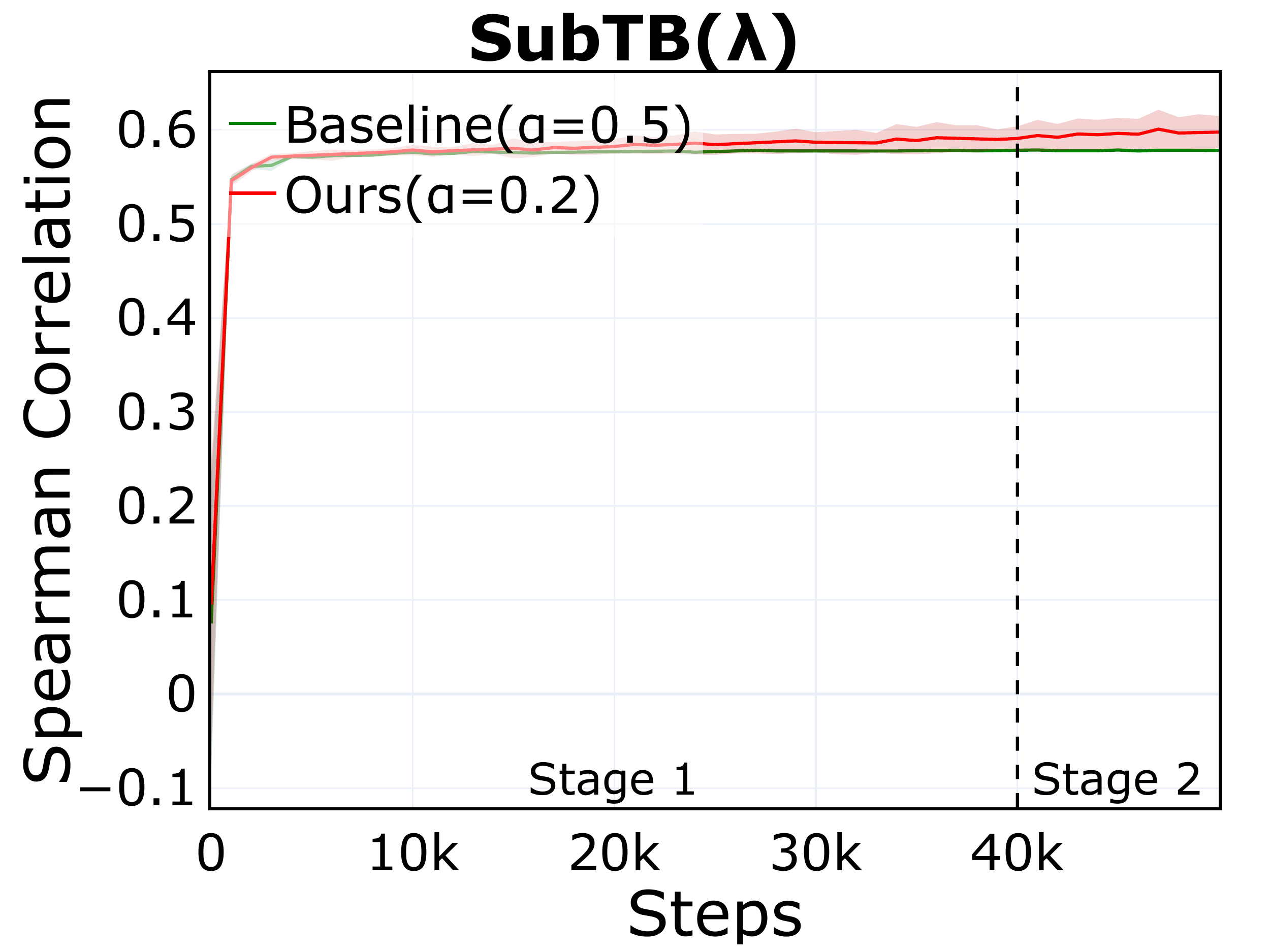} &
    \includegraphics[width=.2\textwidth]{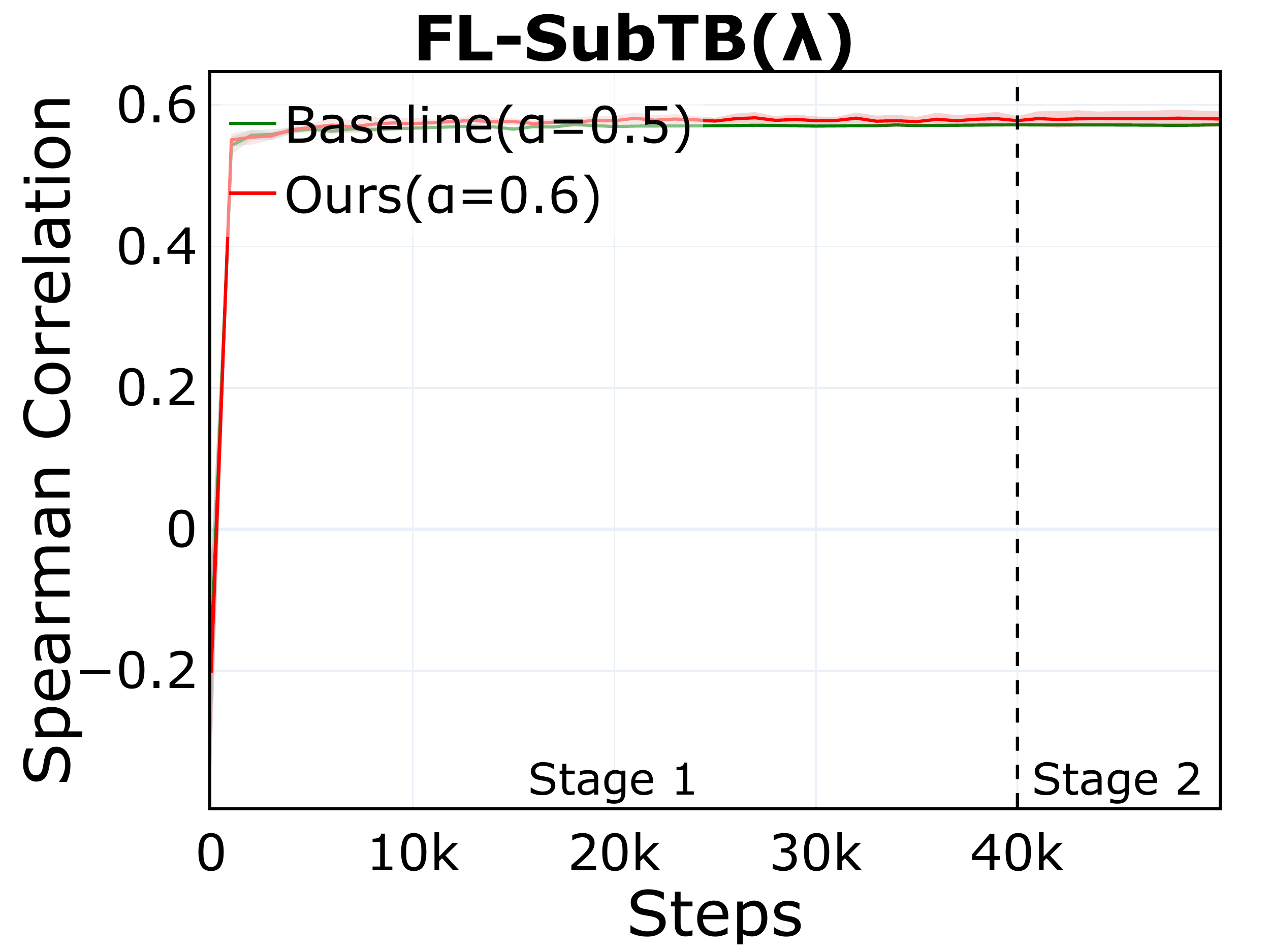} &
    \includegraphics[width=.2\textwidth]{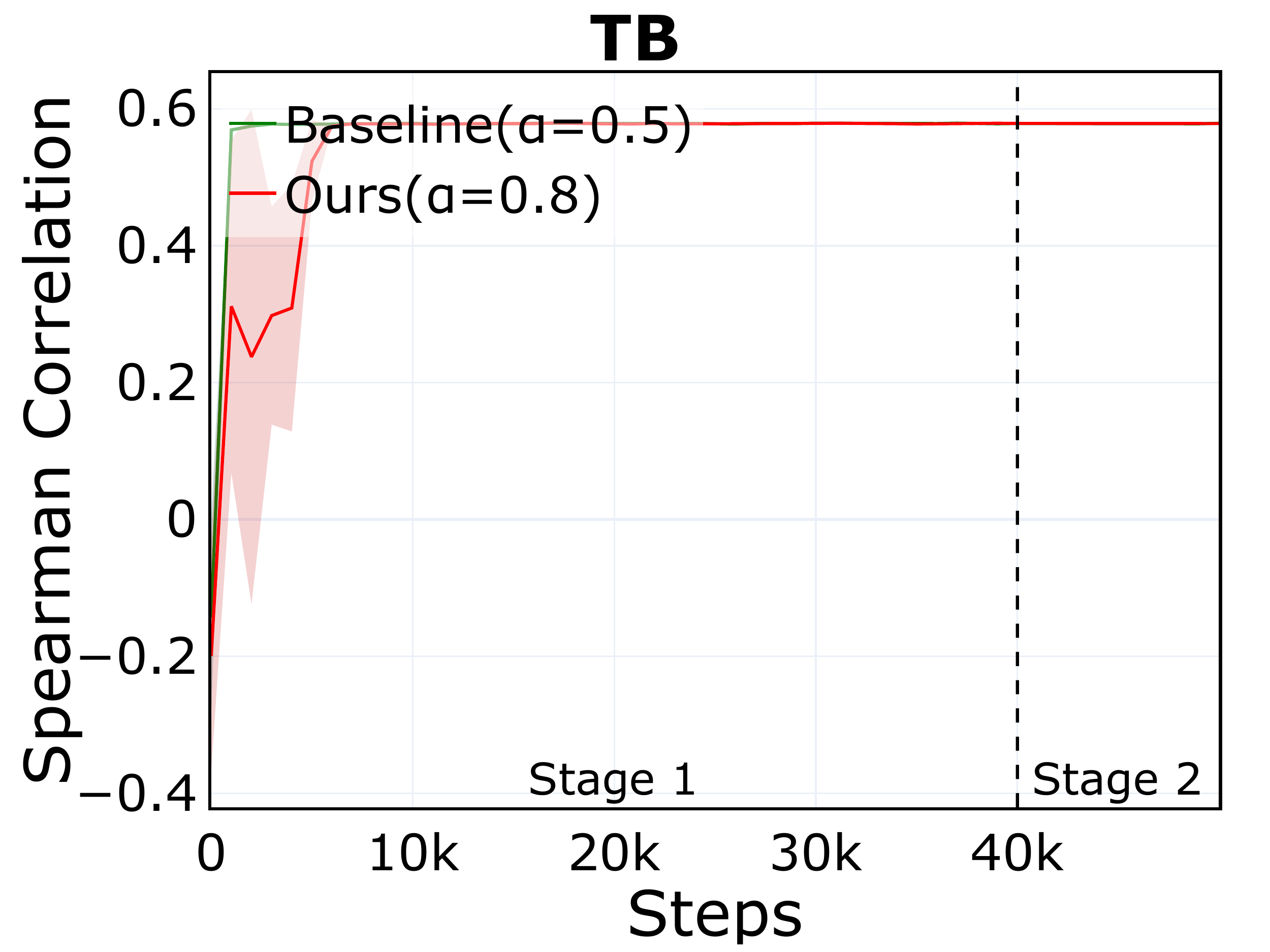} \\
  \end{tabular}
  \caption{\textbf{Spearman Correlation} vs Training Steps in \textbf{Bit Sequence Generation} across different objectives.}
  \label{fig:bit_metric_spearman_corr_test}
\end{figure}

\begin{figure}[htbp]
  \centering
  \setlength{\tabcolsep}{0pt}
  \begin{tabular}{@{}c@{\hspace{0pt}}c@{\hspace{0pt}}c@{\hspace{0pt}}c@{\hspace{0pt}}c@{}}
    \includegraphics[width=.2\textwidth]{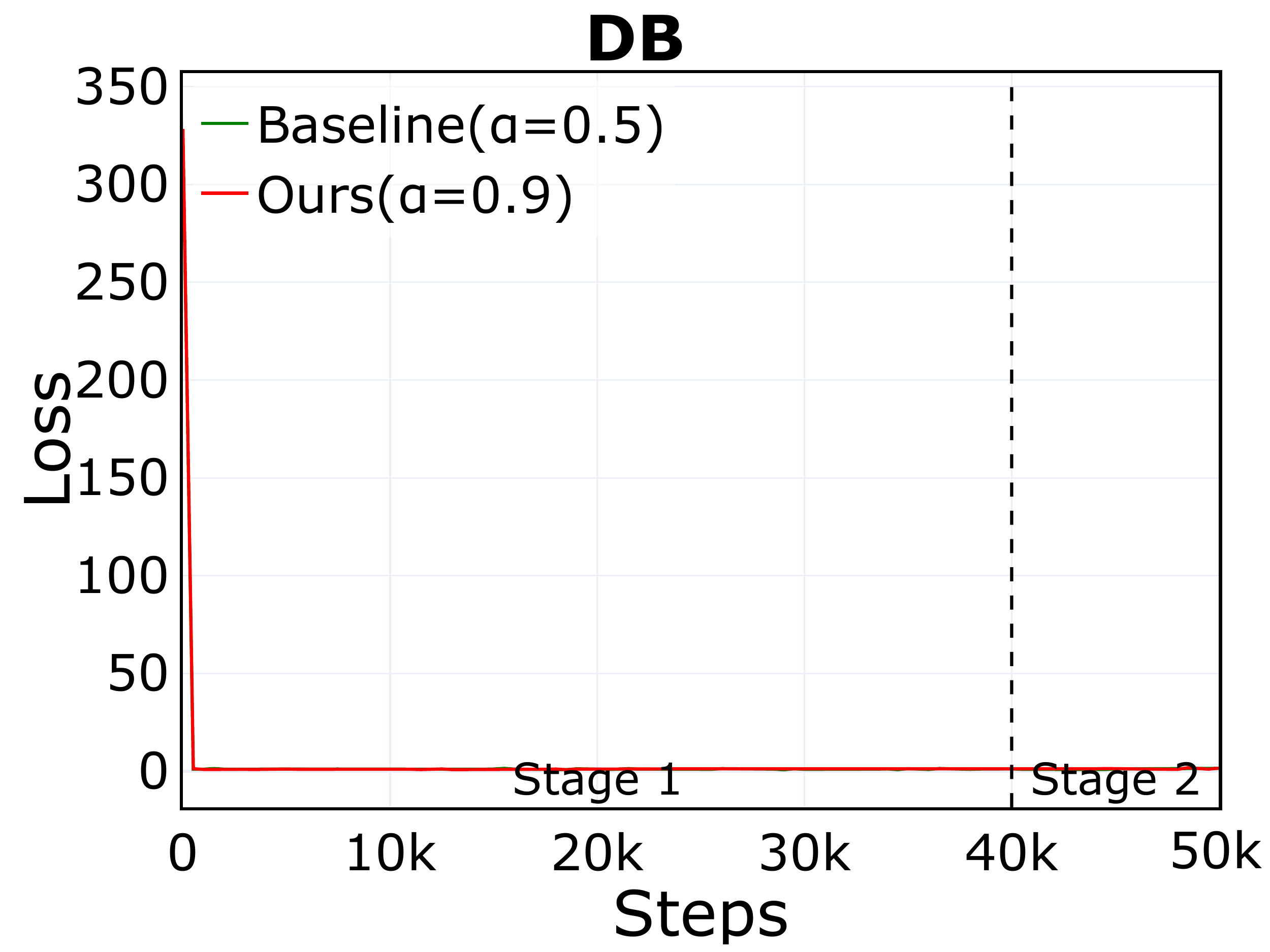} &
    \includegraphics[width=.2\textwidth]{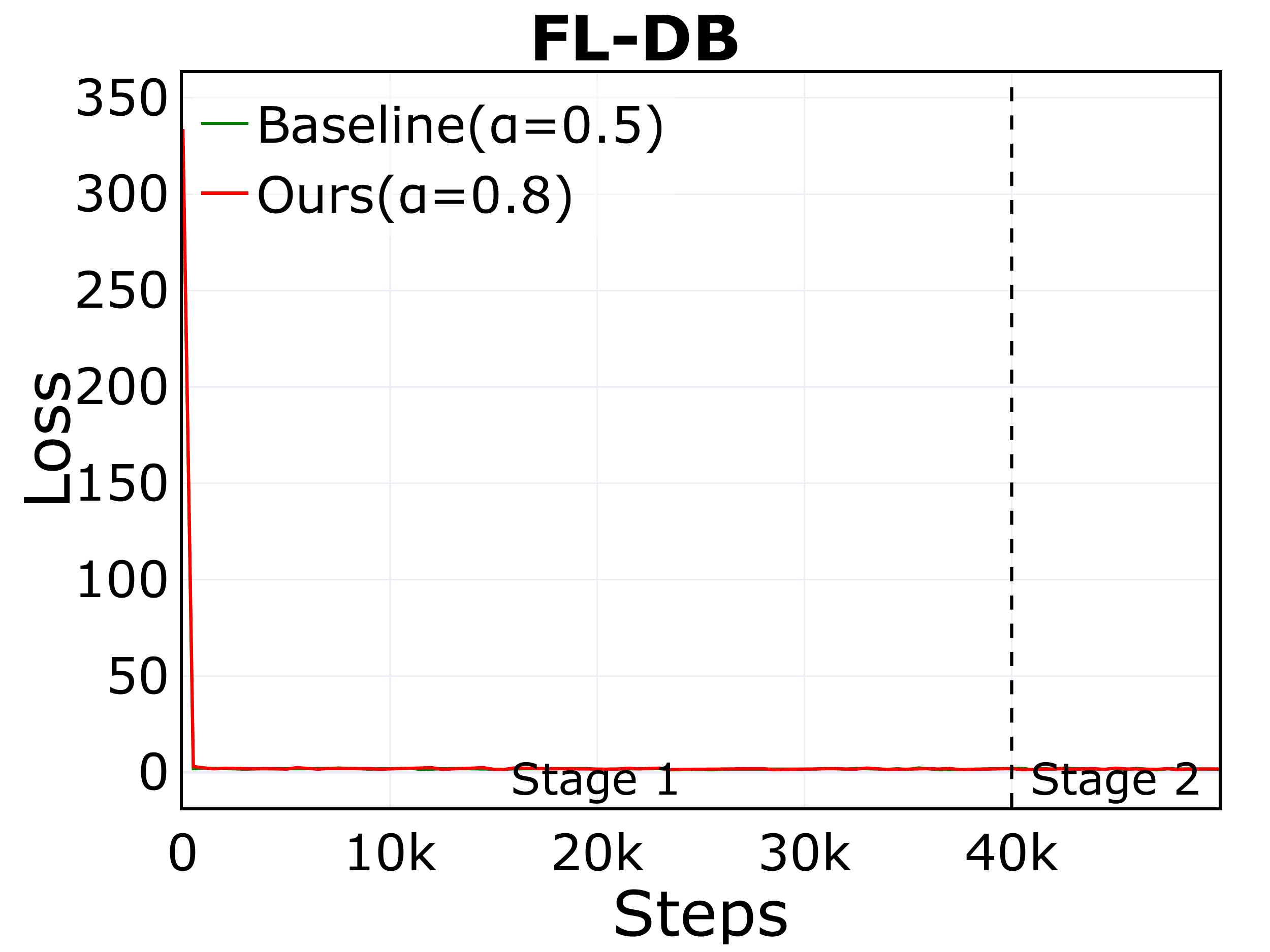} &
    \includegraphics[width=.2\textwidth]{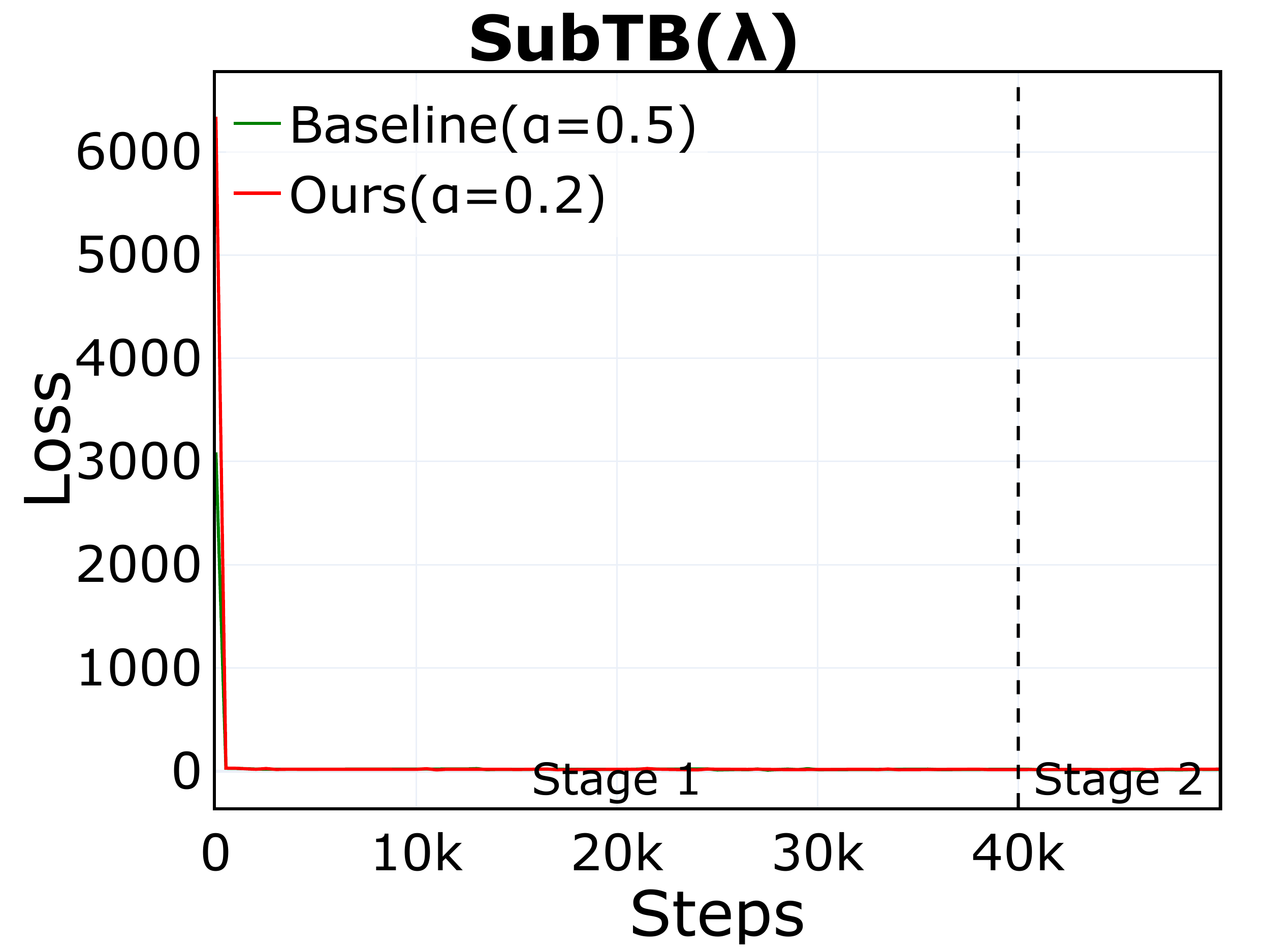} &
    \includegraphics[width=.2\textwidth]{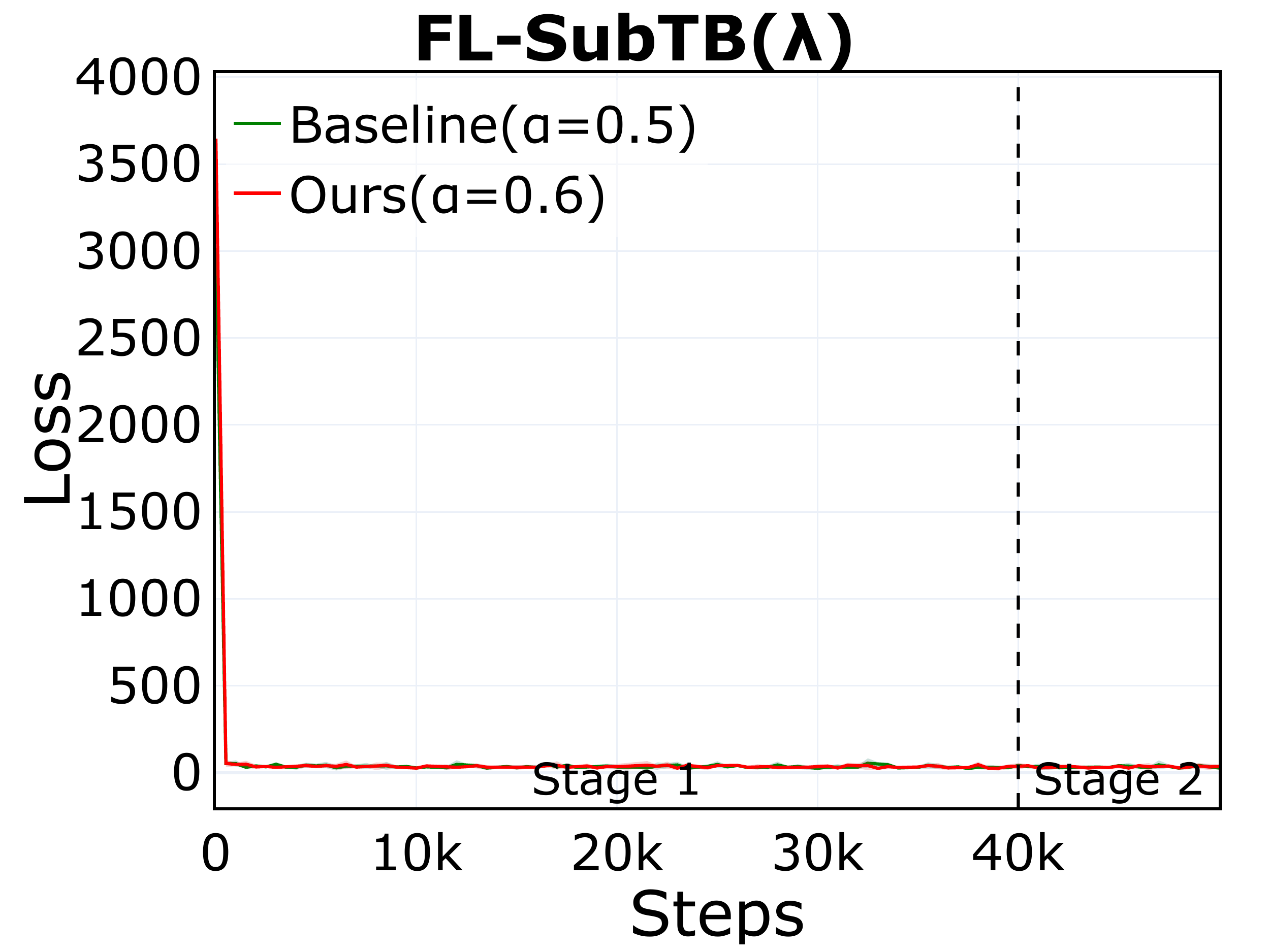} &
    \includegraphics[width=.2\textwidth]{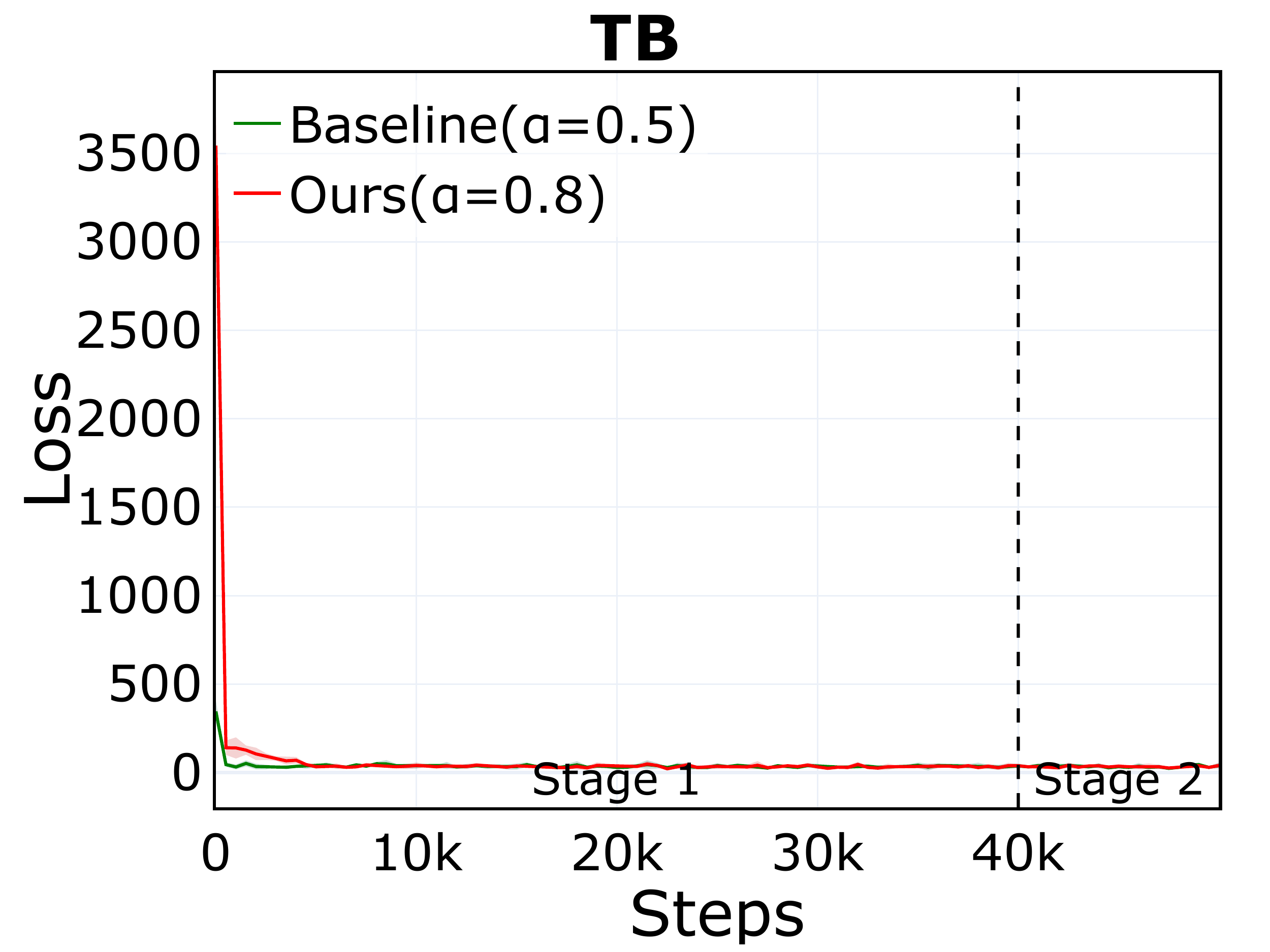} \\
  \end{tabular}
  \caption{\textbf{Loss} vs Training Steps in \textbf{Bit Sequence Generation} across different objectives.}
  \label{fig:bit_metric_loss}
\end{figure}

\begin{figure}[htbp]
  \centering
  \setlength{\tabcolsep}{0pt}
  \begin{tabular}{@{}c@{\hspace{0pt}}c@{\hspace{0pt}}c@{\hspace{0pt}}c@{\hspace{0pt}}c@{}}
    \includegraphics[width=.2\textwidth]{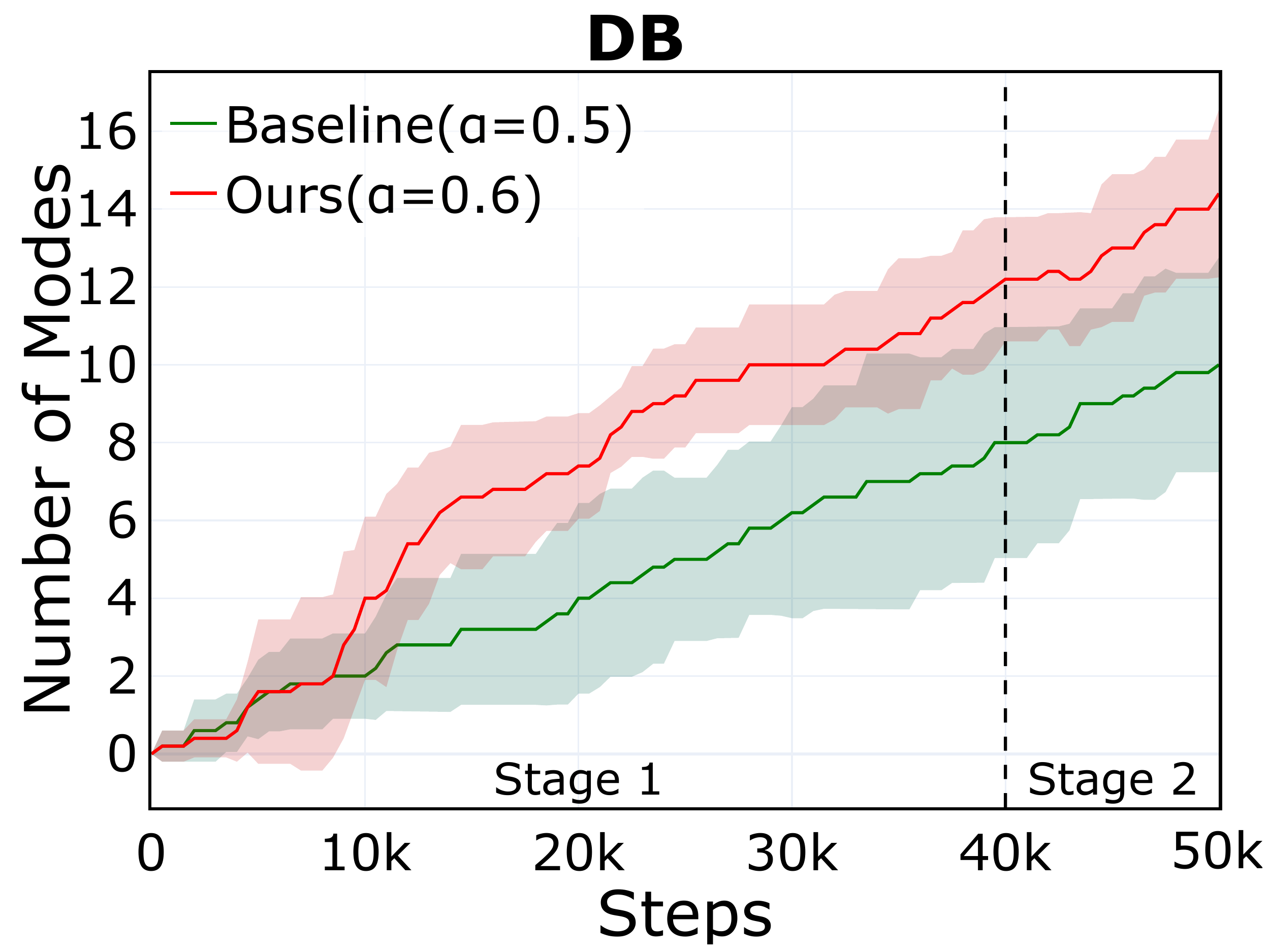} &
    \includegraphics[width=.2\textwidth]{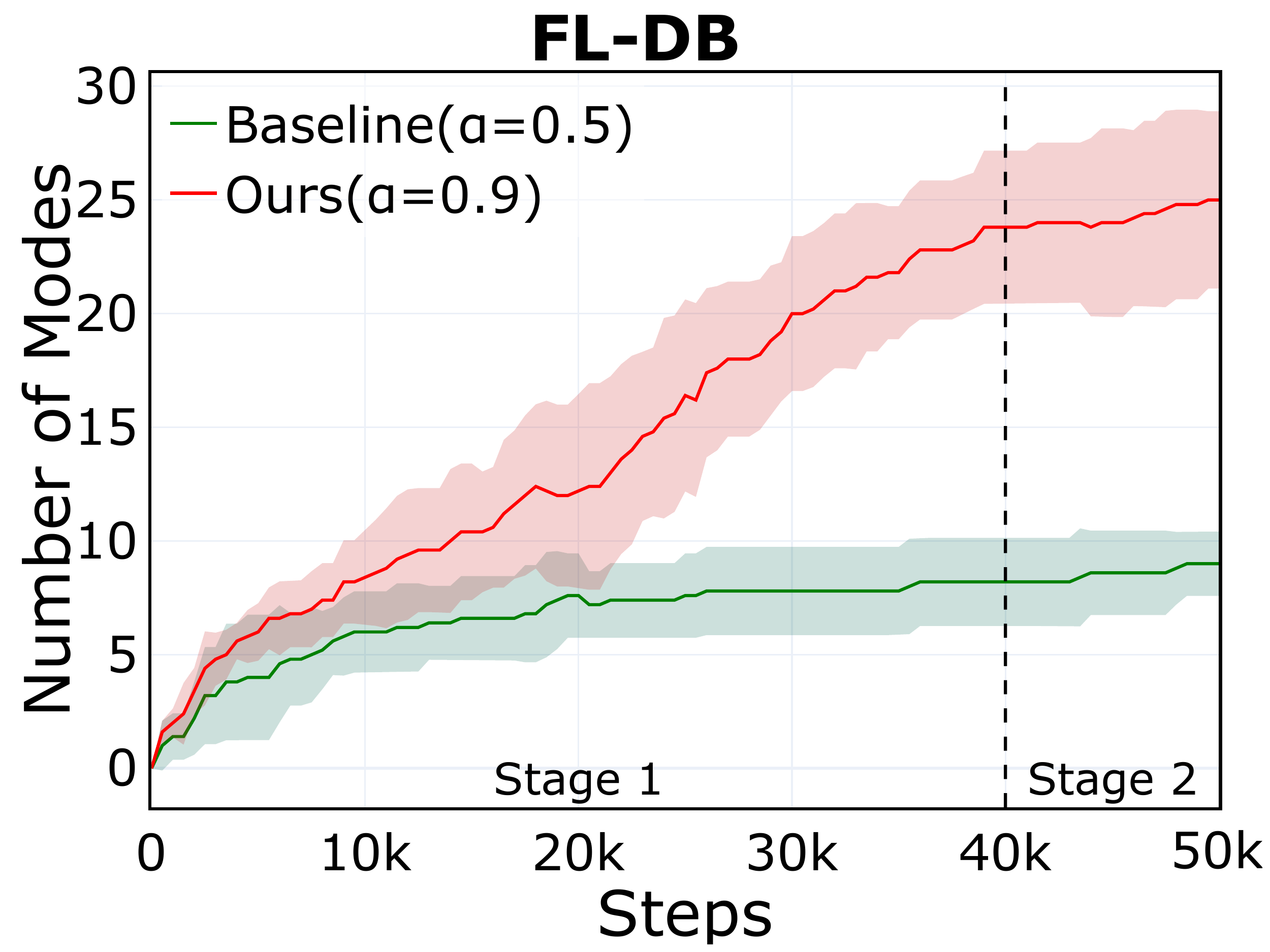} &
    \includegraphics[width=.2\textwidth]{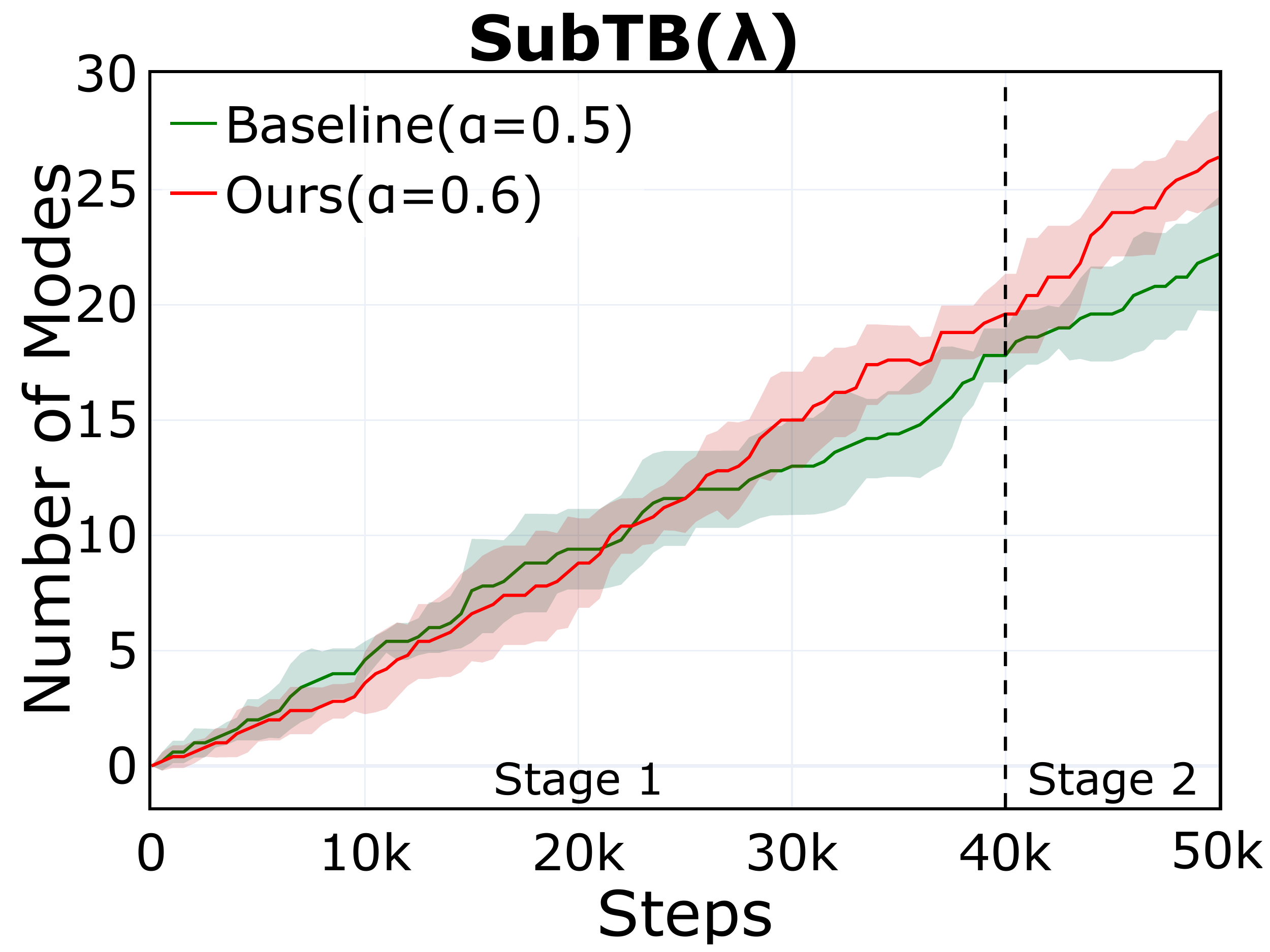} &
    \includegraphics[width=.2\textwidth]{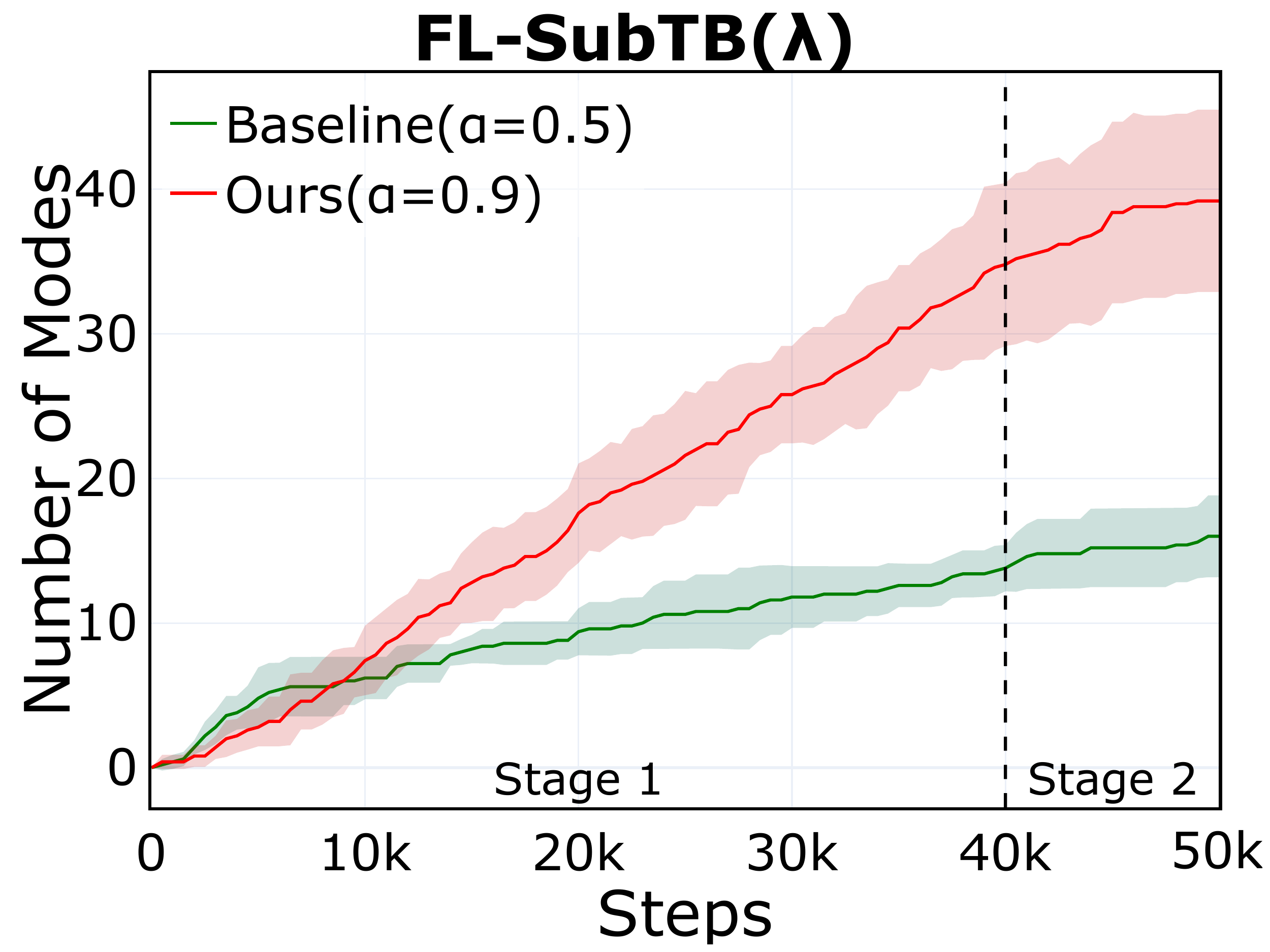} &
    \includegraphics[width=.2\textwidth]{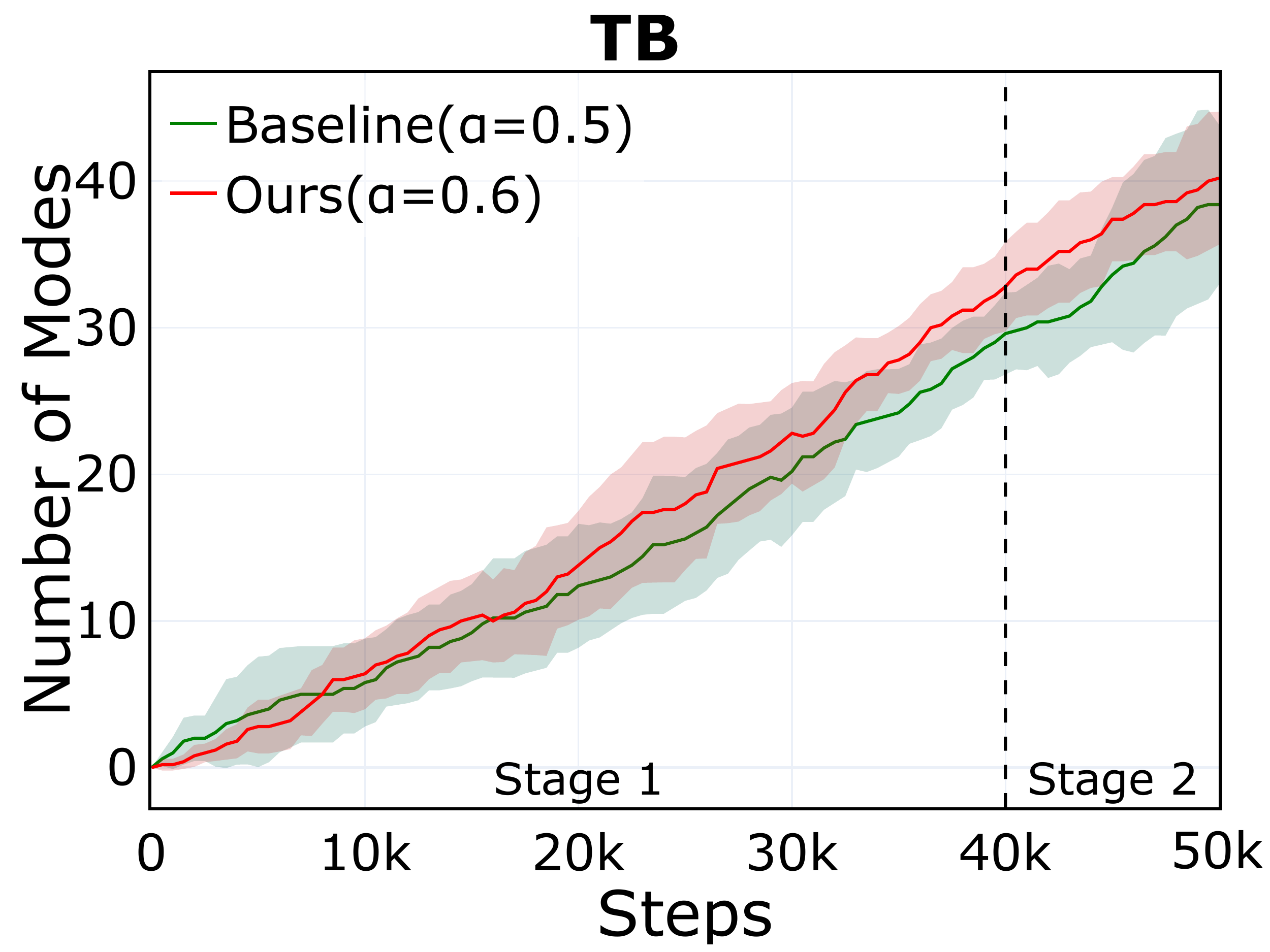} \\
  \end{tabular}
  \caption{\textbf{Number of Modes} vs Training Steps in \textbf{Molecule Generation} across different objectives.}
  \label{fig:mols_metric_num_modes_eval}
\end{figure}

\begin{figure}[htbp]
  \centering
  \setlength{\tabcolsep}{0pt}
  \begin{tabular}{@{}c@{\hspace{0pt}}c@{\hspace{0pt}}c@{\hspace{0pt}}c@{\hspace{0pt}}c@{}}
    \includegraphics[width=.2\textwidth]{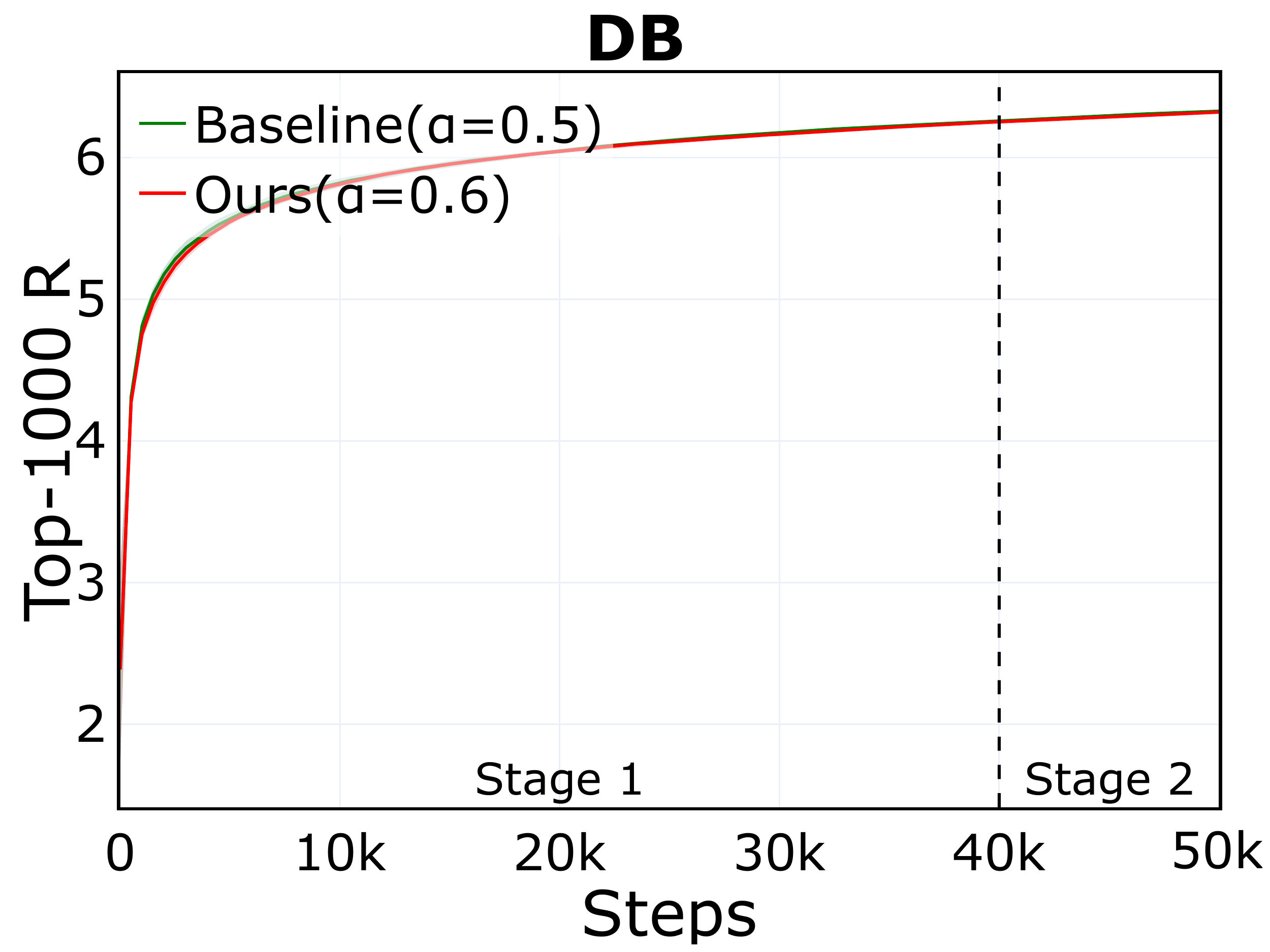} &
    \includegraphics[width=.2\textwidth]{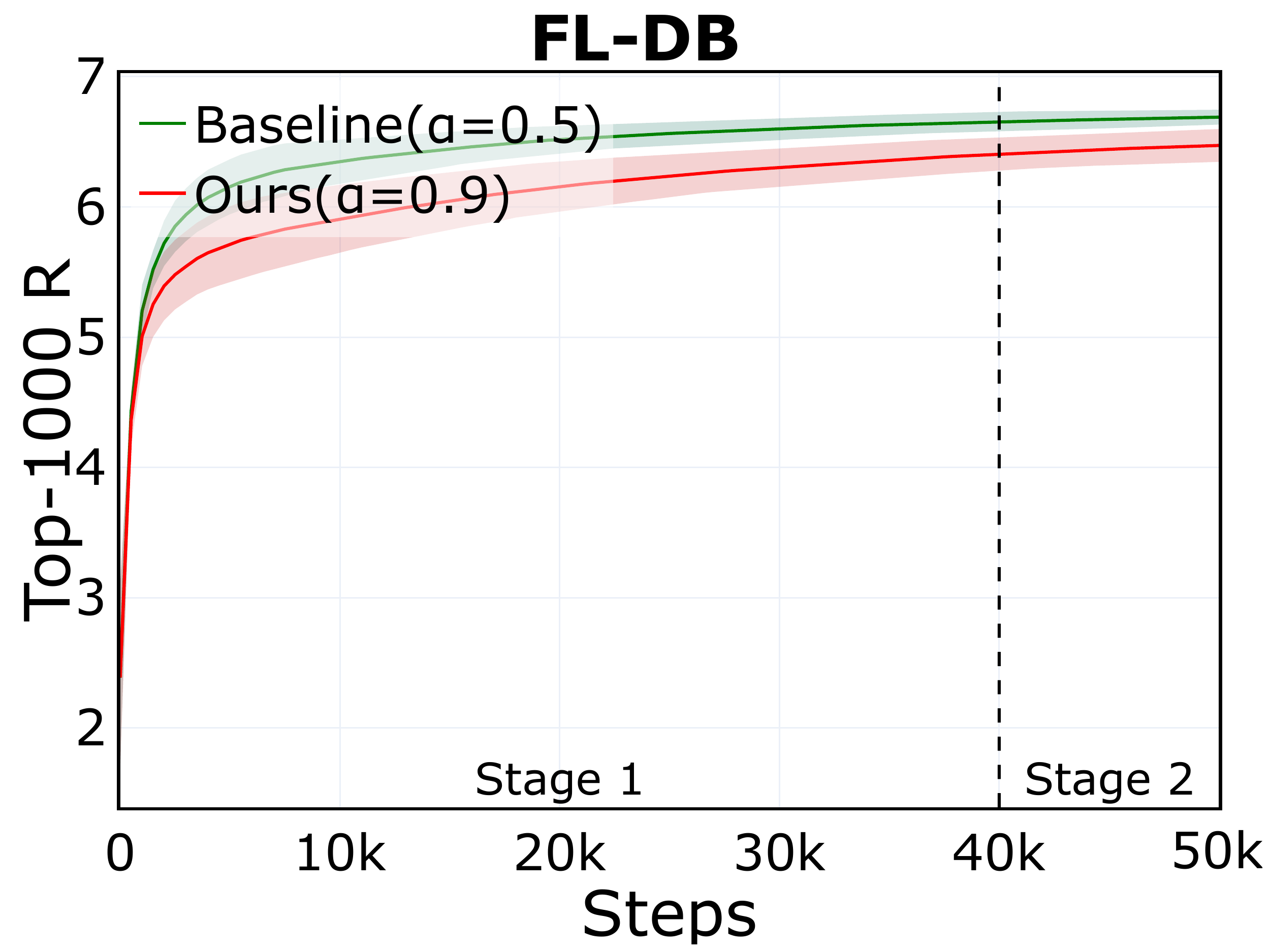} &
    \includegraphics[width=.2\textwidth]{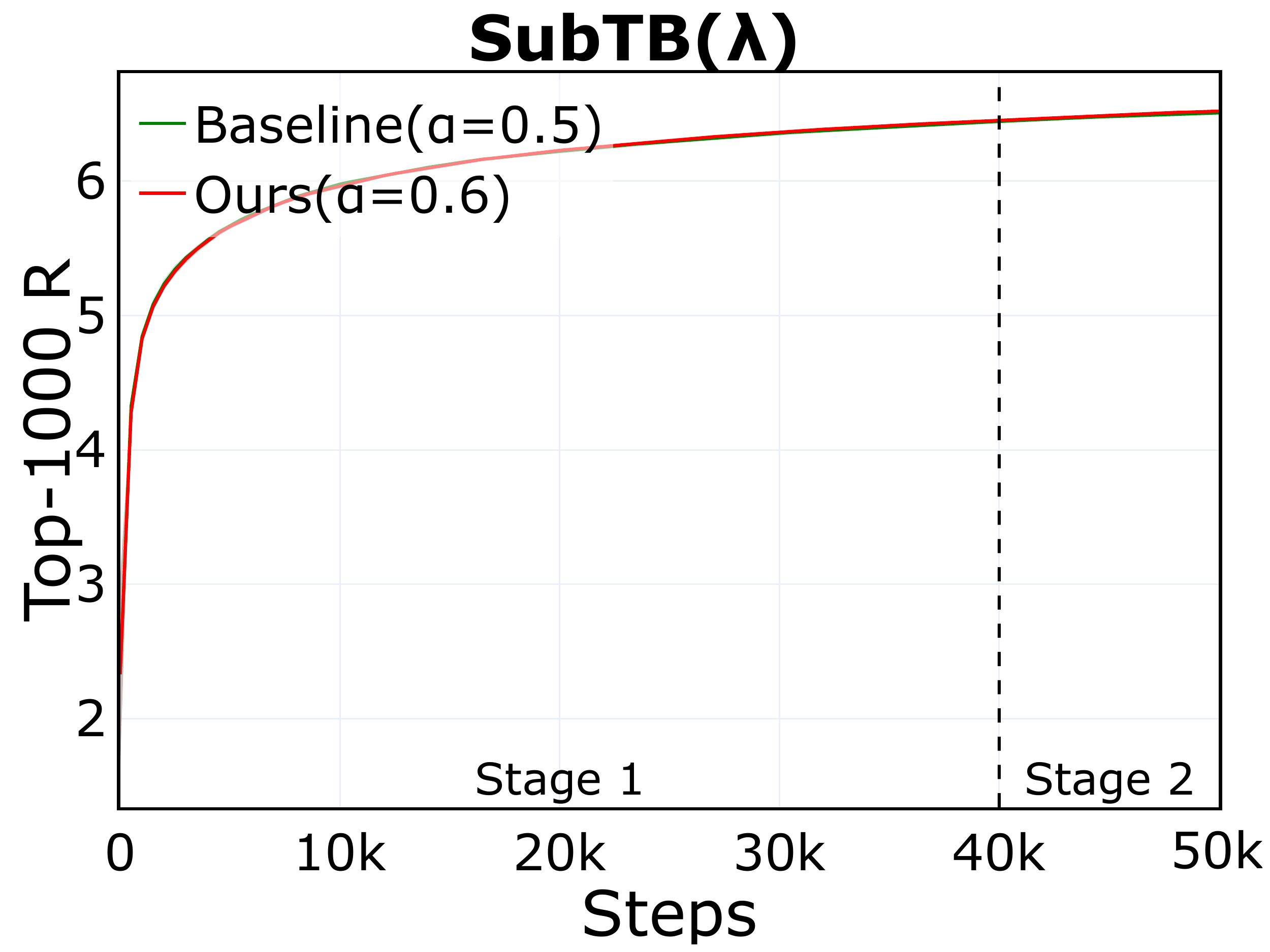} &
    \includegraphics[width=.2\textwidth]{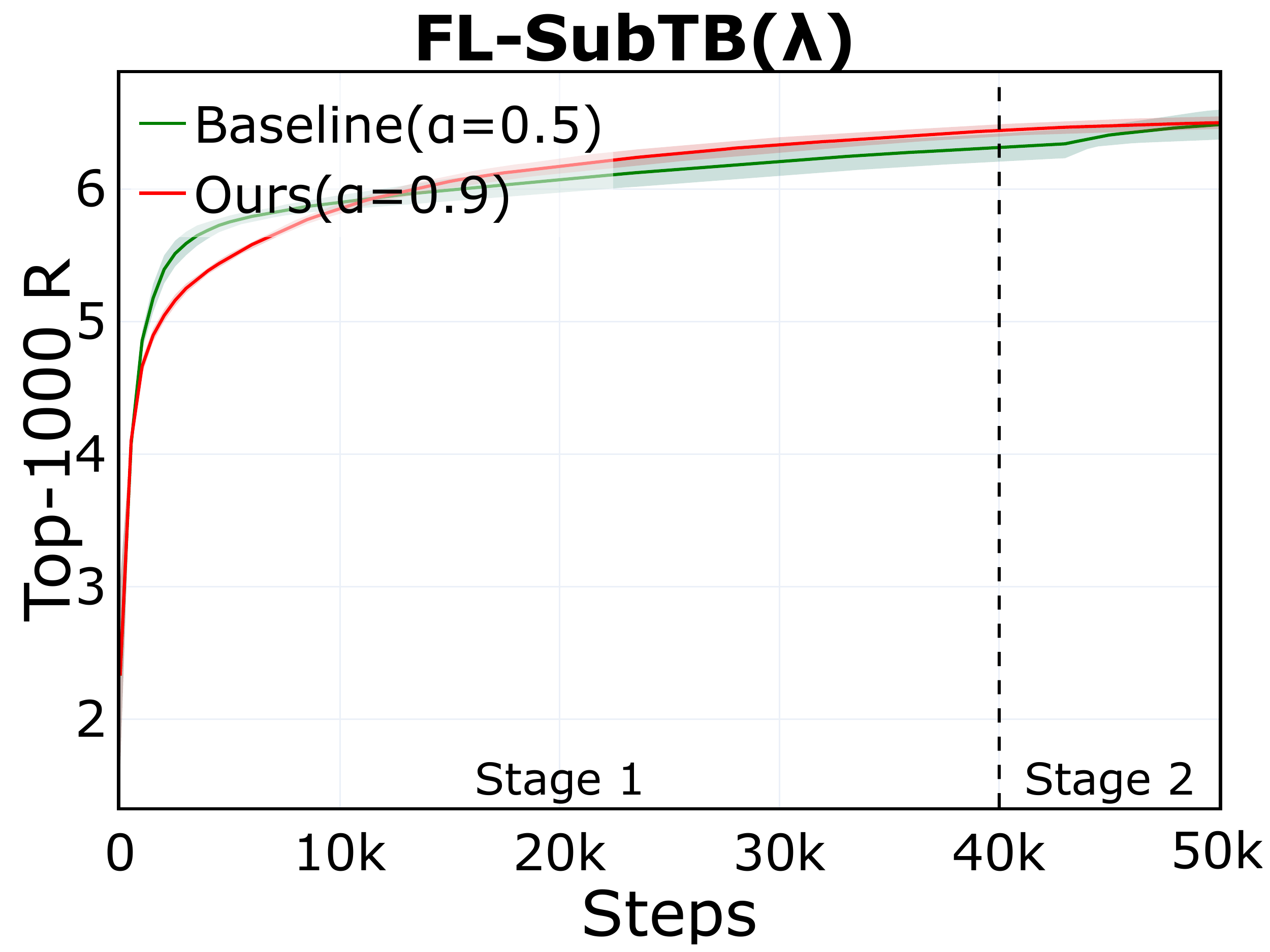} &
    \includegraphics[width=.2\textwidth]{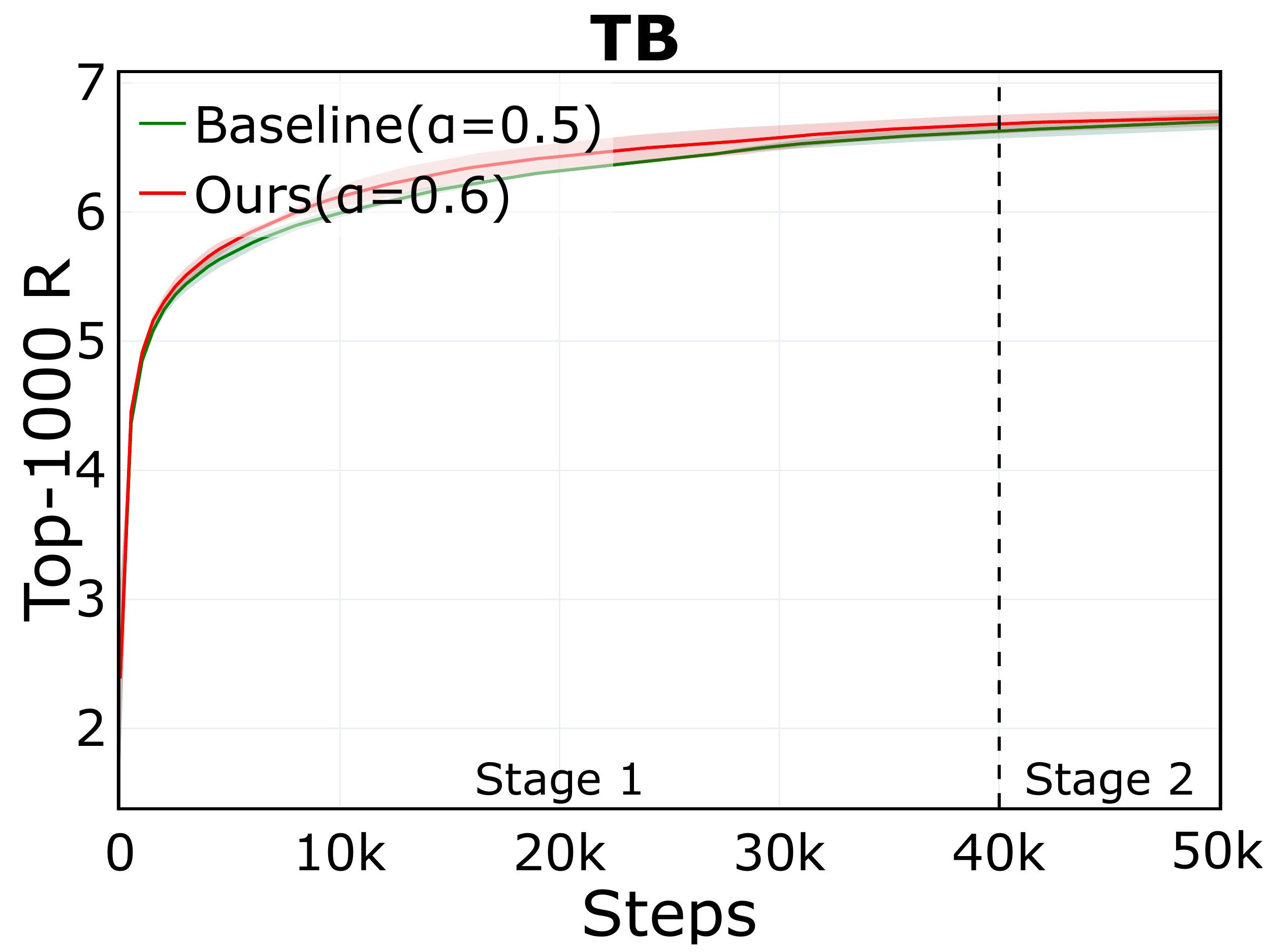} \\
  \end{tabular}
  \caption{\textbf{Top-1000 R} vs Training Steps in \textbf{Molecule Generation} across different objectives.}
  \label{fig:mols_metric_top_100_avg_reward_eval}
\end{figure}

\begin{figure}[htbp]
  \centering
  \setlength{\tabcolsep}{0pt}
  \begin{tabular}{@{}c@{\hspace{85pt}}c@{\hspace{85pt}}c@{}}
    \includegraphics[width=.2\textwidth]{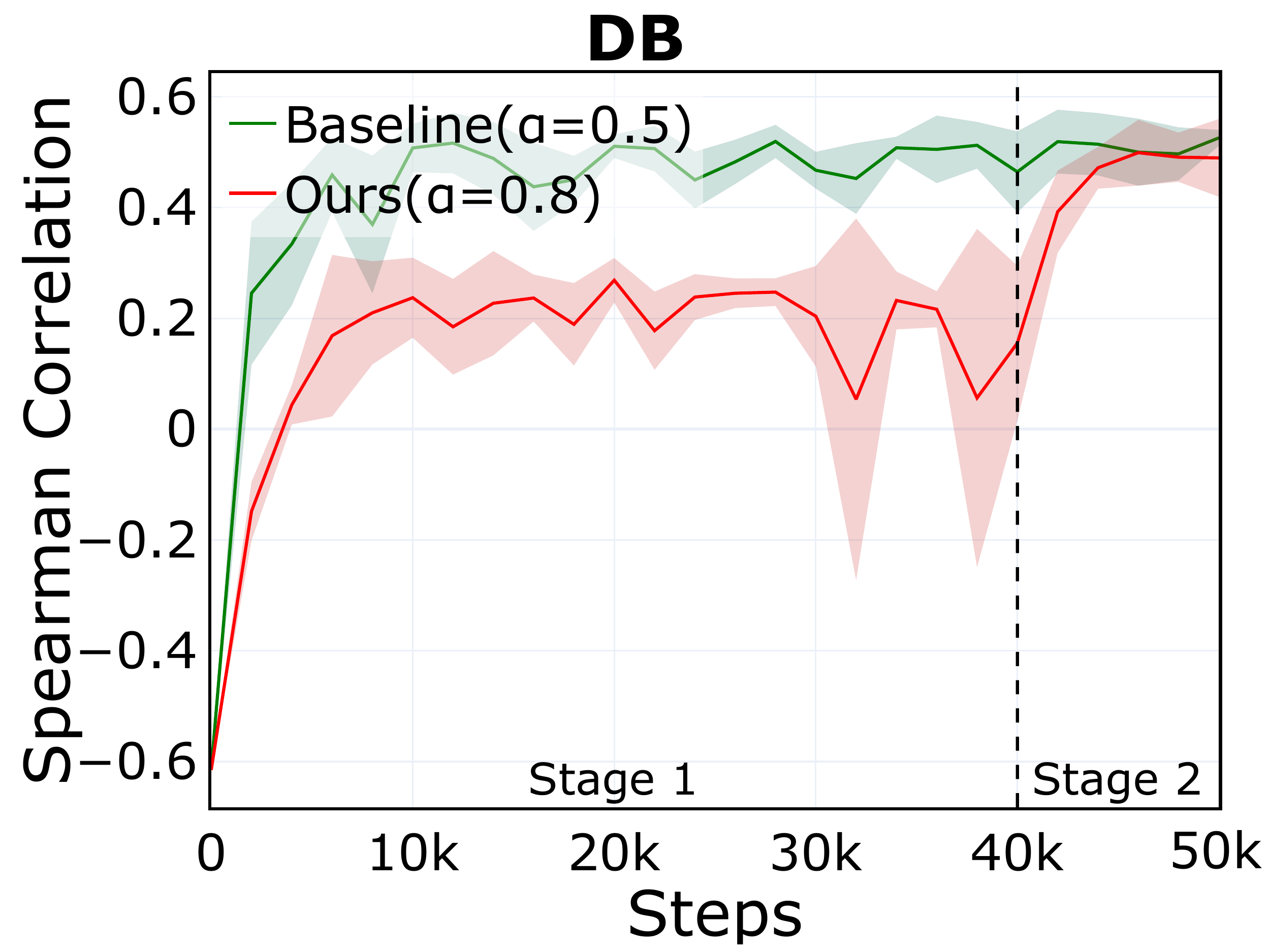} &
    \includegraphics[width=.2\textwidth]{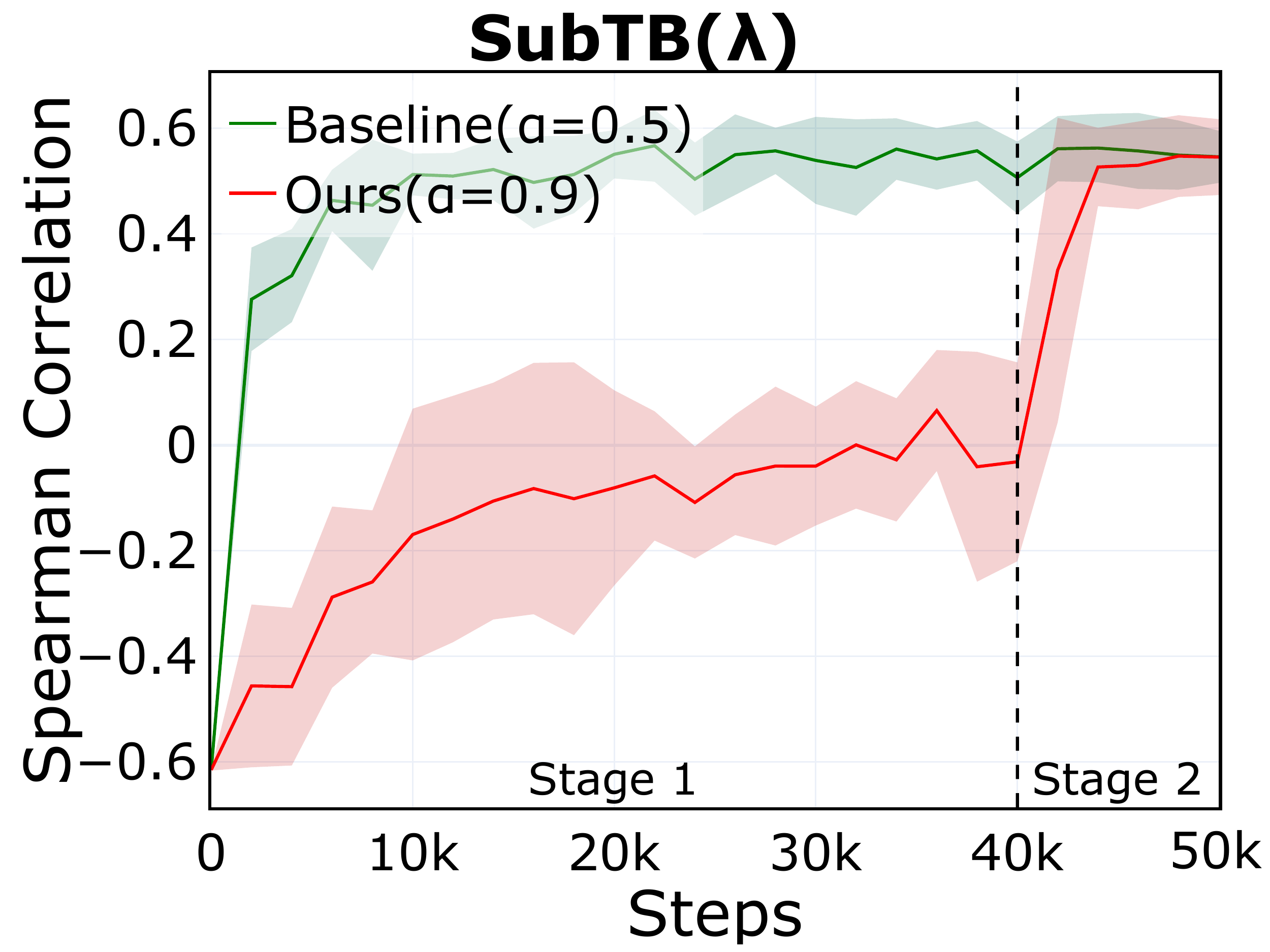} &
    \includegraphics[width=.2\textwidth]{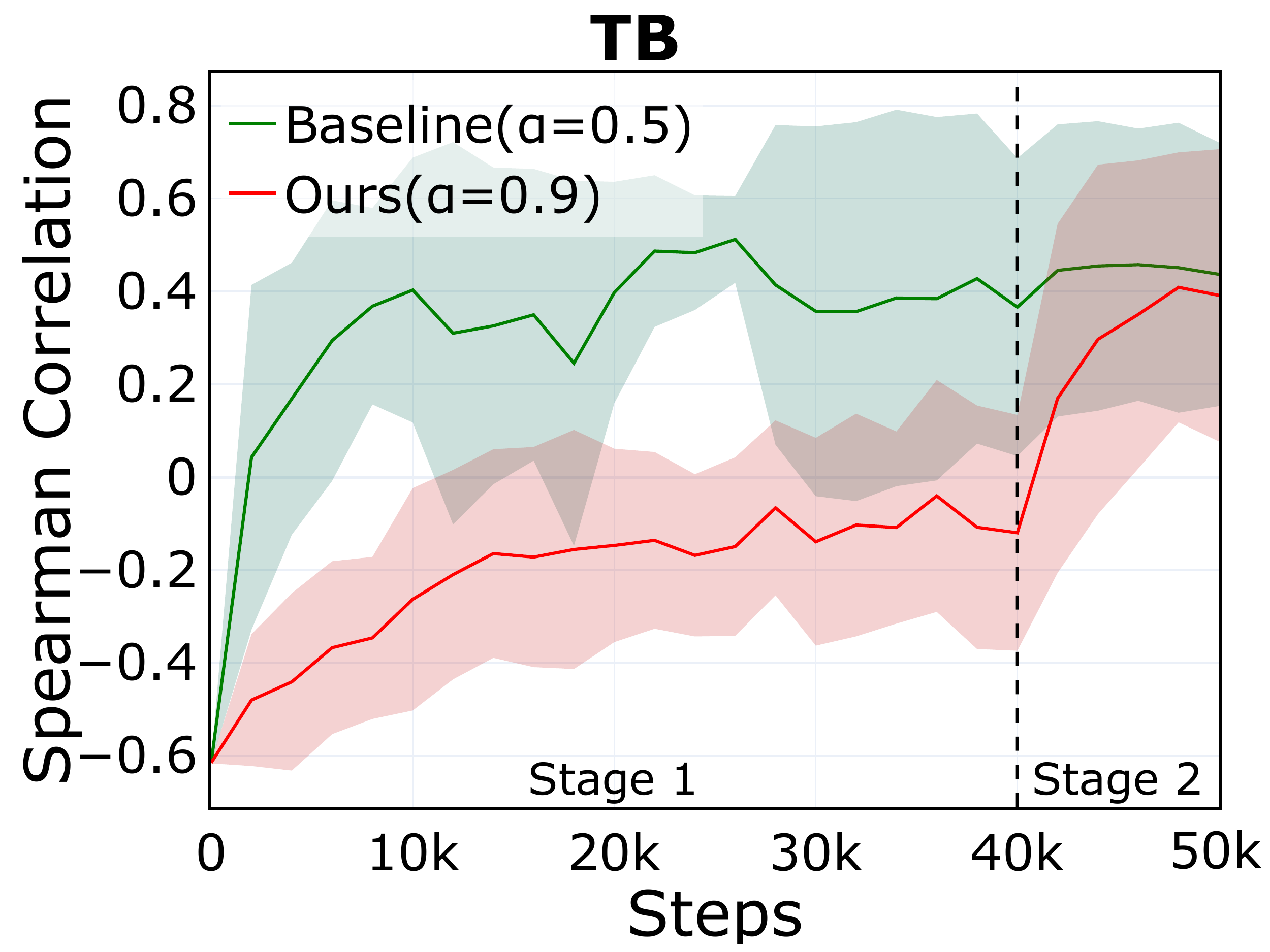} \\
  \end{tabular}
  \caption{\textbf{Spearman Correlation} vs Training Steps in \textbf{Molecule Generation} across different objectives. FL-DB and FL-SubTB($\lambda$) are omitted due to their biased target~\citep{silva2025when}.}
  \label{fig:mols_metric_spearman_corr_test}
\end{figure}

\begin{figure}[t]
  \centering
  \setlength{\tabcolsep}{0pt}
  \begin{tabular}{@{}c@{\hspace{0pt}}c@{\hspace{0pt}}c@{\hspace{0pt}}c@{\hspace{0pt}}c@{}}
    \includegraphics[width=.2\textwidth]{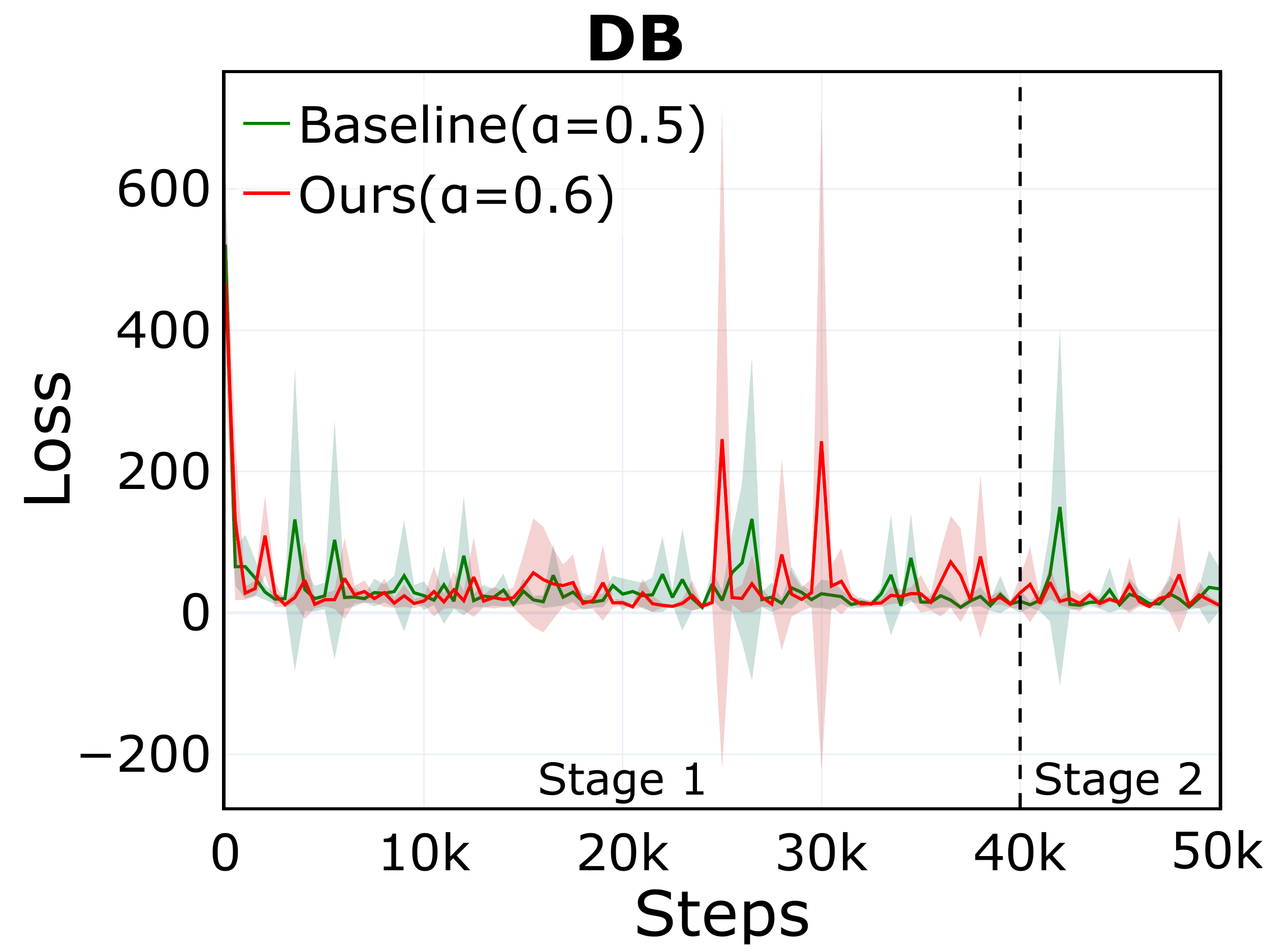} &
    \includegraphics[width=.2\textwidth]{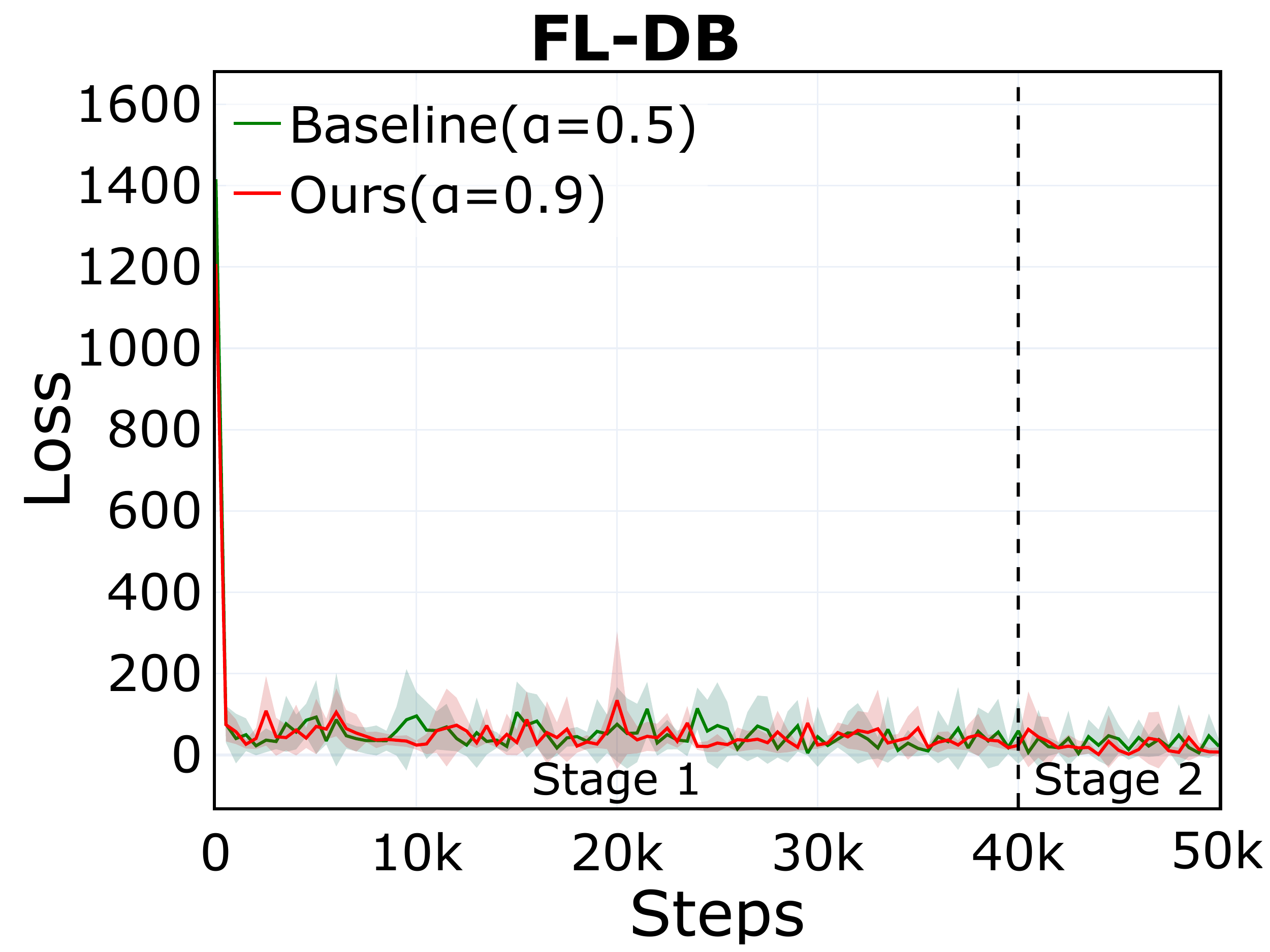} &
    \includegraphics[width=.2\textwidth]{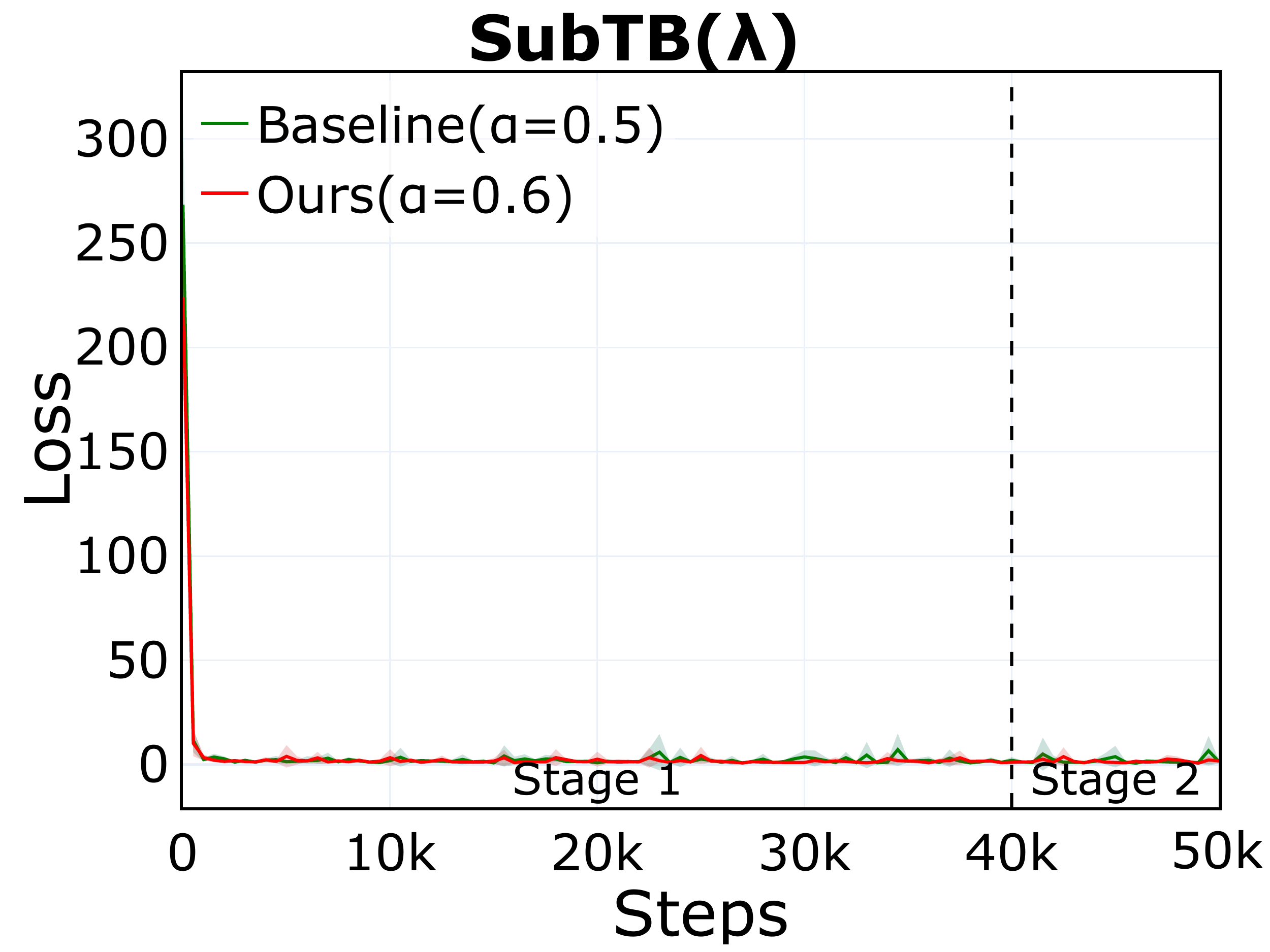} &
    \includegraphics[width=.2\textwidth]{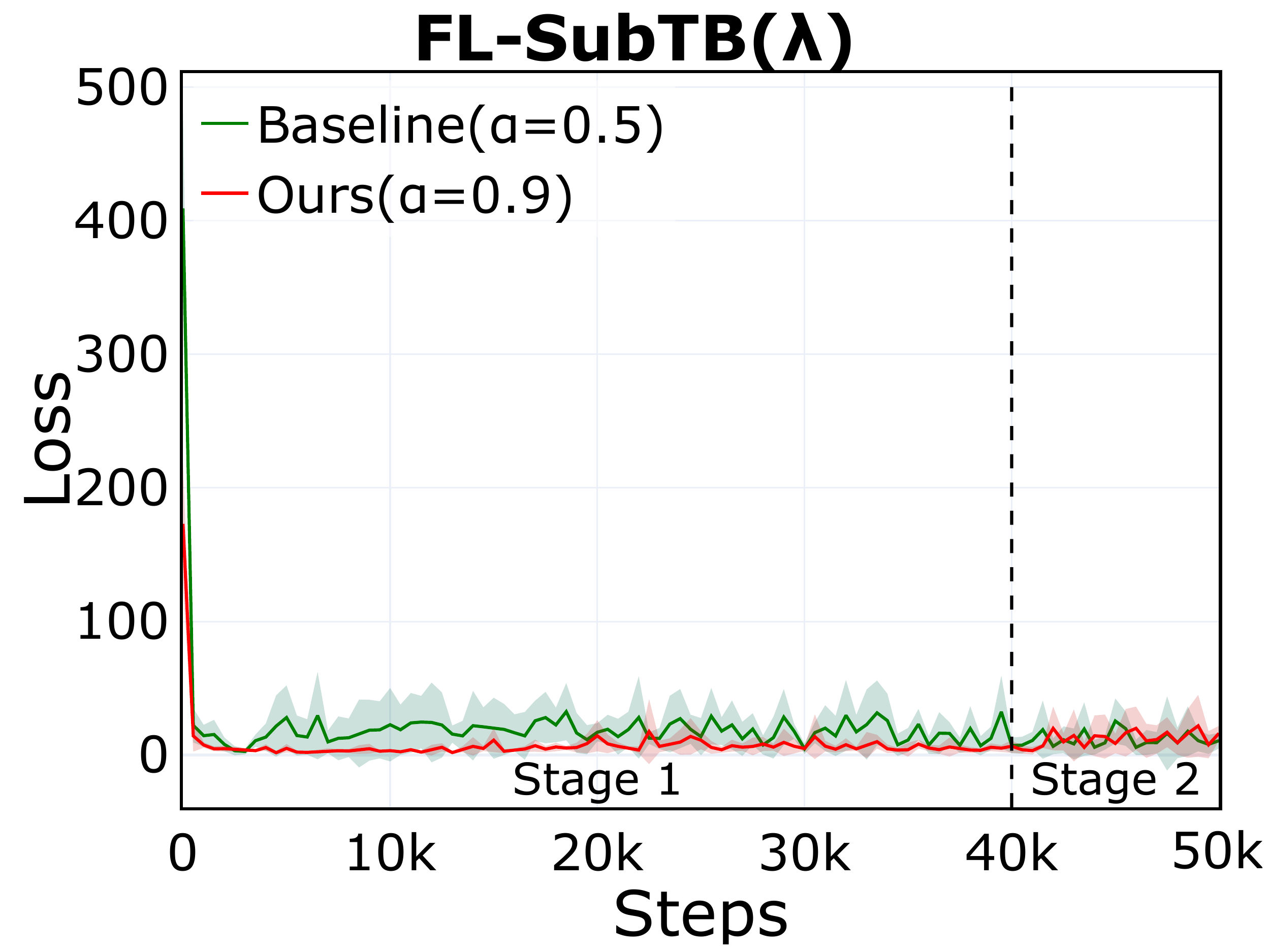} &
    \includegraphics[width=.2\textwidth]{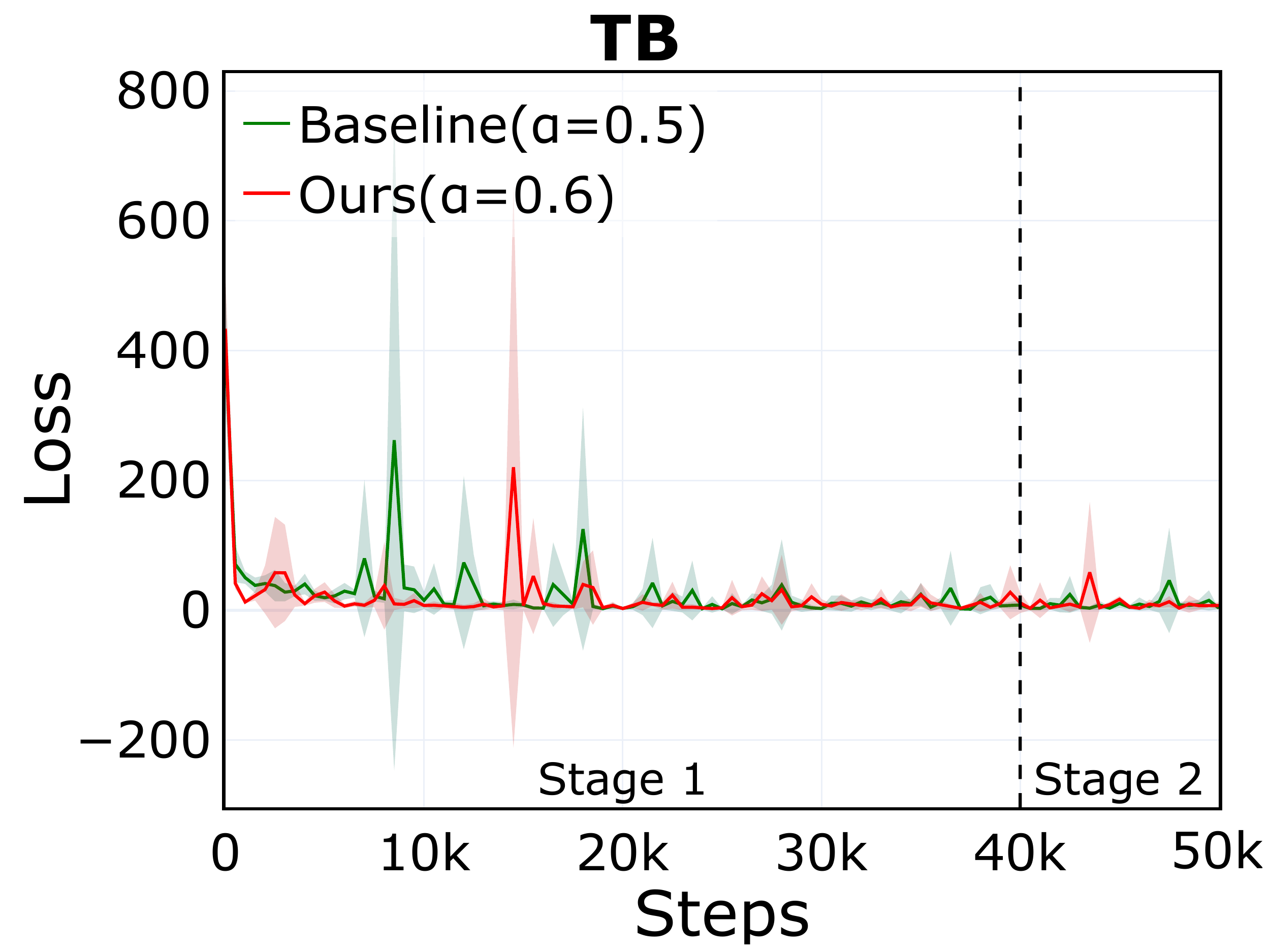} \\
  \end{tabular}
  \caption{\textbf{Loss} vs Training Steps in \textbf{Molecule Generation} across different objectives.}
  \label{fig:mols_metric_current_loss}
\end{figure}

\begin{figure}[htbp]
  \centering
  \setlength{\tabcolsep}{0pt}
  \begin{tabular}{@{}c@{\hspace{0pt}}c@{\hspace{0pt}}c@{\hspace{0pt}}c@{\hspace{0pt}}c@{}}
    \includegraphics[width=.2\textwidth]{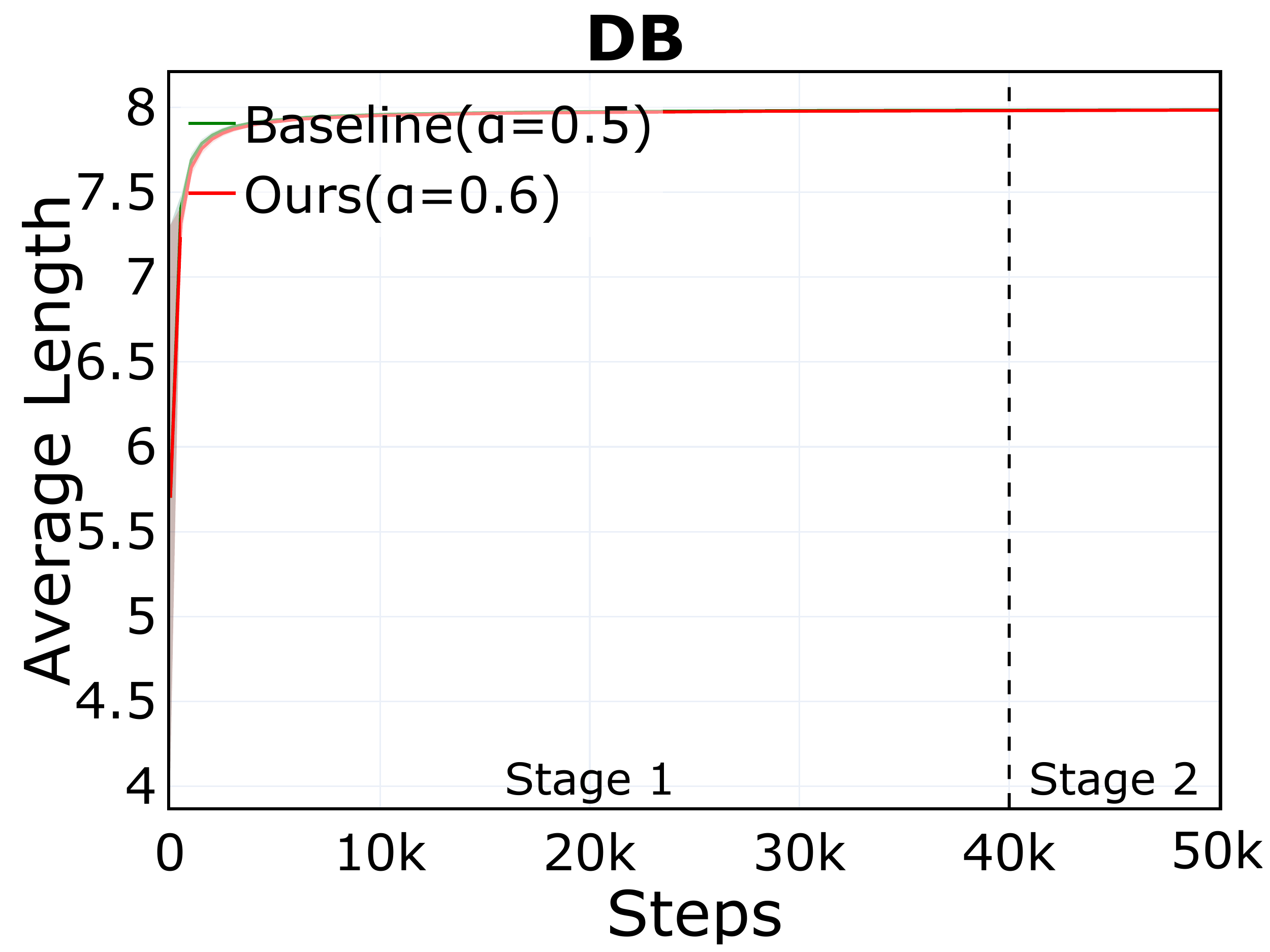} &
    \includegraphics[width=.2\textwidth]{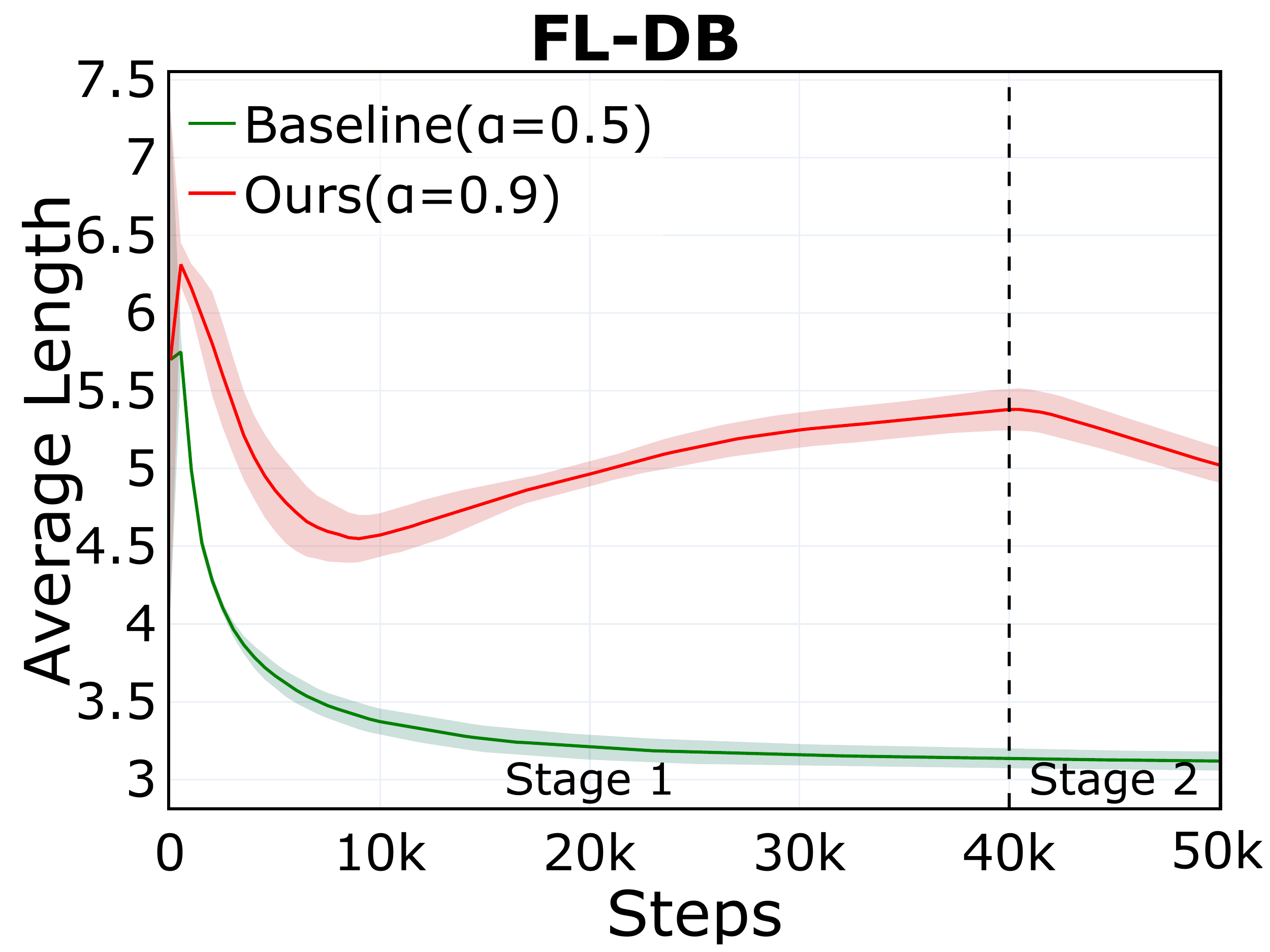} &
    \includegraphics[width=.2\textwidth]{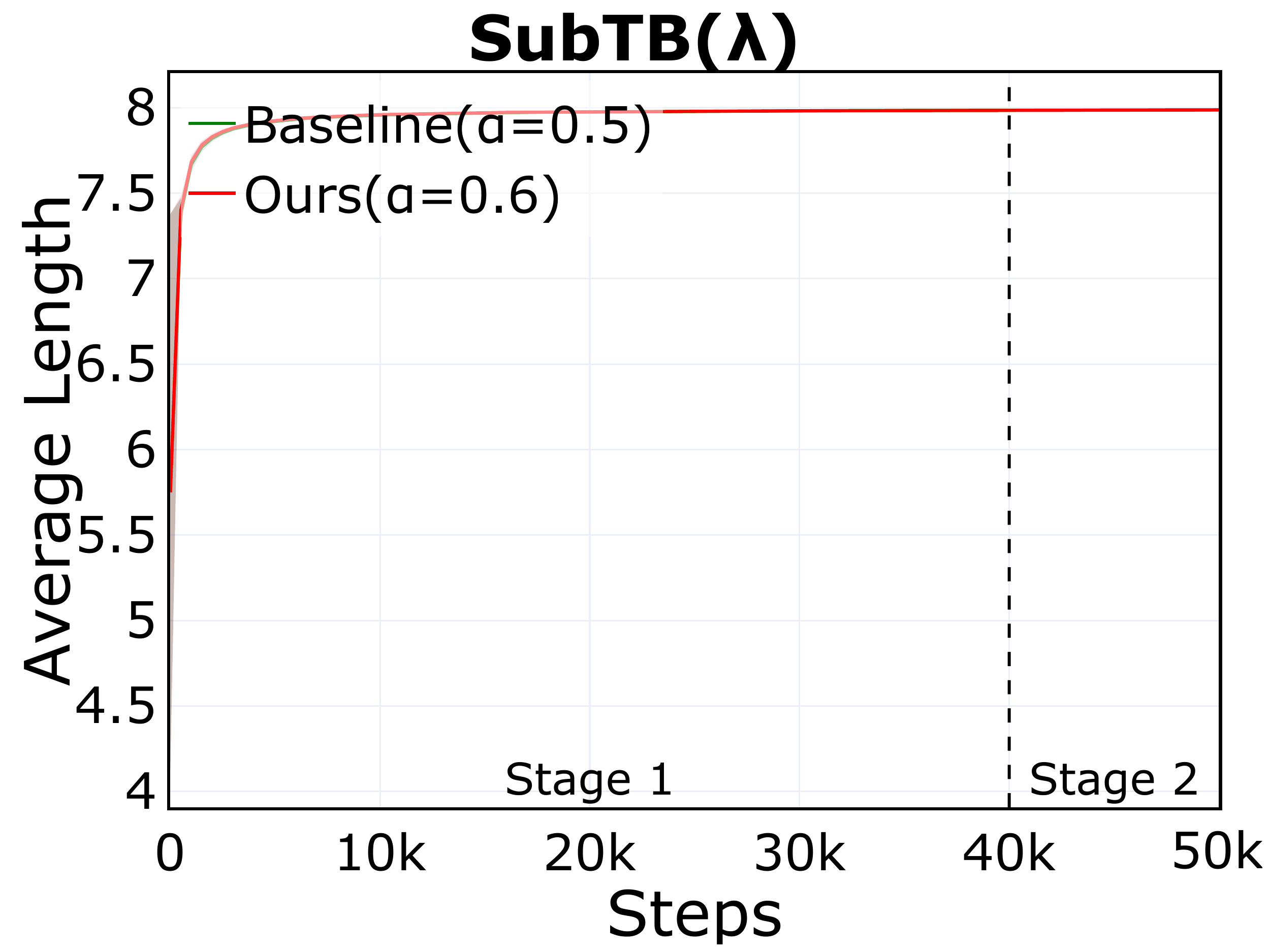} &
    \includegraphics[width=.2\textwidth]{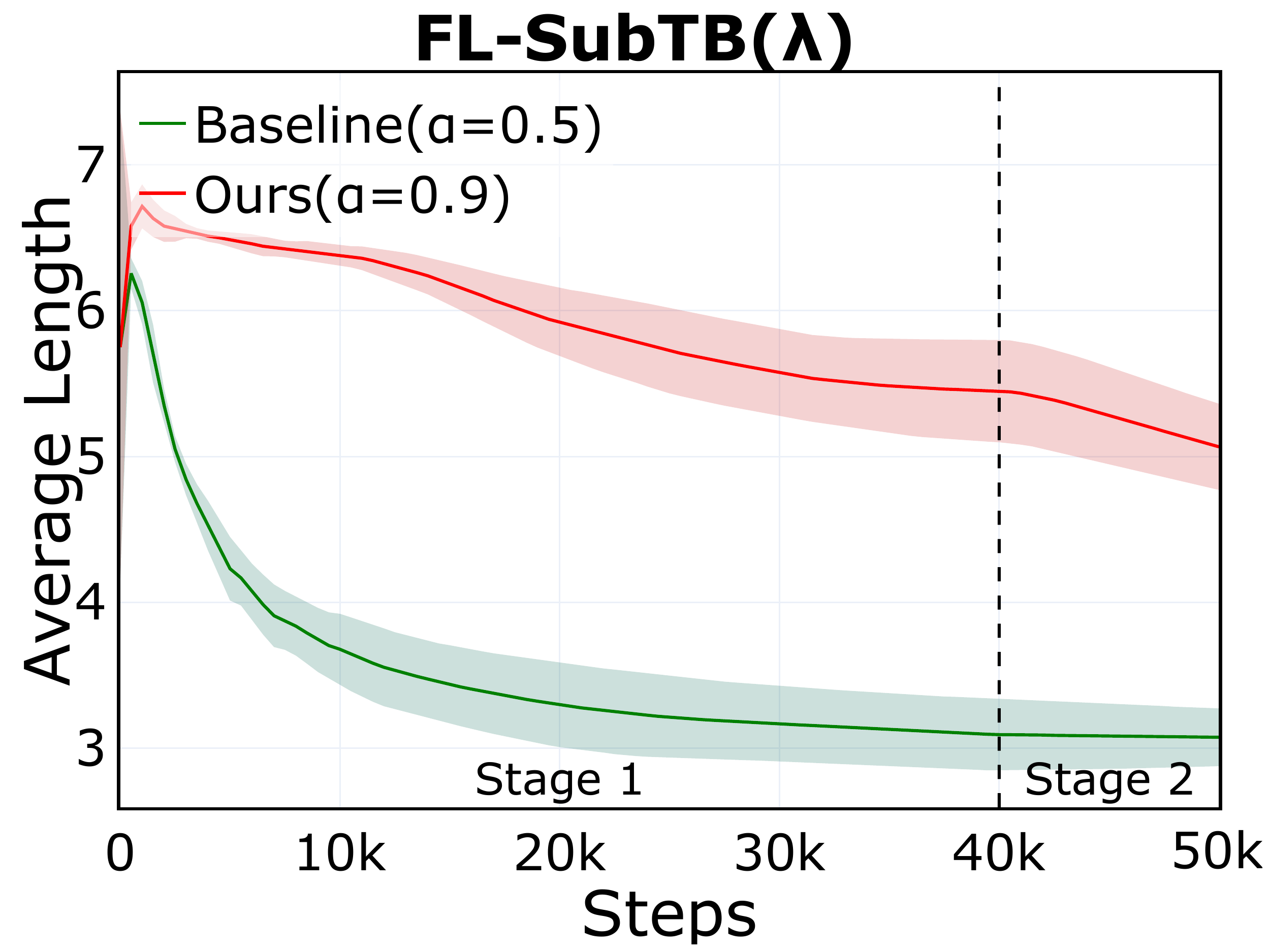} &
    \includegraphics[width=.2\textwidth]{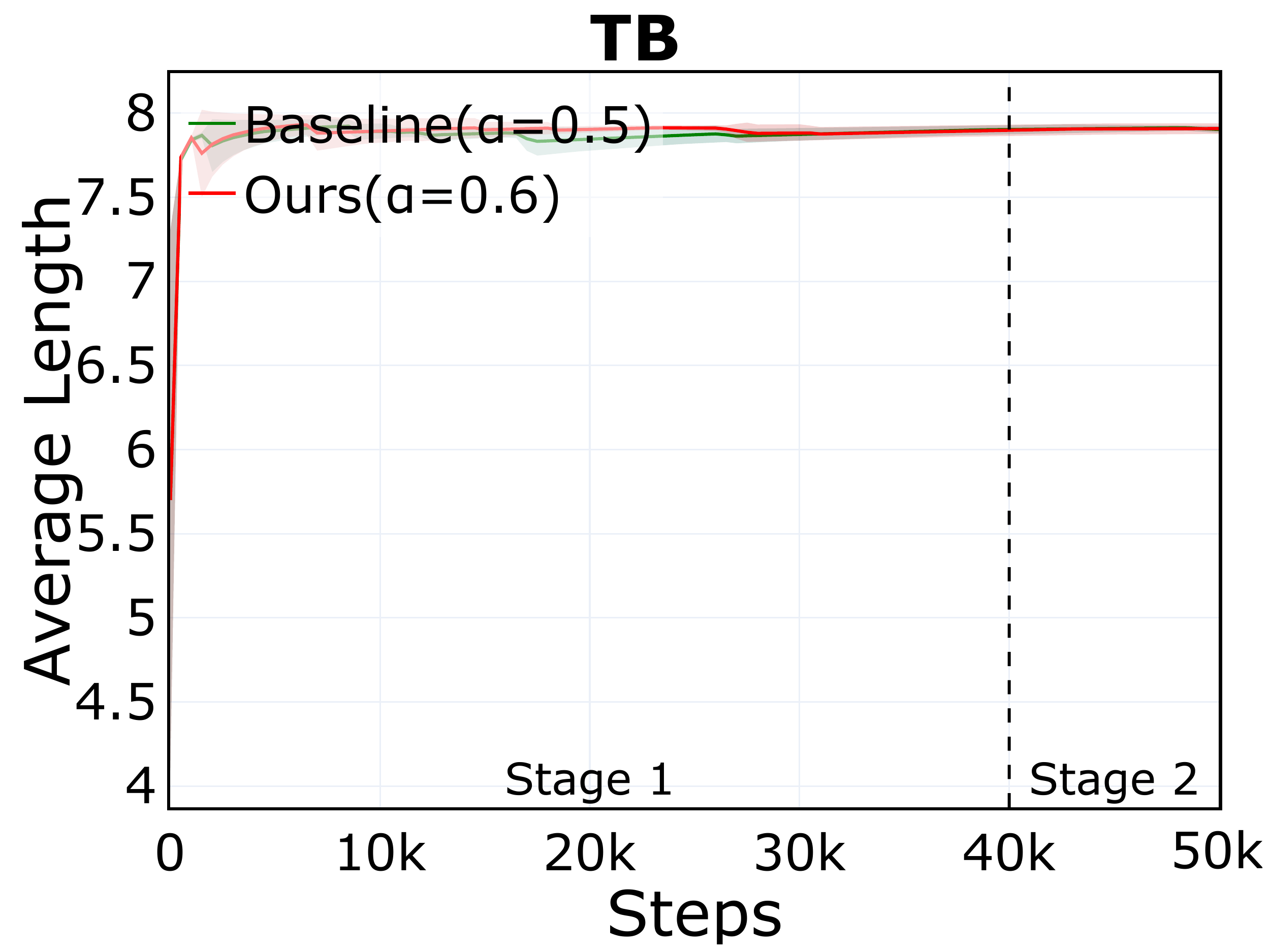} \\
  \end{tabular}
  \caption{\textbf{Average Sample Length} vs Training Steps in \textbf{Molecule Generation} across different objectives.}
  \label{fig:mols_metric_all_samples_avg_length_eval}
\end{figure}

\newpage

\subsection{Compatibility with Prior Methods and Versatility across Domains}
\label{app:compatibility}

To evaluate the broader utility of $\alpha$-GFNs, we demonstrate that our approach is not only compatible with various state-of-the-art training recipes but also provides a versatile mechanism for modulating the exploration-exploitation trade-off in diverse domains.

\paragraph{Compatibility with Existing GFlowNet Enhancements.}
We first evaluate whether $\alpha$-GFN provides performance gains when integrated with established training frameworks. Rather than competing with existing optimizations, $\alpha$-GFN is designed to be orthogonal to them. To demonstrate this, we incorporate our objective into several representative state-of-the-art recipes, such as Adaptive Teachers~\citep{kim2024adaptive} and QGFN~\citep{lau2024qgfn}.
Specifically, for Adaptive Teachers, we use $\alpha$-TB in the student , and vanilla TB in the teacher. For QGFN, we use $\alpha$-TB with the p-greedy sampling method. We set the first 80\% steps to be stage 1, and use an exponential annealing function in stage 2. Other experimental details follow the default settings in the open-source codebases of \citet{kim2024adaptive} and \citet{lau2024qgfn}. Note that although the sEH tasks share the same name in both codebases, their underlying implementations differ. For QGFN, 
As shown in Table~\ref{tab:adaptive-teachers-qgfn-modes}, the addition of the $\alpha$-GFN objective leads to a consistent increase in the number of discovered modes. This suggests that our method can serve as a versatile "plug-and-play" component that bolsters the performance of various GFlowNet training pipelines.

\begin{table}[htbp]
\caption{\textit{$\alpha$-GFN applied to Adaptive Teachers and QGFN.} }
\centering
\small
\setlength{\tabcolsep}{6pt}
\renewcommand{\arraystretch}{1.2}
\resizebox{0.7\linewidth}{!}{%
\begin{tabular}{@{\hspace{.6em}}ccccc@{\hspace{.6em}}}
\toprule
\multirow{1}{*}{\textbf{sEH tasks}} 
& \multicolumn{2}{c}{\textbf{Adaptive Teachers}~\citep{kim2024adaptive}} &\multicolumn{2}{c}{\textbf{QGFN}~\citep{lau2024qgfn}} \\
\cmidrule(lr){2-5}
\textbf{Metric}& Baseline & Ours & Baseline & Ours \\
\midrule
\textbf{Modes$\uparrow$} & \msl{103.50}{5.07} &\msl{\textbf{110.00}}{2.55} & \msl{435.80}{176.90} & \msl{\textbf{613.40}}{269.85} \\
\bottomrule
\end{tabular}
}
\label{tab:adaptive-teachers-qgfn-modes}
\end{table}

\paragraph{Orthogonality to Reward Temperature Scaling.}
Another common technique for balancing exploration and exploitation is reward temperature scaling ($R^{1/\tau}$), which reshapes the reward landscape by adjusting its "peakiness." However, we explicitly distinguish $\alpha$-GFN from such reward-side modifications. While temperature scaling alters the target distribution itself, the $\alpha$ parameter modulates the learning process, influencing how the agent exploits the distribution during training. To demonstrate this, we evaluated the Set Generation task across DB, FL-DB, and TB objectives using reward temperatures of $0.5$ and $2$ ($0.5\times$ and $2\times$ the default temperature 1). As illustrated in Table~\ref{tab:set-exp-results-temp}, $\alpha$-GFN yields consistent performance gains over the vanilla model, regardless of the reward temperature. Such results demonstrate that $\alpha$-GFN provides a complementary advantage that is orthogonal to reward reshaping.

\begin{table}[t]
\caption{\textit{Performance on Set Generation across varying reward temperatures.} $\alpha$-GFN consistently improves both reward and mode discovery compared to the vanilla baseline under different temperatures. We \textbf{bold} the better results, and mark standard deviations \textcolor{gray}{gray}.}
\centering
\small
\setlength{\tabcolsep}{6pt}
\renewcommand{\arraystretch}{1.2}
\resizebox{\linewidth}{!}{%
\begin{tabular}{@{\hspace{.6em}}cclcccccc@{\hspace{.6em}}}
\toprule
\multicolumn{3}{c}{} &
\multicolumn{2}{c}{\textbf{DB}} & \multicolumn{2}{c}{\textbf{FL-DB}} & \multicolumn{2}{c}{\textbf{TB}} \\
\cmidrule(lr){4-5}\cmidrule(lr){6-7}\cmidrule(lr){8-9}
\textbf{Temperature} & \textbf{Set Size} & \textbf{Metric} & Baseline & Ours & Baseline & Ours & Baseline & Ours \\
\midrule
\multirow{12}{*}{0.5$\times$} & \multirow{3}{*}{Small} & Modes$\uparrow$ & \msl{89.6}{0.5} & \msl{\textbf{90.0}}{0.0} & \msl{90.0}{0.0} & \msl{90.0}{0.0} & \msl{86.6}{2.3} & \msl{\textbf{90.0}}{0.0} \\
 &  & Top-1000 R$\uparrow$ & \msl{0.221}{0.000} & \msl{0.221}{0.000} & \msl{0.221}{0.000} & \msl{0.221}{0.000} & \msl{0.220}{0.000} & \msl{\textbf{0.221}}{0.000} \\
 &  & Spearman & \msl{0.998}{0.001} & \msl{0.997}{0.001} & \msl{0.936}{0.009} & \msl{0.832}{0.015} & \msl{0.999}{0.000} & \msl{0.999}{0.000} \\
\cmidrule(lr){2-9}
 & \multirow{3}{*}{Medium} & Modes$\uparrow$ & \msl{20.4}{6.2} & \msl{\textbf{415.2}}{124.0} & \msl{1254.2}{398.3} & \msl{\textbf{8125.0}}{1435.0} & \msl{1.2}{0.8} & \msl{\textbf{3861.4}}{796.5} \\
 &  & Top-1000 R$\uparrow$ & \msl{5.3}{0.113} $\times 10^{5}$ & \msl{\textbf{7.05}}{0.148} $\times 10^{5}$ & \msl{7.63}{0.182} $\times 10^{5}$ & \msl{\textbf{8.75}}{0.141} $\times 10^{5}$ & \msl{2.75}{0.0267} $\times 10^{5}$ & \msl{\textbf{8.22}}{0.127} $\times 10^{5}$ \\
 &  & Spearman & \msl{0.995}{0.001} & \msl{0.989}{0.001} & \msl{0.883}{0.017} & \msl{0.743}{0.025} & \msl{0.996}{0.001} & \msl{0.994}{0.001} \\
\cmidrule(lr){2-9}
 & \multirow{3}{*}{Large} & Modes$\uparrow$ & \msl{35.4}{9.2} & \msl{\textbf{8727.0}}{3678.3} & \msl{12932.2}{1297.2} & \msl{\textbf{19088.0}}{1377.1} & \msl{0.2}{0.4} & \msl{\textbf{58.0}}{8.1} \\
 &  & Top-1000 R$\uparrow$ & \msl{4.38}{0.115} $\times 10^{5}$ & \msl{\textbf{8.69}}{0.0998} $\times 10^{5}$ & \msl{8.75}{0.0268} $\times 10^{5}$ & \msl{\textbf{8.76}}{0.0299} $\times 10^{5}$ & \msl{1.02}{0.0137} $\times 10^{5}$ & \msl{\textbf{4.82}}{0.0852} $\times 10^{5}$ \\
 &  & Spearman & \msl{0.990}{0.002} & \msl{0.975}{0.007} & \msl{0.797}{0.014} & \msl{0.716}{0.022} & \msl{0.991}{0.002} & \msl{0.990}{0.001} \\
\midrule
\multirow{12}{*}{2.0$\times$} & \multirow{3}{*}{Small} & Modes$\uparrow$ & \msl{5.6}{2.1} & \msl{\textbf{13.2}}{6.5} & \msl{28.8}{4.8} & \msl{\textbf{74.2}}{12.8} & \msl{5.0}{2.8} & \msl{\textbf{5.2}}{2.8} \\
 &  & Top-1000 R$\uparrow$ & \msl{0.133}{0.002} & \msl{\textbf{0.156}}{0.009} & \msl{0.186}{0.005} & \msl{\textbf{0.213}}{0.006} & \msl{0.130}{0.002} & \msl{\textbf{0.131}}{0.002} \\
 &  & Spearman & \msl{0.993}{0.002} & \msl{0.986}{0.004} & \msl{0.987}{0.006} & \msl{0.991}{0.004} & \msl{0.999}{0.000} & \msl{0.999}{0.000} \\
\cmidrule(lr){2-9}
 & \multirow{3}{*}{Medium} & Modes$\uparrow$ & \msl{0.0}{0.0} & \msl{\textbf{6.0}}{9.0} & \msl{0.4}{0.9} & \msl{\textbf{4.4}}{3.6} & \msl{0.0}{0.0} & \msl{\textbf{16.0}}{20.5} \\
 &  & Top-1000 R$\uparrow$ & \msl{10866}{591} & \msl{\textbf{4.69}}{0.616} $\times 10^{5}$ & \msl{2.54}{0.747} $\times 10^{5}$ & \msl{\textbf{4.58}}{0.403} $\times 10^{5}$ & \msl{7847}{388} & \msl{\textbf{4.76}}{1.29} $\times 10^{5}$ \\
 &  & Spearman & \msl{0.981}{0.009} & \msl{0.931}{0.011} & \msl{0.947}{0.020} & \msl{0.989}{0.003} & \msl{0.998}{0.000} & \msl{0.988}{0.002} \\
\cmidrule(lr){2-9}
 & \multirow{3}{*}{Large} & Modes$\uparrow$ & \msl{0.0}{0.0} & \msl{\textbf{75.4}}{139.3} & \msl{0.8}{0.8} & \msl{\textbf{57.4}}{27.6} & \msl{0.0}{0.0} & \msl{\textbf{0.4}}{0.5} \\
 &  & Top-1000 R$\uparrow$ & \msl{4368}{262} & \msl{\textbf{4.48}}{1.13} $\times 10^{5}$ & \msl{1.77}{0.448} $\times 10^{5}$ & \msl{\textbf{4.72}}{0.502} $\times 10^{5}$ & \msl{1663}{29} & \msl{\textbf{1.91}}{0.36} $\times 10^{5}$ \\
 &  & Spearman & \msl{0.963}{0.004} & \msl{0.837}{0.037} & \msl{0.892}{0.023} & \msl{0.949}{0.017} & \msl{0.997}{0.000} & \msl{0.998}{0.001} \\
\bottomrule
\end{tabular}
}
\label{tab:set-exp-results-temp}
\end{table}

\paragraph{Versatility of Exploration-Exploitation Control in Scale-up Scenarios.} 
To evaluate the scalability and versatility of $\alpha$-GFN, we apply it to FlowRL \citep{zhu2025flowrl}, a recent framework for LLM reasoning. Specifically, we integrate $\alpha$-GFN into the FlowRL objective using Qwen2.5-3B-Instruct~\citep{qwen2.5} with a fixed $\alpha \in \{0.1, 0.5, 0.9\}$ throughout a 200-step training process on the VeRL recipe \citep{sheng2024hybridflow}. The training objective is
\begin{equation}
    \mathcal{L}_{\alpha-\text{FlowRL}} =
w \cdot \left( 
\log Z_\phi(\mathbf{x}) 
+ \frac{1}{|\mathbf{y}|} \log \pi_\theta(\mathbf{y} \mid \mathbf{x}) 
- \beta \hat{r}(\mathbf{x}, \mathbf{y}) 
- \frac{1}{|\mathbf{y}|}\log \pi_{\mathrm{ref}}(\mathbf{y} \mid \mathbf{x}) + \log \frac{\alpha}{1-\alpha}
\right)^2.
\end{equation}
where the definitions of the notations directly follows~\citep{zhu2025flowrl}.
Results are shown in Table~\ref{tab:flowrl-results}, where
$\alpha$ serves as an effective lever for balancing exploitation and exploration, even within the high-variance environment of LLM reasoning. We observe a performance trend: lower $\alpha$ values (e.g., $0.1$) enhance the broader exploration necessary for general benchmarks like MATH500~\citep{lightman2023lets}, Minerva~\citep{minerva}, and Olympiad~\citep{he2024olympiadbench}. These datasets consist of a large volume of problems (e.g., $500$ to $2000+$ instances) covering a wide range of difficulty levels, where maintaining exploration abilities are beneficial for performance improvements. On the other hand, higher $\alpha$ values (e.g., $0.9$) facilitate stronger exploitation of the reward model, yielding better performance on challenging, competition-level benchmarks such as AIME2024/2025~\citep{AIME}. These benchmarks are significantly more constrained, featuring only $30$ extremely challenging, competition-level problems where the reward landscape is sparse and the model must focus on high-reward reasoning paths to succeed. Intermediate tasks like AMC23~\citep{AMC} achieve peak performance at $\alpha=0.5$, representing a middle ground between the two regimes.

To further elucidate this mechanism, we analyze key training metrics at step 200 in Table~\ref{tab:flowrl-results}. A pattern emerges: both the average reward and the magnitude of the KL divergence from the reference policy ($|\texttt{ref\_kl}|$) tend to scale with $\alpha$. Specifically, higher $\alpha$ values tend to exhibit larger rewards and more substantial deviations from the reference policy, signaling aggressive exploitation of the reward model. In contrast, lower $\alpha$ values are predisposed to smaller rewards and maintain a closer proximity to the reference model (lower $|\texttt{ref\_kl}|$), reflecting a more exploratory training regime. These results restate and validate the exploration-exploitation analysis in Sec. 4.4, confirming that the $\alpha$-modulated exploration-exploitation trade-off remains effective across large-scale task environments.

\begin{table}[htbp]
\caption{\textit{Performance of $\alpha$-GFN on mathematical reasoning tasks.} Results demonstrate how $\alpha$ acts as a control lever to shift the focus between exploration and exploitation. We evaluate FlowRL with a Qwen2.5-3B-Instruct backbone across distinct $\alpha$ settings, and evaluation metrics for the benchmarks are the same as~\citep{zhu2025flowrl}. We \textbf{bold} the better results.}
\centering
\small
\setlength{\tabcolsep}{6pt}
\renewcommand{\arraystretch}{1.2}
\resizebox{0.5\linewidth}{!}{%
\begin{tabular}{@{\hspace{.6em}}ccccc@{\hspace{.6em}}}
\toprule
\multicolumn{2}{c}{} 
& \multicolumn{3}{c}{\textbf{$\alpha$}}  \\
\cmidrule(lr){3-5}
\textbf{Stage} &\textbf{Metric/Benchmark}& 0.1 & 0.5 & 0.9 \\
\midrule
\multirow{2}{*}{\textbf{Train}} & Avg Reward & -0.506
& -0.399 & -0.431 \\
& \texttt{ref\_kl} & -0.298 & -0.2111 & -0.680 \\
\midrule
\multirow{6}{*}{\textbf{Test}}
&AIME2024 & 0.054167 & \textbf{0.056250} & \textbf{0.056250} \\
&AIME2025 & 0.033333 & 0.031250 & \textbf{0.045833} \\
&AMC23 & 0.385937 & \textbf{0.542188} & 0.496875 \\
&MATH500 & \textbf{0.618875} & 0.587375 & 0.547500 \\
&Minerva & \textbf{0.240119} & 0.232537 & 0.186351 \\
&Olympiad & \textbf{0.273090} & 0.243694 & 0.231825
\\
\bottomrule
\end{tabular}
}
\label{tab:flowrl-results}
\end{table}

\end{document}